\documentclass[12pt]{ociamthesis} 

\usepackage{etoolbox}
\makeatletter
\patchcmd{\@chapter}{\addtocontents{lof}{\protect\addvspace{10\p@}}}{}{}{}
\patchcmd{\@chapter}{\addtocontents{lot}{\protect\addvspace{10\p@}}}{}{}{}
\makeatother

\usepackage{multirow, hyperref, dsfont, bbm, caption, listings, float, titletoc, stmaryrd, amsmath, amsfonts, amsthm, amssymb, amsthm, epsfig, epstopdf, url, fancyhdr, inputenc, pgfplotstable, listings, listings, color, bm, graphicx, epstopdf, subcaption, makecell, mathtools, sectsty, array, pifont, hyperref, fixltx2e, bold-extra, enumitem, mdframed, boldline, rotating, cleveref, titlesec}

\usepackage[square,numbers,sort]{natbib}
	
\usepackage[nottoc,notlof,notlot]{tocbibind}
\usepackage[font=footnotesize,labelfont=bf]{caption}
\usepackage{xcolor,colortbl}

\theoremstyle{theorem}

\newtheorem{lemma}{Lemma}[chapter]
\newtheorem{proposition}{Proposition}[chapter]
\newtheorem{theorem}{Theorem}[chapter]

\theoremstyle{definition}

\newtheorem{descrip}{Description}[chapter]
\newtheorem{example}{Example}[chapter]
\newtheorem{remark}{Remark}[chapter]

\theoremstyle{plain}
\newtheorem{definition}{Definition}[chapter]

\chapternumberfont{\LARGE} 
\chaptertitlefont{\LARGE}

\def\cpar{\hss\egroup\line\bgroup\hss}

\title{Investigating Graph Neural Networks and Classical Feature-Extraction Techniques in Activity-Cliff and Molecular Property Prediction}   

\author{Markus Ferdinand Dablander}             
\college{Mansfield College}  

\degree{Doctor of Philosophy}  		 
\degreedate{October 2023}         

\begin{document}

\baselineskip=18pt plus1pt

\setcounter{secnumdepth}{3}
\setcounter{tocdepth}{3}

\maketitle                  
\include{acknowledgements}  

$\text{}$
\newpage
\begin{acknowledgements}
	First and foremost, I would like to say thank you to my scientific supervisors Prof.~Garrett M.~Morris, Prof.~Renaud Lambiotte, and Dr~Thierry Hanser, who always took the time to have deep discussions, who supported me in my years of doctoral research to explore novel ideas, and who encouraged me in the right moments to push forward.
	
	I would like to thank Garrett for sharing with me his vast knowledge of the cheminformatics literature, teaching me important chemistry facts, providing exceptionally detailed and dedicated feedback on my work, helping me to navigate all the formal milestones of my doctoral journey, and always making scientific meetings with me and his other doctoral students a priority. Throughout the years, I have immensely benefitted from Garrett's knowledge and have more than once felt the genuine care he actively shows for the well-being of his students.
	
	I would like to thank Renaud for offering his profound and rigorous mathematical expertise, continually encouraging me throughout the ups and downs of my research, reliably finding kind and motivational words for me during the most intensive parts of my doctoral studies, and sharing his insights with me about how to develop an effective and creative mindset for advanced research. One of the most impactful pieces of advice I received from Renaud during my early days as a doctoral student was that the best ideas often come when working on something concrete.

	I would like to thank Thierry for consistently giving precise advice on the deepest technical aspects of my work, supporting me to find the right research direction whenever I was facing a difficult crossroad, cordially welcoming me into his industrial research group when I visited him at my partner company Lhasa, and always being interested in talking about new and fascinating ideas from science and beyond. I will warmly remember the long in-person conversations I had with Thierry that did not only revolve around my scientific work but also around philosophy, cosmology, consciousness, history, music, technological progress and the future of artificial intelligence.
	
	In addition to thanking my scientific supervisors, I would like to express further gratitude to: The InFoMM CDT directors, Chris Breward and Colin Please, for granting me the privilege to do research at the Mathematical Institute and for guiding me and many other students through our formal doctoral path. The EPSRC~(EP/L015803/1) as well as my industrial partner Lhasa for their financial support. Mansfield College for enriching my life at Oxford with a great community, engaging events and a beautiful place to retreat. Stéphane Werner and Jean-Francois Marchaland from Lhasa for always being approachable and staying up to date with my scientific work. My mother and my father for the sacrifices they made without which I would have never been able to become a mathematician. My grandparents and my aunt for their continuous loving support. My brother for always having my back and believing in me. My sister-in-law and my little niece for brightening the mood after a long day. My cousin who is also a mathematician for going on long walks with me in the Austrian countryside to discuss science and life. My old friends from UCL, who have since scattered around the world but have still managed to preserve and foster the seed of our friendship. My new friends from Oxford who have accompanied me along my path. All my other friends from Austria, the UK, and elsewhere who have grown with me throughout the years. And finally, my girlfriend who has deeply shared with me all the highs and lows of my time as a doctoral student, and who has always been there for me with her whole heart.
	
	\thispagestyle{empty}

\end{acknowledgements}

$\text{}$
\newpage

\begin{abstract}
	
	Molecular featurisation refers to the process of transforming molecular data into numerical feature vectors. It is one of the key research areas in molecular machine learning and computational drug discovery. Recently, message-passing graph neural networks~(GNNs) have emerged as a novel method to learn differentiable features directly from molecular graphs. While such graph-based techniques hold great theoretical promise, further investigations are needed to clarify if and when they indeed manage to definitively outcompete classical molecular featurisations such as extended-connectivity fingerprints~(ECFPs) and physicochemical-descriptor vectors~(PDVs). 
	
	In this thesis, we systematically explore and further develop classical as well as graph-based molecular featurisation methods for two important tasks: the well-studied problem of molecular property prediction, in particular, quantitative structure-activity relationship~(QSAR) prediction, and the largely unexplored challenge of activity-cliff~(AC) prediction. We first give a mathematical description and critical analysis of PDVs, ECFPs and message-passing GNNs, with a focus on graph isomorphism networks~(GINs). We then conduct a rigorous computational study to compare the performance of PDVs, ECFPs and GINs for QSAR and AC-prediction. Following this, we mathematically describe a novel twin neural network model for AC-prediction and experimentally evaluate ECFP-based and GIN-based versions of this dual architecture. In an additional project, we introduce \textit{substructure pooling} as a general mathematical operation for the vectorisation of structural fingerprints that represents a natural counterpart to graph pooling in GNN architectures. We propose \textit{Sort \& Slice} as a simple substructure-pooling technique for ECFPs that robustly outperforms hashing at molecular property prediction. Finally, we outline two ideas for future research: (i) a graph-based self-supervised learning strategy to make classical molecular featurisations trainable, and (ii) trainable substructure-pooling via differentiable self-attention.

\end{abstract}

\newpage $\text{}$
\newpage

\begin{romanpages}          

\tableofcontents            

\listoffigures

\listoftables

\newpage
{\noindent \LARGE \textbf{List of Abbreviations}}
\newline
\begin{itemize}
	\itemsep-0.7em
	
	\item AC = Activity Cliff
	\item AUPRC = Area Under the Precision Recall Curve
	\item AUROC = Area Under the Receiver Operating Characteristic Curve
	\item ECFP = Extended-Connectivity Fingerprint
	\item GCN = Graph Convolutional Network
	\item GIN = Graph Isomorphism Network
	\item GNN = Graph Neural Network
	\item InChI = International Chemical Identifier
	\item kNN = k-Nearest Neighbour
	\item MAE = Mean Absolute Error
	\item MCC = Matthews Correlation Coefficient
	\item MIM = Mutual-Information Maximisation
	\item MLP = Multilayer Perceptron
	\item MMP = Matched Molecular Pair
	\item NFP = Neural Fingerprint
	\item PD = Potency Direction
	\item PDV = Physicochemical-Descriptor Vector
	\item QSAR = Quantitative Structure-Activity Relationship
	\item RF = Random Forest
	\item SAR = Structure-Activity Relationship
	\item SELFIE = Self-Referencing Embedded String
	\item SMILES = Simplified Molecular-Input Line-Entry System
	\item TPE = Tree-structured Parzen Estimator
	\item WL = Weisfeiler-Lehman
	
\end{itemize}

\newpage $\text{}$ \thispagestyle{empty}
\newpage
\end{romanpages}            

\newpage
\chapter[Introduction]{Introduction} \label{chap: intro}

How to best represent a chemical compound in computational form is one of the most fundamental questions in cheminformatics and molecular machine learning. Molecules are complex physical entities whose properties depend on a large number of intertwined factors such as the types of their involved atoms and bonds, their graph-theoretic connectivity structure, their three-dimensional geometry, and the quantum-mechanical behaviour of their associated electrons. Whether a given computational molecular representation is useful intricately depends on the task one wants to solve, on its required data format and on its level of abstraction.

In this thesis, we are interested in the representation of chemical compounds as real-valued feature vectors for the purpose of molecular machine learning. We focus on the process of \textit{molecular featurisation} which describes the computational transformation of a non-vectorial molecular representation (often a detailed string of symbols or mathematical graph that fully encodes the chemical composition and structure of an input compound) into an abstract representation as a numerical vector than can readily be processed by a machine-learning system. As we will see, the impact of the chosen molecular featurisation on the predictive performance of a machine-learning system is often profound.

In the field of computer vision, the advent of convolutional neural networks has led to dramatic performance breakthroughs in the last decade~\citep{krizhevsky2012imagenet, simonyan2014very, zeiler2014visualizing, szegedy2015going, he2016deep}. These breakthroughs were caused by the remarkable ability of convolutional architectures to automatically extract relevant high-level features directly from raw visual data. For machine-learning tasks on images, this has since led to an almost universal shift from traditional expert-based \textit{feature engineering} towards differentiable and automatic \textit{feature learning}. This can be seen in contrast to the area of molecular machine learning, where trainable feature learning techniques that operate directly on molecular graphs without intermediate feature-engineering steps have only emerged in the last few years and have still not managed to conclusively outperform classical non-trainable featurisation methods in a variety of studies~\citep{stepivsnik2021comprehensive, mayr2018large, jiang2021could, wang2021molclr, menke2021using, chithrananda2020chemberta, sabando2021using, winter2019learning}. For machine learning on molecular data, breakthroughs comparable to the ones caused by the feature-extraction capabilities of convolutional neural networks in computer vision remain elusive.

For chemical prediction tasks, molecules have classically been represented via physicochemical-descriptor vectors~(PDVs)~\citep{todeschini2008handbook} or via structural fingerprints such as extended-connectivity fingerprints~(ECFPs)~\citep{bajusz2017fingerprints, rogers2010extended}. Both featurisation types rely on non-trainable algorithms to compute specific properties of an input molecule. In the case of physicochemical descriptors, these properties are typically high-level characteristics of an input compound such as the predicted molecular partition coefficient $\log(P)$ that quantifies lipophilicity; in the case of structural fingerprints, these properties are frequently binary features that indicate the presence or absence of certain chemical substructures. The computed properties are collected in numerical feature vectors that can be subsequently used for a downstream machine-learning task. While certain descriptors and fingerprints can lead to considerable predictive performance~\citep{stepivsnik2021comprehensive, mayr2018large, jiang2021could, menke2021using}, they also have some serious drawbacks. The most significant of these is that the algorithmic transformation of a detailed molecular representation into a mere vector of selected structural or physicochemical properties can lead to a loss of crucial chemical information. Classical descriptor and fingerprint-based featurisations thus form an information bottleneck which separates detailed molecular data on one side from its computationally processed form on the other side.

Recently, message-passing graph neural networks~(GNNs)~\citep{gilmer2017neural, kipf2016semi, kearnes2016molecular,duvenaud2015convolutional, hu2019strategies,yang2019analyzing, wu2020comprehensive, wieder2020compact, li2015gated,battaglia2016interaction,defferrard2016convolutional,liu2019chemi} have entered the field of cheminformatics as a potential solution for this limitation. GNNs are trainable deep-learning architectures that allow the differentiable extraction of continuous features directly from molecular graphs. These graphs can fully specify the connectivity structure and chemical composition of a molecule; even simple stereochemical properties such as tetrahedral R-S chirality and E-Z double-bond geometry can easily be encoded in the atom and bond features of molecular graphs. Such graphs thus form highly explicit and detailed molecular representations. Learning differentiable features directly from molecular graph structures holds the promise of overcoming potential information bottlenecks during molecular featurisation. Unfortunately, GNNs come with their own set of technical challenges: Some popular GNN models lack theoretical expressivity and thus cannot learn to distinguish certain simple graph structures~\citep{xu2018powerful}; many GNN architectures cannot be made deep due to a tendency of successively convolved node features to become indistinguishable~\citep{liu2020towards}; GNNs have to learn a reasonable embedding of chemical space from scratch every time they are trained on a novel task; almost all current GNN models are based on a local neighbourhood-aggregation scheme~\citep{gilmer2017neural} which hinders information flow between distant graph nodes; and GNNs require a global pooling step to eventually reduce the graph to a vector that can potentially form a dangerous information bottleneck in and of itself~\citep{navarin2019universal}. A variety of studies have found GNNs to exhibit superior predictive performance compared to the simpler and more computationally efficient fingerprint and descriptor-based featurisations~\citep{duvenaud2015convolutional, gilmer2017neural,yang2019analyzing, rao2022quantitative, hop2018geometric, shang2018edge, li2017learning, xiong2019pushing}; however, other experiments have pointed towards the exact opposite~\citep{stepivsnik2021comprehensive, mayr2018large, jiang2021could, wang2021molclr, menke2021using, chithrananda2020chemberta, sabando2021using, winter2019learning, van2022exposing}. The superiority of GNNs over classical techniques thus remains questionable and requires more investigation.

In this work, we aim to explore and further develop classical as well as graph-based molecular featurisation methods with a focus on two important challenges in computational drug discovery: the canonical problem of molecular property prediction, in particular quantitative structure-activity relationship~(QSAR) prediction, and the largely unexplored task of activity-cliff~(AC) prediction. Molecular property prediction encompasses a diverse number of tasks such as the prediction of lipophilicity, aqueous solubility, toxicity, or mutagenicity of a chemical compound. QSAR-prediction represents a special case of molecular property prediction and refers to the problem of predicting the biological activity of a compound with respect to a given pharmacological target from its chemical structure. QSAR modelling often takes the form of predicting the binding affinity of a molecule for a specific protein target such as an enzyme or a receptor. ACs refer to pairs of molecular compounds whose chemical structures only differ by a small change at a specific site but which exhibit an unexpectedly high difference in their binding affinity for a given pharmacological target~\citep{silipo1991qsar, maggiora_outliers_2006, sheridan_experimental_2020, cruz-monteagudo_activity_2014, stumpfe_recent_2014, stumpfe_evolving_2019, stumpfe_advances_2020}. AC-prediction does not only refer to the classification whether a given compound pair forms and AC or not but usually also implicitly encompasses the prediction of the potency direction~(PD) of the pair (i.e.~which of both compounds is more active). Molecular property prediction, QSAR-prediction, and AC-prediction represent three important challenges of high relevance to the \textit{in silico} identification and optimisation of novel pharmacological compounds. In particular, our work aims to contribute to the development of computational tools that can accurately predict important properties of novel compounds and whether they will bind tightly to a biological target. This could help to tackle one of the central questions in drug discovery: what should a medicinal chemist synthesise next?

\section{Outline of Thesis and Research Contributions}

The rest of thesis is organised as follows:

\begin{itemize}

\item In Chapter~\ref{chap: mol_reps} we first introduce SMILES strings and molecular graphs which both constitute detailed molecular representations that fully encode the chemical composition and structure of an input compound. We then give a technical introduction to PDVs, ECFPs and GNNs as the three most frequently used molecular featurisation methods in the current literature. We critically discuss all three featurisations and contrast their respective strengths and weaknesses. In the case of GNNs, we also shortly discuss the issue of theoretical expressivity in mathematical terms and give a description of the graph isomorphism network~(GIN) as perhaps the simplest and most prototypical GNN in the $1$-Weisfeiler-Lehman class.

\item In Chapter~\ref{chap: qsar_ac_study} we conduct a rigorous computational study to explore the relative performance of PDVs, ECFPs and GINs for QSAR and AC-prediction. We investigate nine separate QSAR models by integrating each of the three studied featurisations with a random forest, a multilayer perceptron, and a k-nearest neighbour regressor. We use each model to classify whether a compound pair forms an AC, to classify which compound in the pair is more active, and to predict the binding affinities of individual molecules for three pharmacological targets: dopamine receptor D2, factor Xa, and SARS-CoV-2 main protease. We further develop a novel pair-based data-splitting strategy for the evaluation of distinct AC-prediction scenarios. 

Our results strongly support the hypothesis that QSAR models frequently fail to predict ACs. We observe low AC-sensitivity when the activities of both compounds are unknown, but a substantial increase in AC-sensitivity when the actual activity of one of the compounds is given. GIN features are found to be competitive with or superior to classical molecular featurisations for AC-classification. For general QSAR-prediction, however, ECFPs still consistently deliver the best performance amongst the tested featurisations.

Our study represents the first work that investigates the capabilities of QSAR models to classify between ACs and non-ACs. Moreover, it provides additional evidence that standard message-passing GNNs may need to be improved further to robustly outperform classical ECFPs. Our study results have been published as a \href{https://doi.org/10.1186/s13321-023-00708-w}{peer-reviewed research paper}~\citep{dablander2023exploring} in the Journal of Cheminformatics and as a \href{http://dx.doi.org/10.13140/RG.2.2.35914.34241}{scientific poster}~\citep{dablander2023exploringqsaracpredposter} at the 10th International Congress on Industrial and Applied Mathematics (ICIAM 2023, Tokyo).

\item In Chapter~\ref{chap: twin_net_ac_pred} we design and evaluate a novel twin neural network model for AC-prediction. We provide a mathematical description of our proposed dual architecture along with proofs of its built-in symmetry properties. We then conduct computational experiments to compare the AC-prediction performance of the twin architecture and the strongest QSAR models from the previous chapter on a data set of SARS-CoV-2 main protease inhibitors. We explore four distinct molecular featurisations for the twin model, based on either ECFP, GINs, or two transfer learning techniques we developed.

Our experiments suggest that the twin architecture outperforms standard QSAR models at AC-prediction in a variety of scenarios. In particular, twin networks appear to strike a superior balance between AC-sensitivity and AC-precision which increases their practical utility.

To the best of our knowledge, this work represents the first application of twin neural networks to AC-prediction and the first investigation of a novel AC-prediction technique that includes important control experiments based on standard QSAR models. We have presented our twin architecture as a \href{http://dx.doi.org/10.13140/RG.2.2.18137.60000}{scientific poster}~\citep{dablander2021siamese} at the 4th RSC-BMCS / RSC-CICAG Artificial Intelligence in Chemistry Symposium (2021, virtual) where we were subsequently awarded the prize for the best scientific poster.

\item In Chapter~\ref{chap: ecfps_sort_and_slice} we introduce a mathematical operation called \textit{substructure pooling} that formalises the transformation of sets of substructures (i.e.~unordered set representations of structural fingerprints) into numerical vectors. We draw an analogy between substructure pooling for structural fingerprints and global graph pooling for GNN architectures. However, unlike global graph pooling, substructure pooling remains largely unexplored. We show that the standard hashing procedure for the vectorisation of ECFPs is a form of substructure pooling. We go on to describe a straightforward alternative to hashing for ECFP vectorisation called \textit{Sort \& Slice}. This technique is simply based on first sorting all detected circular substructures according to their frequency of appearance in the training set and then only accepting the most frequent substructures into the final vectorial fingerprint. This naturally leads to a higher interpretability than hashing due to an absence of bit collisions in the final vector representation.

Under reasonable theoretical assumptions that hold approximately true for real-world data sets, we mathematically prove that Sort \& Slice only selects the most informative substructures from an entropic point of view. Furthermore, we demonstrate via a set of strict computational experiments that Sort \& Slice robustly leads to higher predictive performance than hashing as well as two advanced supervised substructure selection schemes for molecular property prediction across a large number of settings. In particular, we observe a predictive advantage of Sort \& Slice over hashing across varying data sets, data splitting techniques, machine-learning regressors, and ECFP hyperparameters. This advantage seems to increase with the expected number of bit collisions in the hashed ECFP.

Based on our current knowledge, the work in this chapter constitutes the first study that robustly demonstrates the existence of a technically simple and more interpretable alternative to hashing for the vectorisation of ECFPs that leads to higher predictive performance at supervised molecular property prediction. We have only recently become aware of one other work~\citep{macdougall2022reduced} that has investigated a technique similar to Sort \& Slice which we acknowledge. We intend to submit the results in this chapter for publication in the near future.

\item In Chapter~\ref{chap: future_research} we shortly present two potential ideas for future research in the domain of molecular featurisation: 

Firstly, we describe a self-supervised graph-based learning strategy that can intuitively be interpreted as a method to make classical molecular featurisations differentiable and trainable. Our central idea involves first pre-training a modern GNN to predict classical featurisations such as ECFPs or PDVs directly from molecular graphs and then fine-tuning the GNN on a supervised task. 

Secondly, we outline a trainable substructure-pooling method based on a differentiable self-attention mechanism that might enable the learning of contextual substructural features. This technique could for instance be combined with MACCS fingerprints or ECFPs. It could potentially provide a useful inductive bias for the learning of task-specific compound-level featurisations that depend not only on individually present substructures but also on their interactions.

\item In Chapter~\ref{chap: conclusions} we give some final conclusions and further thoughts on our work.

\end{itemize}

\newpage $\text{}$ 
\newpage
\chapter[Molecular Representations and Featurisations for Machine Learning]{Molecular Representations and Featurisations for Machine Learning} \label{chap: mol_reps}

\section{Overview} \label{sec: mol_reps_overview}

In its most general interpretation, \textit{molecular representation} refers to the problem of describing a molecular compound in an abstract machine-readable format for the purpose of computational processing. In this Chapter, we focus on three of the most common types of molecular representation:
\begin{enumerate}
	\item \textbf{Graph-based representation methods} that represent molecules as mathematical graphs with numerical node and edge features.
	\item \textbf{String-based representation methods} that represent molecules as linear chains of symbols.
	\item \textbf{Feature-based representation methods} that represent molecules as vectors of real numbers.
\end{enumerate}
There are other types of molecular representations, including image-based representations~\citep{goh2017chemception,yoshimori2021prediction,iqbal_prediction_2021} and three-dimensional coordinate-based representations~\citep{schrier2020can,zhou2022unimol}. However, the above three categories cover a large part of modern use cases in cheminformatics and currently dominate the field of molecular machine learning. 

The key advantage of most non-feature-based representations like graphs or strings is that they allow for highly explicit representations of molecular compounds. Such representations are generally designed to fully encode the chemical structure of real molecules and as a result contain high levels of information. However, standard machine learning techniques such as random forests and k-nearest neighbours cannot directly utilise this rich information content because they are only able to process feature-based representations that consist of vectors of real numbers. One of the central challenges for machine learning in chemistry is thus the developments of \textit{molecular featurisation methods}. 
\begin{definition}[Molecular Featurisation]
A mathematical or algorithmic technique that turns a non-feature-based molecular representation $\mathcal{R}$ (such as a graph or a string) into a feature-based representation $\mathcal{F}$,
$$\mathcal{R} \mapsto \mathcal{F} = (f_1,...,f_l) \in \mathbb{R}^{l},$$
is called a molecular featurisation method.
\end{definition}
The usual purpose of molecular featurisation is to encode useful predictive information within the feature vector $\mathcal{F}$ for a downstream molecular machine learning task such as activity prediction. The chosen molecular featurisation method is often the key ingredient that determines the relative success or failure of a predictive algorithm. Current molecular featurisation methods can be subdivided into two distinct categories: \textit{trainable-} and \textit{non-trainable}. Non-trainable methods are based on fixed non-differentiable algorithms that compute specific predefined properties of an input molecule. The most widely-used examples in this category are physicochemical-descriptor vectors and structural fingerprints. Trainable methods, on the other hand, are based on recently developed deep-learning architectures that can learn to extract molecular features in a differentiable and task-specific manner. At present, message-passing GNNs that operate on top of molecular graphs are the most widely-used method in this class.

We start off this chapter by describing molecular graphs as mathematical representations of chemical compounds. We then go on to discuss SMILES strings which are the prevalent string-based molecular representation in cheminformatics. Finally, we give a technical description and critical discussion of the three most important molecular featurisation methods in the current literature: physicochemical descriptors~\citep{todeschini2008handbook}, extended-connectivity fingerprints~\citep{rogers2010extended} and message-passing GNNs~\citep{gilmer2017neural}. In recent years, all three of these competing techniques have led to state-of-the-art results in a wide variety of molecular machine learning tasks~\citep{stepivsnik2021comprehensive, mayr2018large, jiang2021could, wang2021molclr, menke2021using, chithrananda2020chemberta, sabando2021using, winter2019learning,duvenaud2015convolutional,kearnes2016molecular,gilmer2017neural,yang2019analyzing, wieder2020compact, liu2019chemi}.

\section{Molecular Graphs} \label{sec: mol_graphs}

The molecular graph of a chemical compound is a formal representation of its chemical structure via the language of mathematical graph theory.
\begin{definition}[Molecular Graph] 
Let $\mathcal{M}$ be a molecule composed of $n$ atoms and let 
$$A = \{a_1, ..., a_n\} $$
be a mathematical set of elements representing these atoms. Furthermore, let
$$ B = \{\{a, \tilde{a}\} \ \vert \ a,\tilde{a} \in A  \textit{ and there is a chemical bond between } a \text{ and } \tilde{a} \text{ in } \mathcal{M} \} $$
be a set of unordered pairs of elements of $A$ which describe existing chemical bonds between the atoms in $\mathcal{M}$. Then the pair $\mathcal{G} = (A, B)$ is called the molecular graph of~$\mathcal{M}$. If all hydrogen atoms are removed from $A$ before constructing $\mathcal{G}$, then $\mathcal{G}$ is called hydrogen-depleted.
\end{definition}

Those familiar with the field of graph theory will recognise a molecular graph as a connected, undirected, unweighted graph without self-loops and without multiple edges. The idea to view molecules as abstract networks of linked entitites has deep historical roots and was already employed in the 19th century by the British mathematician Arthur Cayley~\citep{biggs1986graph}. Using the hydrogen-depleted version of a molecular graph usually reduces the number of involved atoms significantly while preserving the essential chemical connectivity pattern between the heavy atoms. Hydrogen-depleted molecular graphs are therefore routinely preferred over their complete counterparts, especially in settings where computational cost plays a role. In this work, we too exclusively use hydrogen-depleted molecular graphs unless otherwise specified.

In almost all cases, the pure structural information contained in the molecular graph is enriched via atom and bond feature vectors.
\begin{definition}[Atom and Bond Feature Vectors]
Let $\mathcal{G} = (A, B)$ be the molecular graph of a molecule $\mathcal{M}$ and let
$$f : A \to \mathbb{R}^k$$
be a function that assigns real-valued vectors in $\mathbb{R}^k$ to the atoms in $A$. Then $f$ is called an atom featurisation function. For each atom $a \in A$, the vector $f(a) \in \mathbb{R}^k$ is called its atom feature vector. Similarly, let
$$g : B \to \mathbb{R}^j$$
be a function that assigns real-valued vectors in $\mathbb{R}^j$ to the chemical bonds in $B$. Then $g$ is called a bond featurisation function. For each chemical bond $\{a,\tilde{a}\} \in B$ the vector $g(\{a,\tilde{a}\}) \in \mathbb{R}^j$ is called its bond feature vector. 
\end{definition}

\begin{table}[h]
	\centering
	{\renewcommand{\arraystretch}{1.4}
	\begin{tabular}{V{3} p{6.2cm}|p{5.1cm}|p{2.4cm} V{3}}
		
		\hlineB{3}
		\multicolumn{3}{V{3} c V{3}}{\textbf{Selected Atom and Bond Features for Molecular Graphs}} \\
		\hline 
		
		\textbf{Atom Feature} & \textbf{Encoding} & \textbf{Dimensions}\\
		\hline
		
		element type   & one-hot encoding    & 43 \\
		
		number of heavy neighbours   & one-hot encoding    & 6 \\
		
		number of hydrogen neighbours   & one-hot encoding    & 6 \\
		
		formal charge   & one-hot encoding    & 8 \\
		
		hybridisation type   & one-hot encoding    & 7 \\
		
		R-S chirality type   & one-hot encoding    & 4 \\
		
		is part of a ring   & binary integer    & 1 \\
		
		is aromatic   & binary integer    & 1 \\
		
		atomic mass   & min-max scaled real number    & 1 \\
		
		van der Waals radius   & min-max scaled real number     & 1 \\
		
		covalent radius   & min-max scaled real number     & 1 \\
		\hline
		\textbf{Bond Feature}  &   -   &  - \\
		\hline
		
		bond type   & one-hot encoding    & 4 \\
		
		E-Z double-bond geometry  & one-hot encoding    & 4 \\
		
		is part of a ring  & binary integer    & 1 \\
		
		is conjugated  & binary integer    & 1 \\
		
		\hlineB{3}

	\end{tabular}}
	
	\caption[Atom and bond features for molecular graphs.]{Overview of the chosen atom and bond features for the molecular graphs used in our experiments.}
	
	\label{tab: atom_bond_features}
\end{table}

The purpose of atom and bond feature vectors is to add extra physicochemical information to the molecular graph. At the very least, the atom feature vectors should allow for a distinction between different atomic numbers, i.e.~element types
$$\{\text{C, N, O, S, F, Si, Cl, Br,...}\}$$
and the bond feature vectors should allow for a distinction between different bond types
$$\{\text{single, double, triple, aromatic}\}.$$ 
However, the exact composition of atom and bond feature vectors beyond basic indications of atom and bond types can vary substantially from author to author. For example,~\citet{kearnes2016molecular} include tetrahedral R-S chirality in their atom feature vectors while~\citet{gilmer2017neural} do not. Such differences can make an objective comparison of results across studies challenging.~\citet{pocha2020comparison} have recently made a systematic effort to compare the effects of differing atom feature vectors on the performance of a popular GNN architecture, a graph convolutional network~(GCN)~\citep{kipf2016semi}, in several molecular property prediction tasks. They found that the overall influence of the chosen atom featurisation on downstream performance was modest. However, representations that included more atomic features (and thus more information) tended to deliver better results.

Motivated by this, we decided to extensively leverage the information-storing capacities of molecular graphs in our own work and employ $11$ atom and four bond features. An overview of the chosen features that were used throughout our experiments is given in Table~\ref{tab: atom_bond_features}. In our atom featurisation, we considered $42$ common element types of heavy atoms, alond with a generic \textit{other-element} type for atoms whose element types were not on our predefined list. Element types were thus represented via $43$-dimensional one-hot encoded vectors. Note that while we used the usual hydrogen-depleted form of the molecular graph, we still implicitly considered hydrogen-related information by including the number of hydrogen neighbours in the list of atom features.

\section{SMILES Strings} \label{sec: smiles}

The most frequently used method to store molecules in digital form is via Simplified Molecular-Input Line-Entry System (SMILES) strings~\citep{weininger1988smiles}. A SMILES string is a sequence of ASCII characters which is generated by a fixed set of rules in order to specify the chemical identity of an input molecule. The most common version of the SMILES string can be interpreted as a sequential encoding of a (hydrogen-depleted) molecular graph that takes into account atom and bond types. The SMILES string of a molecule is thus fully sufficient to construct a molecular graph equipped with basic atom and bond features. On the other hand, every molecular graph with basic atom and bond features can be used to generate a SMILES string.

\begin{descrip}[SMILES String]

A valid SMILES string can be obtained from a molecular graph by turning the graph into a spanning tree via breaking all existing cycles, and then printing out the (capitalised or lower case) atom symbols encountered in a depth-first traversal of the spanning tree. Branches are specified via parentheses while double and triple bonds are written using the symbols $\{=,\#\}$ respectively; the existence of aromatic bonds is either inferred from the context or is expressed via the use of lower case for the involved atoms. The SMILES algorithm for an example molecule is illustrated in Figure~\ref{fig:smiles_generation_example}. For the exact procedure we refer the reader to the original paper series of Weininger et al.~\citep{weininger1988smiles, weininger1989smiles, weininger1990smiles}. 
\end{descrip} 
\begin{figure}[!t]
	\centering
	\includegraphics[width=0.69\linewidth]{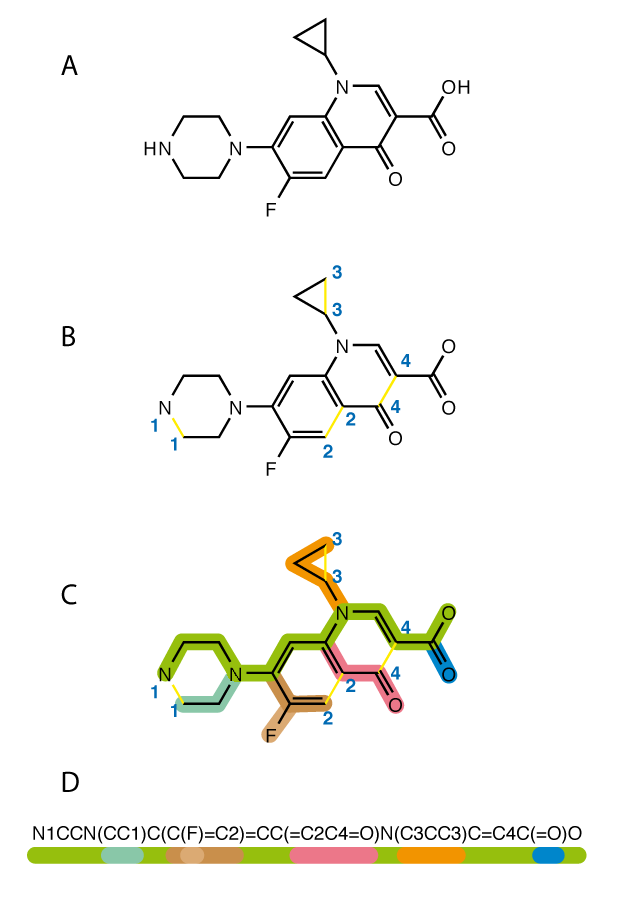}
	\caption[Generation of a SMILES string.]{Generation of a Simplified Molecular-Input Line-Entry System (SMILES) string from the molecular graph of the antibiotic molecule Ciprofloxacin. First the molecular graph is reduced to its hydrogen-depleted version. Then cycles are broken to turn the graph into a spanning tree. Finally, a depth-first traversal of the spanning tree (here starting with the leftmost nitrogen atom as a root) produces the SMILES string whereby branches are specified via parentheses. The integers in the SMILES string indicate which ring bonds were broken to produce the spanning tree and the equality signs indicate double bonds. Image source:~\citep{smileswikifigure2023}.}
	\label{fig:smiles_generation_example}
\end{figure}
The SMILES string of a molecule depends on the bonds chosen to break rings, the atom chosen as a starting point for the depth-first traversal of the spanning tree, and the order in which encountered branches are listed. Thus, a molecule can have many distinct and valid SMILES strings. However, canonicalisation algorithms exist that make the mapping from compound to SMILES string unambiguous. Basic SMILES strings do not contain explicit information about the three-dimensional positioning of the atoms in a molecule, but extensions of the SMILES system such as \textit{isomeric} SMILES strings exist which are able to express tetrahedral R-S chirality and E-Z double-bond geometry. Throughout our work, we exclusively use isomeric SMILES strings that include specifications for these two important stereochemical features.

The SMILES notation has become the default molecular representation in cheminformatics due to its relative simplicity, readability and expressivity. It also allows for the computational storage of large numbers of molecular graphs in a highly compressed and efficient format. However, the SMILES methodology does have some technical drawbacks. For example, since the SMILES string of a molecule is not unique, one has to arbitrarily commit to a fixed canonicalisation algorithms in order to guarantee the consistency of the SMILES representation across applications. In comparison, the international chemical identifier (InChI)~\citep{heller2013inchi} string which was initially developed by the International Union of Pure and Applied Chemistry~(IUPAC) is more expressive and is guaranteed to produce a unique string-label for every chemical structure. These advantages, however, come at the cost of reduced simplicity and readability. 

Another weakness of the SMILES notation is that a slight perturbation of symbols can easily lead to either a syntactically invalid expression or an impossible molecular structure that violates the basic laws of chemistry. This fragility may cause SMILES-based deep generative models trained for \textit{de novo} molecular design to output large fractions of broken and invalid SMILES strings; a challenge that can already be observed in the pioneering work of~\citet{gomez2018automatic}. To address this issue, more robust string representations such as the self-referencing embedded string~(SELFIE) notation~\citep{krenn2020self} and the DeepSMILES notation~\citep{oboyle2018} have recently been developed. These string systems have been specifically designed to facilitate the generation of valid molecules by deep generative models and may thus be preferrable to SMILES strings in this context.

\section{Physicochemical-Descriptor Vectors} \label{sec: pdvs}

\subsection{Mathematical Description}\label{subsec: pdvs_description}

Physicochemical descriptors~\citep{todeschini2008handbook} are likely the oldest and most established molecular featurisation method and play a fundamental role in classical cheminformatics. Such descriptors are able quantify an abundant number of important characteristics of a molecule such as lipophilicity, druglikeness, molecular refractivity, electrotopological state, molecular graph structure, fragment-profile, molecular charge, and molecular surface properties. Descriptor-based molecular featurisations have been used in QSAR-modelling~\citep{puzyn2010recent}, toxicity prediction~\citep{hong2008mold2}, virtual screening~\citep{xue2000molecular}, combinatorial chemistry~\citep{xue2000molecular}, the organisation of latent spaces of deep generative models~\citep{gomez2018automatic} by chemical properties of interest, and countless other areas of cheminformatics. In the context of our work, we formally define physicochemical-descriptor vectors~(PDVs) as follows:

\begin{definition}[Physicochemical-Descriptor Vector]
Let $\mathcal{R}$ be a non-feature-based molecular representation (such as a SMILES string or a molecular graph). A physico\-chemical-descriptor function is an algorithmic procedure $f^{\text{desc}}$ that transforms $\mathcal{R}$ into a real number $f^{\text{desc}}(\mathcal{R}) \in \mathbb{R}$ that quantitatively describes a physicochemical property of the input molecule. A vector $\mathcal{F}$ composed of the outputs of $l \in \mathbb{N}$ distinct physicochemical-descriptor functions,
$$ \mathcal{F} \coloneqq \mathcal (f_1^{\text{desc}}(\mathcal{R}), ..., f_l^{\text{desc}}(\mathcal{R})) \in \mathbb{R}^{l}, $$
is called a physicochemical-descriptor vector for the molecule represented by $\mathcal{R}$.
\end{definition}

Ideally physicochemical descriptors should be interpretable, conceptually simply, and computationally efficient to generate. Moreover, they should preferably give distinct values for distinct input molecules, not require experimental measurements, and vary continuously with continuous changes in molecular structure. Several thousand physicochemical descriptors have been developed in the last decades~\citep{consonni2010molecular} which fulfill these desired properties to varying degrees. Below we take a closer look at two well-known descriptors that have been frequently used throughout the literature.

\begin{example}[Predicted $\log(P)$] \label{ex: logP}
	The $1$-octanol-water partition coefficient $P$ of a molecule $\mathcal{M}$ describes the ratio of its concentrations in an (immiscible) mixture of $1$-octanol and water at equilibrium:
	$$ P \coloneqq \frac{[\mathcal{M}]_{\text{$1$-octanol}}}{[\mathcal{M}]_{\text{water}}} . $$
	Since $1$-octanol is a fatty alcohol, the decadic logarithm $\log(P)$ is used as a measure of the lipophilicity of $\mathcal{M}$. Lipophilicity in turn plays an important role in drug discovery since orally active drugs cannot be too lipophilic; the famous Lipinski's rule of five~\citep{lipinski1997experimental} states that a $\log(P)$ over $5$ is a risk factor for poor absorption or permeation of an oral drug.
	
	It is common to determine $\log(P)$ experimentally, but it can also be treated as a physicochemical descriptor that can be approximately calculated \textit{in-silico}. A simple method for the computational estimation of $\log(P)$ has been given by~\citet{wildman1999prediction} who employed a summation over the contributions of individual atoms:
	$$ \log(P) = \sum_{a \in A} c(f(a)). $$
	Here $A$ is a set containing symbolic representations of the atoms in $\mathcal{M}$. The function 
	$$f : A \to \{1, ..., 68\} $$
	assigns each atom to one of $68$ developed classes according to properties related to element type, neighbouring atoms and aromaticity. The function
	$$c : \{1, ..., 68\} \to \mathbb{R}$$
	then maps each atomic class to its additive contribution towards the overall $\log(P)$ of~$\mathcal{M}$. The number of atomic classes as well as the functions $f$ and $c$ were constructed and calibrated on the basis of a training set comprising of experimental $\log(P)$-data. 
	
	While it might seem simplistic at first sight to try to predict the lipophilicity of a compound merely via a sum over atomic contributions, this method has been shown to provide reasonably accurate and computationally fast $\log(P)$-estimations for many compounds. The intuitive motivation behind the summation of atomic contributions is that the $\log(P)$ may be approximately thought of as reflecting a balance between hydrophobic and hydrophilic parts of a compound whose contributions can be summed up. However, there are limitations to this strategy, and predictions might be less accurate for molecules that significantly deviate from the training set. The Wildmann-Crippen method has been implemented in the \texttt{Python} cheminformatics-library \texttt{RDKit}~\citep{landrum2006rdkit}. This implementation performs a series of substructure searches on an input compound (usually given via its SMILES string) to assign the appropriate atom classes used for the calculation of the $\log(P)$ estimate.

\end{example}

\begin{example}[Balaban Index] \label{ex: balaban_index}
	Let $\mathcal{M}$ be a molecule represented via an hydrogen-depleted molecular graph $\mathcal{G} = (A,B)$ with $\vert A \vert \coloneqq n$ atoms and $\vert B \vert \coloneqq m$ bonds and let $g : B \to \mathbb{R}$ be a scalar bond featurisation function that maps single bonds to $1$, double bonds to $1/2$, triple bonds to $1/3$ and aromatic bonds to $2/3$. For $a, \tilde{a} \in A$ let 
	$$d_g(a, \tilde{a}) \coloneqq \min \left\{ \sum_{\{a_1, a_2\} \in P} g(\{a_1, a_2\}) \ \vert \ P \text{ is  a shortest path between } a \text{ and } \tilde{a} \text{ in } B	\right\} $$
	be the minimal bond-order-weighted graph distance between $a$ and $\tilde{a}$ in $\mathcal{G}$ and let
	$$s_a  \coloneqq \sum_{\bar{a} \in A} d_g(a, \bar{a}). $$
	Then the Balaban index~\citep{balaban1982highly} of $\mathcal{M}$ is defined via
	$$J^{\text{bal}}(\mathcal{G}, g) \coloneqq \frac{m}{m-n+2} \sum_{\{a, \tilde{a}\} \in B} \frac{1}{\sqrt{s_a s_{\tilde{a}}}}. $$
	It may not be straightforward to interpret the quantity $J^{\text{bal}}(\mathcal{G}, g)$ in structural terms. However, it can be shown to be extraordinary useful in detecting the existence of structural differences: for two distinct molecular graphs $\mathcal{G}_1 \neq \mathcal{G}_2$ it almost always holds that $$J^{\text{bal}}(\mathcal{G}_1, g) \neq J^{\text{bal}}
	(\mathcal{G}_2, g).$$
	The feature $J^{\text{bal}}(\mathcal{G}, g)$ may in some cases encode useful structural information for a downstream prediction task and could then help machine-learning models to detect relevant differences even between very similar input molecules. The Balaban index of a compound can be calculated quickly from its SMILES string using \texttt{RDKit}~\citep{landrum2006rdkit}. 
	
\end{example}

For the machine-learning tasks in our work, we employed a $200$-dimensional PDV composed of a diverse selection of $l = 200$ descriptors that together give a robust overall physicochemical profile of a molecule. Our chosen descriptors are identical to the ones used in the study of~\citet{fabian2020molecular} who successfully used a self-supervised transformer-architecture to learn a useful latent embedding of chemical space via the prediction of physicochemical properties from SMILES strings of chemical compounds. Each selected descriptor can be computed quickly from the compound SMILES-string using the \texttt{Python} cheminformatics-package \texttt{RDKit}~\citep{landrum2006rdkit}. A full list of our chosen descriptors, which includes the predicted $\log(P)$ from Example~\ref{ex: logP} and the Balaban index from Example~\ref{ex: balaban_index}, can be found in Table~\ref{tab: physchem_descriptors_rdkit}.

\begin{table}
	\centering
	{\renewcommand{\arraystretch}{1}
		\begin{tabular}{V{3} p{14.6cm} V{3}}
			
			\hlineB{3}
			\multicolumn{1}{V{3} c V{3}}{\textbf{Selected} \textbf{\texttt{RDKit}} \textbf{Physicochemical Descriptors}} \\
			\hlineB{3}
			
			\footnotesize BalabanJ, BertzCT, Chi0, Chi0n, Chi0v, Chi1, Chi1n, Chi1v, Chi2n, Chi2v,
			Chi3n, Chi3v, Chi4n, Chi4v, EState\_VSA1, EState\_VSA10, EState\_VSA11,
			EState\_VSA2, EState\_VSA3, EState\_VSA4, EState\_VSA5, EState\_VSA6, EState\_VSA7,
			EState\_VSA8, EState\_VSA9, ExactMolWt, FpDensityMorgan1, FpDensityMorgan2,
			FpDensityMorgan3, FractionCSP3, HallKierAlpha, HeavyAtomCount, HeavyAtomMolWt,
			Ipc, Kappa1, Kappa2, Kappa3, LabuteASA, MaxAbsEStateIndex, MaxAbsPartialCharge,
			MaxEStateIndex, MaxPartialCharge, MinAbsEStateIndex, MinAbsPartialCharge,
			MinEStateIndex, MinPartialCharge, MolLogP, MolMR, MolWt, NHOHCount, NOCount,
			NumAliphaticCarbocycles, NumAliphaticHeterocycles, NumAliphaticRings,
			NumAromaticCarbocycles, NumAromaticHeterocycles, NumAromaticRings, NumHAcceptors,
			NumHDonors, NumHeteroatoms, NumRadicalElectrons, NumRotatableBonds,
			NumSaturatedCarbocycles, NumSaturatedHeterocycles, NumSaturatedRings,
			NumValenceElectrons, PEOE\_VSA1, PEOE\_VSA10, PEOE\_VSA11, PEOE\_VSA12, PEOE\_VSA13,
			PEOE\_VSA14, PEOE\_VSA2, PEOE\_VSA3, PEOE\_VSA4, PEOE\_VSA5, PEOE\_VSA6, PEOE\_VSA7,
			PEOE\_VSA8, PEOE\_VSA9, RingCount, SMR\_VSA1, SMR\_VSA10, SMR\_VSA2, SMR\_VSA3,
			SMR\_VSA4, SMR\_VSA5, SMR\_VSA6, SMR\_VSA7, SMR\_VSA8, SMR\_VSA9, SlogP\_VSA1,
			SlogP\_VSA10, SlogP\_VSA11, SlogP\_VSA12, SlogP\_VSA2, SlogP\_VSA3, SlogP\_VSA4,
			SlogP\_VSA5, SlogP\_VSA6, SlogP\_VSA7, SlogP\_VSA8, SlogP\_VSA9, TPSA, VSA\_EState1,
			VSA\_EState10, VSA\_EState2, VSA\_EState3, VSA\_EState4, VSA\_EState5, VSA\_EState6,
			VSA\_EState7, VSA\_EState8, VSA\_EState9, fr\_Al\_COO, fr\_Al\_OH, fr\_Al\_OH\_noTert,
			fr\_ArN, fr\_Ar\_COO, fr\_Ar\_N, fr\_Ar\_NH, fr\_Ar\_OH, fr\_COO, fr\_COO2, fr\_C\_O,
			fr\_C\_O\_noCOO, fr\_C\_S, fr\_HOCCN, fr\_Imine, fr\_NH0, fr\_NH1, fr\_NH2, fr\_N\_O,
			fr\_Ndealkylation1, fr\_Ndealkylation2, fr\_Nhpyrrole, fr\_SH, fr\_aldehyde,
			fr\_alkyl\_carbamate, fr\_alkyl\_halide, fr\_allylic\_oxid, fr\_amide, fr\_amidine,
			fr\_aniline, fr\_aryl\_methyl, fr\_azide, fr\_azo, fr\_barbitur, fr\_benzene,
			fr\_benzodiazepine, fr\_bicyclic, fr\_diazo, fr\_dihydropyridine, fr\_epoxide,
			fr\_ester, fr\_ether, fr\_furan, fr\_guanido, fr\_halogen, fr\_hdrzine, fr\_hdrzone,
			fr\_imidazole, fr\_imide, fr\_isocyan, fr\_isothiocyan, fr\_ketone, fr\_ketone\_Topliss,
			fr\_lactam, fr\_lactone, fr\_methoxy, fr\_morpholine, fr\_nitrile, fr\_nitro,
			fr\_nitro\_arom, fr\_nitro\_arom\_nonortho, fr\_nitroso, fr\_oxazole, fr\_oxime,
			fr\_para\_hydroxylation, fr\_phenol, fr\_phenol\_noOrthoHbond, fr\_phos\_acid,
			fr\_phos\_ester, fr\_piperdine, fr\_piperzine, fr\_priamide, fr\_prisulfonamd,
			fr\_pyridine, fr\_quatN, fr\_sulfide, fr\_sulfonamd, fr\_sulfone, fr\_term\_acetylene,
			fr\_tetrazole, fr\_thiazole, fr\_thiocyan, fr\_thiophene, fr\_unbrch\_alkane, fr\_urea,
			qed \\
			
			\hlineB{3}
			
	\end{tabular}}
	
	\caption[Physicochemical descriptors for PDVs.]{Overview of the $200$ distinct molecular descriptors selected for the $200$-dimensional physicochemical-descriptor vectors~(PDVs) used in our experiments. Each descriptor is named after its associated command in \texttt{RDKit}~\citep{landrum2006rdkit}. The descriptors are identical to the ones used in the work of~\citet{fabian2020molecular} and cover a broad spectrum of molecular properties related to lipophilicity, druglikeness, electrotopological state, molecular refractivity, molecular surface, molecular-graph structure, charge, and fragment count. }
	
	\label{tab: physchem_descriptors_rdkit}
\end{table}

Before using the PDV for any prediction task, we canonically derived the cumulative distribution function for each individual descriptor from the training set and used it to normalise the associated descriptor-values. The final PDV $\mathcal{F}$ was thus always contained in the hypercube $[0,1]^{200}$. This normalisation-step prevents certain machine-learning algorithms such as multilayer perceptrons from putting disproportionate attention on large-range features while ignoring small-range features.

\subsection{Critical View} \label{subsec: pdvs_pro_con}

Below we give a list of advantages $\bm{(+)}$ and disadvantages $\bm{(-)}$ of PDVs as a molecular featurisation method.

\begin{itemize}

\item[$\bm{(+)}$] \textbf{Low computational cost.} PDVs can be computed quickly, even for large molecular data sets.

\item[$\bm{(+)}$] \textbf{Interpretability (in some cases).} Certain descriptors quantify straightforward physicochemical properties of a compound such as its molecular weight or its heavy atom count. Such descriptors can be immediately understood and interpreted by chemical experts. However, it is important to note that there also exists a large number of potentially useful descriptors whose chemical interpretation is not obvious at all, such as the Balaban index discussed in Example~\ref{ex: balaban_index}.

\item[$\bm{(+)}$] \textbf{Simplicity of implementation.} PDVs are easy to use and can be automatically generated without sophisticated technical knowledge via publicly available cheminformatics libraries such as the \texttt{Python}-package \texttt{RDKit}~\citep{landrum2006rdkit}.

\item[$\bm{(+)}$] \textbf{Chemically reasonable initial embedding of chemical space.} PDVs tend to automatically calculate basic physicochemical features of molecules and thus immediately provide an \textit{a priori} embedding of chemical space that is useful across a wide range of cheminformatics tasks. In particular, this means that basic chemical features that are relevant across many applications do not need to be inferred statistically from a molecular representation from scratch every time a model is trained on a new data set. This can be seen in contrast to trainable methods like GNNs which continuously need to relearn a reasonable embedding of chemical space at every new training cycle (unless they have been combined with suitable self-supervised or transfer learning approaches).

\item[$\bm{(+)}$] \textbf{Low dimensionality.} PDVs with a dimensionality $l \leq 200$ often already provide useful feature vectors for molecular machine learning tasks. This can be seen in contrast to extended-connectivity fingerprints (ECFPs) which we will introduce in Section~\ref{sec: ecfps} and which normally require a length of at least $l \geq 1024$ to reach acceptable performance. The relatively low dimensionality of PDVs can decrease the risk of costly downstream-computations and overfitting.

\item[$\bm{(+)}$] \textbf{Global receptive field.} While ECFPs and message-passing GNNs, which we will both introduce in Section~\ref{sec: ecfps} and Section~\ref{sec: gnns}, are only able to extract features from local circular substructures, PDVs can directly express global high-level properties of a molecule such as its predicted $\log(P)$.

\item[$\bm{(-)}$] \textbf{Non-differentiability.} PDVs are generally based on fixed algorithms that cannot be trained in a differentiable manner to learn features from more explicit molecular representations. In this particular sense we refer to PDVs as non-trainable. It is worth noting though that computed PDVs can still be adapted to a given data set to some extent via the application of additional downstream methods such as data-dependent feature selection and normalisation procedures (like the removal of low-variance or correlated descriptors). Strictly speaking such common procedures can certainly also be considered forms of training (i.e.~forms or learning from input data). The non-differentiability of PDVs might cause an information bottleneck through which important chemical information cannot pass.

\item[$\bm{(-)}$] \textbf{Necessity for feature selection.} Since PDVs are not differentiable and normally do not contain tunable hyperparameters, one of the few ways to adapt them to a given prediction task is via the choice of physicochemical-descriptor functions that compose the final PDV. However, the optimal descriptor-choice is almost never obvious and usually requires the application of expert-knowledge or technical feature selection algorithms.

\item[$\bm{(-)}$] \textbf{Finite number of descriptors.} Trainable molecular featurisation methods like message-passing GNNs are parametric families of functions that can represent an infinite number of potential feature mappings into a continuous real vector space. The composition of a PDV on the other hand is restricted to a finite number of a few thousand human-engineered non-trainable non-tunable physicochemical-descriptor functions. As a result, PDVs can only represent a finite number of distinct feature mappings defined by the initial descriptor-choices. This finite family of feature mappings might not always be sufficient to describe all of the relevant information for a chemical prediction task; a suitable descriptor-function for the problem might simply not be available.

\end{itemize}

\section{Extended-Connectivity Fingerprints} \label{sec: ecfps}

\subsection{A Short Background on Structural Fingerprints}

In this section, we will discuss the extended-connectivity fingerprint~(ECFP)~\citep{morgan1965generation,rogers2010extended}, which is part of a larger family of molecular featurisation methods called \textit{structural fingerprints}. The essential idea behind structural fingerprints is to algorithmically detect and encode features that express parts of the chemical structure of a molecule. Often, structural fingerprints can be represented as bit vectors that express the presence or absence of certain chemical substructures within the input compound. 

Well-known structural fingerprints include the $166$-bit MACCS fingerprint~\citep{durant2002reoptimization} and the $881$-bit PubChem fingerprint~\citep{pubchemfp2009,han2008developing}. Both of these examples belong to a family of \textit{dictionary-based} fingerprints that check the presence or absence of substructures from a predefined dictionary. The dictionaries underlying MACCS and PubChem fingerprints include $166$ and $881$ unique substructures respectively, which explains the dimensionalities of these featurisations.

Another important family of structural fingerprints is given by \textit{enumeration-based} fingerprints. These fingerprints exhaustively enumerate all substructures of a molecule that can be detected via a certain search strategy. An example are Daylight fingerprints~\citep{daylightfp} which enumerate all substructures within an input compound that correspond to groups of atoms and bonds that form paths up to a chosen length. 
As we will see in this section, ECFPs too fall into the category of enumeration-based fingerprints since they enumerate all circular substructures in a molecule up to a chosen radius.
Because enumeration-based fingerprints are not constrained by a fixed-sized dictionary of substructures, they can often detect an extremely large number of distinct fragments. It is thus not immediately obvious how the substructures found within a particular compound should be transformed into a reasonably-sized bit-vector. The standard technique to address this issue is by using a hash function to fold the detected fragments into a bit-vector representation of feasible dimensionality and accept the fact that colliding bits might lead to a certain level of ambiguity. In Chapter~\ref{chap: ecfps_sort_and_slice} and in Section~\ref{sec: diff_substruc_pool}, we will explore some interesting alternatives to this hashing procedure for the vectorisation of structural fingerprints.

In our work, we focus on ECFPs out of all available structural fingerprints. Our reasons for this include that ECFPs have been used widely in the area of cheminformatics and molecular machine-learning and are well-known to perform well on a diverse set of tasks~\citep{riniker2013open,duvenaud2015convolutional,webel2020revealing,rogers2005using,alvarsson2014ligand,gilmer2017neural}. Moreover, as we will discover in Section~\ref{sec: gnns}, ECFPs also exhibit some striking similarities with modern message-passing-GNN architectures. In some sense, ECFPs can be interpreted as a non-differentiable version of message-passing GNNs, which makes a comparison between these two approaches interesting. We will use the rest of this section to give a mathematical description of ECFPs along with a critical discussion of their strengths and weaknesses.

\subsection{Mathematical Description}\label{subsec: ecfps_description}

The ECFP algorithm is a molecular featurisation method that maps a SMILES string (or its associated molecular graph) to a structural fingerprint that can be represented as a bit vector. The fundamental ideas underlying ECFPs were originally introduced by~\citet{morgan1965generation} in $1965$; however, modern implementations of the ECFP are largely based on a detailed technical description given by~\citet{rogers2010extended} in $2010$. ECFPs provide a simple yet powerful technique to encode information about the chemical structure of an input compound. The algorithm is dependent on two predefined hyperparameters: the desired fingerprint length $l \in \mathbb{N}$ and the maximum radius $R \in \mathbb{N}_{0}$ of the receptive field. An ECFP of length $l$ takes the form of a binary vector
$$\mathcal{F} = (f_1,...,f_l) \in \{0,1\}^{l}.$$ 
Up to a certain level of ambiguity due to bit collisions which we will discuss below, each component $f_i$ in $\mathcal{F}$ is associated with the presence or absence of a particular \textit{circular subgraph}, equipped with specific atom and bond features inherited from the input molecule, centered around a given atom.
\begin{definition}[Circular Subgraph]
	Let $\mathcal{G} = (A,B)$ be a molecular graph with atom set $A$ and bond set $B$ and let $a \in A$. Denote with
	$$N(a) \coloneqq \{\tilde{a} \in A \ \vert \ \{ \tilde{a}, a	\} \in B \} $$
	the set of neighbouring atoms of $a$ and denote with
	$$ M(a) \coloneqq \{\{a_1, a_2\} \in B \ \vert \ a \in \{a_1, a_2\} \} $$
	the set of bonds that are attached to $a$. Now let $(A^a_r, B^a_r)$ be a sequence of subgraphs of $(A,B)$ for $r \in \mathbb{N}_{0}$ constructed via the following iterative scheme:
	\begin{gather*}
	(A^a_0, B^a_0) \coloneqq (\{a\}, \{\}), \\[8pt]
	A^a_r = A^a_{r-1} \cup \bigcup_{\tilde{a} \in A^a_{r-1}} N(\tilde{a}), \quad
	B^a_r = B^a_{r-1} \cup \bigcup_{\tilde{a} \in A^a_{r-1}} M(\tilde{a}).
	\end{gather*}
	Then $(A^a_r, B^a_r)$ is called the circular subgraph of $\mathcal{G}$ with center atom $a$ and radius $r$. If $\mathcal{G}$ is equipped with atom and bond feature vectors, then these are inherited by $(A^a_r, B^a_r)$.
\end{definition}
Circular subgraphs with varying radii for an example compound are depicted in Figure~\ref{fig:circular_subgraphs_example}.
\begin{figure}[ht]
	\centering
	\includegraphics[width = 0.8 \linewidth]{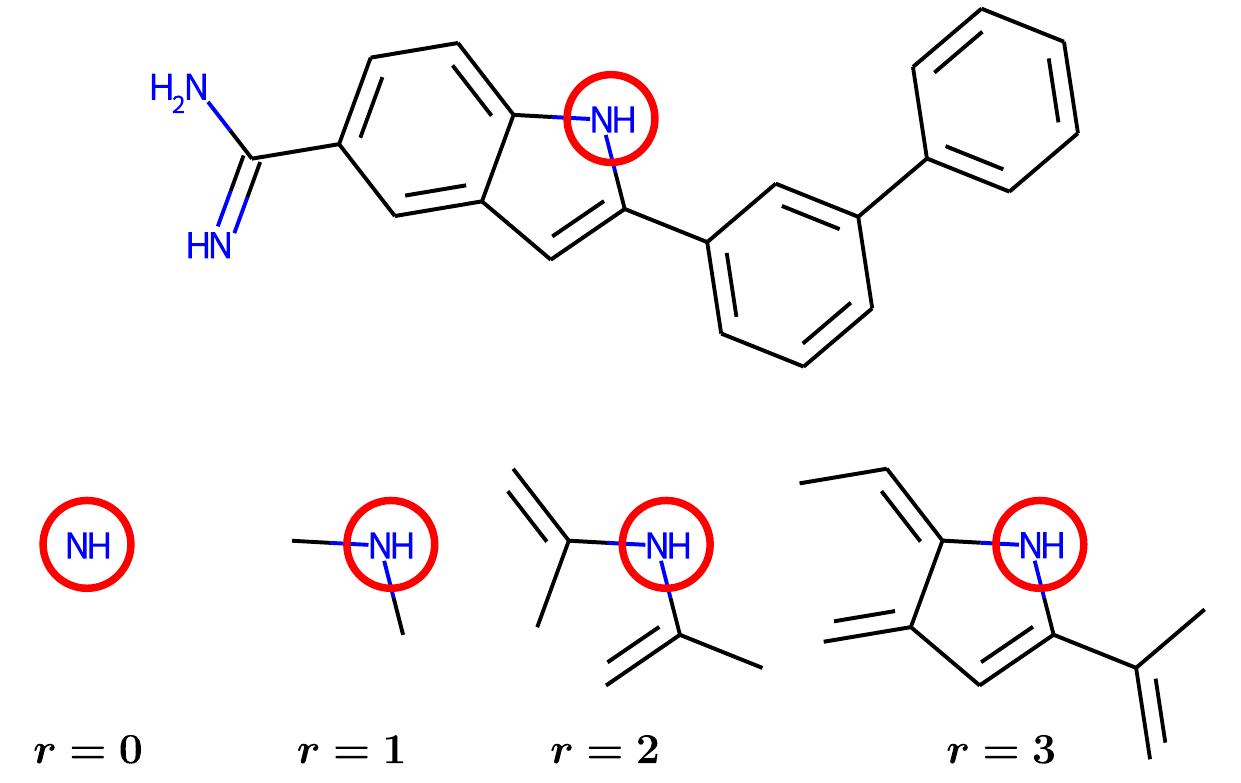}
	\caption[Circular chemical subgraphs of varying radii.]{Circular subgraphs of varying radii for a central nitrogen atom in an example molecule.}
	\label{fig:circular_subgraphs_example}
\end{figure}
The fingerprint hyperparameter $R$ defines the maximum radius of any circular subgraph whose presence or absence is indicated in the fingerprint $\mathcal{F}$. Circular subgraphs that are structurally isomorphic are further distinguished according to their inherited atom and bond features, i.e.~two structurally isomorphic circular subgraphs with distinct atom or bond features correspond to different components of $\mathcal{F}$. For chemical bonds, this distinction is made on the basis of simple bond types: single, double, triple, or aromatic. To distinguish atoms, ECFPs usually either use standard or pharmacophoric atom features and we discuss both variants below. 

There also exists a non-binary version of the ECFP in which each component of $\mathcal{F}$ does not simply indicate the presence or absence of a circular subgraph, but the exact number of occurrences of the subgraph within the input compound. This version is sometimes referred to as \textit{ECFP with counts} as opposed to the binary version mentioned above which is also called \textit{ECFP without counts}. Throughout this thesis, unless specifically stated otherwise, we focus on ECFPs without counts, as they are used significantly more frequently in practice and allow for an easier algorithmic description. However, it is straightforward to extend the methods described in this section to ECFPs with counts.

We now give a technical description of the algorithm used for the generation of ECFPs as specified in the article of~\citet{rogers2010extended}. \vspace{6pt}

\begin{mdframed} \vspace{10pt}
\begin{center} \textbf{ECFP algorithm} \end{center}
\noindent \hrulefill \vspace{10pt}

\noindent \textbf{List of inputs:}
\begin{itemize}
	\item A molecular compound $\mathcal{M}$, usually represented via a SMILES string $\mathcal{S}$. Note that $\mathcal{S}$ contains the structural information of the molecular graph $\mathcal{G} = (A,B)$.
	
	\item A function $f$ specifying the atom feature vectors for $A$ using either standard or pharmacophoric features.
	
	\item A function $g$ specifying the bond types for $B$ using encodings for labels in the set $\{\text{single, double, triple, aromatic}\}$.
	
	\item A fingerprint length $l \in \mathbb{N}$, often chosen to be in $\{1024, 2048\}$.
	
	\item A fingerprint radius $R \in \mathbb{N}_{0}$, often chosen to be in $\{1, 2, 3\}$.
\end{itemize}

\noindent \textbf{Initialisation:}
A deterministic hash function $h$ is chosen which maps integer vectors of arbitrary lengths into a large index set such as
$$ \{1,...,2^{32}\}$$ in a pseudo-random and uniformly distributed manner. The map $h$ is then used to hash each atom feature vector $f(a)$ to a single integer $ h(f(a)) \in \{1,...,2^{32}\}$. The integers in the set
$$I_{0} \coloneqq \{ h(f(a)) \ \vert \ a \in A\} \subseteq \{1,...,2^{32}\} $$
are called \textit{initial atom identifiers}.  \vspace{10pt}

\noindent \textbf{Iterative phase for substructure-enumeration:}
The initialisation step is followed by $R$ iterative steps, whereby at each step all atom identifiers are updated and saved. At each step $r \in \{1,..., R\}$, each atom $a \in A$ first starts off equipped with its respective atom identifier in $I_{r-1}$. The new integer atom identifier for $a$ is generated as follows:

\begin{enumerate}

\item A vector of integers $J_{a,r}$ is created that contains the integer $r$ in the first position and the current atom identifier of $a$ in the second position.

\item All neighbouring atoms that have a bond to $a$ are ordered in lexicographic order according to their bond types (single $= 1$, double $ = 2$, triple $ = 3$, aromatic $ = 4$) and then according to their current integer atom identifiers. 

\item A loop is performed over all neighbouring atoms according to this established order; for each neighbouring atom, first the bond type and then the current atom identifier is appended to the end of the integer vector $J_{a,r}$. 

\item The integer $h(J_{a,r}) \in \{1,...,2^{32}\}$ is computed and interpreted as the new atom identifier for $a$. 

\end{enumerate}
After the set of new atom identifiers
$$ I_{r} \coloneqq \{ h(J_{a,r}) \ \vert \ a \in A	\} \subseteq \{1,...,2^{32}\} $$
has been computed, the current atom identifiers are simultaneously updated to correspond to the elements in $I_{r}$. Note that by construction each atom identifier in $I_{r}$ represents an integer label that can be mapped to a particular circular subgraph with radius $r$ and associated atom and bond features. Thus, if two circular subgraphs receive the same integer label they can be assumed to be isomorphic (ignoring rare ambiguities caused by possible hash collision that could lead to two distinct circular subgraphs having the same integer label).

After the completion of $R$ iterations, the atom identifiers from all iterative stages are collected in one set:
$$ \mathcal{I} \coloneqq  \bigcup_{r = 0}^{R} I_r \subseteq \{1, ..., 2^{32}\}.$$
\newline
\noindent \textbf{Structural-duplicate removal:}
Note that it is possible for several distinct atom identifiers in $\mathcal{I}$ to correspond to the same circular subgraph, for example when an iteratively constructed circular subgraph can be traced back to two or more possible center atoms. To eliminate this redundancy, all but one of the structural-duplicate-identifiers are systematically removed from $\mathcal{I}$ using a method based on sets of bond features. For sake of brevity, further details of this technical removal process are omitted from this description, but full details can be found in the article from~\citet{rogers2010extended}.\newline

\noindent \textbf{Creation of hashed vectorial fingerprint:}
After the elimination of all structural duplicates in $\mathcal{I}$, yet another (arbitrary) standard hash function
$$\tilde{h} : \{1, ..., 2^{32}\} \to \{1, ..., l\} $$
is chosen to map the integers in $\mathcal{I}$ into the much smaller set
$$\{1, ..., l\}. $$
This folding procedure creates a fingerprint-set
$$F \coloneqq \{\tilde{h}(i) \ \vert \ i \in \mathcal{I} \} \subseteq \{1, ..., l\}. $$
Finally, $F$ is transformed into a binary fingerprint-vector
$$\mathcal{F} \coloneqq (f_1, ..., f_l) \in \{0,1\}^l $$
by setting
$$
f_i \coloneqq \left\{
\begin{array}{ll}
1 & \quad i \in F, \\
0 & \quad i \notin F, \\
\end{array}
\right. \quad \quad  \forall i \in \{1,...,l\}.  $$
\end{mdframed}
\vspace{19pt}

The total number of detected substructures in a chemical data set naturally increases with the fingerprint radius $R$. Note that the larger the fingerprint dimension $l$ gets compared to the number of detected substructures (as controlled by $R$), the more likely it becomes that there are no bit collisions. In other words, the more likely it becomes that each dimensional component $f_i$ of $\mathcal{F}$ informs about the presence or absence of one particular atom identifier in $\mathcal{I}$ and thus about the presence or absence of one unambiguous circular subgraph with atom and bond features within the input compound (ignoring rare ambiguities where two distinct circular subgraphs end up with the same atom identifier due to hash collisions). The bit $f_i$ is then set to $1$ if and only if the circular substructure is present anywhere in the molecule, otherwise $f_i$ is set to~$0$. However, if $l$ becomes small relative to the number of detected substructures then more and more hash collisions start to occur, which degrades the quality of the fingerprint. Such hash collisions cause a fingerprint-component $f_i$ to become ambiguous and correspond to one out of several possible atom identifiers and thus to distinct circular substructures. Therefore, $l$ must be chosen sufficiently large as to guarantee the expressivity of the ECFP. 

In the literature, ECFP featurisations with radius $R$ are often written in the form ECFP$2R$ with $2R$ being the fingerprint diameter. For example, the frequently used $1024$-bit ECFP$4$-featurisation describes an ECFP with radius $R = 2$ and length $l = 1024$. In our work, we used the \texttt{Python} cheminformatics-library \texttt{RDKit}~\citep{landrum2006rdkit} to generate ECFPs from SMILES strings. 

\subsection{Standard and Pharmacophoric Atom Features} \label{subsec: ecfp_and_fcfp_atom_features}

To distinguish atoms, ECFPs as implemented in \texttt{RDKit}~\citep{landrum2006rdkit} use six standard features. Optionally, the algorithm also allows for the stereochemical distinction between atoms with respect to tetrahedral R-S chirality. There also exist alternative binary atom features that were designed to be more reflective of the abstract function that an atom might play in pharmacological chemistry. When these \textit{pharmacophoric} atom features~\citep{rogers2010extended} are used instead of the standard features, then one speaks of functional-connectivity fingerprints (FCFPs). The standard and pharmacophoric atom features for the ECFP algorithm are listed in Table~\ref{tab: ecfp_atom_features}.
\begin{table}[h]
	\centering
	{\renewcommand{\arraystretch}{1.3}
		\begin{tabular}{V{3} p{7cm}| p{7cm} V{3}}
			
			\hlineB{3}
			\multicolumn{1}{V{3} c|}{\textbf{Standard Atom Features}} & \multicolumn{1}{c V{3}}{\textbf{Pharmacophoric Atom Features}}  \\
			\hline 
			Atomic number & Hydrogen-bond acceptor (yes/no) \\
			Total degree  & Hydrogen-bond donor (yes/no) \\
			Number of hydrogen neighbours & Negatively ionisable (yes/no) \\
			Formal charge & Positively ionisable (yes/no) \\
			Isotope & Aromatic (yes/no) \\
			Part of a ring (yes/no) & Halogen (yes/no) \\
			Optional: tetrahedral R-S chirality & Optional: tetrahedral R-S chirality   \\ 
			
			\hlineB{3}
			
	\end{tabular}}
	
	\caption[Atom features for ECFPs.]{Standard and pharmacophoric atom features used for the two versions of the ECFP algorithm.}
	
	\label{tab: ecfp_atom_features}
\end{table}

As can be seen, there is an overlap between the standard atom features for ECFPs and the atom features in our molecular graphs. In certain molecular machine learning applications, replacing standard with pharmacophoric atom features might lead to increased performance and decreased learning time since important high-level atomic properties are presented to the learning model from the start and do not need to be inferred statistically. However, standard atom features contain more detailed information that could still be relevant for the prediction task and thus be leveraged by the learning algorithm.

\subsection{Critical View} \label{subsec: ecfps_pro_con}

Below we give a list of advantages $\bm{(+)}$ and disadvantages $\bm{(-)}$ of the ECFP algorithm as a molecular featurisation method.

\begin{itemize}
	\item[$\bm{(+)}$] \textbf{Low computational cost.} ECFPs can be computed rapidly, even for large molecular data sets.
	
	\item[$\bm{(+)}$] \textbf{Interpretability.} Since ECFPs contain information about the presence or absence of concrete chemical substructures (up to hash collisions) they can often be understood and interpreted in a straightforward manner.
	
	\item[$\bm{(+)}$] \textbf{Simplicity of implementation.} ECFPs are easy to use and can be automatically generated without sophisticated technical knowledge via publicly available cheminformatics libraries such as the \texttt{Python}-package \texttt{RDKit}~\citep{landrum2006rdkit}.
	
	\item[$\bm{(+)}$] \textbf{Chemically reasonable initial embedding of chemical space.} ECFPs automatically express basic structural features of molecules and thus immediately provide an \textit{a priori} embedding of chemical space that is useful across a wide range of cheminformatics tasks. In particular, this implies that basic structural features that are important across many applications do not need to be relearned from a molecular representation from scratch every time a model is trained on a new data set. This can be seen in contrast to trainable methods like GNNs which continuously need to relearn a reasonable embedding of chemical space at every new training cycle (unless they have been combined with suitable self-supervised or transfer learning approaches).
	
	\item[$\bm{(+)}$] \textbf{Arbitrary circular subgraphs.} ECFPs do not simply check the existence of substructures from a predefined finite list of chemical substructures (like dictionary-based structural fingerprints do), but are able to distinguish between an essentially infinite number of chemical subgraphs, albeit only circular ones.
	
	\item[$\bm{(-)}$] \textbf{Non-differentiability.} After its hyperparameters have been chosen, the ECFP transformation becomes a fixed method that produces the same task-agnostic features for a given compound across data sets. In particular, ECFPs cannot be trained in a differentiable manner to learn features from more explicit molecular representations. In this particular sense we refer to ECFPs as non-trainable. It should be noted though that computed ECFPs can nevertheless be adapted to a given data set to some degree via the application of additional downstream methods such as data-dependent feature selection and normalisation procedures; strictly speaking such procedures can certainly also be considered forms of training (i.e.~forms or learning from input data) and we will explore some of these techniques in Chapter~\ref{chap: ecfps_sort_and_slice}. The non-differentiability of ECFPs might cause an information bottleneck through which important chemical information cannot pass.

	\item[$\bm{(-)}$] \textbf{Trade-off between dimensionality and hash collisions.} Decreasing the fingerprint dimension~$l$ relative to the number of detected substructures increases the number of hash collisions. These collisions cause a loss of information and interpretability due to an increasing inability of the fingerprint to distinguish between non-identical circular substructures. Thus, lengths of $l \geq 1024$ are often a necessity to reach an acceptable performance. In a machine learning setting, this high dimensionality can lead to costly downstream-computations and increased risk of overfitting.
	
	\item[$\bm{(-)}$] \textbf{Locality of receptive field.} ECFPs are based on a local neighbourhood-aggregation scheme for atoms in a molecule. By design, they are thus only capable of indicating the existence of local circular subgraphs. In particular, they cannot directly express global properties of a molecule.

\end{itemize}

\section{Message-Passing Graph Neural Networks} \label{sec: gnns}

\subsection{Mathematical Description} \label{subsec: gnns_description}

In $2015$, \citet{duvenaud2015convolutional} published an influential article where they proposed an adaptive counterpart to ECFPs. Their goal was to overcome some of the inherent limitations of the ECFP featurisation such as non-differentiability and high dimensionality. Their work resulted in a trainable GNN architecture that could directly process molecular graphs of varying sizes as input; they reported competitive performance of their method relative to classical ECFPs at several canonical molecular prediction tasks. After this initial step, a multitude of other promising GNN architectures appeared rapidly within the molecular machine-learning community~\citep{kipf2016semi, kearnes2016molecular, hu2019strategies,yang2019analyzing, wu2020comprehensive, wieder2020compact, li2015gated,battaglia2016interaction,defferrard2016convolutional,liu2019chemi}. In a seminal article from $2017$,~\citet{gilmer2017neural} managed to subsume almost all modern GNN architectures under an overarching mathematical framework which was further generalised and brought into its modern form by~\citet{bronstein2017geometric,bronstein2021geometric}. Below we give a brief technical description of this framework known as \textit{message-passing}. \vspace{18pt}

\begin{mdframed} \vspace{10pt}
\begin{center} \textbf{Message-passing-GNN algorithm} \end{center}
\noindent \hrulefill \vspace{10pt}

\noindent \textbf{List of inputs:}
\begin{itemize}
	\item A chemical compound $\mathcal{M}$, represented via its molecular graph $\mathcal{G} = (A,B)$.
	
	\item An function $f$ specifying the atom feature vectors for $A$.
	
	\item A function $g$ specifying the bond feature vectors for $B$.
	
	\item A fingerprint length $l \in \mathbb{N}$, often chosen to be in $\{50,...,500\}$.
	
	\item A fingerprint radius $R \in \mathbb{N}_{0}$, often chosen to be in $\{1, 2, 3, 4, 5\}$.
\end{itemize}

\noindent \textbf{Message-passing phase:}
The first part of the algorithm consists of a message-passing phase in which the atom feature vectors of $\mathcal{G}$ are iteratively updated over $r \in \{1,...,R\}$ steps via a local neighbourhood-aggregation scheme. We denote the set of neighbours of an atom $a \in A$ as
$$N(a)\coloneqq \{\tilde{a} \in A: \{a, \tilde{a}\} \in B\}. $$
At each step $r$, every atom $a \in A$ updates its associated feature vector from $f_{r-1}(a)$ to $f_{r}(a)$ according to the following recursive relations:
\begin{equation}
\begin{gathered} \label{eq: mpnn_equations}
f_{0}(a) \coloneqq f(a), \\[12pt]
m_{r}(a) = \bigoplus \ \left\{w_r(f_{r-1}(a), f_{r-1}(\tilde{a}), g(\{a, \tilde{a}\})) \ \vert \ \tilde{a} \in N(a) \right\}_{\text{mul}}, \\[12pt]
f_{r}(a) = u_r(f_{r-1}(a), m_{r}(a)). \\[4.5pt]
\end{gathered}
\end{equation}
The vector-valued functions $(w_r)_{r = 1}^{R}$ are called message-passing functions and the vector-valued functions $(u_r)_{r = 1}^{R}$ are called atom-updating functions; both types of maps frequently contain trainable neural networks. The vector $m_{r}(a)$ can be interpreted as an aggregated message that the atom $a$ receives at step $r$ from neighbouring atoms along its associated chemical bonds to update its current feature vector.

Multisets, i.e.~sets that can contain multiple instances of the same element, denoted via $\{\}_{\text{mul}}$ instead of $\{\}$, are required in the above recursion. This is because two atom neighbours $\tilde{a}_1, \tilde{a}_2$ of $a$ could in theory produce identical messages
$$w_r(f_{r-1}(a), f_{r-1}(\tilde{a}_1), g(\{a, \tilde{a}_1\})) = w_r(f_{r-1}(a), f_{r-1}(\tilde{a}_2), g(\{a, \tilde{a}_2\})), $$
but two such messages should still be counted as separate elements. 

The $\bm{\oplus}$-symbol represents a placeholder for a vector-valued and permutation-invariant multiset-function. Note that permutation-invariance is a basic requirement that each multiset functions must possess in order to be a well-defined function in the first place. The permutation-invariance is necessary to guarantee that the aggregated message $m_{r}(a)$ is invariant under any (arbitrary) ordering imposed on the neighbours of~$a$. Possible operators for $\bm{\oplus}$ include summation, averaging, or the componentwise computation of maxima. In most cases, however, the sum-operator is selected,
\begin{equation} \label{eq: sum_op_pool}
\bigoplus \coloneqq \sum,
\end{equation}
and we assume this choice by default unless stated otherwise.

The multiset of atom feature vectors at a particular step $r \in \{0,...,R\}$,
$$f_r(A)_{\text{mul}} \coloneqq \{f_r(a) \ \vert \ a \in A\}_{\text{mul}},$$ is often imagined to be located at the $r$-th layer of a multilayer graph with $R + 1$ layers. 
Note that the multiset of bond feature vectors 
$$g(B)_{\text{mul}} \coloneqq \{ g(\{a,\tilde{a}\}) \ \vert \ \{a, \tilde{a}\} \in B	\}_{\text{mul}}$$ usually does not get updated, although this rule was successfully broken by~\citet{kearnes2016molecular}. \newline

\noindent \textbf{Global-pooling step:}
The message-passing phase is followed by a \textit{global-pooling} step at which the multisets of computed atom feature vectors $(f_r(A)_{\text{mul}})_{r = 0}^{R}$ are summarised to a single $l$-dimensional vectorial representation of $\mathcal{G}$. In the most general case, this pooling operation is accomplished via the application of a sequence of vector-valued multiset-functions $(\tilde{\bm{\oplus}}_r)_{r=0}^{R}$ along the sequence of graph layers. The layerwise outputs can then be combined via another vector-valued function $q$ to a final graph-level fingerprint $\mathcal{F}$ of length~$l$. In mathematical terms the global pooling operation thus takes the following general form:
\begin{equation} \label{eq: mpnn_pooling}
q(\tilde{\bm{\oplus}}_0 f_0(A)_{\text{mul}},...,\tilde{\bm{\oplus}}_R f_R(A)_{\text{mul}}) \eqqcolon \mathcal{F} \in \mathbb{R}^{l}.
\end{equation}

Similar to before, the pooling maps $(\tilde{\bm{\oplus}}_r)_{r=0}^{R}$ are chosen to be permutation-invariant multiset functions as to guarantee their invariance under any (arbitrary) atom ordering imposed on the input graph $\mathcal{G}$ and to guarantee that multiple instances of identical atom feature vectors are treated as separate elements. Often each $\tilde{\bm{\oplus}}_r$ simply represents a summation or averaging operator, or the componentwise computation of maxima, although more powerful and expressive pooling maps exist. An overview of the total molecular-featurisation process via message-passing GNNs is depicted in Figure~\ref{fig:mpgnn_overview}.
\end{mdframed} \vspace{19pt}
\begin{figure}[h]
	\centering
	\includegraphics[width = 0.98 \linewidth]{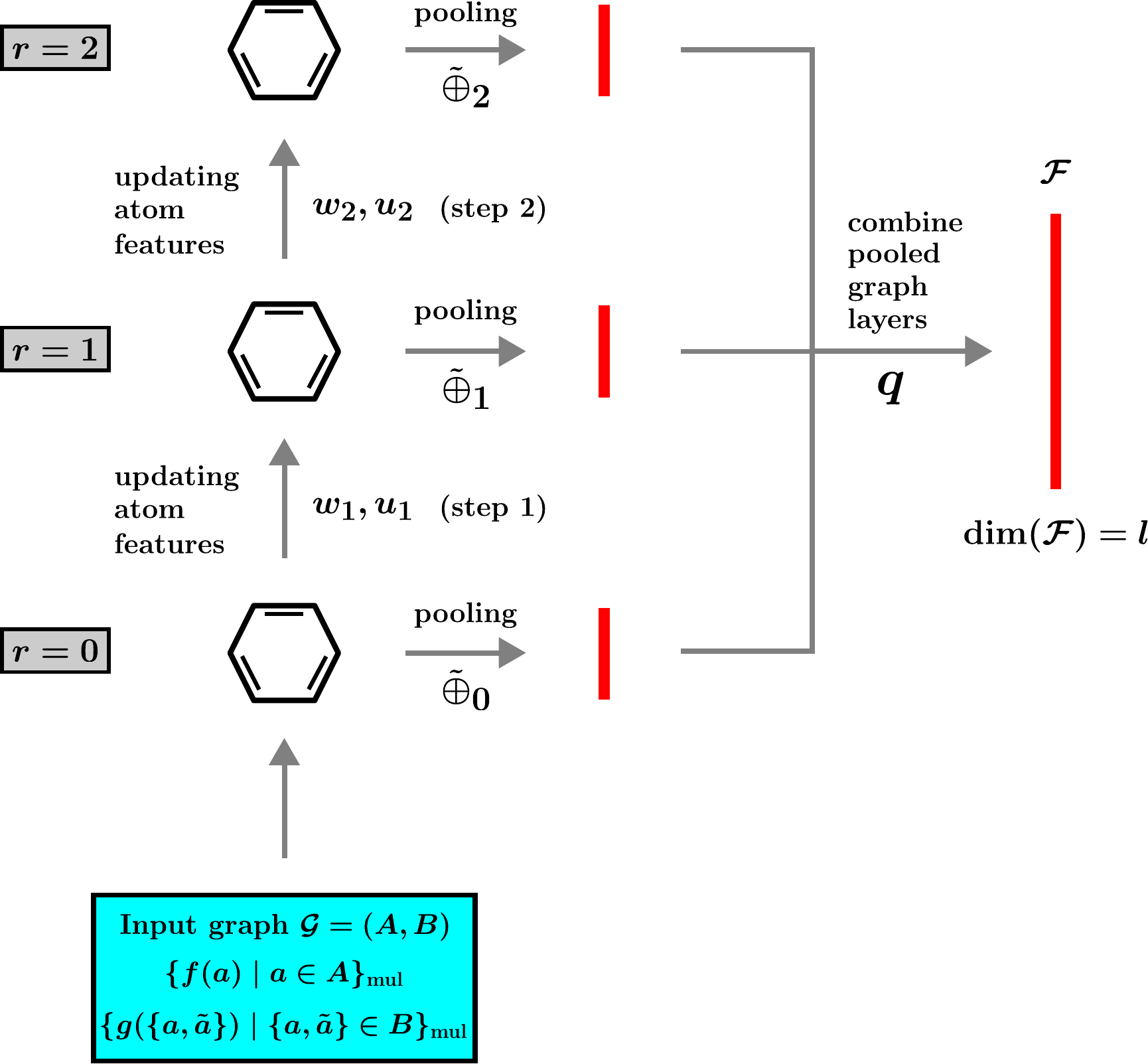}
	\caption[Molecular featurisation via message-passing GNN.]{Schematic overview of the molecular-featurisation mechanism of a message-passing graph neural network~(GNN) with radius $R = 2$. All depicted functions may contain trainable deep-learning components.}
	\label{fig:mpgnn_overview}
\end{figure}

We can see that message-passing GNNs exhibit some striking similarities with ECFPs. The fingerprint radius $R$ of a GNN defines the maximum diameter of its circular receptive field and plays an analogous role to the hyperparameter $R$ of the same name for ECFPs. In both featurisation methods a local neighbourhood-aggregation scheme is applied to iteratively update the atom feature vectors of the molecular graph. Unlike ECFPs, however, message-passing GNNs can learn from graph-shaped input data in a differentiable manner: the functions $(w_r)_{r = 1}^{R}, (u_r)_{r = 1}^{R}$ normally contain a trainable deep-learning component. The graph-level feature vector $\mathcal{F}$ can be fed into a neural prediction head (such as a multilayer perceptron) which enables the resulting end-to-end architecture to train its weights on a supervised task via backpropagation and standard gradient-based optimisation algorithms.

\subsection{Graph Convolutional Networks} \label{subsec: gcns_description}

We now describe an early GNN model that has been used frequently in the literature due to its relative simplicity and computational efficiency: the graph convolutional network~(GCN) introduced in $2016$ by~\citet{kipf2016semi}. Let

\begin{itemize}

\item $S \in \{0,1\}^{n \times n}$ be the adjacency matrix of a molecular graph $\mathcal{G} = (A,B)$ with atoms $A = \{a_1, ..., a_n\}$,

\item $\bar{S} \coloneqq S + I$ be a modified version of the adjacency matrix with $1$s on the diagonal (here $I$ is the $n$-dimensional identity-matrix which is added to guarantee self-loops in the graph so that during the node-feature updating-process each node does not only take into account the features of its neighbouring nodes but also its own features),

\item $\bar{D}$ be the diagonal degree-matrix of $\bar{S}$ defined via $\bar{D}_{i,i} \coloneqq \sum_{j = 1}^n \bar{S}_{i,j}$ and $\bar{D}_{i,j} = 0$ for $i \neq j$,

\item $F_{r} \in \mathbb{R}^{n \times k}$ be a matrix whose rows contain the transposed atom feature vectors $(f_r(a_i)^{T})_{i = 1}^{n}$ at step $r$,

\item $W_{r} \in \mathbb{R}^{k \times l}$ be a matrix of trainable weights associated with the $r$-th graph layer, and

\item $\text{ReLU}(x) \coloneqq \max\{0,x\}$ be the well-known rectified linear unit~(ReLU) activation-function.

\end{itemize}
Then the atom feature vector updates in a GCN are governed by the following dynamics:
\begin{equation} \label{eq: gcn_matrix_recursion}
F_{r} \coloneqq \text{ReLU}(\bar{D}^{-\frac{1}{2}} \bar{S} \bar{D}^{-\frac{1}{2}} F_{r-1} W_{r-1}).
\end{equation}
This updating rule can be motivated by a first-order approximation of localised spectral filters on graphs. For details on this derivation, we refer the reader to the original article~\citep{kipf2016semi}. We now show that GCNs are indeed message-passing GNNs. 

\begin{proposition}[Message-Passing for GCNs] \label{prop: gcns_are_mpgnns}
The GCN model falls into the class of message-passing GNNs, i.e.~the atom feature vector updating process of GCNs can be mathematically expressed via the message-passing scheme described in Equations~\ref{eq: mpnn_equations}.
\end{proposition}

\begin{proof}
Note that for each $a_i \in A$ the associated entry of the diagonal matrix $\bar{D}$ takes the form
$$\bar{D}_{i,i} = \sum_{j = 1}^n \bar{S}_{i,j} = \deg(a_i) + 1 $$
and we can thus write $\bar{D}_{i,i} \coloneqq \bar{D}(a_i)$. Now let $a_i, a_j \in A$ be two neighbouring atoms in $\mathcal{G}$, i.e.~$a_j \in N(a_i)$. We choose
$$w_r(f_{r-1}(a_i), f_{r-1}(a_j), g(\{a_i, a_j\})) \coloneqq \bar{D}(a_i)^{-1/2} \bar{D}(a_j)^{-1/2} f_{r-1}(a_j) $$
as our message-passing functions and
$$u_r(f_{r-1}(a_i), m_{r}(a_i)) \coloneqq \text{ReLU}(W^{T}_{r-1}(\bar{D}(a_i)^{-1} f_{r-1}(a_i) + m_r(a_i))) $$
as our atom-updating functions. To guarantee that the functions $w_r$ and $u_r$ are indeed well-defined, we assume without loss of generality that the node degree information $\bar{D}(a_i)$ is implicitly included in each initial atom feature vector $f_0(a_i)$ and then simply gets copied from $f_{r-1}(a_i)$ to $f_r(a_i)$ to assure its availability at each iterative feature update. The aggregated message for $a_i$ is now given by
\begin{align*}
m_{r}(a_i) &= \sum \{ \bar{D}(a_i)^{-1/2} \bar{D}(a_j)^{-1/2} f_{r-1}(a_j) \ \vert \ a_j \in N(a_i) \}_{\text{mul}} \\ &= \sum_{j = 1}^n \bar{D}(a_i)^{-1/2} S_{i,j} \bar{D}(a_j)^{-1/2} f_{r-1}(a_j) .
\end{align*}
As explained in~\ref{eq: sum_op_pool} above, here the expression of the form $\sum \{...\}_{\text{mul}}$ simply represents the sum of all elements in the multiset $\{...\}_{\text{mul}}$, which means that the summation-operator is used as a permutation-invariant multiset-function. 

If we denote the $i$-th row-vector of the matrix $\bar{D}^{-\frac{1}{2}} \bar{S} \bar{D}^{-\frac{1}{2}}$ with $[\bar{D}^{-\frac{1}{2}} \bar{S} \bar{D}^{-\frac{1}{2}}]_{[i,:]}$ then it follows that
\begin{align*}
f_r(a_i) &= \text{ReLU}\big(W^{T}_{r-1}(\bar{D}(a_i)^{-1} f_{r-1}(a_i) + m_r(a_i))\big) \\
&= \text{ReLU}\bigg(W^{T}_{r-1} \sum_{j = 1}^n \bar{D}(a_i)^{-1/2} \bar{S}_{i,j} \bar{D}(a_j)^{-1/2} f_{r-1}(a_j) \bigg) \\
&= \text{ReLU}\bigg(W^{T}_{r-1} \big( [\bar{D}^{-\frac{1}{2}} \bar{S} \bar{D}^{-\frac{1}{2}}]_{[i,:]} F_{r-1} \big)^{T} \bigg) \\
&= \text{ReLU}\bigg(W^{T}_{r-1} F_{r-1}^{T} [\bar{D}^{-\frac{1}{2}} \bar{S} \bar{D}^{-\frac{1}{2}}]_{[i,:]}^{T} \bigg)
\end{align*}
We can now finally conclude that
$$ f_r(a_i)^{T} = \text{ReLU}\bigg([\bar{D}^{-\frac{1}{2}} \bar{S} \bar{D}^{-\frac{1}{2}}]_{[i,:]} F_{r-1} W_{r-1} \bigg) $$
and therefore
$$ F_{r} = \text{ReLU}(\bar{D}^{-\frac{1}{2}} \bar{S} \bar{D}^{-\frac{1}{2}} F_{r-1} W_{r-1}). $$
\end{proof}
We constructed our proof of Proposition~\ref{prop: gcns_are_mpgnns} as a slightly adapted and more detailed version of the original proof we found in the article of~\citet{gilmer2017neural}. Note that while GCNs are sensitive to the connectivity structure of their input graph, like many GNNs they do not take into account bond feature vectors. GCNs have been demonstrated to work reasonably well for a variety of graph-based prediction tasks~\citep{wang2021molclr, dwivedi2020benchmarking,tavakoli2020continuous,liu2020towards,hu2020open,kipf2016semi}.

\subsection{Graph Isomorphism Networks} \label{subsec: gins_description}

In 2018, \citet{xu2018powerful} published an influential article in which they introduced the graph isomorphism network (GIN) model. The GIN was developed to overcome some of the theoretical shortcomings of GCNs and other popular GNN architectures available at the time. We will discuss the highly relevant motivation of Xu et al.~and their theoretical insights in the next section, but will first give a brief description of their proposed model. The message-passing mechanism of GINs expressed via Equations~\ref{eq: mpnn_equations} is defined in a straightforward manner via
$$w_r(f_{r-1}(a), f_{r-1}(\tilde{a}), g(\{a, \tilde{a}\})) = f_{r-1}(\tilde{a}) $$
and
$$u_r(f_{r-1}(a), m_{r}(a)) = \phi_{r}((1 + \epsilon_r) f_{r-1}(a) + m_r(a)). $$
Here $\epsilon_r$ is a small (optionally trainable) parameter and $\phi_{r}$ is a trainable multilayer perceptron with at least one hidden layer. Note that $\phi_r$ is specifically required to be more powerful than the shallow neural network $x^{T} \mapsto \text{ReLU}(x^{T}W)$ used in the definition of GCNs in Equation~\ref{eq: gcn_matrix_recursion}, which has no hidden layer and therefore does not fulfill the requirements of the universal approximation theorem for feedforward neural networks~\citep{hornik1989multilayer}. One can write the atom feature updating process for GINs in the compact recursive form
$$f_r(a) = \phi_{r}\Big((1 + \epsilon_r) f_{r-1}(a) + \sum \ f_{r-1}(N(a))_{\text{mul}} \Big).  $$
Like GCNs, standard GINs also ignore bond feature vectors, although modifications of the GIN architecture that do consider bond features have been used successfully~\citep{hu2019strategies}. GINs are easy to implement, work well in practice and are currently reaching state-of-the-art performance in a variety of applications~\citep{kim2020understanding, wang2021molclr,dwivedi2020benchmarking, xu2018powerful}. 

\subsection{Theoretical Expressivity of Graph Neural Networks} \label{subsec: gnn_expressivity}

The relatively old GCNs are still popular GNN models. However, they come with a significant caveat. This caveat was pointed out by \citet{xu2018powerful} in their seminal $2018$ paper and used as a motivation for the design of the GIN. They proved that GCNs (as well as a plethora of other commonly-used GNN architectures) sometimes map non-isomorphic graphs to identical vectorial representations and thus suffer from a lack of theoretical expressivity. In other words, they discovered that GCNs and many other popular GNN models cannot learn to distinguish certain simple graph structures, to the point where in some instances they severely underfit the training set. The GIN model was specifically developed to overcome this problematic lack of expressivity. And indeed, \citet{xu2018powerful} managed to prove that GINs are strictly more expressive than GCNs and a number of other popular GNN models. They specifically demonstrated that GINs can distinguish certain non-isomorphic graphs that GCNs cannot. In addition they showed that GINs are as expressive as standard message-passing GNNs can ever be; GINs are in the subclass of \textit{maximally} expressive message-passing GNNs and are exactly as powerful at distinguishing non-isomorphic graphs as the canonical $1$-Weisfeiler-Lehman~($1$-WL) graph isomorhism test~\citep{weisfeiler1968reduction}.

\begin{descrip}[$1$-Weisfeiler-Lehman Test] \label{descrip: 1-WL test}
	Let $\mathcal{G} = (A,B)$ be a finite graph with a node featurisation function $f : A \to \mathbb{R}^k$. We now iteratively construct a sequence of functions $(c_r)_{r \in \mathbb{N}_{0}}$ on the set of graph nodes $A$. Each value $c_r(a)$ is imagined to be a momentary colouring of the node $a \in A$ at step $r$. The colouring functions are constructed via the following iterative scheme:
	\begin{gather*}
	c_0(a) \coloneqq f(a), \\
	c_{r}(a) = h(c_{r-1}(a), c_{r-1}(N(a))_{\text{mul}}).
	\end{gather*}
	Here $a \in A$ is a node of $\mathcal{G}$ and $h$ is an injective hash function that maps each vertex colour along with its multiset of neighbouring colours $c_{r-1}(N(a))_{\text{mul}}$ to a new unique colour. This procedure terminates after a finite number of $R$ steps when a stable colouring $c_R$ is reached that cannot be changed by subsequent iterations. Let
	$$c_R(A)_{\text{mul}}  \coloneqq  \{ c_R(a) \ \vert \ a \in A \}_{\text{mul}} $$
	be called the multiset of final vertex colours. If the multisets of final vertex colours of two input graphs $\mathcal{G}_1, \mathcal{G}_2$ are distinct, then the graphs are guaranteed not to be isomorphic. However, if the multisets of final vertex colours are identical, then the graphs are potentially but not necessarily isomorphic. If two non-isomorphic graphs reach identical multisets of final vertex colours, we say that the $1$-WL test cannot distinguish these graphs.
\end{descrip}

A pair of non-isomorphic graphs that nevertheless cannot be distinguished by the $1$-WL test is depicted in Figure~\ref{fig:wlcounterexample}. 
\begin{figure}[h]
	\centering
	\includegraphics[width=0.9\linewidth]{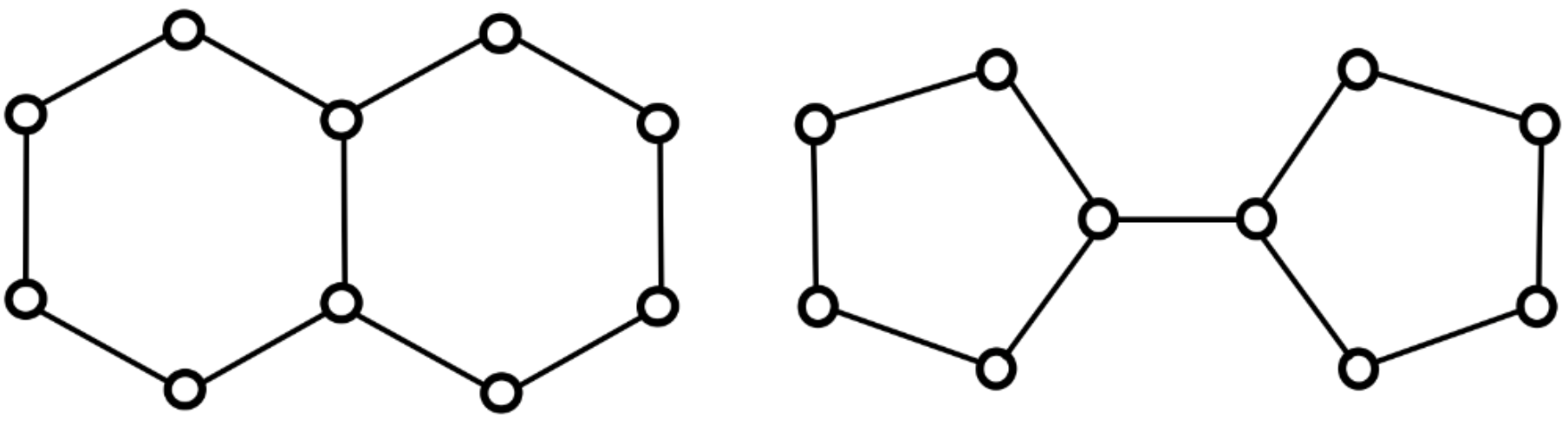}
	\caption[Non-isomorphic graphs that cannot be distinguished by the $1$-WL test.]{Example of two non-isomorphic graphs that cannot be distinguished by the $1$-WL test if all nodes are assumed to have identical initial colourings. Image source:~\citep{bouritsas2022improving}.}
	\label{fig:wlcounterexample}
\end{figure}
However, in spite of the existence of such counterexamples, the $1$-WL test is still an elegant and classical method to tackle the graph isomorphism problem that can efficiently tell apart a large number of non-isomorphic graph structures. Moreover, the $1$-WL test provides a natural benchmark against which the theoretical expressivity of GNNs can be compared.

\begin{theorem}[GNN-Conditions for $1$-WL Power] \label{theorem: gnn_expressivity_1-WL}
A message-passing GNN can be shown to be maximally expressive and (equivalently) as powerful as the $1$-WL test at distinguishing non-isomorphic graphs if the following injectivity-conditions hold: 

\begin{itemize}

\item \textbf{Injective aggregation and update}. For each graph layer $r \in \{1,...,R\}$ there must exist an injective multiset-function $\bm{\oplus}_r$ and an injective function $\psi_r$ such that the updating-procedure for atom features defined in Equations~\ref{eq: mpnn_equations} can be written in the form
$$f_{r}(a) = \psi_r(f_{r-1}(a), \bm{\oplus}_r f_{r-1}(N(a))_{\text{mul}}).$$
This condition guarantees that distinct atom-neighbourhoods are always mapped to distinct atom feature vector updates. 

\item \textbf{Injective pooling.} The multiset function $\tilde{\bm{\oplus}}_R$ used in the pooling step~\ref{eq: mpnn_pooling} at the $R$-th (i.e.~last) graph layer must be injective and the final pooling function $q$ must be injective in the argument corresponding to the $R$-th graph layer. This condition guarantees that graphs with distinct final multisets of updated atom feature vectors are always mapped to distinct graph-level vectorial representations.

\end{itemize}

\end{theorem}

\begin{proof}[Proof]
	
	We will show that the GNN can distinguish every pair of graphs that the $1$-WL test can. Let $\mathcal{G} = (A,B)$ be a (finite) graph with a node featurisation function $f : A \to \mathbb{R}^k$. We will prove that for each graph layer $r$ there exists an injective map $\alpha_r$ such that
	$$f_r(a) = \alpha_{r}(c_r(a)) \quad \forall a \in A. $$
	This means that the colour $c_r(a)$ of a node $a \in A$ at step $r$ (as computed by the WL algorithm) uniquely determines the node feature vector $f_r(a)$ at step $r$ (as computed by the GNN). Furthermore, the injectivity condition on $\alpha_r$ guarantees that nodes with distinct colours have distinct feature vectors.
	
	We show the existence of $\alpha_r$ via induction over $r$. For $r = 0$, it holds for all $a \in A$ that $f_0(a) = \alpha_0(a)$ so we can simply choose the identity map for $\alpha_0$. If we now assume for the induction step that $\alpha_{r-1}$ exists, then we can write
	$$f_r(a) =  \psi_r(f_{r-1}(a), \bm{\oplus}_r f_{r-1}(N(a))_{\text{mul}}) = \psi_r(\alpha_{r-1}(c_{r-1}(a)), \bm{\oplus}_r \alpha_{r-1}(c_{r-1}(N(a))_{\text{mul}})_{\text{mul}}).$$
	Since $\psi_r$, $\bm{\oplus}_r$ and $\alpha_{r-1}$ are injective and the composition of injective functions is again injective, it follows that there must be an injective function $\beta_r$ such that
	$$f_r(a) = \beta_r(c_{r-1}(a), c_{r-1}(N(a))_{\text{mul}}).$$
	We now set $\alpha_r \coloneqq \beta_r \circ h^{-1}$ whereby $h^{-1}$ is the inverse of the hash function used in the $1$-WL algorithm. Once again, $\alpha_r$ is injective since $\beta_r$ and $h^{-1}$ are injective. Moreover, it now holds for all $a \in A$ that
	$$\alpha_r(c_r(a)) =  \beta_r (h^{-1}( c_r(a))) = \beta_r(c_{r-1}(a), c_{r-1}(N(a))_{\text{mul}}) = f_r(a) $$
	which concludes the proof that $\alpha_r$ exists and has the desired properties.
	
	Now let $\mathcal{G} = (A,B)$ and $\tilde{\mathcal{G}} = (\tilde{A}, \tilde{B})$ be two graphs that can be distinguished with the $1$-WL test. We want to show that $\mathcal{G}$ and $\tilde{\mathcal{G}}$ are then mapped to different graph-level feature vectors by the message-passing GNN. The fact that $\mathcal{G}$ and $\tilde{\mathcal{G}}$ can be distinguished by the $1$-WL test means that at some graph layer $R$ their respective multisets of node colours must be different:
	$$c_R(A)_{\text{mul}}  \neq \tilde{c}_R(\tilde{A})_{\text{mul}} . $$
	Due to the injectivity of $\alpha_R$ this means that the multisets of updated atom feature vectors at the $R$-th layer must also be different:
	$$f_R(A)_{\text{mul}} = \alpha_R(c_R(A)_{\text{mul}} )_{\text{mul}} \neq \alpha_R(\tilde{c}_R(A)_{\text{mul}})_{\text{mul}} = f_R(\tilde{A}_{\text{mul}} ). $$
	Since the multiset-function $\tilde{\bm{\oplus}}_R$ that is used in the pooling step is assumed to be injective and since the final pooling function $q$ is assumed to be injective in its $R$-th argument as well, we can finally conclude that
	$$ \mathbb{R}^{l} \ni q(\tilde{\bm{\oplus}}_0 f_0(A)_{\text{mul}},...,\tilde{\bm{\oplus}}_R f_R(A)_{\text{mul}}) \neq q(\tilde{\bm{\oplus}}_0 f_0(\tilde{A})_{\text{mul}},...,\tilde{\bm{\oplus}}_R f_R(\tilde{A})_{\text{mul}}) \in \mathbb{R}^{l} $$
	which proves that the final vectorial GNN embeddings for $\mathcal{G}$ and $\tilde{\mathcal{G}}$ must be different.

	We have shown that all graphs that can be distinguished by the $1$-WL test can also be distinguished by a message-passing GNN under suitable injectivity conditions which means that such GNNs are at least as powerful as the $1$-WL test. For the full proof, including the converse that all graphs that can be distinguished by a message-passing GNN can also be distinguished by the $1$-WL test, we refer the reader to the original article of~\citet{xu2018powerful}.
\end{proof}

For simplicity we have omitted bond feature vectors in Description~\ref{descrip: 1-WL test} and Theorem~\ref{theorem: gnn_expressivity_1-WL}, but it is easy to formulate and prove equivalent versions of both statements that take these into account. Unlike GCNs, GINs fulfill the injectivity conditions from Theorem~\ref{theorem: gnn_expressivity_1-WL} and are thus in the class of maximally expressive message-passing GNNs that are as powerful as the $1$-WL test. \citet{xu2018powerful} suggested that the technical reasons why GCNs fail to meet the injectivity criteria of Theorem~\ref{theorem: gnn_expressivity_1-WL} is partly connected with the fact that their atom-updating scheme described in~\ref{eq: gcn_matrix_recursion} only involves a shallow neural network of the form $x^{T} \mapsto \text{ReLU}(x^{T}W)$ that lacks hidden layers. We hypothesise that the increased expressivitiy of GINs compared to theoretically weaker models such as GCNs most strongly translates into superior performance when training data is abundant. This idea is greatly supported by recent experiments in the realm of self-supervised learning where GINs consistently beat GCNs by a large margin when fine-tuned on supervised tasks after pre-training on millions of unlabelled molecular graphs~\citep{hu2019strategies, wang2021molclr}.

Arguably the most important contribution of \citet{xu2018powerful} was not the invention of the GIN, but the formulation of a mathematical framework to analyse the expressivity of GNNs that was based on well-known concepts from classical graph theory such as the graph isomorphism problem and the $1$-WL test. Notably, this breakthrough was achieved in parallel with \citet{morris2019weisfeiler} who published similar insights and proposed $k$-GNNs as a generalisation of traditional message-passing GNNs. The $k$-GNN architecture was designed to break through the expressivity barrier posed by the $1$-WL test and move up the \textit{WL hierarchy}. 
\begin{descrip}[WL Hierarchy]
The WL hierarchy is constituted by the $k$-WL tests for $k \in \mathbb{N}$. The $k$-WL tests represent a family of polynomial-time graph isomorphism tests of strictly increasing expressivity (with the exception of the $2$-WL test that is as powerful as the $1$-WL test). Each $k$-WL test represents an extension of the $1$-WL test from Description~\ref{descrip: 1-WL test} that operates on $k$-tuples of nodes. 
\end{descrip}
One can prove that the $k$-GNN model is as expressive as the $k$-WL test; however, the $k$-GNN architecture does not fall under the umbrella of local message-passing GNNs since it involves higher-order operations on sets of graph nodes. This makes $k$-GNNs prohibitively computationally expensive and normally impossible to use in practice.

A potential way forward was recently proposed by~\citet{bouritsas2022improving} in the form of graph substructure networks (GSNs). GSNs are associated with the emerging field of geometric deep learning~\citep{bronstein2017geometric,bronstein2021geometric} whose aim it is to effectively generalise deep learning architectures to non-Euclidean domains such as graph structures. GSNs were designed with the goal of being simultaneously practically applicable \textit{and} more expressive than the $1$-WL test. The key idea of GSNs is to compute messages in a manner that depends on the structural relationships between the considered nodes. This can be achieved by adding precomputed atomic structural descriptors for both nodes to the message-passing function $(w_r)_{r = 1}^{R}$ in the procedure outlined in Equations~\ref{eq: mpnn_equations}. Each node descriptor counts how often the node appears in a particular topological role in all subgraphs of the input graph that are isomorphic to a particular predefined (small) graph. The node descriptors do not get updated and remain static across all graph layers. The \textit{a priori} choice of basic subgraphs to look out for in the input graphs allows for the establishment of a useful inductive bias tailored to the prediction task at hand: for example, triangles might be relevant in social networks while ring-structures might play an important role in molecular graphs. Since GSNs only represent a slight generalisation of the standard local message-passing framework, their application is computationally feasible; and since they still allow for the inclusion of higher-order structural information in each message-passing step, they can be made strictly more powerful than the $1$-WL test (which then for example enables them to distinguish the graphs from Figure~\ref{fig:wlcounterexample}). The expressivity of GSNs cannot be easily formulated though within the larger WL hierarchy since suitably designed GSNs can distinguish (for example) \textit{some} graph pairs in the $3$-WL or $4$-WL class but not all. GSNs have already shown encouraging performance on chemical prediction tasks~\citep{bouritsas2022improving}.

\subsection{Critical View} \label{subsec: gnns_pro_con}

Below we give a list of advantages $\bm{(+)}$ and disadvantages $\bm{(-)}$ of message-passing GNNs as a molecular featurisation method.

\begin{itemize}
	
	\item[$\bm{(+)}$] \textbf{Differentiability.} Message-passing GNNs are able to learn features directly from highly explicit graph representations of molecules in a differentiable and task-specific manner. In this particular sense we refer to GNNs as trainable. From a mathematical perspective, the training process for GNNs takes the form of a gradient-based continuous optimisation problem. The differentiability of GNNs could potentially allow them to extract substantially larger amounts of valuable chemical information from molecular graphs than non-differentiable featurisation methods such as ECFPs and PDVs.
	
	\item[$\bm{(+)}$] \textbf{Low dimensionality.} Many message-passing GNNs still produce useful features for downstream prediction tasks even if the fingerprint length is limited to $l \leq 100$. This can be seen in contrast to ECFPs where the minimum fingerprint length to reach an acceptable level of performance is usually considered to be $l \geq 1024$. The relatively low dimensionality of GNN-based feature vectors can decrease the risk of costly downstream-computations and overfitting.
	
	\item[$\bm{(+)}$] \textbf{Ability to capture intricate patterns due to model complexity.} GNN architectures are often complex and involve large numbers of trainable parameters. This may enable them to detect nuanced and intricate patterns, especially when the training set is large.
	
	\item[$\bm{(+)}$] \textbf{High expressivity of some models.} The recent development of the GSN~\citep{bouritsas2022improving} as discussed in Section~\ref{subsec: gnn_expressivity} has shown that GNN models can be made both practical and substantially more powerful than the $1$-WL test (and thus also more powerful than ECFPs) at distinguishing graph structures.

	\item[$\bm{(-)}$] \textbf{Low expressivity of some models.} As discussed in Section~\ref{subsec: gnn_expressivity}, many common GNN models such as GCNs are substantially less powerful than the $1$-WL test at distinguishing graph structures and in some cases severely underfit the training set.
	
	\item[$\bm{(-)}$]\textbf{Risk of overfitting due to model complexity.} The complexity of GNN architectures may lead them to require relatively large amounts of data and/or careful regularisation in order to avoid overfitting and generalise effectively.
	
	\item[$\bm{(-)}$] \textbf{Limited depth.} The success of modern deep-learning architectures lies to a large part in their ability to automatically extract abstract high-level features directly from raw input data. This ability is crucially reliant on the sufficient \textit{depth} of the neural model since learned features tend to become more abstract and task-specific as they are moving through consecutive network layers. GNNs cannot yet leverage the power of deep architectures as the predictive utility of learned node embeddings tends to decrease sharply after a few GNN layers. The exact reasons and possible solutions for this pathology are still under research~\citep{liu2020towards,jin2022feature,zhang2021evaluating,godwin2021simple, chen2020measuring}. A likely cause might lie in the phenomenon of \textit{oversmoothing} which refers to a tendency of successively updated node features to eventually become indistinguishable.
	
	\item[$\bm{(-)}$] \textbf{Difficulty of interpreting learned features.} The graph-level featurisation generated by the global pooling step of a GNN usually consists of a vector of obscure real numbers that cannot be directly interpreted by human experts in any straightforward manner.
	
	\item[$\bm{(-)}$] \textbf{No chemically reasonable initial embedding of chemical space.} 
	The molecular featurisations extracted by untrained GNNs essentially represent extremely noisy random projections of the information contained in the initial molecular graph features. GNNs therefore need to relearn a reasonable embedding from chemical space from scratch every time they are trained on a new problem (unless they have been combined with a suitable self-supervised or transfer learning approach). In particular, this might force GNNs to continuously relearn basic chemical features that are useful across a wide set of tasks. This can be seen in contrast to classical methods like ECFPs and PDVs which automatically generate embeddings of chemical space that have a basic utility for many applications.
	
	\item[$\bm{(-)}$] \textbf{Difficulty of implementation.} PDVs and ECFPs can be generated easily via ready-to-use algorithms implemented in widely-used cheminformatics libraries such as the \texttt{Python}-package \texttt{RDKit}~\citep{landrum2006rdkit}. The implementation of a tailored message-passing-GNN model, on the other hand, usually requires familiarity with technically advanced graph-based deep-learning libraries such as \texttt{PyTorch Geometric}~\citep{fey2019fast}.
		
	\item[$\bm{(-)}$] \textbf{High computational cost.} Like most deep-learning models, GNNs involve large numbers of trainable parameters and costly computations related to the multiplication of large matrices. This can make GNN models much slower to train than PDV or ECFP-based models, although this difference in speed can be significantly mitigated if a suitable GPU can be leveraged during GNN training.
	
	\item[$\bm{(-)}$] \textbf{Locality of receptive field.} Just like ECFPs, message-passing GNNs are based on a local neighbourhood-aggregation scheme for the nodes in a graph. By design, they are thus only capable of learning node embeddings that contain features from local circular subgraphs; this prevents information flow between distant nodes during the atom feature updating process.
	
	\item[$\bm{(-)}$] \textbf{Potential information loss during graph pooling.} The global pooling step in a GNN that combines the multisets of updated node embeddings to a final graph-level feature vector can easily turn into a (non-injective) information bottleneck if not designed carefully. Imagine for example two multisets of identical node feature vectors $\{f, f \}_{\text{mul}}$ and $\{f, f, f\}_{\text{mul}}$; if the popular averaging-operator is used to pool these two multisets, they lead to the same output representation $f$ even though both multisets are different and should thus be mapped to distinct representations. A potential way to avoid such pitfalls may be to employ differentiable graph pooling techniques that (at least in theory) can provably approximate any sufficiently regular pooling function~\citep{navarin2019universal}.
		
\end{itemize}

\newpage
\section{Molecular Featurisations: Critical Overview} \label{sec: mol_feats_crit_overview}

We have summarised the critical analyses of PDVs~(see Section~\ref{subsec: pdvs_pro_con}), ECFPs~(see Section~\ref{subsec: ecfps_pro_con}) and message-passing GNNs~(see Section~\ref{subsec: gnns_pro_con}) in Table~\ref{tab: mol_feats_pro_cons} to provide a final overview before moving on to the systematic computational experiments in Chapter~\ref{chap: qsar_ac_study}.
\begin{table}[H] \small
	\centering
	{\renewcommand{\arraystretch}{1.5}
		\begin{tabular}{ V{3} m{1.3cm}|m{6.1cm}|m{6.1cm} V{3}}
			
			\hlineB{3}
			\multicolumn{3}{V{3} c V{3}}{\textbf{Molecular Featurisation Methods: Critical Overview}} \\
			
			\hline
			\multicolumn{1}{V{3} c|}{\textbf{Method}} & \multicolumn{1}{c|}{\textbf{Advantages} $\bm{(+)}$} & \multicolumn{1}{c V{3}}{\textbf{Disadvantages} $\bm{(-)}$} \\
			\hline
			
			\multicolumn{1}{V{3} c|}{PDV}
			&\begin{itemize}[leftmargin=0.47cm] \setlength\itemsep{-0.3em}
				\item Low computational cost
				\item Interpretability (in some cases)
				\item Simplicity of implementation
				\item Chemically reasonable initial embedding of chemical space 
				\item Low dimensionality 
				\item Global receptive field \vspace{-0.535pc}
			\end{itemize}
			&
			\begin{itemize}[leftmargin=0.47cm] \setlength\itemsep{-0.3em}
				\item Non-differentiability
				\item Necessity for feature selection
				\item Finite number of descriptors \vspace{-0.535pc}
			\end{itemize}
			
			\\ \hline
			
			\multicolumn{1}{V{3} c|}{ECFP}   
			&
			\begin{itemize}[leftmargin=0.47cm] \setlength\itemsep{-0.3em}
				\item Low computational cost
				\item Interpretability 
				\item Simplicity of implementation
				\item Chemically reasonable initial embedding of chemical space 
				\item Arbitrary circular subgraphs \vspace{-0.535pc}
			\end{itemize}
			&
			\begin{itemize}[leftmargin=0.47cm] \setlength\itemsep{-0.3em}
				\item Non-differentiablity
				\item Trade-off between dimensionality and hash collisions
				\item Locality of receptive field \vspace{-0.535pc}
			\end{itemize}
			
			\\ \hline
			
			\multicolumn{1}{V{3} c|}{GNN}
			&
			\begin{itemize}[leftmargin=0.47cm] \setlength\itemsep{-0.3em}
				\item Differentiability
				\item Low dimensionality
				\item Ability to capture intricate patterns due to model complexity
				\item High expressivity of some models \vspace{-0.535pc}
			\end{itemize}
			&
			\begin{itemize}[leftmargin=0.47cm] \setlength\itemsep{-0.3em}
				\item Low expressivity of some models
				\item Risk of overfitting due to model complexity
				\item Limited depth
				\item Difficulty of interpreting learned features
				\item No chemically reasonable initial embedding of chemical space
				\item Difficulty of implementation
				\item High computational cost
				\item Locality of receptive field
				\item Potential information loss during graph pooling \vspace{-0.535pc}
			\end{itemize}
			\\
			\hlineB{3}
			
	\end{tabular}}
	
	\caption[Strengths and weaknesses of PDVs, ECFPs and message-passing GNNs.]{Advantages and disadvantages of physicochemical-descriptor vectors~(PDVs), extended-connectivity fingerprints~(ECFPs) and message-passing graph neural networks~(GNNs) for molecular featurisation.}
	
	\label{tab: mol_feats_pro_cons}
\end{table}

\newpage $\text{}$ 
\newpage
\chapter[Exploring Molecular Featurisations for QSAR and Activity-Cliff Prediction: A~Computational Study]{Exploring Molecular Featurisations for QSAR and Activity-Cliff Prediction: A~Computational Study} \label{chap: qsar_ac_study}

\noindent \textit{We have published our findings from this chapter as a \href{https://doi.org/10.1186/s13321-023-00708-w}{peer-reviewed research paper}~\citep{dablander2023exploring} in the Journal of Cheminformatics. Most of the figures and tables, as well as major parts of the text contained in this chapter are thus either identical or similar to the content of our published article. We have also presented results from this chapter as a \href{http://dx.doi.org/10.13140/RG.2.2.35914.34241}{scientific poster}~\citep{dablander2023exploringqsaracpredposter} at the 10th International Congress on Industrial and Applied Mathematics~(ICIAM 2023, Tokyo). }
	
\section{Overview} \label{sec: qsar_ac_study_overview}

In Chapter~\ref{chap: mol_reps} we have given a detailed technical description of state-of-the-art featurisation methods in molecular machine learning along with a critical discussion of their theoretical strengths and weaknesses. However, perhaps the most pressing question has not yet been addressed: which featurisation method actually leads to the strongest performance at chemical prediction tasks? There is still disagreement in the computational-chemistry community whether modern trainable message-passing GNNs do in fact outperform classical non-trainable featurisation methods such as ECFPs and PDVs at important molecular machine-learning tasks. A multitude of studies have found that GNNs do indeed clearly outcompete ECFPs and PDVs~\citep{duvenaud2015convolutional, gilmer2017neural,yang2019analyzing, rao2022quantitative, hop2018geometric, shang2018edge, li2017learning, xiong2019pushing}. However, a considerable number of other studies have found evidence pointing towards the exact opposite~\citep{stepivsnik2021comprehensive, mayr2018large, jiang2021could, wang2021molclr, menke2021using, chithrananda2020chemberta, sabando2021using, winter2019learning,van2022exposing}.

In this chapter, we present a series of carefully designed computational experiments to systematically investigate the predictive powers of PDVs, ECFPs and GINs for two important tasks in computational drug discovery: the well-researched problem of quantitative structure-activity relationship~(QSAR) prediction and the largely unexplored and difficult challenge of activity-cliff (AC) prediction. 
QSAR-prediction refers to the problem of using experimental data to learn a mapping from a computational representation $\mathcal{R}$ of a chemical compound to its biological activity value against a given pharmacological target such as an enzyme or a receptor. A QSAR model usually takes the form of a machine-learning model that can be decomposed into a molecular featurisation method followed by a regression technique:
$$\mathcal{R} \quad \overset{\longmapsto}{\footnotesize \text{featurisation}} \quad \mathcal{F} = (f_1, ..., f_l) \in \mathbb{R}^l \quad \overset{\longmapsto}{\footnotesize \text{regression}} \quad \text{activity (pK\textsubscript{i}, pIC\textsubscript{50}, ...)} \in \mathbb{R}. $$
ACs are pairs of very similar compounds whose molecular structure only differs by a small change at a specific site but which exhibit a very large difference in their activity against a given pharmacological target. ACs explicitly encode small structural changes that abruptly change a biological effect and are thus rich in pharmacological information. AC-prediction primarily refers to the task of classifying whether a given pair of structurally similar compounds forms an AC or not, but usually also implicitly encompasses the classification of the potency direction~(PD) of the pair (i.e.~which of both compounds is more active). Every QSAR model can be repurposed as an AC-prediction model by using it to individually predict the activities of two structurally similar compounds (which gives the PD-classification) and then thresholding the absolute difference of the two predicted activities (which gives the AC-classification). Accurate QSAR-prediction and AC-prediction models would represent valuable tools in the computer-aided search for novel pharmacological compounds with desired properties.

In this chapter, we conduct a rigorous computational study to evaluate the QSAR and AC-prediction performance of nine modern QSAR models on three curated pharmacological data sets~(dopamine receptor D2, factor Xa and SARS-CoV-2 main protease). Each QSAR model is generated by merging one of three molecular featurisation methods (PDVs, ECFPs, or GINs) with one of three canonically used regression techniques (random forests~(RFs), k-nearest neighbours~(kNNs), or multilayer perceptrons~(MLPs)). Our experimental setup thus allows for a systematic comparison of the three studied featurisations across three regression techniques, three pharmacological targets, and three distinct chemical prediction tasks (QSAR-prediction, AC-classification, and PD-classification). Our experiments are thus organised according to a robust combinatorial methodology of the form
$$\vert \{\text{featurisers}\} \vert \times \vert \{\text{regressors}\} \vert \times \vert \{\text{data sets}\} \vert \times \vert \{\text{tasks}\} \vert = 3 \times 3 \times 3 \times 3 $$
that to the best of our knowledge has not been described before in the literature. The QSAR-prediction, AC-classification and PD-classification performance of each featuriser-regressor combination on each data set is measured using a strict data splitting and evaluation strategy involving a full algorithmic hyperparameter optimisation. To evaluate AC and PD-classification performance, we develop a novel pair-based data-splitting method that operates on top of data splits for individual molecules in a natural and interpretable manner. Our proposed data split for sets of molecular pairs is conceptually simple, yet allows one to make important distinctions between several types of compound pairs with respect to their molecular overlap with an underlying training set of individual compounds.

It has been hypothesised that ACs form one of the major sources of prediction error in QSAR modelling~\citep{cruz-monteagudo_activity_2014, maggiora_outliers_2006} but so far only a few studies have attempted to generate empirical evidence for this claim~\citep{golbraikh2014data, sheridan_experimental_2020, van2022exposing}. However, these studies follow an indirect approach by measuring QSAR performance on \textit{individual} compounds involved in ACs instead of \textit{pairs} of similar compounds. Our published work~\citep{dablander2023exploring} closes a gap in the current QSAR and AC literature by providing the first computational study to investigate the capabilities of state-of-the-art QSAR models to classify whether a given pair of similar compounds forms an AC or not. The main aim of the work in this chapter is to answer the following question:

\begin{itemize}
	
	\item Which molecular featurisation method performs best for QSAR or AC-prediction across different regression techniques and data sets? In particular, when (if at all) do trainable GINs outperform non-trainable PDVs and ECFPs?

\end{itemize}
Besides this, we are also interested in the following questions:
\begin{itemize}

	\item When (if at all) are common QSAR models capable of predicting the existence of ACs?
	
	\item When (if at all) are common QSAR models capable of predicting which of two similar compounds is more active?
	
	\item Which QSAR model shows the strongest AC-prediction performance, and should thus be used as a baseline against which to compare tailored AC-prediction models?
	
	\item What is the quantitative relationship between QSAR and AC-prediction performance for QSAR models?
	
\end{itemize}

\section[Introduction to Activity Cliffs and Activity-Cliff Prediction]{Introduction to Activity Cliffs and \\ Activity-Cliff Prediction} \label{sec: qsar_ac_study_intro_acs}

As mentioned above, activity cliffs~(ACs) are pairs of small molecules that exhibit high structural similarity but at the same time show an unexpectedly large difference in their binding affinity against a given pharmacological target~\citep{silipo1991qsar, maggiora_outliers_2006, sheridan_experimental_2020, cruz-monteagudo_activity_2014, stumpfe_recent_2014, stumpfe_evolving_2019, stumpfe_advances_2020}. The existence of ACs directly defies the intuitive idea that chemical compounds with similar structures should have similar activities, often referred to as the \textit{molecular similarity principle}. An example of an AC between two inhibitors of blood coagulation factor Xa~\citep{leadley2001coagulation} is depicted in Figure~\ref{fig:ac_example_factor_Xa_CHEMBL658338}; a small chemical modification involving the addition of a hydroxyl group leads to an increase in binding affinity of almost three orders of magnitude.
\begin{figure}[h]
	\centering
	\includegraphics[width=1\linewidth]{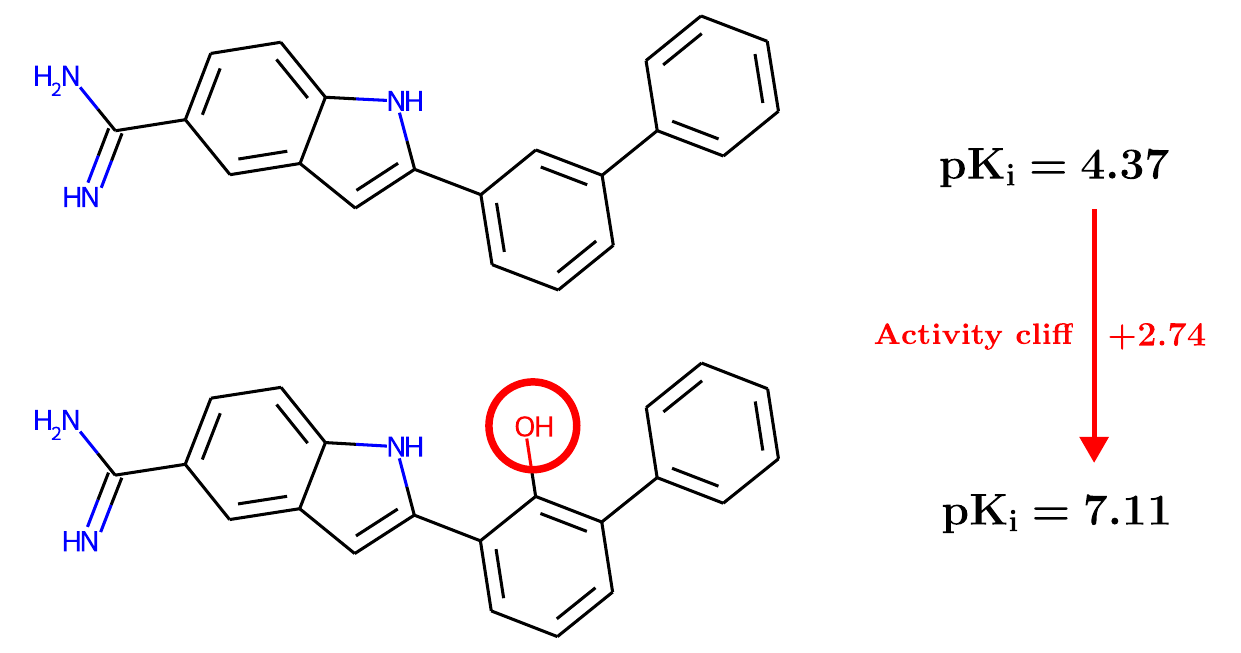}
	\caption[Example of an AC for factor Xa.]{Example of an activity cliff~(AC) for blood coagulation factor Xa. A small structural change in the upper compound leads to an increase in binding affinity of almost three orders of magnitude. Here binding affinity is quantified via the commonly-used pK\textsubscript{i}-value, which represents the negative decadic logarithm of the dissociation constant K\textsubscript{i} of the drug-target complex. Both compounds can be found in the same ChEMBL assay with ID 658338.}
	\label{fig:ac_example_factor_Xa_CHEMBL658338}
\end{figure} 

For medicinal chemists, ACs can be puzzling and confound their understanding of structure-activity relationships~(SARs)~\citep{stumpfe_recent_2014, vogt2011activity, dimova_activity_2015}. ACs reveal small compound-modifications with large biological impact and thus represent rich sources of pharmacological information. Mechanisms by which a small structural change can give rise to an AC include a drastic change in 3D-conformation and/or the switching to a different binding mode or even binding site. Another mechanism that could potentially induce ACs is the stabilisation of a single molecular conformation as a result of a small structural change; such an effect could possibly increase the equilibrium concentration of the binding conformation, thus leading to an AC without creating any specific change in binding interaction. One can speculate that this might in fact be the mechanism underlying the AC shown in Figure~\ref{fig:ac_example_factor_Xa_CHEMBL658338}; it is imaginable that the depicted addition of a hydroxyl group leads to a stabilising intramolecular N-H to O interaction. 

ACs form discontinuities in the SAR-landscape and can therefore have a crucial impact on the success of lead-optimisation programmes. While knowledge of ACs can be powerful when trying to escape from flat regions of the SAR-landscape, their presence can be detrimental in later stages of the drug development process, when multiple molecular properties beyond mere activity need to be balanced carefully to arrive at a safe and effective compound~\citep{cruz-monteagudo_activity_2014, stumpfe_recent_2014}. In the field of cheminformatics, ACs are suspected to form one of the major roadblocks for successful QSAR modelling~\citep{golbraikh2014data, cruz-monteagudo_activity_2014, maggiora_outliers_2006, sheridan_experimental_2020}. In practice, abrupt changes in activity are expected to negatively influence the abilities of QSAR methods to learn general SAR-trends. On the other hand, in theory ACs can be seen as an opportunity for such methods to extract precious SAR-knowledge. During the development of QSAR models, ACs are sometimes dismissed as measurement errors~\citep{medina2013activity}, but simply removing ACs from a training data set can result in a loss of large amounts of pharmacological information~\citep{cruz2016probing}. 

\citet{golbraikh2014data} developed the simple MODI metric which quantifies the smoothness of the SAR-landscape of a given binary molecular classification data set. They subsequently showed that SAR-landscape smoothness is a strong determinant for downstream QSAR-modelling performance across a large number of data sets. In a conceptually related work,~\citet{sheridan_experimental_2020} found that the density of ACs in a given molecular data set is strongly predictive for its overall modelability by classical descriptor and fingerprint-based QSAR methods. Furthermore, they demonstrated that such methods incur a significant drop in performance when the molecular test set is restricted to only include ``cliffy" compounds which form a large number of ACs. In a recent and more extensive study,~\citet{van2022exposing} observed a similar drop in performance when testing a large number of classical and graph-based QSAR techniques on sets of compounds involved in ACs. Notably, in both studies this performance drop was also observed for highly nonlinear and adaptive deep learning models. Moreover, van Tilborg reports that descriptor-based QSAR methods do in fact outperform more complex deep learning models on ``cliffy" compounds associated with ACs. This runs counter to earlier hopes expressed in the literature that the approximation powers of highly parametric deep networks might ameliorate the problem of ACs~\citep{winkler2017performance}.

While these works provide valuable insights into the detrimental effects of SAR discontinuity on QSAR models, they consider ACs mainly indirectly by focussing on \textit{individual} compounds involved in ACs. Arguably, a distinct and more natural approach would be to investigate ACs directly at the level of compound \textit{pairs}. This approach has been followed in the AC-prediction field which is concerned with the development of techniques to classify the existence and direction of potential ACs. An effective AC-prediction method would be of great value for drug development with important applications in rational compound optimisation and automatic SAR-knowledge acquisition. 

The AC-prediction literature is still very thin compared to the QSAR-prediction literature. An attempt to conduct an exhaustive literature review on AC-prediction techniques revealed a total of $15$ methods~\citep{heikamp_prediction_2012, tamura_ligand-based_2020, de_la_vega_de_leon_prediction_2014, beck2014quantitative, namasivayam_searching_2012, namasivayam_prediction_2013, husby_structure-based_2015, horvath_prediction_2016, perez-benito_predicting_2019, asawa_prediction_2020, keyvanpour2021pcac, iqbal_prediction_2021, park2022acgcn, chen2022deepac, dablander2021siamese}, all of which have been published since 2012. Current AC-prediction methods are often based on creative ways to extract features from pairs of molecular compounds in a manner suitable for standard machine learning pipelines. For example,~\citet{horvath_prediction_2016} used condensed graphs of reactions~\citep{hoonakker2011condensed,jauffret1990machine}, a representation technique originally introduced for modelling of chemical reactions, to encode pairs of similar compounds and subsequently predict ACs. Another method was recently described by~\citet{iqbal_prediction_2021} who investigated the abilities of convolutional neural networks operating on 2D images of compound pairs to distinguish between ACs and non-ACs. Interestingly, none of the AC-prediction methods we identified employ feature extraction techniques built on GNNs with the exception of~\citet{park2022acgcn} who recently applied graph convolutional methods to compound-pairs to predict ACs.

In spite of the existence of various technically complex AC-prediction models there are significant gaps left in the current AC-prediction literature. Note that any given QSAR model can immediately be repurposed as an AC-prediction model by using it to individually predict the activities of two structurally similar compounds and then thresholding the difference of both predicted activities. Nevertheless, at the moment there is no study that uses this straightforward technique to rigorously investigate the potential of modern QSAR models to classify whether a given pair of compounds forms an AC or not. Importantly, this also entails that the most salient AC-prediction models~\citep{heikamp_prediction_2012, de_la_vega_de_leon_prediction_2014, horvath_prediction_2016, tamura_ligand-based_2020, iqbal_prediction_2021} have not been compared to a simple QSAR-modelling baseline. It is thus an open question to what extent (if at all) these tailored AC-prediction techniques outcompete repurposed state-of-the-art QSAR methods at the detection of ACs. This question is especially relevant in light of the fact that several published AC-prediction models~\citep{heikamp_prediction_2012, de_la_vega_de_leon_prediction_2014, iqbal_prediction_2021} are evaluated via compound-pair-based data splits which incur a significant overlap between the training set and test set at the level of individual molecules. This type of data split should strongly favour standard QSAR models for AC-prediction, yet a comparison to such baseline methods is lacking. We address these problems by providing the first computational study that explores the capabilities of modern QSAR models to predict ACs. The results of our study establish a natural baseline against which more advanced AC-prediction models can be compared. Moreover, we disentangle the prevalent data-splitting issues in AC-prediction settings by introducing a novel splitting technique. This method, which we recommend as the standard for future publications, allows one to make important distinctions between several types of compound-pair test-sets with respect to their molecular overlaps with the underlying training set.

\section{Experimental Methodology} \label{sec: qsar_ac_study_exp_methods}

\subsection{Molecular Data Sets} \label{subsec: qsar_ac_study_data_sets}

We built three binding affinity data sets of small-molecule inhibitors of dopamine receptor D2, factor Xa, and SARS-CoV-2 main protease. Dopamine receptor D2 is the main site of action for classic antipsychotic drugs which act as antagonists of the D2 receptor~\citep{seeman1987dopamine}. Factor Xa is an enzyme in the coagulation cascade and a canonical target for blood-thinning drugs~\citep{leadley2001coagulation}. SARS-CoV-2 main protease is one of the key enzymes in the viral replication cycle of the SARS coronavirus~2, that recently caused the unprecedented COVID-19 pandemic; it is one of the most promising targets for antiviral drugs against this coronavirus~\citep{ullrich2020sars}. The protein structures of dopamine receptor D2, factor Xa and SARS-CoV-2 main protease are visualised in~\Cref{fig:d2,fig:fxa,fig:mpro}, respectively.
\begin{figure}
	\centering
	\vspace*{-30mm}
	\includegraphics[width=0.97\linewidth]{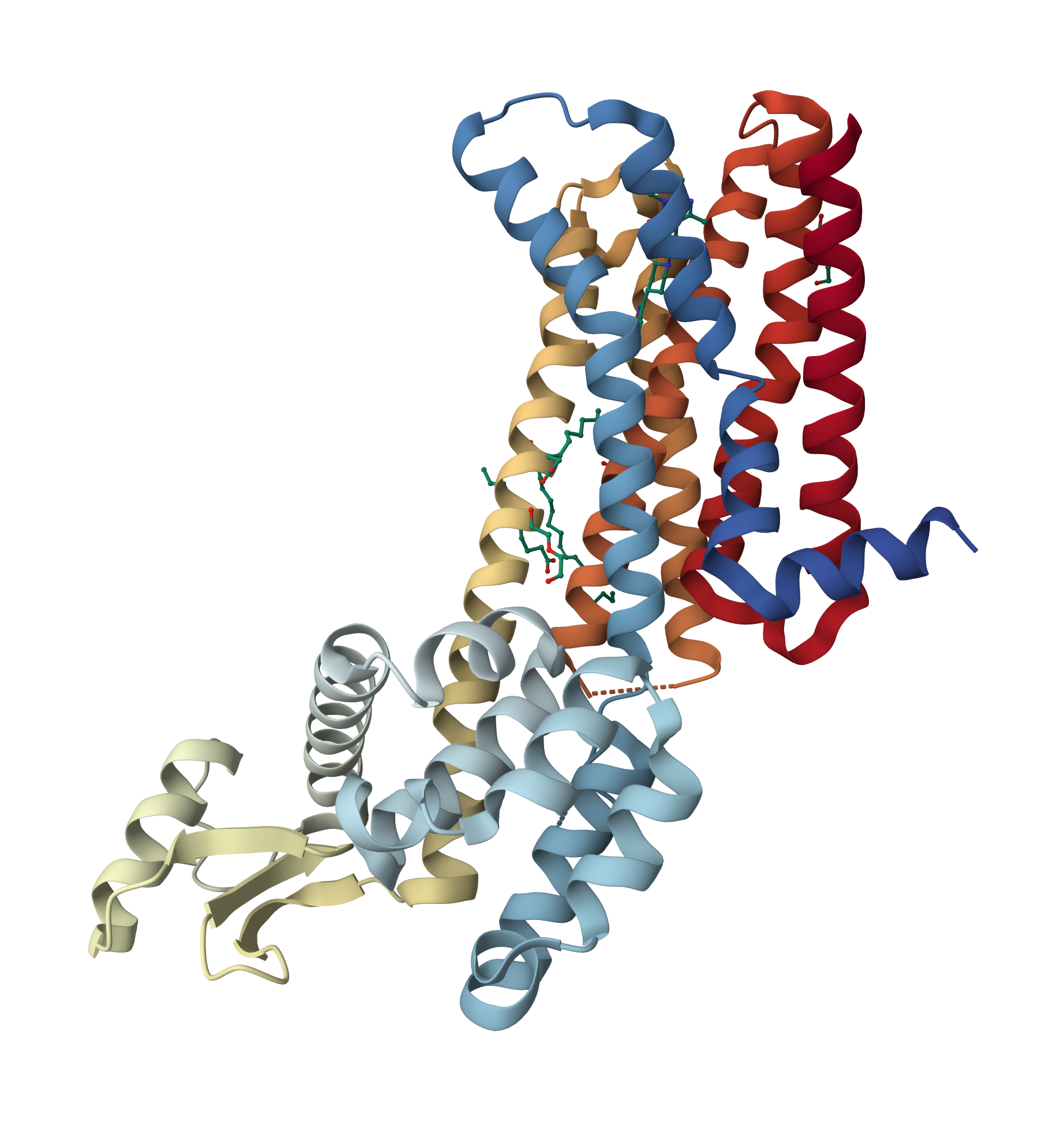}
	\caption[Protein structure of dopamine receptor D2.]{Protein structure of \textbf{dopamine receptor D2}. Extracted from the Research Collaboratory for Structural Bioinformatics Protein Data Bank~(RCSB PDB)~\citep{berman2000protein}. PDB ID: 6CM4.}
	\label{fig:d2}
\end{figure}
\begin{figure}
	\centering
	\vspace*{-30mm}
	\includegraphics[width=0.97\linewidth]{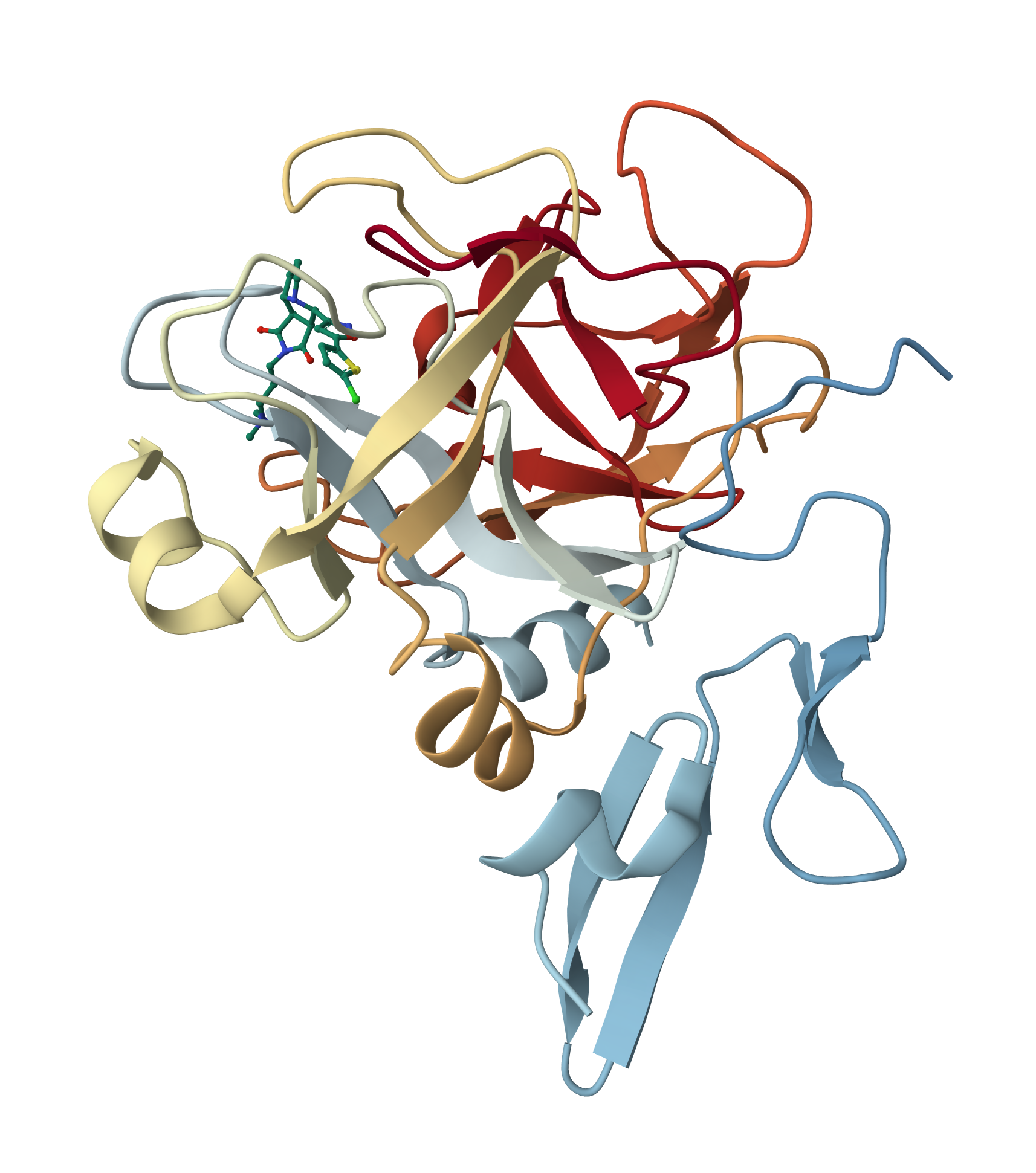}
	\caption[Protein structure of factor Xa.]{Protein structure of \textbf{factor Xa}. Extracted from the Research Collaboratory for Structural Bioinformatics Protein Data Bank~(RCSB PDB)~\citep{berman2000protein}. PDB ID: 2JKH.}
	\label{fig:fxa}
\end{figure}
\begin{figure}
	\centering
	\vspace*{-30mm}
	\includegraphics[width=0.97\linewidth]{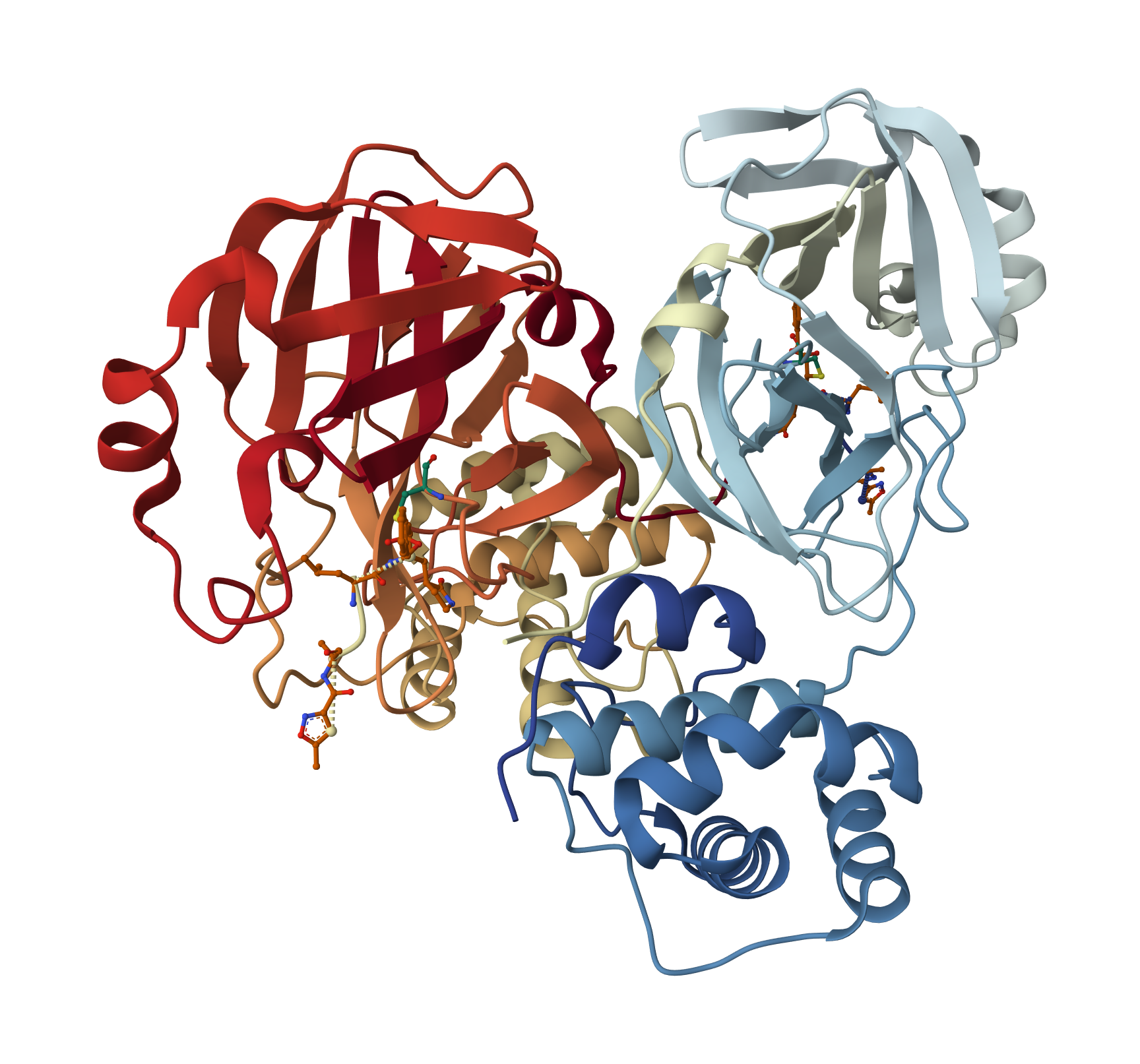}
	\caption[Protein structure of SARS-CoV-2 main protease.]{Protein structure of \textbf{SARS-CoV-2 main protease}. Extracted from the Research Collaboratory for Structural Bioinformatics Protein Data Bank~(RCSB PDB)~\citep{berman2000protein}. PDB ID: 6LU7.}
	\label{fig:mpro}
\end{figure}

For dopamine receptor D2 and factor Xa, data was extracted from the ChEMBL database~\citep{liu2007bindingdb} in the form of SMILES strings with associated K\textsubscript{i}~[nM] values. All activity values extracted from ChEMBL were associated with the equality standard relation ``$=$"; qualified activity values associated with other relations such as ``$<$" or ``$>$" were not used. In the case of SARS-CoV-2 main protease, data was obtained from the COVID moonshot project~\citep{achdout2020covid} in the form of SMILES strings with associated IC\textsubscript{50}~[µM] values. SMILES strings were standardised and desalted via the ChEMBL structure pipeline~\citep{bento2020open}. This step also removed solvents and isotopic information. Following this, SMILES strings that produced error messages when turned into an \texttt{RDKit} mol object were deleted. Finally, a scan for duplicate molecules was performed: If the activities in a set of duplicate molecules were within the same order of magnitude then the set was unified via geometric averaging. Otherwise, the measurements were considered unreliable and the corresponding set of duplicate molecules was removed. This procedure reduced the data set for dopamine receptor D2 / factor Xa / SARS-CoV-2 main protease from $8883$ / $4116$ / $1926$ compounds to $6333$ / $3605$ / $1924$ unique compounds whereby $174$ / $21$ / $0$ sets of duplicate SMILES were removed and the rest was unified.

\subsection{Definition of Binary Activity-Cliff Classification-Tasks} \label{subsec: qsar_ac_study_ac_bin_class_prob}

The exact definition of an AC hinges on two concepts: structural similarity and large activity difference. An elegant technique to measure structural similarity in the context of AC analysis is given by the matched molecular pair~(MMP) formalism~\citep{kenny2005structure, hu_extending_2012}. An MMP is a pair of compounds that share a common structural core but differ by a small structural change at a specific site. Figure~\ref{fig:ac_example_factor_Xa_CHEMBL658338} depicts an example of an MMP whose variable parts are formed by a hydrogen atom and a hydroxyl group. To detect MMPs algorithmically, we used the \texttt{mmpdb} \texttt{Python}-package provided by~\citet{dalke2018mmpdb}. We restricted ourselves to the commonly-used definition of MMPs~\citep{heikamp_prediction_2012, horvath_prediction_2016, tamura_ligand-based_2020} which employs relatively generous size constraints:
the MMP core was required to contain at least twice as many heavy atoms as either of the two variable parts; each variable part was required to contain no more than $13$ heavy atoms; the maximal size difference between both variable parts was set to eight heavy atoms; and bond cutting was restricted to single exocyclic bonds. Note that the decomposition of an MMP into a core part and two variable parts is not necessarily unique as the size of the chosen core part can vary. To guarantee a well-defined mapping from each MMP to a unique structural core, we canonically chose the core that contained the largest number of heavy atoms whenever there was ambiguity.

Based on the ratio of the activity values of both MMP compounds, each MMP was assigned to one of three classes: ``AC", ``non-AC" or ``half-AC". In accordance with the literature~\citep{heikamp_prediction_2012, namasivayam_prediction_2013, horvath_prediction_2016, vogt2011activity, bajorath_exploring_2014} we assigned an MMP to the ``AC"-class if both activity values differed by at least a factor of $100$. If both activity values differed by no more than a factor of $10$, then the MMP was assigned to the ``non-AC"-class. 
In the residual case the MMP was assigned to the ``half-AC"-class. To arrive at a well-separated binary classification task, we labelled all ACs as positives and all non-ACs as negatives. The half-ACs were removed and not considered further in our experiments. It is also relevant to know the direction of a potential activity cliff, i.e.~which of the compounds in the pair is more active. We thus assigned a binary label to each MMP indicating its potency direction~(PD). PD-classification is a balanced binary classification task. Table~\ref{tab: qsar_ac_study_datsets_overview} gives an overview of all our curated data sets.
\begin{table*}[h]
	\centering
	\normalsize
	{\renewcommand{\arraystretch}{1.3}
		\begin{tabular}{p{3.4cm} V{4} p{2.93cm}V{4}p{2.93cm}V{4}p{2.98cm}}

			\textbf{Data Set} & 
			\textbf{Dopamine Receptor D2} & 
			\textbf{Factor Xa} & 
			\textbf{SARS-CoV-2 \newline Main Protease} 
			\\ \hlineB{4}
			
			\textbf{Compounds} & 
			\multicolumn{1}{rV{4}}{$6333$} & 
			\multicolumn{1}{rV{4}}{$3605$} & 
			\multicolumn{1}{r}{$1924$}
			\\ \hline
			
			\textbf{MMPs} & 
			\multicolumn{1}{rV{4}}{$35484$} & 
			\multicolumn{1}{rV{4}}{$21292$} & 
			\multicolumn{1}{r}{$12594$}
			\\ \hline
			
			\textbf{ACs} & 
			\multicolumn{1}{rV{4}}{$461$} & 
			\multicolumn{1}{rV{4}}{$1896$} & 
			\multicolumn{1}{r}{$521$}
			\\ \hline
			
			\textbf{Half-ACs} & 
			\multicolumn{1}{rV{4}}{$3804$} & 
			\multicolumn{1}{rV{4}}{$4693$} & 
			\multicolumn{1}{r}{$1762$}
			\\ \hline
			
			\textbf{Non-ACs} & 
			\multicolumn{1}{rV{4}}{$31219$} & 
			\multicolumn{1}{rV{4}}{$14703$} & 
			\multicolumn{1}{r}{$10311$}
			\\ \hline
			
			\textbf{ACs : Non-ACs} & 
			\multicolumn{1}{rV{4}}{$\approx 1 : 68$}& 
			\multicolumn{1}{rV{4}}{$\approx 1 : 8$} & 
			\multicolumn{1}{r}{$\approx 1: 20$} \\ \hline
			
			\textbf{Measurement} & 
			\multicolumn{1}{rV{4}}{K\textsubscript{i}~[nM]} & 
			\multicolumn{1}{rV{4}}{K\textsubscript{i}~[nM]} & 
			\multicolumn{1}{r}{IC\textsubscript{50}~[µM]}
			
	\end{tabular}}
	
	\vspace*{6mm}
	
	\caption[Data sets for AC-prediction study.]{Sizes of our curated data sets and their respective numbers of matched molecular pairs~(MMPs), activity cliffs~(ACs), half-activity-cliffs~(half-ACs) and non-activity-cliffs~(non-ACs).}
	
	\label{tab: qsar_ac_study_datsets_overview}
\end{table*}

It is worth emphasizing again that the definition of an AC hinges on the employed measure of structural similarity. While the MMP formalism currently represents the most widespread similarity criterion in the field of AC research, other techniques have regularly been employed in the past. In particular, the Tanimoto similarity of binary structural fingerprints (such as ECFPs or MACCS fingerprints) has originally frequently been used to define ACs~\citep{hu_mmp-cliffs_2012,stumpfe_evolving_2019}. One obvious advantage of this approach is its computational simplicity. Another advantage may be that ACs can be defined directly within the input feature space of a potential machine learning model. As a result, ACs based on this definition might intuitively map very well to compound pairs whose activity difference is indeed challenging to predict for fingerprint-based QSAR models. However, using Tanimoto similarity in feature space to define ACs also has a variety of serious drawbacks that have been effectively summarised by \citet{stumpfe_evolving_2019}: Tanimoto similarity varies continuously in the interval $[0,1]$ and thus requires the choice of a subjective threshold value for the classification of compound pairs as similar. Furthermore, different fingerprints generally lead to different Tanimoto similarities; this makes the classification of compound pairs as similar dependent on the employed fingerprint. Finally, Tanimoto similarity values can sometimes be difficult to interpret from a chemical perspective and might not always accurately reflect the intuitions of a chemical expert. In contrast, note that MMPs do not require the choice of a similarity threshold, are not dependent on a particular fingerprint featurisation, and have a clear chemical interpretation in terms of structural cores and variable parts. These reasons might explain the shift away from Tanimoto similarity towards the MMP formalism that could be observed in recent years in the field of AC research.

\subsection{Developed Pair-Based Data Splitting Technique} \label{subsec: qsar_ac_study_data_splitting}

ACs consist of molecular pairs rather than single molecules; it is thus not obvious how best to split up a chemical data set into non-overlapping training and test sets for the fair evaluation of an AC-prediction method. There seems to be no consensus about which data splitting strategy should be canonically used. Several authors~\citep{heikamp_prediction_2012, de_la_vega_de_leon_prediction_2014, iqbal_prediction_2021} have employed a random split at the level of compound pairs. While this technique is conceptually straightforward, it must be expected to incur a significant overlap between training and test set at the level of individual molecules. For example, randomly splitting up a set of three MMPs $\{(\mathcal{R}_1, \mathcal{R}_2), (\mathcal{R}_2, \mathcal{R}_3), (\mathcal{R}_1, \mathcal{R}_3)\}$ into a training and a test set may lead to $(\mathcal{R}_1, \mathcal{R}_2)$ and $(\mathcal{R}_1, \mathcal{R}_3)$ getting assigned to the training set and $(\mathcal{R}_2, \mathcal{R}_3)$ getting assigned to the test set. This corresponds to a full inclusion of the test set in the training set at the level of individual molecules. A molecular overlap of this kind is problematic for at least three reasons: Firstly, it likely leads to overly optimistic performance estimates of AC-prediction methods since they will have already encountered some of the test compounds during training. Secondly, it does not model the natural situation encountered by medicinal chemists who it is assumed do not know the activity value of at least one compound in a test-set MMP. Thirdly, the mentioned molecular overlap should lead to strong AC-prediction results for standard QSAR models, but to the best of our knowledge, no such control experiments have been conducted in the literature.

\citet{horvath_prediction_2016} and~\citet{tamura_ligand-based_2020} have made efforts to address the shortcomings of a compound-pair-based random split. They proposed advanced data splitting algorithms designed to mitigate the described molecular-overlap problem by either managing distinct types of test sets according to compound membership in the training set or by designing splitting techniques based on the structural cores of MMPs. However, their data splitting schemes exhibit a relatively high degree of technical complexity which can make their implementation and interpretation non-straightforward.

For our study, we propose a novel data splitting method which may represent a favourable trade-off between rigour, interpretability and simplicity. Our technique shares some of its concepts with the methods proposed by~\citet{horvath_prediction_2016} and~\citet{tamura_ligand-based_2020} but might be simpler to implement and interpret. We first split the data into a training and test set at the level of individual molecules and then use this basic split to distinguish several types of test sets at the level of compound pairs. Let
$$\mathfrak{D} = \{\mathcal{R}_1, \mathcal{R}_2,...\}$$
be the given data set of individual compounds. One can for instance think of $\mathfrak{D}$ as a set of SMILES strings or molecular graphs. Furthermore, let 
$$\mathfrak{M} \subseteq \{ (\mathcal{R}, \tilde{\mathcal{R}}) \ \vert \ \mathcal{R}, \tilde{\mathcal{R}} \in \mathfrak{D}	\} $$
be the set of MMPs in $\mathfrak{D}$ that are eligible for the AC-classification task. This means that $\mathfrak{M}$ represents the set of MMPs in $\mathfrak{D}$ that have either been labelled as ACs or as non-ACs. Then each MMP $(\mathcal{R}, \tilde{\mathcal{R}}) \in \mathfrak{M}$ consists of two structurally similar compounds $\mathcal{R}$ and $\tilde{\mathcal{R}}$ that share a common structural core which we denote as $\text{core}(\mathcal{R}, \tilde{\mathcal{R}})$. To avoid redundancy, we associate each MMP with an (arbitrary) ordering of its two involved compounds and only contain one of both orderings in $\mathfrak{M}$, i.e.~if $(\mathcal{R}, \tilde{\mathcal{R}}) \in \mathfrak{M}$ then $(\tilde{\mathcal{R}}, \mathcal{R}) \notin \mathfrak{M}$.

We now use a uniform random split to partition $\mathfrak{D}$ into a training set $\mathfrak{D}_{\text{train}}$ and a test set $\mathfrak{D}_{\text{test}}$ such that $\mathfrak{D}_{\text{train}} \cap \mathfrak{D}_{\text{test}} = \emptyset$ and $\mathfrak{D}_{\text{train}} \cup \mathfrak{D}_{\text{test}} = \mathfrak{D}$. On the basis of this split, we define the following MMP sets:
\begin{align*}
\mathfrak{M}_{\text{train}} &= \{ (\mathcal{R}, \tilde{\mathcal{R}}) \in \mathfrak{M} \ \vert \ \mathcal{R}, \tilde{\mathcal{R}} \in \mathfrak{D}_{\text{train}}	\},	\\
\mathfrak{M}_{\text{test}} &= \{ (\mathcal{R}, \tilde{\mathcal{R}}) \in \mathfrak{M} \ \vert \ \mathcal{R}, \tilde{\mathcal{R}} \in \mathfrak{D}_{\text{test}}	\},\\
\mathfrak{M}^{}_{\text{inter}} &=  \mathfrak{M} \setminus (\mathfrak{M}_{\text{train}} \cup \mathfrak{M}_{\text{test}} ),\\
\mathfrak{M}_{\text{cores}} &= \{ (\mathcal{R}, \tilde{\mathcal{R}}) \in \mathfrak{M}_{\text{test}} \ \vert \ \text{core}(\mathcal{R}, \tilde{\mathcal{R}}) \notin \mathfrak{C}_{\text{train}} 	\}.
\end{align*}
Here $$\mathfrak{C}_{\text{train}} = \{ \text{core}(\mathcal{R}, \tilde{\mathcal{R}}) \ \vert \ (\mathcal{R}, \tilde{\mathcal{R}}) \in \mathfrak{M}_{\text{train}} \cup \mathfrak{M}_{\text{inter}}	\} $$
describes the set of structural MMP cores that appear in $\mathfrak{D}_{\text{train}}$.

Note that $\mathfrak{M}_{\text{train}} \cup \mathfrak{M}_{\text{inter}} \cup \mathfrak{M}_{\text{test}} = \mathfrak{M}$. The pair $(\mathfrak{D}_{\text{train}}, \mathfrak{M}_{\text{train}})$ describes the training space at the level of individual molecules and MMPs, and can be used to train a QSAR or AC-prediction method. MMPs in $\mathfrak{M}_{\text{test}}$, $\mathfrak{M}_{\text{inter}}$ and $\mathfrak{M}_{\text{cores}}$ can then be classified via a trained AC-predictor. $\mathfrak{M}_{\text{test}}$ models an AC-prediction setting where the activities of both MMP compounds are unknown. $\mathfrak{M}_{\text{cores}}$ represents the subset of MMPs in $\mathfrak{M}_{\text{test}}$ whose structural cores do not appear in $\mathfrak{M}_{\text{train}} \cup \mathfrak{M}_{\text{inter}}$; $\mathfrak{M}_{\text{cores}}$ thus models the difficult task of predicting ACs within MMPs that do not contain near analogs to MMP compounds in the training set. Finally, $\mathfrak{M}_{\text{inter}}$ represents an AC-prediction scenario where the activity of one MMP compound is given \textit{a priori}; this can be interpreted as a compound-optimisation task where one strives to predict small AC-inducing modifications of a query compound with known activity. Arguably the scenario modelled by $\mathfrak{M}_{\text{inter}}$ is the one that is most representative of real-world applications. An illustration of our data splitting strategy is given in~\ref{fig:ac_pred_data_splitting_strategy}.

\begin{figure}[!t]
	\centering
	\includegraphics[width=0.98\linewidth]{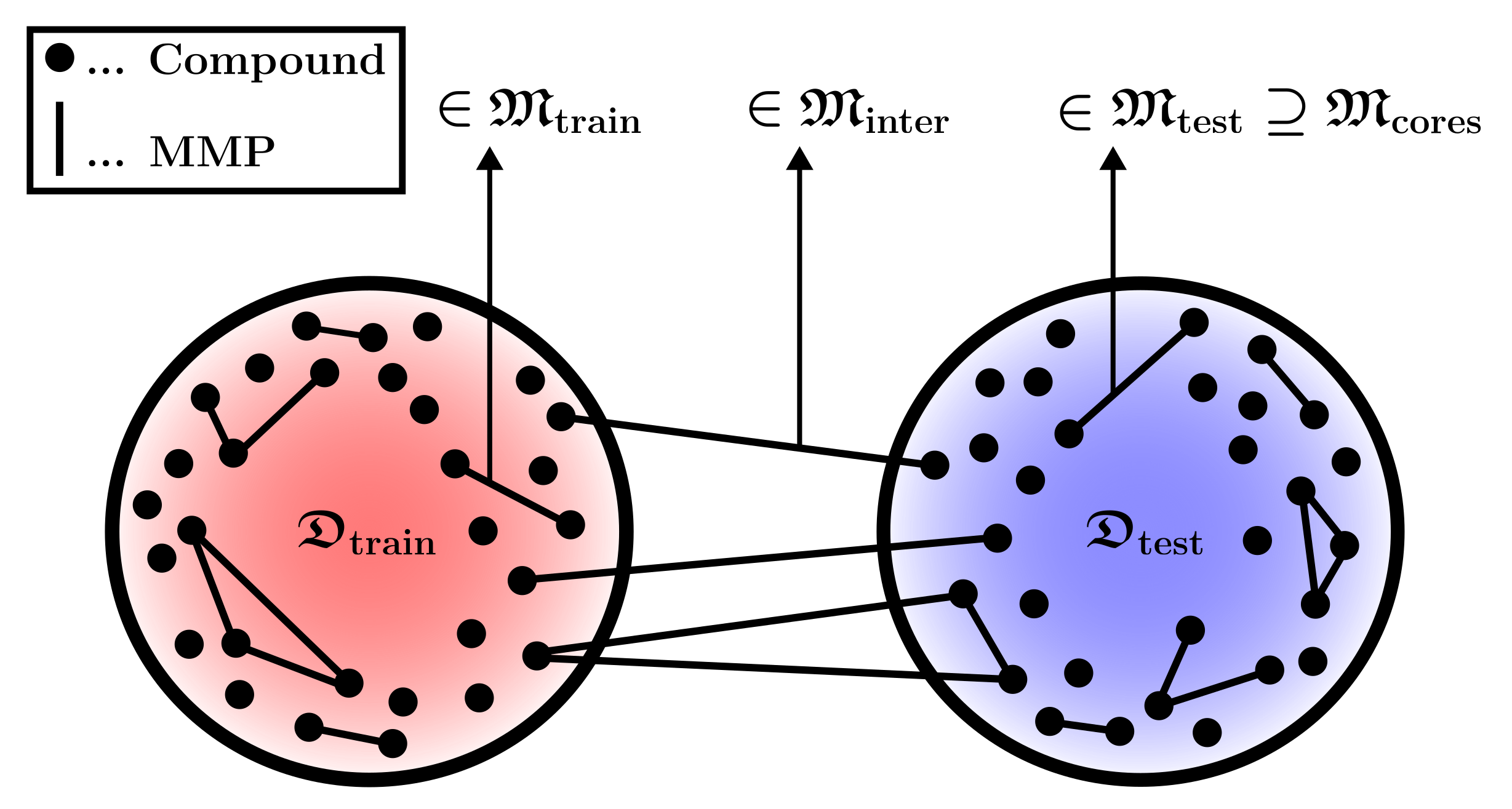}
	\caption[Data splitting strategy for AC-prediction study.]{Illustration of our data splitting strategy for activity-cliff~(AC) and potency-direction~(PD) classification. We distinguish between three sets of matched molecular pairs~(MMPs), $\mathfrak{M}_{\text{train}}, \mathfrak{M}_{\text{inter}}$ and $\mathfrak{M}_{\text{test}}$, depending on whether both MMP compounds are in $\mathfrak{D}_{\text{train}}$, one MMP compound is in $\mathfrak{D}_{\text{train}}$ and the other one is in $\mathfrak{D}_{\text{test}}$, or both MMP compounds are in $\mathfrak{D}_{\text{test}}$. We additionally consider a fourth MMP set, $\mathfrak{M}_{\text{cores}}$, consisting of the MMPs in $\mathfrak{M}_{\text{test}}$ whose structural cores do not appear in $\mathfrak{M}_{\text{train}} \cup \mathfrak{M}_{\text{inter}}$. }
	\label{fig:ac_pred_data_splitting_strategy}
\end{figure}

We implemented our data splitting strategy within a $k$-fold cross validation scheme repeated with $m$ random seeds. This generated data splits of the form 
$$\mathfrak{S}^{i,j} = (\mathfrak{D}^{i,j}_{\text{train}}, \mathfrak{D}^{i,j}_{\text{test}}, \mathfrak{M}^{i,j}_{\text{train}},\mathfrak{M}^{i,j}_{\text{test}}, \mathfrak{M}^{i,j}_{\text{inter}}, \mathfrak{M}^{i,j}_{\text{cores}}) $$
for $i \in \{1,...,m\}$ and $j \in \{1,...,k\}$ whereby the pair $(\mathfrak{D}^{i,j}_{\text{train}}, \mathfrak{D}^{i,j}_{\text{test}})$
represents the $j$-th split with the $i$-th random seed of the underlying data set $\mathfrak{D}$ in a $k$-fold cross validation scheme repeated with $m$ random seeds. The overall performance of each model for all prediction tasks was recorded as the average over the $mk$ training and test runs based on all data splits $\mathfrak{S}^{1,1}, ..., \mathfrak{S}^{m,k}$. We chose the configuration $(m,k) = (3,2)$ which gave a good trade-off between computational cost and accuracy and reasonable numbers of MMPs in the compound-pair-sets. In particular, random cross-validation with $k = 2$ gave expected relative sizes of:
$$\vert \mathfrak{M}_{\text{train}} \vert : \vert \mathfrak{M}_{\text{inter}} \vert : \vert \mathfrak{M}_{\text{test}} \vert = 1 : 2 : 1 .$$ 
On average, $12.7\%$, $11.91\%$, and $6.84\%$ of MMPs in $\mathfrak{M}_{\text{test}}$ were also in $\mathfrak{M}_{\text{cores}}$ for dopamine receptor D2, factor Xa, and SARS-CoV-2 main protease, respectively.

\subsection{Prediction Strategies} \label{subsec: qsar_ac_study_pred_strategies}

In a data split of the form
$$\mathfrak{S} = (\mathfrak{D}_{\text{train}}, \mathfrak{D}_{\text{test}}, \mathfrak{M}_{\text{train}},\mathfrak{M}_{\text{test}}, \mathfrak{M}_{\text{inter}}, \mathfrak{M}_{\text{cores}})$$
each MMP 
$$(\mathcal{R}, \tilde{\mathcal{R}}) \in \mathfrak{M}_{\text{train}} \cup \mathfrak{M}_{\text{inter}} \cup \mathfrak{M}_{\text{test}} = \mathfrak{M}$$
comes with a binary label, $\text{AC}(\mathcal{R}, \tilde{\mathcal{R}}) \in \{\text{Non-AC}, \text{AC}\}$, indicating whether $(\mathcal{R}, \tilde{\mathcal{R}})$ is an AC or not and another binary label, $\text{PD}(\mathcal{R}, \tilde{\mathcal{R}}) \in \{\text{Right}, \text{Left}\}$, indicating which of both compounds is more active. Furthermore, each individual compound
$$\mathcal{R} \in \mathfrak{D}_{\text{train}} \cup \mathfrak{D}_{\text{test}} = \mathfrak{D}$$ 
can be associated with a numerical activity label $\text{act}(\mathcal{R}) \in \mathbb{R}$ which we define as the negative decadic logarithm of the experimentally measured activity of $\mathcal{R}$.
We stuck with the original units used in the ChEMBL database and the COVID moonshot project before applying the logarithm ([nM] for K\textsubscript{i} and [µM] for IC\textsubscript{50}); each activity label $\text{act}(\mathcal{R})$ thus represents a standard pK\textsubscript{i} or pIC\textsubscript{50} value with a slight additive shift towards $0$ caused by the use of [nM] or [µM]-units instead of the canonical [M]-units; this shift towards~$0$ might slightly benefit prediction techniques initialised around the origin.

We are interested in QSAR-prediction functions
$$ Q : \mathfrak{D} \to \mathbb{R}$$
that map a given molecular representation $\mathcal{R} \in \mathfrak{D}$ to an estimate of its binding affinity: 
$$Q(\mathcal{R}) \approx \text{act}(\mathcal{R}).$$ In our study, the molecular representation $\mathcal{R}$ is either a SMILES string or a molecular graph. The mapping $Q$ is found via an algorithmic training process on the labelled data set 
$$\{ (\mathcal{R}, \text{act}(\mathcal{R}))	\ \vert \ \mathcal{R} \in \mathfrak{D}_{\text{train}} \}	$$
and can then either be used to predict the activity labels of compounds in $\mathfrak{D}_{\text{test}}$, or it can be repurposed to classify whether an MMP forms an activity cliff~(AC-classification) and what the potency direction of an MMP is~(PD-classification). 

If $(\mathcal{R}, \tilde{\mathcal{R}}) \in \mathfrak{M}_{\text{inter}}$, then one can assume that the activity label of one of the compounds, say $\text{act}(\mathcal{R})$, is known. $Q$ is then used to generate an AC-classification for $(\mathcal{R}, \tilde{\mathcal{R}})$ via
$$ (\mathcal{R}, \tilde{\mathcal{R}}) \mapsto
\begin{cases}
\text{Non-AC} \quad &\text{if} \ \vert \text{act}(\mathcal{R}) - Q(\tilde{\mathcal{R}}) \vert \leq d_{\text{crit}}, \\
\text{AC} \quad &\text{else}.
\end{cases} $$
Here $d_{\text{crit}} \in \mathbb{R}_{> 0}$ is a critical threshold above which an MMP is classified as an AC. Throughout this work we use $d_{\text{crit}} = 1.5$ (in pK\textsubscript{i} or pIC\textsubscript{50} units) since this value represents the middle point between the intervals $[0, 1]$ and $[2, \infty)$ which correspond to absolute logarithmic activity differences associated with non-ACs and ACs respectively. If $(\mathcal{R}, \tilde{\mathcal{R}}) \in \mathfrak{M}_{\text{test}} \cup \mathfrak{M}_{\text{cores}}$ then $\mathcal{R}, \tilde{\mathcal{R}} \in \mathfrak{D}_{\text{test}}$ and therefore the activities of both compounds are unknown. We hence perform the AC-classification for $(\mathcal{R}, \tilde{\mathcal{R}})$ via
$$ (\mathcal{R}, \tilde{\mathcal{R}}) \mapsto
\begin{cases}
\text{Non-AC} \quad &\text{if} \ \vert Q(\mathcal{R}) - Q(\tilde{\mathcal{R}}) \vert \leq d_{\text{crit}}, \\
\text{AC} 	\quad &\text{else}.
\end{cases} $$
The PD-classification for an MMP $(\mathcal{R}, \tilde{\mathcal{R}}) \in \mathfrak{M}_{\text{inter}}$ with $Q$ is performed in a straightforward manner by simply comparing the binding affinity prediction of the test compound $\tilde{\mathcal{R}}$ with the experimentally measured binding affinity of the training compound~$\mathcal{R}$:
$$ (\mathcal{R}, \tilde{\mathcal{R}}) \mapsto
\begin{cases}
\text{Right} \quad &\text{if} \ \text{act}(\mathcal{R}) \leq Q(\tilde{\mathcal{R}}), \\
\text{Left} \quad &\text{else}.
\end{cases} $$
Similarly, the PD-classification for an MMP $(\mathcal{R}, \tilde{\mathcal{R}}) \in \mathfrak{M}_{\text{test}} \cup \mathfrak{M}_{\text{cores}}$ with $Q$ is performed via:
$$ (\mathcal{R}, \tilde{\mathcal{R}}) \mapsto
\begin{cases}
\text{Right} \quad &\text{if} \ Q(\mathcal{R}) \leq Q(\tilde{\mathcal{R}}), \\
\text{Left} \quad &\text{else}.
\end{cases} $$

\subsection{Performance Metrics} \label{subsec: qsar_ac_study_ac_perform_measures}

The performance of $Q$ when used as a standard QSAR method for the prediction of the activity labels of individual molecules in $\mathfrak{D}_{\text{test}}$ was measured via the mean absolute error~(MAE):
$$ \frac{1}{\vert \mathfrak{D}_{\text{test}} \vert} \sum_{\mathcal{R} \in \mathfrak{D}_{\text{test}}} \vert Q(\mathcal{R}) - \text{act}(\mathcal{R}) \vert. $$
For the balanced PD-classification problem we could rely on simple accuracy as a suitable performance metric:
$$ \frac{\text{number of correct predictions}}{\text{number of predictions}}.$$
However, when using $Q$ for the naturally highly imbalanced AC-classification task, a more nuanced set of performance metrics had to be chosen. Denote with TP, TN, FP, and FN the numbers respectively representing true positives, true negatives, false positives and false negatives for the AC-classification task. We then used the Matthews correlation coefficient~(MCC) as a suitable overall performance metric:
$$\frac{TP * TN - FP*FN}{\sqrt{(TP + FP)(TP + FN)(TN + FP)(TN + FN)}}. $$
In addition, we tracked AC-sensitivity
$$\frac{TP}{TP + FN} $$
and AC-precision
$$ \frac{TP}{TP + FP}.$$
The MCC represents a summary statistic for the confusion matrix of a binary classification problem that is reasonably robust against imbalanced class labels. Sensitivity~(also known as recall) can be interpreted as an approximation for the probability of the classifier to classify an actually positive instance as a positive one. Finally, precision~(also known as positive predictive value) approximates the probability that a positively classified instance is indeed positive. 

For the relatively small SARS-CoV-2 main protease data set we sometimes encountered the edge case where $\text{TP} + \text{FP} = 0$, i.e.~where there were no positive predictions. In this situation we set $\text{MCC} = 0$ and ignored the ill-defined precision measurements when averaging the performance metrics.

\subsection*{Molecular Featurisation Methods and Regression Techniques} \label{subsec: qsar_ac_study_mol_reps_regr_techn}

We constructed nine QSAR models (i.e.~nine versions of $Q$) via a robust combinatorial methodology that systematically combines three molecular featurisation methods with three regression techniques. This setup allows one to systematically compare the performance of molecular featurisations across regression techniques, data sets and predictions tasks. For molecular featurisation, we used PDVs~(\ref{sec: pdvs}) and ECFPs (\ref{sec: ecfps}), both generated from SMILES strings~(\ref{sec: smiles}), as well as GINs~(\ref{subsec: gins_description}) on top of molecular graphs~(\ref{sec: mol_graphs}). Both the ECFPs and the PDVs were computed via \texttt{RDKit}~\citep{landrum2006rdkit}. The ECFPs used a radius of two, a length of $2048$ bits, and standard atom features with active tetrahedral R-S chirality flags. The PDVs had a dimensionality of $200$, were constructed using the list of descriptors specified in Table~\ref{tab: physchem_descriptors_rdkit}, and were normalised to lie in the hypercube~$[0,1]^{200}$ via their componentwise cumulative distribution functions derived from the training set as explained in Section~\ref{sec: pdvs}. The PDV descriptor-list encompassed properties related to druglikeness, logP, molecular refractivity, electrotopological state, molecular graph structure, fragment-profile, molecular charge, and molecular surface. The GIN was implemented via \texttt{PyTorch Geometric}~\citep{fey2019fast}. The atom features of the underlying molecular graph objects can be found in Table~\ref{tab: atom_bond_features}. For global graph pooling we chose the componentwise maximum over all atom feature vectors in the final graph-layer.

Each molecular featurisation was used as an input featurisation for three regression techniques: random forests (RFs), k-nearest neigbours (kNNs) and multilayer perceptrons (MLPs). The RF and kNN models were implemented via scikit-learn~\citep{pedregosa2011scikit} and the MLP models via \texttt{PyTorch}~\citep{paszke2019pytorch}. The MLPs used ReLU activations and batch normalisation at each hidden layer. The GIN was combined with the regression techniques as follows: For MLP regression, the GIN was trained with the MLP as a projection head after the pooling step in the usual end-to-end manner. For RF or kNN regression, the GIN was first trained with a single linear layer added after the global pooling step that directly mapped the graph-level representation to an activity prediction. After this training phase the weights of the GIN were frozen and it was used as a static feature extractor. The RF or kNN regressor was then trained on top of the features extracted by the frozen GIN. Figure~\ref{fig:qsar_ac_study_methods_overview} depicts an overview of all investigated QSAR models.
\begin{figure}[h]
	\centering
	\includegraphics[width=1\linewidth]{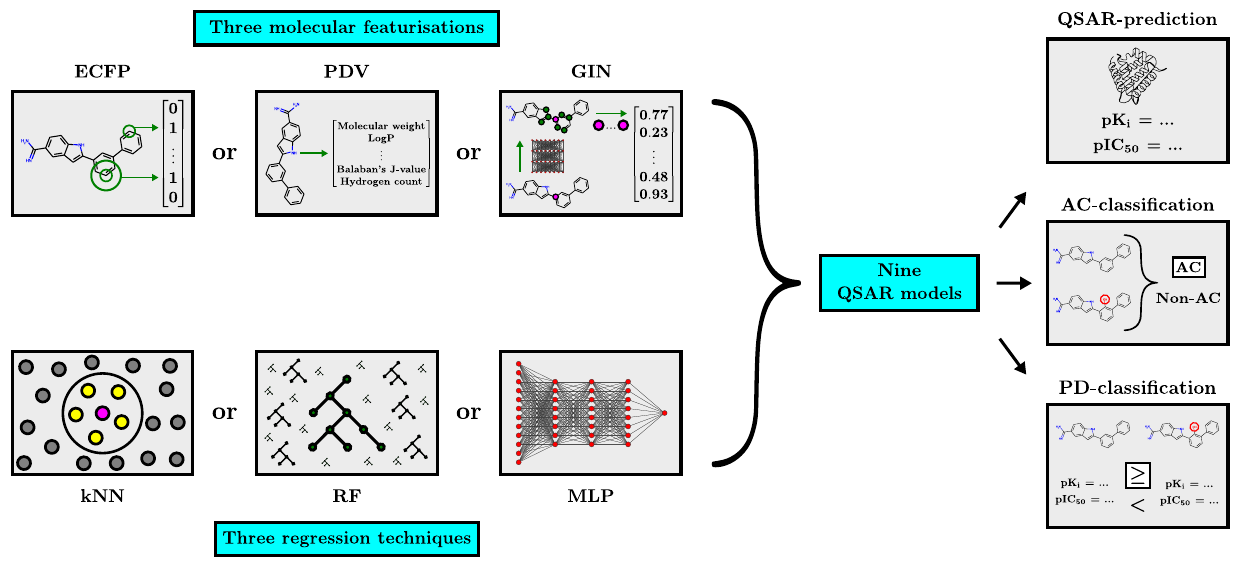}
	\caption[Overview of investigated QSAR models for AC-prediction study.]{Schematic showing the combinatorial experimental methodology used for the study. Each molecular featurisation method is systematically combined with each regression technique, giving a total of nine quantitative structure-activity relationship~(QSAR) models. Each QSAR model is trained and evaluated for QSAR-prediction, activity-cliff~(AC) classification and potency-direction~(PD) classification via $2$-fold cross validation repeated with $3$ random seeds. For each of the $2*3 = 6$ trials, an extensive inner hyperparameter-optimisation loop on the training set is performed for each QSAR model.}
	\label{fig:qsar_ac_study_methods_overview}
\end{figure}

\subsection{Model Training and Hyperparameter Optimisation} \label{subsec: qsar_ac_study_training_and_hyperparams}

The three investigated regression techniques as well as the GIN feature-extractor come with a multitude of training and model hyperparameters which need to be tuned properly as to allow each QSAR model to unfold its maximum performance. Since our presented study is comparative in nature, it was essential to include a systematic hyperparameter-optimisation procedure within the training routine of each model. While such a hyperparameter optimisation greatly increased the computational cost and the difficulty of implementation for our numerical experiments, it formed an indispensable part of a fair and objective model comparison.

As described previously, each QSAR model was evaluated within a $k$-fold cross validation scheme repeated with $m$ random seeds. This implies that an independent version of each QSAR model was trained on each molecular training set $\mathfrak{D}^{i,j}_{\text{train}}$ for $i \in \{1,...,m\}$ and $j \in \{1,...,k\}$ and the results were then averaged for each model over all $mk$ trials. At each of the $mk$ training rounds, each model was first trained via an inner hyperparameter-optimisation loop on $\mathfrak{D}^{i,j}_{\text{train}}$. The determined set of hyperparameters was then used to train a model with optimised architecture on $\mathfrak{D}^{i,j}_{\text{train}}$ and this optimised model was used for evaluation. In the implementation of the inner hyperparameter-optimisation loop we distinguished between the four classical machine learning models (ECFPs or PDVs combined with either RFs or kNNs) and the five models that contained a deep learning component (a GIN or an MLP) in order to save computational resources. 

In the classical case we used a five-fold inner cross validation split on $\mathfrak{D}^{i,j}_{\text{train}}$; $10$ models with distinct hyperparameter settings were then trained within this cross validation scheme. The $10$ hyperparameter configurations were sampled uniformly at random from a predefined grid using the RandomizedSearchCV routine implemented in scikit-learn~\citep{pedregosa2011scikit}. The hyperparameters that minimised the MAE over the inner cross validation loop were subsequently chosen to train the final model. For RF regression, we chose a forest size of $500$ trees and optimised the maximum tree depth, the minimum number of samples required to split an internal node, the minimum number of samples required to be at a leaf node, the number of features to consider for the best split, and whether bootstrap samples should be used or not when building trees. For the kNN algorithm, we optimised the number of considered neighbours, the power parameter for the underlying Minkowski distance measure, and whether to give uniform or inverse-distance weights to neighbours.

In the deep learning case, we employed a $4$:$1$ split of $\mathfrak{D}^{i,j}_{\text{train}}$ into an inner training set and an inner validation set. We then trained $20$ models with distinct hyperparameter configurations on the inner training set. The $20$ hyperparameter sets were sampled from a predefined grid using the tree-structured Parzen estimator~(TPE) algorithm implemented in the Optuna hyperparameter-optimisation package~\citep{akiba2019optuna}. The hyperparameters that minimised the MAE on the inner validation set were subsequently chosen to train the final model. For the MLP architecture, we optimised the number of hidden layers and the number of neurons per hidden layer. Additionally, we chose ReLU as the hidden activation function and used batch normalisation~\citep{ioffe2015batch} throughout the neural network. 

For the GINs~(\ref{subsec: gins_description}) we optimised the total number of graph layers $R$ and the dimensional length of the updated atom feature vectors $f_r(a)$ at the $r$-th graph layer. For each layerwise multilayer perceptron $\phi_r$ that formed part of the GIN we used two internal hidden layers whereby the number of neurons in each hidden layer was equivalent to the dimensionality of the atom feature vectors in the $r$-th layer. For the two hidden layers of each $\phi_r$ we employed batch normalisation and again chose ReLU as the hidden activation function. Each $\epsilon_r$ was set to $0$. To reduce the molecular graph to a single feature vector after the message-passing phase, we employed max-pooling at the $R$-th graph layer which computes the componentwise maximum over all final atom feature vectors. Since we optimised the atom feature dimensionality in the final graph layer we also implicitly optimised the dimensionality $l$ of the final graph-level fingerprint.

All deep learning models were trained for $500$ epochs on a single NVIDIA GeForce RTX 3060 GPU using the mean squared error loss function and AdamW optimisation~\citep{loshchilov2017decoupled}. During training we employed weight decay, learning rate decay and dropout~\citep{srivastava2014dropout} at all hidden layers for regularisation. Batch size, learning rate, learning rate decay rate, weight decay rate, and dropout rate were treated as hyperparameters and subsequently optimised. Note that the training length (the number of gradient updates) was implicitly optimised via tuning the batch size for the fixed number of $500$ training epochs.

\section{Results and Discussion} \label{sec: qsar_ac_study_results_discussion}

The QSAR-prediction, AC-classification and PD-classification results for all three investigated data sets are depicted in~\Cref{fig:ac_results_dopamined2,fig:ac_results_factorxa,fig:ac_results_sarscov2mpro,fig:pd_results_dopamined2,fig:pd_results_factorxa,fig:pd_results_sarscov2mpro} below.

\begin{figure}
	\centering
	\includegraphics[width=1\linewidth]{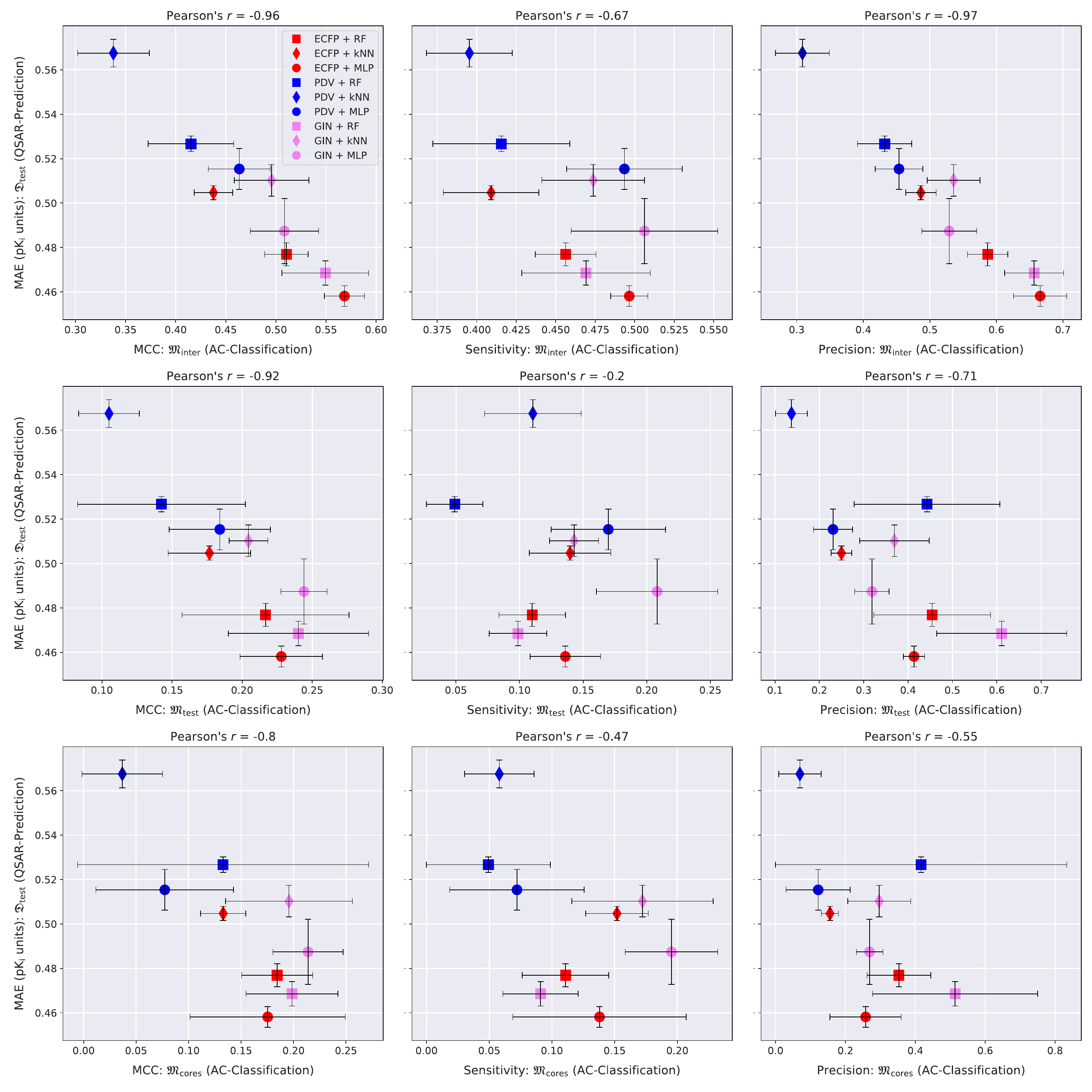}
	\caption[QSAR and AC-prediction results for dopamine receptor D2.]{QSAR-prediction and activity-cliff~(AC) classification results for \textbf{dopamine~receptor~D2}. For each plot, the $x$-axis corresponds to a combination of MMP set and AC-classification performance metric and the $y$-axis shows the QSAR-prediction performance on the molecular test set $\mathfrak{D}_{\text{test}}$. The total length of each error bar equals twice the standard deviation of the performance metric measured over all $mk = 3*2 = 6$ hyperparameter-optimised models. For each plot, the lower right corner corresponds to strong performance at both prediction tasks.}
	\label{fig:ac_results_dopamined2}
\end{figure}
\begin{figure}
	\centering
	\includegraphics[width=1\linewidth]{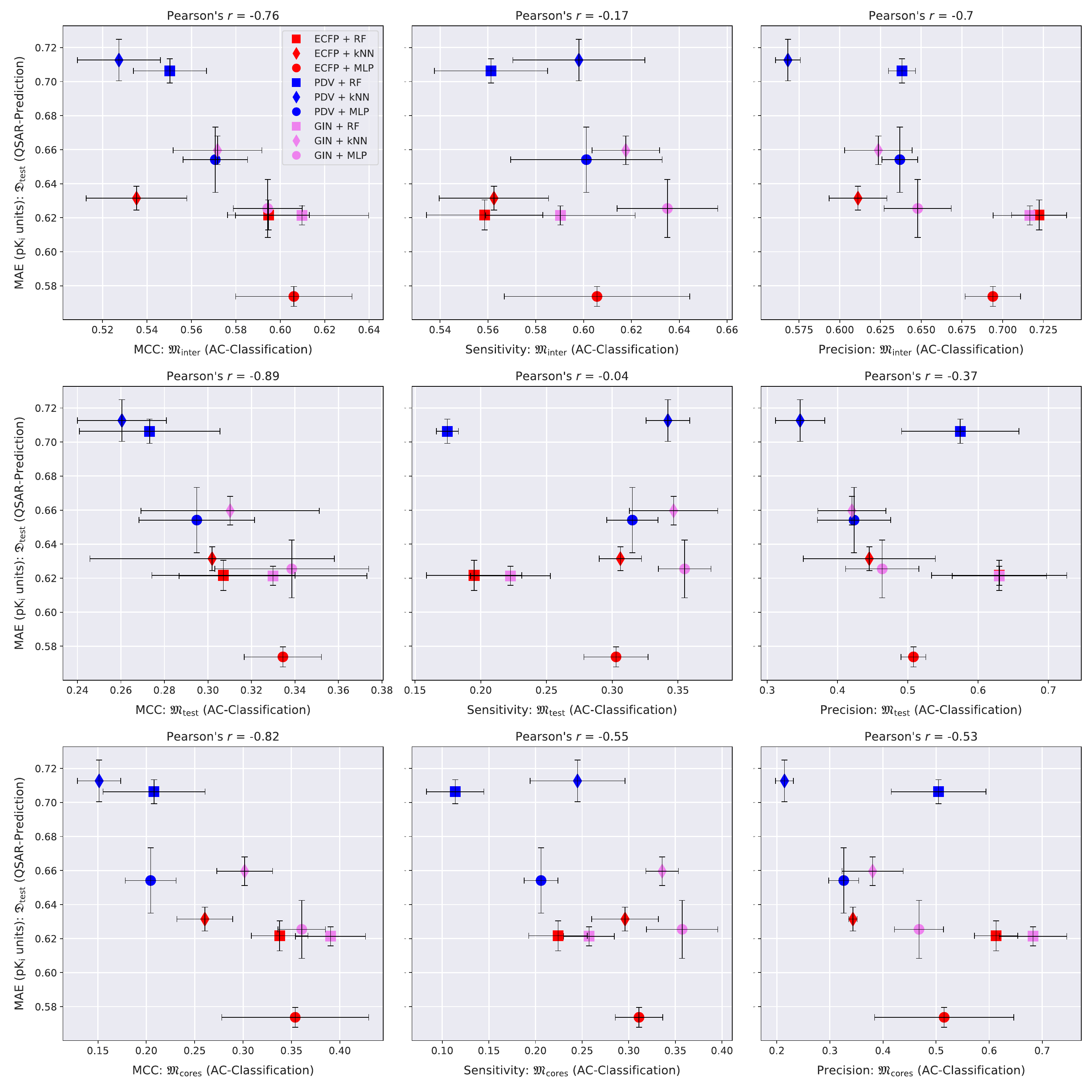}
	\caption[QSAR and AC-prediction results for factor Xa.]{QSAR-prediction and activity-cliff~(AC) classification results for \textbf{factor Xa}. For each plot, the $x$-axis corresponds to a combination of MMP set and AC-classification performance metric and the $y$-axis shows the QSAR-prediction performance on the molecular test set $\mathfrak{D}_{\text{test}}$. The total length of each error bar equals twice the standard deviation of the performance metric measured over all $mk = 3*2 = 6$ hyperparameter-optimised models. For each plot, the lower right corner corresponds to strong performance at both prediction tasks.}
	\label{fig:ac_results_factorxa}
\end{figure}
\begin{figure}
	\centering
	\includegraphics[width=1\linewidth]{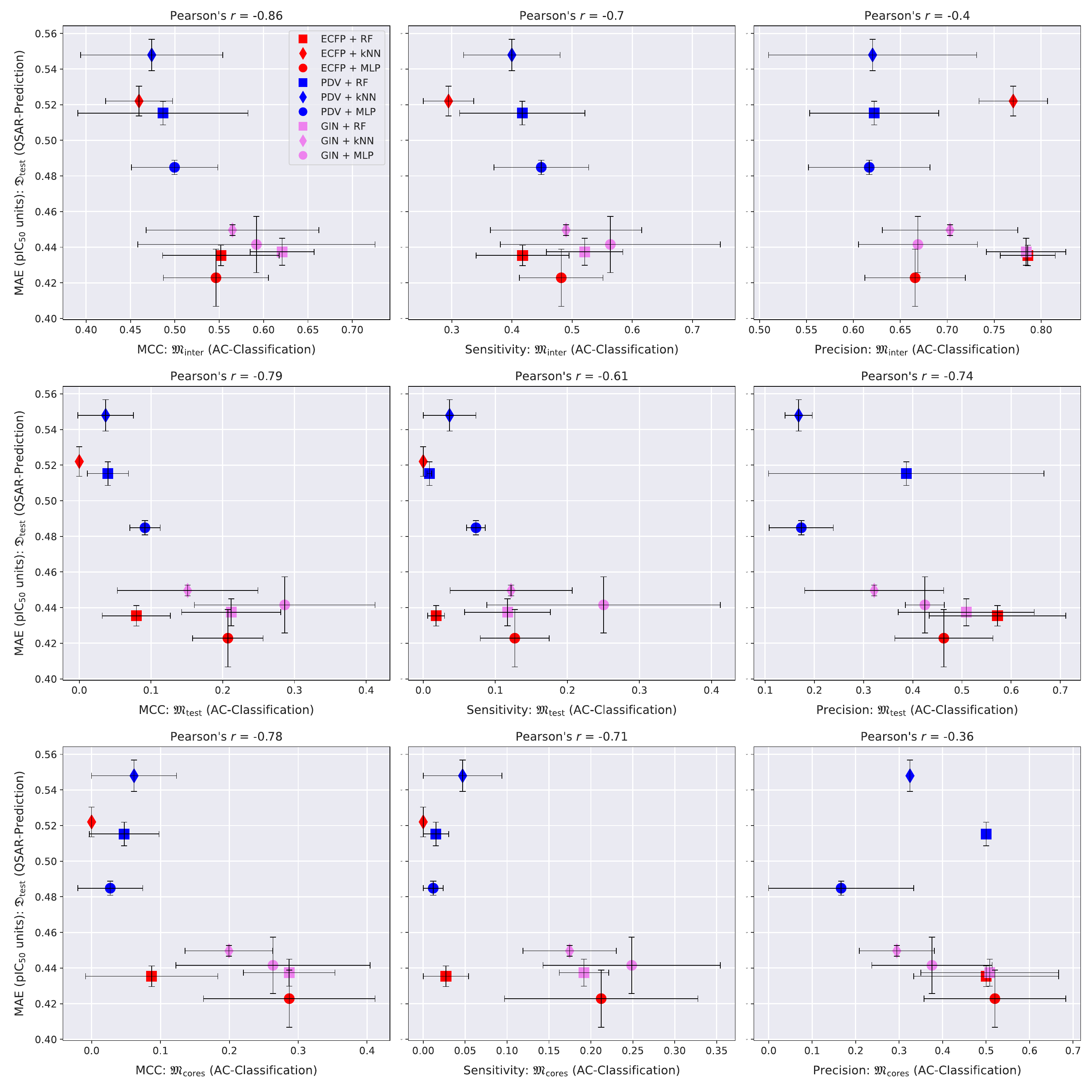}
	\caption[QSAR and AC-prediction results for SARS-CoV-2 main protease.]{QSAR-prediction and activity-cliff~(AC) classification results for \textbf{SARS CoV-2 main protease}. For each plot, the $x$-axis corresponds to a combination of MMP set and AC-classification performance metric and the $y$-axis shows the QSAR-prediction performance on the molecular test set~$\mathfrak{D}_{\text{test}}$. The total length of each error bar equals twice the standard deviation of the performance metric measured over all $mk = 3*2 = 6$ hyperparameter-optimised models. The precision of the AC-classification task is not shown for the ECFP + kNN technique on $\mathfrak{M}_{\text{test}}$ and $\mathfrak{M}_{\text{cores}}$ since this method produced only negative AC-classifications for all trials on this data set. For each plot, the lower right corner corresponds to strong performance at both prediction tasks.}
	\label{fig:ac_results_sarscov2mpro}
\end{figure}
\begin{figure}
	\centering
	\includegraphics[width=1\linewidth]{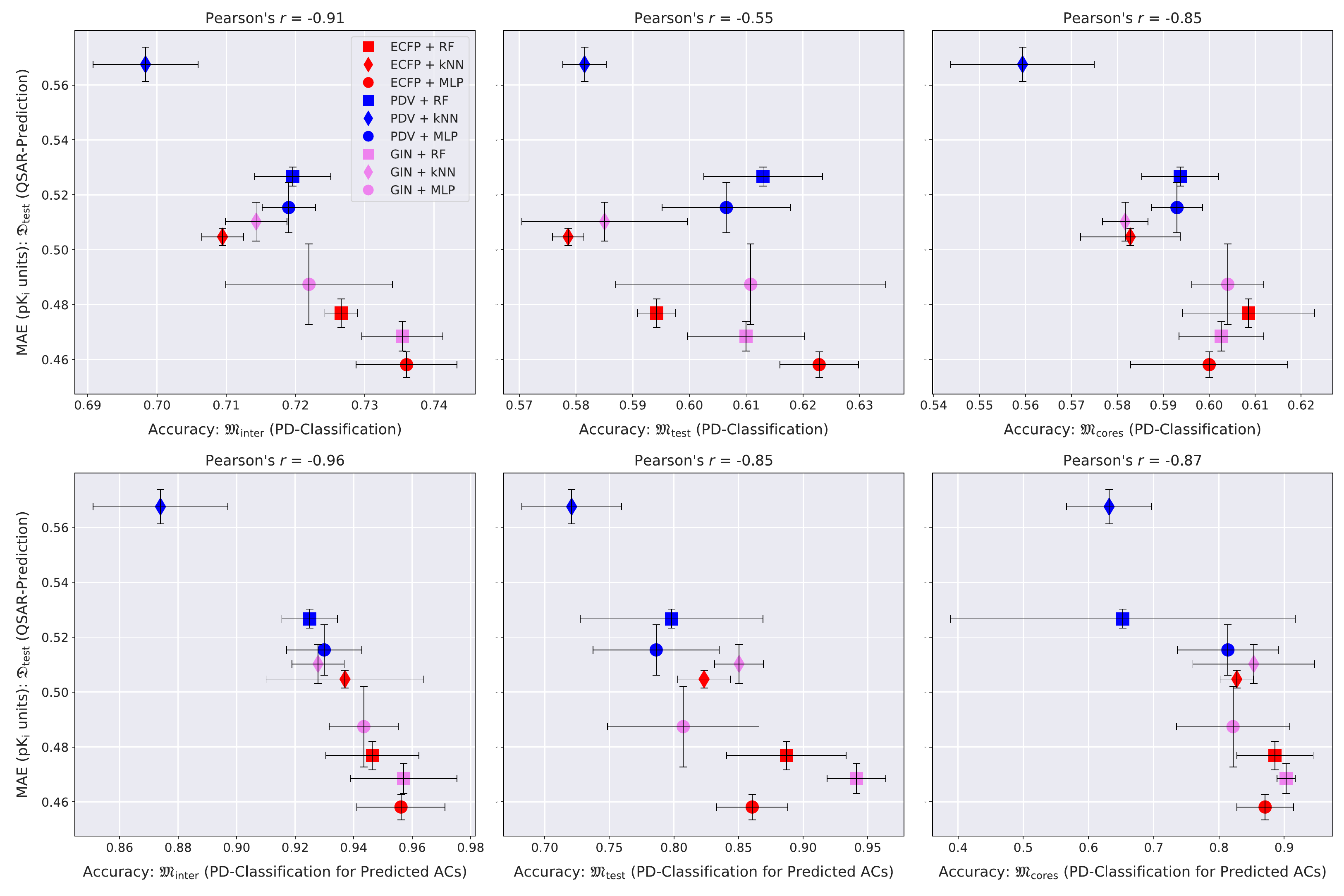}
	\caption[QSAR and PD-prediction results for for dopamine receptor D2.]{QSAR-prediction and potency-direction~(PD) classification results for \textbf{dopamine receptor D2}. Each column corresponds to an upper plot and a lower plot for one of the MMP sets $\mathfrak{M}_{\text{inter}}$, $\mathfrak{M}_{\text{test}}$ or $\mathfrak{M}_{\text{cores}}$. The x-axis of each upper plot indicates the PD-classification accuracy on the full MMP set; the x-axis of each lower plot indicates the PD-classification accuracy on a restricted MMP set only consisting of MMP predicted to be ACs by the respective method. The $y$-axis of each plot shows the QSAR-prediction performance on the molecular test set $\mathfrak{D}_{\text{test}}$. The total length of each error bar equals twice the standard deviation of the performance metrics measured over all $mk = 3*2 = 6$ hyperparameter-optimised models. For each plot, the lower right corner corresponds to strong performance at both prediction tasks.}
	\label{fig:pd_results_dopamined2}
\end{figure}
\begin{figure}
	\centering
	\includegraphics[width=1\linewidth]{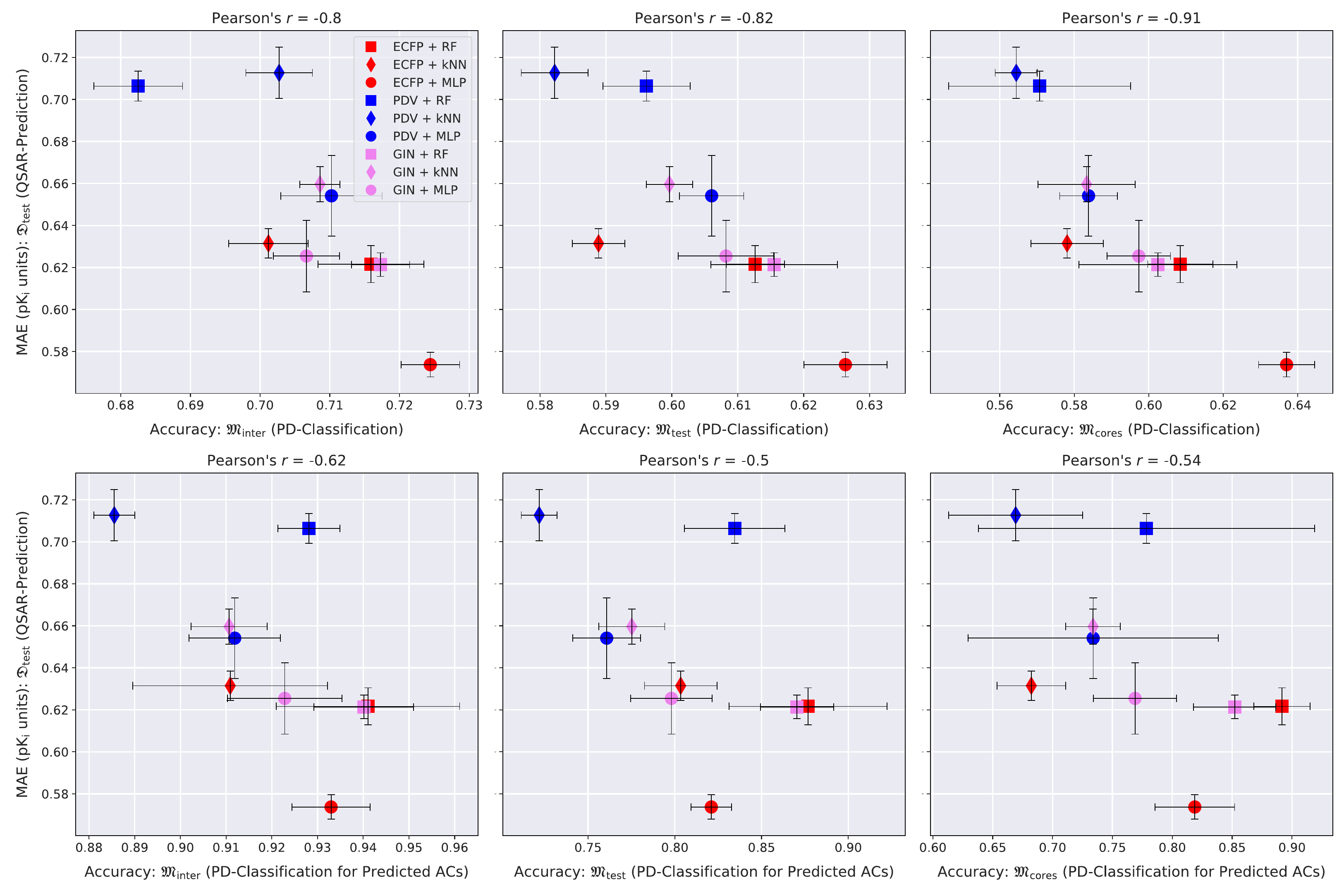}
	\caption[QSAR and PD-prediction results for factor Xa.]{QSAR-prediction and potency-direction~(PD) classification results for \textbf{factor Xa}. Each column corresponds to an upper plot and a lower plot for one of the MMP sets $\mathfrak{M}_{\text{inter}}$, $\mathfrak{M}_{\text{test}}$ or $\mathfrak{M}_{\text{cores}}$. The x-axis of each upper plot indicates the PD-classification accuracy on the full MMP set; the x-axis of each lower plot indicates the PD-classification accuracy on a restricted MMP set only consisting of MMP predicted to be ACs by the respective method. The $y$-axis of each plot shows the QSAR-prediction performance on the molecular test set $\mathfrak{D}_{\text{test}}$. The total length of each error bar equals twice the standard deviation of the performance metrics measured over all $mk = 3*2 = 6$ hyperparameter-optimised models. For each plot, the lower right corner corresponds to strong performance at both prediction tasks.}
	\label{fig:pd_results_factorxa}
\end{figure}
\begin{figure}
	\centering
	\includegraphics[width=1\linewidth]{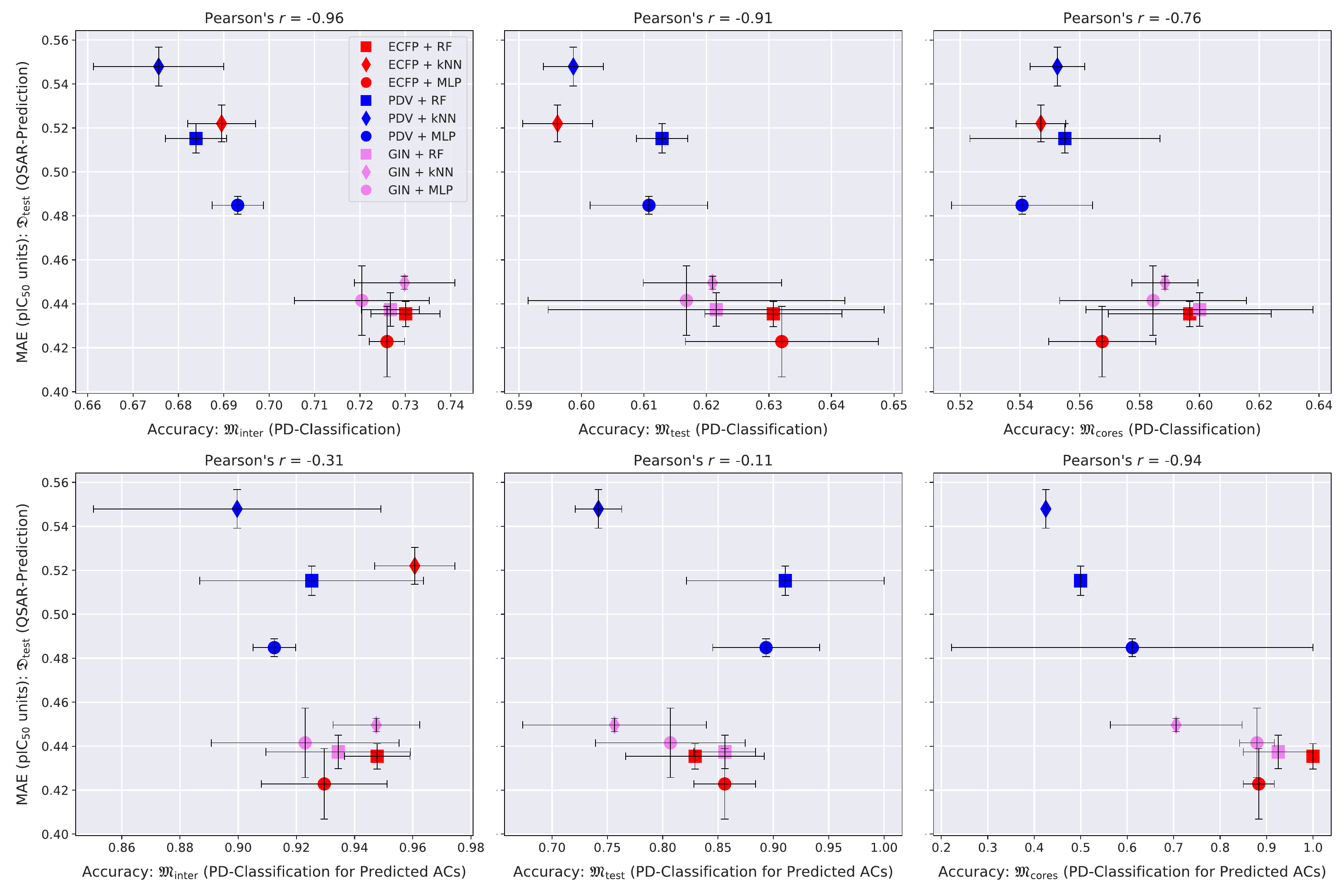}
	\caption[QSAR and PD-prediction results for SARS-CoV-2 main protease.]{QSAR-prediction and potency-direction~(PD) classification results for \textbf{SARS-CoV-2 main protease}. Each column corresponds to an upper plot and a lower plot for one of the MMP sets $\mathfrak{M}_{\text{inter}}$, $\mathfrak{M}_{\text{test}}$ or $\mathfrak{M}_{\text{cores}}$. The x-axis of each upper plot indicates the PD-classification accuracy on the full MMP set; the x-axis of each lower plot indicates the PD-classification accuracy on a restricted MMP set only consisting of MMP predicted to be ACs by the respective method. The $y$-axis of each plot shows the QSAR-prediction performance on the molecular test set $\mathfrak{D}_{\text{test}}$. The total length of each error bar equals twice the standard deviation of the performance metrics measured over all $mk = 3*2 = 6$ hyperparameter-optimised models. The accuracy of the PD-classification task for predicted ACs is not shown for the ECFP + kNN technique on $\mathfrak{M}_{\text{test}}$ and $\mathfrak{M}_{\text{cores}}$ since this method produced only negative AC-classifications for all trials on this data set. For each plot, the lower right corner corresponds to strong performance at both prediction tasks.}
	\label{fig:pd_results_sarscov2mpro}
\end{figure}

\subsection{QSAR-Prediction Performance} \label{subsec: qsar_ac_study_qsar_pred_perform}

When considering the results depicted in~\Cref{fig:ac_results_dopamined2,fig:ac_results_factorxa,fig:ac_results_sarscov2mpro,fig:pd_results_dopamined2,fig:pd_results_factorxa,fig:pd_results_sarscov2mpro} with respect to QSAR-prediction performance, one can see that ECFPs tend to lead to better performance (i.e.~a lower QSAR-MAE) compared to GINs, which in turn tend to lead to better performance compared to PDVs. In particular, the combination ECFP-MLP consistently produced the lowest QSAR-MAE across all three targets. These observations reinforce a growing corpus of literature that suggests that trainable GNNs have not yet reached a level of technical maturity by which they consistently and definitively outperform the much simpler task-agnostic ECFPs at important molecular property prediction problems~\citep{stepivsnik2021comprehensive, mayr2018large, jiang2021could, menke2021using, chithrananda2020chemberta, sabando2021using, winter2019learning}. 

\subsection{AC-Classification Performance} \label{subsec: qsar_ac_study_ac_pred_perform}

The AC-MCC plots in~\Cref{fig:ac_results_dopamined2,fig:ac_results_factorxa,fig:ac_results_sarscov2mpro} reveal surprisingly strong overall AC-classification results on $\mathfrak{M}_{\text{inter}}$. This type of MMP set models a compound-optimisation scenario where a researcher strives to identify small structural modifications with a large impact on the activity of query compounds with known activities. For this task, a substantial number of our QSAR models exhibit an AC-MCC value greater than $0.5$ across targets, which appears impressive considering the simplicity of the approach. Exchanging $\mathfrak{M}_{\text{inter}}$ with either $\mathfrak{M}_{\text{test}}$ or $\mathfrak{M}_{\text{cores}}$ leads to a substantial drop in the AC-MCC to approximately $0.3$ that appears to be mediated by a large drop in AC-sensitivity.

In most cases, GINs perform better than the other molecular featurisation methods with respect to the AC-MCC. Notably, the combination GIN-kNN consistently performs considerably better for AC-classification than the combinations ECFP-kNN and PDV-kNN. This supports the idea that GINs might have a heightened ability to resolve ACs by learning an embedding of chemical space in which the distance between two compounds is reflective of activity difference rather than structural difference. The combinations GIN-MLP, GIN-RF and ECFP-MLP exhibit particularly high AC-MCC values relative to the other methods. We recommend using at least one of these three models as a baseline against which to compare tailored AC-classification models; the practical utility of any AC-classification technique that cannot outperform these three common QSAR methods is questionable.

Across all three targets, AC-sensitivity is moderately high on $\mathfrak{M}_{\text{inter}}$ but universally low on $\mathfrak{M}_{\text{test}}$ and $\mathfrak{M}_{\text{cores}}$. This is consistent with the hypothesis that ACs form one of the major sources of prediction error for QSAR models. The weak AC-sensitivity on $\mathfrak{M}_{\text{test}}$ and $\mathfrak{M}_{\text{cores}}$ indicates that modern QSAR methods are largely blind to ACs formed by two MMP compounds outside the training set and thus lack essential chemical knowledge. GINs clearly outperform the other two more classical molecular featurisations across regression techniques with respect to AC-sensitivity. In particular, the GIN-MLP combination leads to the highest AC-sensitivity in all examined cases and thus discovers the most ACs. The highly parametric nature of GINs that makes them prone to overfitting could at the same time enable them to better model jagged regions of the SAR-landscape that contain ACs than classical task-agnostic representations.

There is a wide gap between distinct prediction techniques with respect to AC-precision: some models achieve a considerable level of AC-precision such that over $50\%$ of positively predicted MMPs in $\mathfrak{M}_{\text{test}}$ and $\mathfrak{M}_{\text{cores}}$ are indeed actual ACs. Other QSAR models, however, seem to fail almost entirely with respect to this metric on $\mathfrak{M}_{\text{test}}$ and $\mathfrak{M}_{\text{cores}}$ and only deliver modest performance on $\mathfrak{M}_{\text{inter}}$. RFs tend to exhibit the strongest AC-precision and the weakest AC-sensitivity. This might be as a result of their ensemble nature which should intuitively lead to conservative but trustworthy predictions of extreme effects such as ACs.

\subsection{PD-Classification Performance} \label{subsec: qsar_ac_study_pd_pred_perform}

The abilities of the evaluated QSAR models to identify which compound in an MMP is more active is universally weak, with PD-accuracies clustering around $0.7$ on $\mathfrak{M}_{\text{inter}}$ and around $0.6$ on $\mathfrak{M}_{\text{test}}$ and $\mathfrak{M}_{\text{cores}}$, as can be seen in the top rows of~\Cref{fig:pd_results_dopamined2,fig:pd_results_factorxa,fig:pd_results_sarscov2mpro}. Predicting the potency direction for two compounds with similar structures and thus usually similar levels of activity must be considered a challenging task. The combination ECFP-MLP reaches the strongest PD-accuracy in the majority of cases and we recommend starting with this model as a baseline for more advanced PD-classification methods.

One can argue that the activity direction of two similar compounds is of little interest if the true activity difference is small, as is often the case. We therefore also restricted PD-classification to predicted ACs. The three plots in the bottom rows of~\Cref{fig:pd_results_dopamined2,fig:pd_results_factorxa,fig:pd_results_sarscov2mpro} depict the PD-accuracy of each QSAR model on the subset of MMPs that were also predicted to be ACs by the same model. In this practically more relevant scenario, PD-classification accuracy tends to exceed $0.9$ on $\mathfrak{M}_{\text{inter}}$ and $0.8$ on $\mathfrak{M}_{\text{test}}$ and $\mathfrak{M}_{\text{cores}}$. The QSAR models investigated here are thus able to identify the correct activity direction of MMPs if they also predict them to be ACs. The relatively rare instances in which the PD of a predicted AC is misclassified, however, reflect severe QSAR-prediction errors.

\subsection[Linear Relationship between QSAR-MAE and AC-MCC]{Linear Relationship between QSAR-MAE and \\ AC-MCC} \label{subsec: qsar_ac_study_linear_relationship}

Our experiments reveal a consistent linear functional relationship between the QSAR-MAE and the AC-MCC as can be seen in the left columns of~\Cref{fig:ac_results_dopamined2,fig:ac_results_factorxa,fig:ac_results_sarscov2mpro}. A potential mechanism driving this effect could be as follows: As the overall QSAR-MAE of a model improves, its accuracy at predicting activity differences between similar molecules could be expected to improve as well. Previously misclassified MMPs whose predicted absolute activity differences were already close to the critical value $d_{\text{crit}} = 1.5$ might then gradually move to the correct side of the decision boundary and increase the AC-MCC. These results suggest that for real-world QSAR models the AC-MCC and the QSAR-MAE are strongly predictive of each other, i.e.~there appears to be a strong positive association between the general ability of a QSAR model to predict activities of individual compounds and its capability to correctly distinguish between ACs and non-ACs. While this observation only rests on nine models, it is highly consistent across MMP sets and pharmacological targets. Since AC-precision also appears to reliably increase as the QSAR-MAE decreases, one can speculate that as the QSAR-prediction performance of a model gets better, it gradually removes ``false spikes" from its generated SAR-landscape that would otherwise result in the prediction of false AC-positives.

\section{Conclusions} \label{sec: qsar_ac_study_conclusions}

To the best of our knowledge this is the first study to investigate the capabilities of QSAR models to classify the existence and direction of ACs within pairs of similar compounds. It is also the first work to explore the quantitative relationship between QSAR-prediction at the level of individual molecules and AC-prediction at the level of compound-pairs. As part of our methodology we have additionally introduced a simple, interpretable, and rigorous data-splitting technique for pair-based prediction problems. 

When the activities of both MMP compounds are unknown (i.e.~absent from the training set) then common QSAR models exhibit low AC-sensitivity which limits their utility for AC-classification. This strongly supports the hypothesis that QSAR methods do indeed regularly fail to predict ACs which might thus form a major source of prediction errors in QSAR modelling~\citep{golbraikh2014data,cruz-monteagudo_activity_2014, maggiora_outliers_2006, sheridan_experimental_2020}. However, in the practically significantly more relevant scenario where the activity of one MMP compound is known (i.e.~present in the training set) AC-sensitivity increases substantially; for query compounds with known activities, QSAR methods can therefore be used as simple AC-classification, compound-optimisation and SAR-knowledge-acquisition tools. Furthermore, based on the observed PD-classification results, we can expect the predicted direction of predicted ACs to have a high degree of accuracy.

With respect to molecular featurisation, we have found PDVs to be consistently inferior to ECFPs and GINs at both QSAR-prediction and AC-classification. It might be the case that simply too much of the explicit structural information that is relevant for both tasks is lost during the task-agnostic PDV transformation. Moreover, we have found robust evidence that precomputed ECFPs do not only outcompete PDVs but also differentiable GINs at general QSAR-prediction. This adds to a growing awareness that standard message-passing GNNs might need to be improved further to definitively beat classical molecular featurisations based on structural fingerprints such as ECFPs~\citep{stepivsnik2021comprehensive, mayr2018large, jiang2021could, menke2021using, chithrananda2020chemberta, sabando2021using, winter2019learning}. One potential angle to achieve this could be self-supervised GNN pre-training, which has recently shown promising results in the molecular domain~\citep{hu2019strategies, wang2021molclr}. However, while GINs appear to be inferior to ECFPs at QSAR-prediction, they tend to be advantageous for AC-classification; their highly parametric nature might simultaneously lead to increased overfitting but to a better modelling of the more jagged regions of the SAR-landscape. We thus recommend using GINs as an AC-classification baseline since such an agreed-upon benchmark is currently lacking.

Finally, the low AC-sensitivity of the tested QSAR models when the activities of both MMP compounds are unknown suggests that such methods are still lacking essential SAR knowledge. On the flip side, one can speculate that it might be possible to considerably boost the performance of common QSAR models in the future by focussing on the development of techniques to specifically increase their AC-sensitivity.

\newpage $\text{}$ 
\newpage
\chapter[A Twin Neural Network Model for Activity-Cliff Prediction]{A Twin Neural Network Model for Activity-Cliff Prediction} \label{chap: twin_net_ac_pred}

\noindent \textit{We have presented an early version of our twin neural network model for activity-cliff prediction described in this chapter at the 4th RSC-BMCS~/ RSC-CICAG Artificial Intelligence in Chemistry Symposium (2021, virtual) where we were subsequently awarded the prize for the best \href{http://dx.doi.org/10.13140/RG.2.2.18137.60000}{scientific poster}~\citep{dablander2021siamese}. The material in this chapter is built on the ideas outlined in our poster.}

\section{Overview} \label{sec: twin_net_ac_pred_overview}

In Chapter~\ref{chap: qsar_ac_study} we have seen that the utility of standard QSAR models for the detection of ACs is limited if the activites of both compounds are unknown and that ACs do in fact form a major source of prediction error in such cases. However, a method to predict ACs accurately \textit{in silico} would still be of great value for computational drug discovery due to its potential utility for tasks such as compound optimisation and SAR-knowledge acquisition. A natural research question to ask is thus how to design an AC-prediction model that provably outperforms the QSAR-modelling baselines established in the previous chapter.

As mentioned in Section~\ref{sec: qsar_ac_study_intro_acs}, the AC-prediction literature is still thin compared to the QSAR-prediction literature. A generous and thorough literature search for computational AC-prediction models revealed a total of $15$ methods~\citep{heikamp_prediction_2012, tamura_ligand-based_2020, de_la_vega_de_leon_prediction_2014, beck2014quantitative, namasivayam_searching_2012, namasivayam_prediction_2013, husby_structure-based_2015, horvath_prediction_2016, perez-benito_predicting_2019, asawa_prediction_2020, keyvanpour2021pcac, iqbal_prediction_2021, park2022acgcn, chen2022deepac}, all of which were published since $2012$. The core difference between these tailored AC-prediction models is their respective method to extract features from pairs of molecular compounds in a manner suitable for standard machine learning pipelines. Pair-based feature extraction techniques that have been employed for AC-prediction include condensed graphs of reactions~\citep{horvath_prediction_2016}, convolutional neural networks operating on 2D images of compound pairs~\citep{iqbal_prediction_2021}, and newly designed kernel functions for support-vector-machine classification of compound pairs~\citep{heikamp_prediction_2012}. One of the major flaws of the published work on AC-prediction is that none of the most salient AC-prediction models~\citep{heikamp_prediction_2012, de_la_vega_de_leon_prediction_2014, horvath_prediction_2016, tamura_ligand-based_2020, iqbal_prediction_2021} have been compared to a simple QSAR model. It is thus unclear whether these technically complex methods do in fact outcompete the simple QSAR-based AC-classification baselines established in Chapter~\ref{chap: qsar_ac_study}. This casts doubt on their practical utility. Running control experiments that compare the AC-prediction model to a straightforward QSAR-modelling baseline is especially relevant in light of the fact that a variety of published AC-prediction methods~\citep{heikamp_prediction_2012, de_la_vega_de_leon_prediction_2014, iqbal_prediction_2021} are tested on compound-pair-based data splits which incur a large overlap between training and test set at the level of individual compounds. This less-than-rigorous data splitting technique should naturally favour QSAR models for AC-prediction. 

In this chapter, we introduce a novel deep learning model for the classification of the existence and direction ACs in chemical space. Our key idea is to employ a twin architecture~\citep{chicco2021siamese, bromley1993signature, koch2015siamese, taigman2014deepface} that is specifically designed to process dual inputs in a natural way. Twin networks are artificial neural network architectures that contain two (or more) indistinguishable copies of the same subnetwork. All subnetworks are required to share the same architecture and trainable parameters, and these parameters are updated jointly for all subnetworks during training. Twin architectures have been shown to be well-suited to learn similarity and distance metrics for pairs of complex data structures~\citep{chicco2021siamese}; such metric-learning approaches can be used for data-scarce prediction tasks in a process called single-shot learning. Successful areas of application for twin neural networks can for example be found in facial recognition~\citep{taigman2014deepface}, signature verification~\citep{bromley1993signature} and single-shot image recognition~\citep{koch2015siamese}. Comparatively little work has been done, however, to study twin neural networks in the context of computational drug discovery. A small number of studies have investigated the potential of twin architectures for drug-drug interaction prediction~\citep{dhami2019predicting, zhong2019graph, schwarz2020attentionddi}, one-shot drug-discovery~\citep{torres2020exploring, baskin2006neural, alvarez2010qt}, protein-protein interaction prediction~\citep{chen2019multifaceted}, bioactivity prediction~\citep{fernandez2021siamese}, drug-response similarity prediction~\citep{jeon2019resimnet}, compound-structure determination~\citep{roberts2018using}, protein-representation learning~\citep{nourani2021tripletprot}, and the identification of drug-target interactions~\citep{gao2018interpretable}. To the best of our knowledge, our work represents the first application of twin neural networks to the problem of activity-cliff prediction. 

Our proposed twin network is jointly trained on two distinct prediction tasks: a ternary AC-classification task to predict whether an input MMP represents an AC, a half-AC or a non-AC, and a binary PD-classification task to predict which compound in the MMP is the more potent one. Important symmetry properties can be hard-coded into the neural architecture of the twin network as a useful inductive bias for pair-based prediction problems, and we give straightforward mathematical proofs that illustrate these properties. The developed twin network can be seamlessly combined with either modern GNNs or classical molecular featurisations such as ECFPs or PDVs, or indeed with any representation of individual molecules. This removes the need to develop complex feature-engineering procedures for compound-pairs, which appears to have been the main technical hurdle in the development of previous AC-prediction models~\citep{heikamp_prediction_2012, horvath_prediction_2016, tamura_ligand-based_2020, iqbal_prediction_2021}.

We experimentally evaluate an ECFP-based and a GIN-based version of the proposed twin model using the SARS-CoV-2 main protease data set described in Section~\ref{subsec: qsar_ac_study_data_sets} and our novel rigorous data splitting technique for pair-based data developed in Section~\ref{subsec: qsar_ac_study_data_splitting}. We additionally experiment with a transfer learning approach to enrich the input features of each of the two model versions in an attempt to further boost performance. We also run strict control experiments to compare the twin network models to two QSAR models that have shown strong performance in Chapter~\ref{chap: qsar_ac_study} when repurposed for AC-classification, namely ECFP-MLP and GIN-MLP. However, to guarantee comparability with the twin network, this time we will also repurpose the QSAR models for \textit{ternary} rather than binary AC-classification. As stated before, such indispensable control experiments involving QSAR models are lacking in other studies~\citep{heikamp_prediction_2012, horvath_prediction_2016, tamura_ligand-based_2020, iqbal_prediction_2021}. 

We start off this chapter by giving a mathematical description and visual illustration of the proposed twin neural network model for AC and PD-classification, along with the transfer learning technique to improve the information content of its input features. We then present our computational experiments and discuss our empirical observations. Finally, we summarise our findings and draw conclusions for the AC-prediction field.

\section[Twin Neural Network: Mathematical Description]{Twin Neural Network: Mathematical \\ Description} \label{sec: twin_net_ac_pred_model_description}

\subsection{Neural Architecture and Symmetry Properties} \label{subsec: twin_net_ac_pred_model_description_neural_arch}

Let $(\mathcal{R}, \tilde{\mathcal{R}})$ be an ordered pair of two molecular representations forming an MMP. An example for an MMP is depicted in the previous chapter in Figure~\ref{fig:ac_example_factor_Xa_CHEMBL658338}. Just like in Chapter~\ref{chap: qsar_ac_study}, we assume that both compounds are associated with activity labels $\text{act}(\mathcal{R}), \text{act}(\tilde{\mathcal{R}}) \in \mathbb{R}$ which quantify their biological activity with respect to the same predefined pharmacological target. More specifically, we consider the expression $\text{act}(\mathcal{R})$ to be the negative decadic logarithm of the experimentally measured activity value of~$\mathcal{R}$. This is equivalent to its pK\textsubscript{i} or pIC\textsubscript{50}-value (up to a minor additive shift if one chooses to use units other than the canonical [M]-units). From $\text{act}(\mathcal{R})$ and $\text{act}(\tilde{\mathcal{R}})$ we can derive a ternary label
$$\text{AC}(\mathcal{R}, \tilde{\mathcal{R}}) \in \{(1,0,0),(0,1,0), (0,0,1)\} \eqqcolon \{\text{AC}, \text{ half-AC},\text{ non-AC}\} $$
indicating whether $(\mathcal{R}, \tilde{\mathcal{R}})$ is an AC, half-AC or non-AC, as well as a binary label 
$$\text{PD}(\mathcal{R}, \tilde{\mathcal{R}}) \in \{0, 1\} \eqqcolon \{\text{Right}, \text{Left} \} $$
indicating which of both compounds is more active. Note that 
$$\text{AC}(\mathcal{R}, \tilde{\mathcal{R}}) = \text{AC}(\tilde{\mathcal{R}}, \mathcal{R})$$ and 
$$\text{PD}(\mathcal{R}, \tilde{\mathcal{R}}) = 1 - \text{PD}(\tilde{\mathcal{R}}, \mathcal{R}).$$
Our goal is to predict both $\text{AC}(\mathcal{R}, \tilde{\mathcal{R}})$ and $\text{PD}(\mathcal{R}, \tilde{\mathcal{R}})$ from the input MMP~$(\mathcal{R}, \tilde{\mathcal{R}})$ using a twin neural network model. Note that unlike in Chapter~\ref{chap: qsar_ac_study} we now also explicitly consider ``half-ACs" in our AC-classification task. As stated in Section~\ref{subsec: qsar_ac_study_ac_bin_class_prob}, half-ACs are defined as MMPs which exhibit a potency difference between one and two orders of magnitude. Until now, half-ACs have been essentially ignored in the AC-prediction literature~\citep{heikamp_prediction_2012, horvath_prediction_2016, tamura_ligand-based_2020}; however, extending our classification task to also include half-ACs increases the amount of available training data and leads to a more complete and practically relevant (albeit more challenging) \textit{ternary} formulation of the AC-classification problem.

A visual introduction to our developed twin neural network architecture is depicted in Figure~\ref{fig:twin_network_architecture}. This illustration might serve as a useful point of reference to facilitate the theoretical discussions in the rest of this section.
\begin{figure}[h]
	\centering
	\includegraphics[width=1\linewidth]{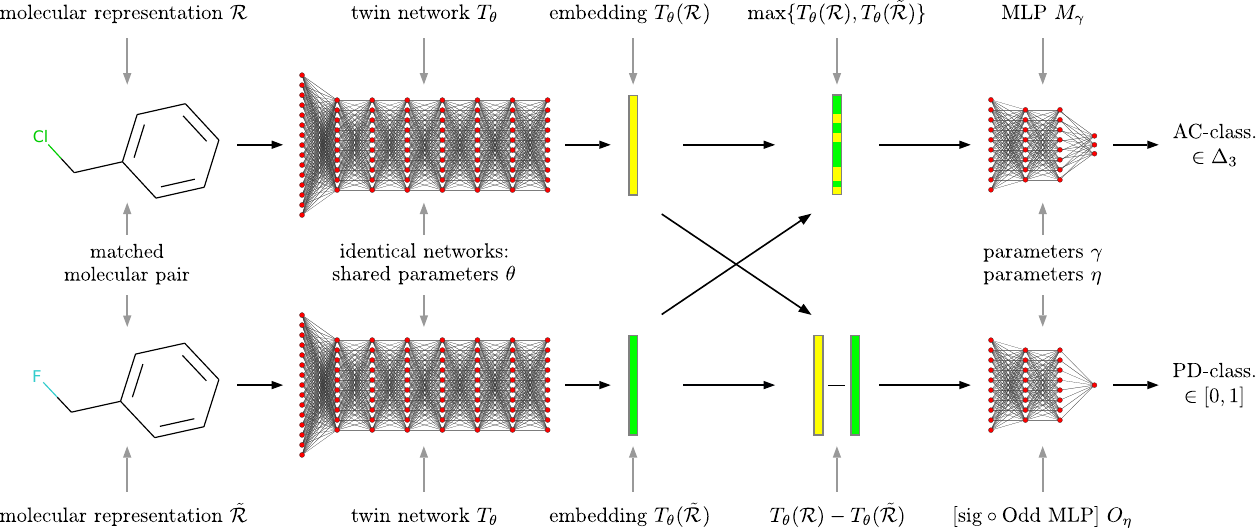}
	\caption[Twin neural network architecture for AC and PD-classification.]{Twin neural network model for activity-cliff~(AC) and potency-direction~(PD) classification. Changing the order of the input compounds leaves the predicted AC-classification label invariant but reverses the predicted PD-classification label.}
	\label{fig:twin_network_architecture}
\end{figure}
From a mathematical point of view, our twin network model can be expressed as the composition of two distinct mappings. The first mapping corresponds to a featurisation step that creates a vectorial embedding for each MMP compound:
\begin{equation} \label{eq: twin_network_feat}
(\mathcal{R}, \tilde{\mathcal{R}}) \quad \mapsto \quad \begin{bmatrix} T_{\theta}(\mathcal{R})\\[0.1cm] T_{\theta}(\tilde{\mathcal{R}}) \end{bmatrix} \in \mathbb{R}^l \times \mathbb{R}^l.\end{equation}
The second mapping corresponds to a classification step that uses both vectorial embeddings to create probabilistic estimates for $\text{AC}(\mathcal{R}, \tilde{\mathcal{R}})$ and $\text{PD}(\mathcal{R}, \tilde{\mathcal{R}})$:
\vspace{7pt}
\begin{equation} \vspace{7pt} \label{eq: twin_network_class}
\begin{bmatrix} T_{\theta}(\mathcal{R})\\[0.1cm] T_{\theta}(\tilde{\mathcal{R}}) \end{bmatrix} \quad \mapsto \quad \begin{bmatrix} M_{\gamma}(\max\{T_{\theta}(\mathcal{R}) , T_{\theta}(\tilde{\mathcal{R}})\}) \\[0.1cm] O_{\eta}(T_{\theta}(\mathcal{R}) - T_{\theta}(\tilde{\mathcal{R}})) \end{bmatrix} \in \Delta^{+}_{3} \times (0, 1).
\end{equation}
The upper component on the right-hand side of Expression~\ref{eq: twin_network_class} corresponds to the predicted AC-classification label,
$$ \text{AC}(\mathcal{R}, \tilde{\mathcal{R}}) \approx M_{\gamma}(\max\{T_{\theta}(\mathcal{R}) , T_{\theta}(\tilde{\mathcal{R}})\}) \eqqcolon \hat{\text{AC}}_{\theta, \gamma}(\mathcal{R}, \tilde{\mathcal{R}}) \in \Delta^{+}_3, $$
and the lower component corresponds to the predicted PD-classification label,
$$ \text{PD}(\mathcal{R}, \tilde{\mathcal{R}}) \approx  O_{\eta}(T_{\theta}(\mathcal{R}) - T_{\theta}(\tilde{\mathcal{R}})) \eqqcolon \hat{\text{PD}}_{\theta, \eta}(\mathcal{R}, \tilde{\mathcal{R}}) \in (0,1). $$

\begin{itemize}
	\item Here $T_{\theta}$ is a trainable deep-learning-based molecular featurisation method and $\theta$ is its associated vector of trainable parameters. $T_{\theta}$ is used to map a given molecular representation $\mathcal{R}$ to a feature vector $T_{\theta}(\mathcal{R}) \in \mathbb{R}^l$ in a differentiable manner. For example, if $\mathcal{R}$ is a molecular graph, then $T_{\theta}$ could take the form of a GIN, and if $\mathcal{R}$ is an ECFP, then $T_{\theta}$ could take the form of a simple MLP. Note that $T_{\theta}$ is applied twice in our model, once to each input compound. $T_{\theta}$ can thus be seen as representing both branches of a twin neural network, as visualised in Figure~\ref{fig:twin_network_architecture} below.
	
	\item The vector $\max\{v_1, v_2\} \in \mathbb{R}^l$ describes their componentwise maximum of two arbitrary vectors $v_1, v_2 \in \mathbb{R}^l$. Note that $\max\{\cdot, \cdot\}$ is thus a permutation-invariant set function, i.e.~$\max\{v_1, v_2\} = \max\{v_2, v_1\}$. In Remark~\ref{remark: choice_of_max} below, we will give some reasons why we specifically chose the max-operator out of all possible permutation-invariant set functions for our model.
	
	\item The symbol $\Delta^{+}_{3} \coloneqq \{(p_1, p_2, p_3) \in \mathbb{R}^3 \ \vert \ p_1, p_2, p_3 > 0 \ \land \ p_1 + p_2 + p_3 = 1 \}$ denotes the set of all positive three-dimensional probability vectors. $\Delta^{+}_{3}$ forms a flat triangular surface in $\mathbb{R}^3$ and is sometimes also referred to as the unit $2$-simplex.
	
	\item $M_{\gamma} : \mathbb{R}^{l} \to \Delta^{+}_{3}$ is an MLP and $\gamma$ is its associated vector of trainable parameters. $M_{\gamma}$ is of the form
	$$M_{\gamma} = \text{softmax} \circ \bar{M}_{\gamma} $$
	whereby $\bar{M}_{\gamma}$ is an MLP with three output neurons in its final layer and
	$$\text{softmax}(x_1, x_2, x_3) \coloneqq \Big( \frac{e^{x_1}}{e^{x_1} + e^{x_2} + e^{x_3}}, \frac{e^{x_2}}{e^{x_1} + e^{x_2} + e^{x_3}}, \frac{e^{x_3}}{e^{x_1} + e^{x_2} + e^{x_3}} \Big) \in \Delta^{+}_3.$$
	$M_{\gamma}$ thus outputs a three-dimensional vector whose components represent the predicted probabilites that the MMP $(\mathcal{R}, \tilde{\mathcal{R}})$ is an AC, half-AC or non-AC respectively.
		
	\item $O_{\eta} : \mathbb{R}^{l} \to (0,1)$ is an MLP and $\eta$ is its associated vector of trainable parameters. 	$O_{\eta}$ outputs the predicted probability that $\text{act}(\mathcal{R}) \geq \text{act}(\tilde{\mathcal{R}})$. We restrict $O_{\eta}$ to be of the functional form
	$$O_{\eta}(v) =  \text{sig} \circ \bar{O}_{\eta}$$
	with
	\begin{equation}  \label{eq: def_odd_mlp}
	\bar{O}_{\eta} \coloneqq W_k \circ \arctan \circ W_{k-1} \circ \ ... \circ \arctan \circ W_1
	\end{equation} 
	Here $W_1,...,W_k$ are trainable weight matrices with no added bias vectors, with $W_1$ containing $l$ columns and $W_k$ containing $1$ row. The expression $\arctan(\cdot)$ is the componentwise applied arctangent map used as a nonlinear activation function. Finally,
	$$\text{sig}(x) \coloneqq \frac{e^x}{e^x +1} \in [0, 1]$$ is the sigmoidal activation function applied to the single output neuron in the final MLP layer of $\bar{O}_{\eta}$. We show below that the above choices make $\bar{O}_{\eta}$ an \textit{odd function}, i.e.~a function with the property that $\bar{O}_{\eta}(-v) = -\bar{O}_{\eta}(v)$. We therefore call $\bar{O}_{\eta}(v)$ an \textit{odd} MLP.

\end{itemize}
	
Note that the twin model consists of three neural network components: $T_{\theta}$, $M_{\gamma}$ and $O_{\eta}$. The feature-extracting network component $T_{\theta}$ is separately applied to each of the two input compounds, leading to a twin architecture which can be thought of as containing two copies of $T_{\theta}$. Both network copies are required to share the same weight parameter vector $\theta$ which means that an update of $\theta$ during training changes both copies of $T_{\theta}$ in the same manner. We refer to neural networks of this type that contain two or more indistinguishable copies of the same subnetwork as \textit{twin neural networks}. Twin neural networks are natural tools in situations where two or more distinct inputs must be processed by a machine learning system simultaneously~\citep{chicco2021siamese, bromley1993signature, koch2015siamese, taigman2014deepface} since they allow for important pair-based symmetry properties to be hard-coded into the model architecture as an inductive bias. We first have a look at a symmetry property relevant for AC-classification.

\begin{proposition}[Order-Invariance of AC-classification] \label{prop: order_inv_ac_twin_net}
	Changing the order of the input compounds from $(\mathcal{R}, \tilde{\mathcal{R}})$ to $(\tilde{\mathcal{R}}, \mathcal{R})$ in the twin network model from Figure~\ref{fig:twin_network_architecture} leaves the predicted AC-classification label unchanged.
\end{proposition}
\begin{proof}
	
	The mathematical definition of the twin network model in Expression~\ref{eq: twin_network_class} specifies that the predicted AC-classification label for a compound pair $(\mathcal{R}, \tilde{\mathcal{R}})$ is given by 
	$$M_{\gamma}(\max\{T_{\theta}(\mathcal{R}) , T_{\theta}(\tilde{\mathcal{R}})\}) \in \Delta^{+}_3.$$
	Since $\max\{\cdot , \cdot\}$ is a permutation-invariant (i.e.~order-independent) set function, one can immediately see that
	$$M_{\gamma}(\max\{T_{\theta}(\mathcal{R}) , T_{\theta}(\tilde{\mathcal{R}})\}) = M_{\gamma}(\max\{T_{\theta}(\tilde{\mathcal{R}}) , T_{\theta}(\mathcal{R})\}) $$
	which proves the claim.
\end{proof}

\begin{remark}[Choice of Max-Operator] \label{remark: choice_of_max}
	The proof of Proposition~\ref{prop: order_inv_ac_twin_net} is based on the fact that $\max\{\cdot , \cdot\}$ is a permutation-invariant set function. Such functions are also commonly used for global graph pooling in modern GNN architectures where an unordered set of node features must eventually be reduced to a single feature vector. Exchanging $\max\{\cdot , \cdot\}$ with another permutation-invariant set function such as a summation, averaging, sorting, absolute-difference or minimum operator would preserve the order-invariance of the AC-classification. When designing our model we thus experimented with a variety of such operators and ultimately converged on the $\max\{\cdot , \cdot\}$-function as it showed the strongest performance in a set of preliminary experiments. While this is an empirical finding, an intuitive (albeit speculative) explanation for this might be that ACs are inherently rare outliers; featurisation methods for MMPs that put an emphasis on extreme values such as the $\max\{\cdot , \cdot\}$-function might therefore be more efficient than simple averaging or summation strategies at catching extreme features relevant for outlier-detection.
\end{remark}
We now turn our attention to a symmetry property built into our model for PD-classification.

\begin{definition}[Odd Function]
	A real function $f : \mathbb{R}^{k} \to \mathbb{R}^{n}$ is called odd if
	$$\forall x \in \mathbb{R}^{k}: \quad f(-x) = -f(x).  $$
\end{definition}
Examples of odd functions include the trigonometric function $\sin(x)$ and the family of monomial functions $x^{2k+1}$ for $k \in \mathbb{N}_{0}$. 

\begin{lemma}[Composition of Odd Functions is Odd] \label{lemma: comp_odd_func}
	Let $f: \mathbb{R}^{k} \to \mathbb{R}^{n}$ and $g : \mathbb{R}^{n} \to \mathbb{R}^{s}$ be two odd functions. Then the composition $h \coloneqq g \circ f : \mathbb{R}^{k} \to \mathbb{R}^{s}$ is also an odd function.
\end{lemma}
\begin{proof}
	Let $x \in \mathbb{R}^k$. Then the oddity of $f$ and $g$ implies that
	$$h(-x) = (g \circ f)(-x) = g(f(-x)) = g(-f(x)) = -g(f(x)) = -(g \circ f)(x) = -h(x) $$
	which shows that $h$ is odd as well.
\end{proof}

\begin{lemma}[Oddity of Odd MLP.] \label{lemma: odd_mlp}
	The multilayer perceptron $\bar{O}_{\eta} : \mathbb{R}^{l} \to \mathbb{R}^{\bar{l}}$ defined in Equation~\ref{eq: def_odd_mlp} via
	$$	\bar{O}_{\eta} \coloneqq W_k \circ \arctan \circ W_{k-1} \circ \ ... \circ \arctan \circ W_1 $$
	is an odd function.
\end{lemma}
\begin{proof}
The neural network 
$\bar{O}_{\eta}$
is a composition of matrix-multiplication functions
$$W_i: \mathbb{R}^{l_{i-1}} \to \mathbb{R}^{l_i}$$ without added bias vectors and the componentwise-applied trigonometric nonlinear activation function
$$\arctan: \mathbb{R} \to \mathbb{R}.$$ Since Lemma~\ref{lemma: comp_odd_func} establishes that the composition of odd functions is again odd, it is sufficient to show that both $W_i$ and $\arctan$ are odd functions.

Since matrix-multiplication is a linear function, it holds in particular that 
$$\forall v \in \mathbb{R}^{l_{i-1}} \quad \forall \lambda \in \mathbb{R}: \quad W_i(\lambda v) = \lambda W_i (v). $$
Setting $\lambda \coloneqq -1$ in the above equation immediately shows that $W_i$ is odd. The oddity of $W_i$ is thus a trivial consequence of its linearity.

To show that $\arctan$ is odd, we first observe that its inverse function $$\tan :  (-\pi/2, \pi/2) \to \mathbb{R}$$ 
is odd since
$$\forall y \in (-\pi/2, \pi/2): \quad \tan(-y) = \frac{\sin(-y)}{\cos(-y)} = \frac{-\sin(y)}{\cos(y)} = - \tan(y). $$
This proof relies on the well-known facts that $\sin(-y) = -\sin(y)$ and $\cos(-y) = \cos(y)$. Now let $x \in \mathbb{R}$ be an arbitrary real number. Since $\tan$ is a bijection, there exists exactly one real number $y_x \in (-\pi/2, \pi/2)$ such that $ x = \tan(y_x)$ and thus $\arctan(x) = y_x$. We can now exploit the oddity of $\tan$ to show that
$$\arctan(-x) = \arctan(- \tan(y_x)) = \arctan(\tan(-y_x)) = -y_x = - \arctan(x) $$
which proves the oddity of $\arctan$.

\end{proof}

\begin{proposition}[Order-Equivariance of PD-classification] \label{prop: order_eq_pd_twin_net}
	Changing the order of the input compounds from $(\mathcal{R}, \tilde{\mathcal{R}})$ to $(\tilde{\mathcal{R}}, \mathcal{R})$ in the twin network model from Figure~\ref{fig:twin_network_architecture} flips the predicted PD-classification label, i.e.~transforms it according to the map $$(0,1) \ni p \mapsto 1- p \in (0,1).$$
\end{proposition}
\begin{proof}
The mathematical definition of the twin network model in Expression~\ref{eq: twin_network_class} specifies that the predicted PD-classification label for a compound pair $(\mathcal{R}, \tilde{\mathcal{R}})$ is given by 
$$O_{\eta}(T_{\theta}(\mathcal{R}) - T_{\theta}(\tilde{\mathcal{R}})) = \text{sig}(\bar{O}_{\eta}(T_{\theta}(\mathcal{R}) - T_{\theta}(\tilde{\mathcal{R}})))  \in (0,1).$$
We start by observing that the Sigmoid activation function fulfills the following functional equation:
$$\text{sig}(x) =  
\frac{e^x}{e^x +1} = 
1- \Big(1-\frac{e^x}{e^x +1}\Big) = 
1- \Big(\frac{1}{e^x +1}\Big) = 
1- \Big(\frac{e^{-x}}{1 + e^{-x}}\Big) = 
1- \text{sig}(-x).$$ 
Based on this relationship and the oddity of $\bar{O}_{\eta}$ which we showed in Lemma~\ref{lemma: odd_mlp}, we can write

\begin{align*} \text{sig}(\bar{O}_{\eta}(T_{\theta}(\mathcal{R}) - T_{\theta}(\tilde{\mathcal{R}}))) &= 1 - \text{sig}(-\bar{O}_{\eta}(T_{\theta}(\mathcal{R}) - T_{\theta}(\tilde{\mathcal{R}}))) \\
&= 1 - \text{sig}(\bar{O}_{\eta}(-(T_{\theta}(\mathcal{R}) - T_{\theta}(\tilde{\mathcal{R}})))) \\
&= 1 - \text{sig}(\bar{O}_{\eta}(T_{\theta}(\tilde{\mathcal{R}}) - T_{\theta}(\mathcal{R})))
\end{align*}
which proves the claim.
\end{proof}

Propositions~\ref{prop: order_inv_ac_twin_net}~and~\ref{prop: order_eq_pd_twin_net} assure that the twin neural network respects the natural symmetry properties of AC and PD-classification: if we change the order of the input compounds, then the predicted AC-classification label does not change while the predicted PD-classification label flips in the expected manner. Since these properties are hard-coded into the neural architecture of our model, they do not need to be inferred statistically during training.

\subsection{Loss Function and Model Training} \label{subsec: twin_net_ac_pred_model_description_training}

Let
$$\Delta_{n} \coloneqq \{(p_1, ..., p_n) \in \mathbb{R}^n \ \vert \ p_1, ..., p_n \geq 0 \ \land \ p_1 + ... + p_n = 1 \} $$
denote the set of $n$-dimensional probability vectors and let 
$$\Delta^{+}_{n} \coloneqq \{(p_1, ..., p_n) \in \mathbb{R}^n \ \vert \ p_1, ..., p_n > 0 \ \land \ p_1 + ... + p_n = 1 \}.$$ denote the set of \textit{positive} $n$-dimensional probability vectors. Then the cross-entropy between two probability vectors is defined as
$$H :  \Delta_{n} \times \Delta^{+}_{n} \to \mathbb{R}, \quad H(p, q) = -\sum_{k=1}^{n} p_i \log(q_i).$$
For a fixed $p_0 \in \Delta_n$, the map
$$q \mapsto H(p_0, q)$$ is minimised if $q \to p_0$; this is one of the reasons why $H$ is canonically used as a loss function for machine-learning-based classification problems. In the binary case with $n = 2$, the definition of $H$ can be simplified to
$$H_{\text{bin}} : [0,1] \times (0,1) \to \mathbb{R}, \quad H_{\text{bin}}(p,q) = -p\log(q) - (1-p)\log(1-q). $$
Our twin neural network from Figure~\ref{fig:twin_network_architecture} is trained via the following loss function:
\begin{gather*}
\mathcal{L}_{(\mathcal{R}, \tilde{\mathcal{R}})}(\theta, \gamma, \eta) \coloneqq \\
w_{\text{AC}}(\text{AC}(\mathcal{R}, \tilde{\mathcal{R}})) H( \text{AC}(\mathcal{R}, \tilde{\mathcal{R}}), \hat{\text{AC}}_{\theta, \gamma}(\mathcal{R}, \tilde{\mathcal{R}})) \ + \\ w_{\text{PD}} H_{\text{bin}}(\text{PD}(\mathcal{R}, \tilde{\mathcal{R}}), \hat{\text{PD}}_{\theta, \eta}(\mathcal{R}, \tilde{\mathcal{R}})).
\end{gather*}
Here we used a variety of abbreviations:
\begin{itemize}
	\item $\text{AC}(\mathcal{R}, \tilde{\mathcal{R}}) \in \{(1,0,0),(0,1,0), (0,0,1)\}$ is the AC-classification label.
	
	\item $\text{PD}(\mathcal{R}, \tilde{\mathcal{R}}) \in \{0,1\}$ is the PD-classification label.
	
	\item $\hat{\text{AC}}_{\theta, \gamma}(\mathcal{R}, \tilde{\mathcal{R}}) \coloneqq M_{\gamma}(\max\{T_{\theta}(\mathcal{R}) , T_{\theta}(\tilde{\mathcal{R}})\})  \in \Delta^{+}_{3}$ is the predicted AC-classification label.
	
	\item $\hat{\text{PD}}_{\theta, \eta}(\mathcal{R}, \tilde{\mathcal{R}}) \coloneqq O_{\eta}(T_{\theta}(\mathcal{R}) - T_{\theta}(\tilde{\mathcal{R}})) \in (0,1)$ is the predicted PD-classification label.
	
\end{itemize}
The function
$$w_{\text{AC}}: \{(1,0,0),(0,1,0), (0,0,1)\} \to \left[0, \infty\right)$$
is used to place distinct weights on ACs, half-ACs and non-ACs during training. It plays an important role in counteracting the class imbalance in the naturally highly imbalanced task of AC-classification and is an essential part of our loss function. We choose the values of $w_{\text{AC}}$ in proportion to the relative frequencies of ACs, half-ACs and non-ACs in the training set to in total give equal weight to each class. This means that if $$n_{\text{AC}},n_{\text{half-AC}},n_{\text{non-AC}} \in \mathbb{N}$$ are the respective numbers of MMPs in the training set that are ACs, half-ACs and non-ACs, then we choose
$$w_{\text{AC}}(1,0,0) = \frac{n_{\text{non-AC}}}{n_{\text{AC}}}, \quad w_{\text{AC}}(0,1,0) = \frac{n_{\text{non-AC}}}{n_{\text{half-AC}}}, \quad w_{\text{AC}}(0,0,1) = 1. $$
The constant $w_{\text{PD}} \in \left[0, \infty\right)$ is chosen in relation to the function values of $w_{\text{AC}}$ to guarantee that the AC-classification and the (well-balanced) PD-classification task receive on average an equal amount of weight during training. We thus choose $w_{\text{PD}}$ to be the expected weight of a training-MMP with respect to AC-classification:
$$w_{\text{PD}} \coloneqq \frac{n_{\text{AC}}w_{\text{AC}}(1,0,0) + n_{\text{half-AC}}w_{\text{AC}}(0,1,0) + n_{\text{non-AC}}w_{\text{AC}}(0,0,1)}{n_{\text{MMP}}} .$$
Here $n_{\text{MMP}} = n_{\text{AC}} + n_{\text{half-AC}} + n_{\text{non-AC}} \in \mathbb{N}$ denotes the total number of MMPs in the training set. Our choice of $w_{\text{PD}}$ guarantees that a randomly chosen training-MMP will on average (i.e.~in expectation) receive equal weights for AC and PD-classification.

Finally, we show how the symmetry properties of the twin neural network translate into a symmetry property of the loss function $\mathcal{L}_{(\mathcal{R}, \tilde{\mathcal{R}})}$ that is highly useful during model training.
\begin{proposition}[Order-Invariance of Loss Function] \label{prop: loss_function_order_inv}
	Let $(\mathcal{R}, \tilde{\mathcal{R}})$ be an MMP and $(\tilde{\mathcal{R}}, \mathcal{R})$ be the same MMP in reversed order. Then it holds for all neural network training parameter configurations $\theta, \gamma, \eta$ that \normalfont
	$$\mathcal{L}_{(\mathcal{R}, \tilde{\mathcal{R}})}(\theta, \gamma, \eta) = \mathcal{L}_{(\tilde{\mathcal{R}}, \mathcal{R})}(\theta, \gamma, \eta).$$
\end{proposition}
\begin{proof}
Using the symmetry properties of the twin network with respect to AC and PD-classification shown in Propositions~\ref{prop: order_inv_ac_twin_net} and~\ref{prop: order_eq_pd_twin_net} along with the identity
\begin{align*}
H_{\text{bin}}(p,q) &= -p\log(q) - (1-p)\log(1-q) \\ &= - (1-p)\log(1-q) - (1-(1-p))\log(1-(1-q)) = H_{\text{bin}}(1 - p,1 - q),
\end{align*}
we can calculate
\begin{align*}
&\mathcal{L}_{(\mathcal{R}, \tilde{\mathcal{R}})}(\theta, \gamma, \eta) = \\
& \text{} \\
&w_{\text{AC}}(\text{AC}(\mathcal{R}, \tilde{\mathcal{R}})) H( \text{AC}(\mathcal{R}, \tilde{\mathcal{R}}), \hat{\text{AC}}_{\theta, \gamma}(\mathcal{R}, \tilde{\mathcal{R}})) \ + \\ 
&w_{\text{PD}} H_{\text{bin}}(\text{PD}(\mathcal{R}, \tilde{\mathcal{R}}), \hat{\text{PD}}_{\theta, \eta}(\mathcal{R}, \tilde{\mathcal{R}})) = \\
& \text{} \\
&w_{\text{AC}}(\text{AC}(\tilde{\mathcal{R}}, \mathcal{R})) H( \text{AC}(\tilde{\mathcal{R}}, \mathcal{R}), \hat{\text{AC}}_{\theta, \gamma}(\tilde{\mathcal{R}}, \mathcal{R})) \ + \\ 
&w_{\text{PD}} H_{\text{bin}}(1 - \text{PD}(\tilde{\mathcal{R}}, \mathcal{R}), 1- \hat{\text{PD}}_{\theta, \eta}(\tilde{\mathcal{R}}, \mathcal{R})) = \\
& \text{} \\
&w_{\text{AC}}(\text{AC}(\tilde{\mathcal{R}}, \mathcal{R})) H( \text{AC}(\tilde{\mathcal{R}}, \mathcal{R}), \hat{\text{AC}}_{\theta, \gamma}(\tilde{\mathcal{R}}, \mathcal{R})) \ + \\ 
&w_{\text{PD}} H_{\text{bin}}(\text{PD}(\tilde{\mathcal{R}}, \mathcal{R}), \hat{\text{PD}}_{\theta, \eta}(\tilde{\mathcal{R}}, \mathcal{R})) = \\
& \text{} \\
&\mathcal{L}_{(\tilde{\mathcal{R}}, \mathcal{R})}(\theta, \gamma, \eta)
\end{align*}
which completes the proof.
\end{proof}
Proposition~\ref{prop: loss_function_order_inv} guarantees that $\mathcal{L}_{(\mathcal{R}, \tilde{\mathcal{R}})}$ is symmetric with respect to the order of the compounds in the input MMP $(\mathcal{R}, \tilde{\mathcal{R}})$. Since the gradients of $\mathcal{L}_{(\mathcal{R}, \tilde{\mathcal{R}})}$ with respect to $\theta, \gamma$ and $\eta$ automatically inherit this symmetry, they too remain unchanged if we flip the order of the input compounds. This property has an important practical consequence: it allows one to only train on one randomly chosen ordering of an input MMP instead of both possible orderings without loss of information.

\subsection{Molecular Featurisations: Four Model Versions} \label{subsec: twin_net_ac_pred_model_description_input_features}

The exact architecture of $T_{\theta}$ in the twin network depicted in Figure~\ref{fig:twin_network_architecture} hinges upon the molecular representation technique used for the input compounds $(\mathcal{R}, \tilde{\mathcal{R}})$ and the subsequent molecular featurisation method applied to the input representations. We imagine that initially we are given a training space of the form
$$(\mathfrak{D}_{\text{train}}, \mathfrak{M}_{\text{train}})$$
where
$$\mathfrak{D}_{\text{train}} = \{\mathcal{R}_1, \mathcal{R}_2,...\}$$
represents a data set of individual molecules represented via SMILES strings 
$$\{\mathcal{S}_1, \mathcal{S}_2,...\}$$
and
$$\mathfrak{M}^{\text{double}}_{\text{train}} = \{ (\mathcal{R}, \tilde{\mathcal{R}}) \ \vert \ (\mathcal{R}, \tilde{\mathcal{R}}) \text{ is MMP and} \ \mathcal{R}, \tilde{\mathcal{R}} \in \mathfrak{D}_{\text{train}} \}$$
is the set of ordered MMPs that are fully contained in $\mathfrak{D}_{\text{train}}$. Note that by construction $\mathfrak{M}^{\text{double}}_{\text{train}}$ contains both orientations of each MMP. However, Proposition~\ref{prop: loss_function_order_inv} guarantees that for our training purposes it is sufficient to only contain one arbitrarily chosen ordering of each MMP. We thus define $\mathfrak{M}_{\text{train}}$ as a proper subset of $\mathfrak{M}^{\text{double}}_{\text{train}}$ of exactly half the size that only contains one arbitrarily chosen ordering of each MMP, i.e.~we demand that if $(\mathcal{R}, \tilde{\mathcal{R}}) \in \mathfrak{M}_{\text{train}}$ then $(\tilde{\mathcal{R}}, \mathcal{R}) \notin \mathfrak{M}_{\text{train}}$. 

The twin network is always trained only on $\mathfrak{M}_{\text{train}}$, ignoring compunds in $\mathfrak{D}_{\text{train}}$ that are not involved in MMPs, but we developed a transfer learning approach that enables the twin network to nevertheless implicitly exploit extra information encapsulated in~$\mathfrak{D}_{\text{train}}$. More specifically, we experimented with four different MMP representations which subsequently led to four distinct version of our twin model.

\begin{itemize}
	\item \textbf{MMP representation 1: ECFPs.} Each individual SMILES string in an MMP $(\mathcal{R}, \tilde{\mathcal{R}}) \in \mathfrak{M}_{\text{train}}$ is transformed into a $1024$-bit ECFP$4$~\ref{sec: ecfps} with active tetrahedral R-S chirality flags using \texttt{RDKit}~\citep{landrum2006rdkit}. The featuriser $T_{\theta}$ takes the form of a deep MLP with input dimension $1024$.
	
	\item \textbf{MMP representation 2: GINs.} Each individual SMILES string in an MMP $(\mathcal{R}, \tilde{\mathcal{R}}) \in \mathfrak{M}_{\text{train}}$ is transformed into a molecular graph using the atom and bond features specified in Table~\ref{tab: atom_bond_features}. The featuriser $T_{\theta}$ takes the form of a GIN-MLP model~\ref{subsec: gins_description} with GIN-radius $R = 2$ and GIN-fingerprint-length $l = 128$. The GIN part uses global max pooling in its final graph layer to produce neural fingerprints that feed into the MLP part.
	
	\item \textbf{MMP representation 3~(supervised): ECFP-NFPs.} Each SMILES string in $\mathfrak{D}_{\text{train}}$ is transformed into a $1024$-bit ECFP$4$~\ref{sec: ecfps} with active tetrahedral R-S chirality flags using \texttt{RDKit}~\citep{landrum2006rdkit}. Then an ECFP-MLP model $Q$ with hidden width $1024$ is trained on $\mathfrak{D}_{\text{train}}$ as a supervised QSAR model to predict the activities of individual compounds. After training, the final layer of $Q$ that maps vectors from a $1024$-dimensional learned feature space onto scalar activity predictions is removed to obtain a feature extractor $Q_{\text{feat}}$. 	We refer to the $1024$-dimensional feature vectors generated by $Q_{\text{feat}}$ for individual compounds as ECFP-neural-fingerprints~(ECFP-NFPs). $Q_{\text{feat}}$ is finally used to map MMPs $(\mathcal{R}, \tilde{\mathcal{R}}) \in \mathfrak{M}_{\text{train}}$ to pairs of ECFP-NFPs on which the twin network is subsequently trained. The featuriser $T_\theta$ then takes the form of a deep MLP with input dimension $1024$.
	
	\item \textbf{MMP representation 4~(supervised): GIN-NFPs.} Each SMILES string in $\mathfrak{D}_{\text{train}}$ is transformed into a molecular graph using the atom and bond features specified in Table~\ref{tab: atom_bond_features}. Then a GIN-MLP model $Q$ with GIN-radius $R = 2$, GIN-fingerprint length $l = 128$, and hidden MLP width $256$ is trained on $\mathfrak{D}_{\text{train}}$ as a supervised QSAR model to predict the activities of individual compounds. After training, the final layer of $Q$ that maps vectors from a $256$-dimensional learned feature space onto scalar activity predictions is removed to obtain a feature extractor $Q_{\text{feat}}$. 	We refer to the $256$-dimensional feature vectors generated by $Q_{\text{feat}}$ for individual compounds as GIN-neural-fingerprints~(GIN-NFPs). $Q_{\text{feat}}$ is finally used to map MMPs $(\mathcal{R}, \tilde{\mathcal{R}}) \in \mathfrak{M}_{\text{train}}$ to pairs of GIN-NFPs on which the twin network is subsequently trained. The featuriser $T_\theta$ then takes the form of a deep MLP with input dimension $256$.
	
\end{itemize}
Consider the set of individual training compounds involved in MMPs:
$$\mathcal{D}^{\text{MMP}}_{\text{train}} \coloneqq \{ \mathcal{R} \in \mathfrak{D}_{\text{train}} \ \vert \ \exists \ \tilde{\mathcal{R}} \in  \mathfrak{D}_{\text{train}} : (\tilde{\mathcal{R}}, \mathcal{R}) \in \mathfrak{M}_{\text{train}} \ \text{or} \ (\mathcal{R}, \tilde{\mathcal{R}}) \in \mathfrak{M}_{\text{train}}	\}. $$
All four introduced MMP representations use the training signal encapsulated in in~$\mathcal{D}^{\text{MMP}}_{\text{train}}$. However, the transfer-learning-based representations ECFP-NFP and GIN-NFP go further: they also allow us to implicitly leverage the information contained in the set of isolated compounds $\mathfrak{D}_{\text{train}} \setminus \mathcal{D}^{\text{MMP}}_{\text{train}}$
since the feature extractor $Q_{\text{feat}}$ must have encountered these compounds during its own preliminary training process on the full training space $\mathfrak{D}_{\text{train}}$. Knowledge about the isolated compounds in $\mathfrak{D}_{\text{train}} \setminus \mathcal{D}^{\text{MMP}}_{\text{train}}$ and their experimentally measured activity labels is thus implicitly encoded in the trained parameters of $Q_{\text{feat}}$ and transferred to the features it extracts.

\section{Computational Experiments} \label{sec: twin_net_ac_pred_num_exp}

In this section, we present a series of computational experiments to investigate the AC and PD-classification capabilities of the four versions of our twin neural network model discussed in Section~\ref{subsec: twin_net_ac_pred_model_description_input_features}. We further compare the twin models to the two strongest QSAR-modelling baselines for AC-classification found in our previous study in Chapter~\ref{chap: qsar_ac_study}: the combinations ECFP-MLP and GIN-MLP.

\subsection{Experimental Methodology} \label{subsec: twin_net_ac_pred_num_exp_meth}

\subsubsection{Molecular Data Set}

For our experiments we employed the SARS-CoV-2 main protease data set introduced in Section~\ref{subsec: qsar_ac_study_data_sets} since it is composed of a single high-quality assay and has a high density of MMPs. Note that this is the exact same data set that we used for our computational study on QSAR models for AC-prediction in Chapter~\ref{chap: qsar_ac_study}. The protein structure of SARS-CoV-2 main protease is visualised in Figure~\ref{fig:mpro}. 

The data was obtained and cleaned in the manner described in Section~\ref{subsec: qsar_ac_study_data_sets} and takes the form of SMILES strings with associated IC\textsubscript{50}~[µM] values. An overview of the numbers of compounds, MMPs, ACs, half-ACs and non-ACs in the data set is given in Table~\ref{tab: qsar_ac_study_datsets_overview}. SARS-CoV-2 main protease is one of the key enzymes in the viral replication cycle of the SARS coronavirus 2 which recently led to the global COVID-19 pandemic. It is a promising target for antiviral drugs against this coronavirus~\citep{ullrich2020sars}.

\subsubsection{Data Splitting Technique and Prediction Tasks}

For data splitting into training and test sets, we employed the novel technique for pair-based data that we developed for our previous computational study in Section~\ref{subsec: qsar_ac_study_data_splitting}. It is visualised in Figure~\ref{fig:ac_pred_data_splitting_strategy} and delivers data splits of the form
$$\mathfrak{S}^{i,j} = (\mathfrak{D}^{i,j}_{\text{train}}, \mathfrak{D}^{i,j}_{\text{test}}, \mathfrak{M}^{i,j}_{\text{train}},\mathfrak{M}^{i,j}_{\text{test}}, \mathfrak{M}^{i,j}_{\text{inter}}, \mathfrak{M}^{\text{i,j}}_{\text{cores}}) $$
for $i \in \{1,...,m\}$ and $j \in \{1,...,k\}$. Here the pair $(\mathfrak{D}^{i,j}_{\text{train}}, \mathfrak{D}^{i,j}_{\text{test}})$ represents the $j$-th random split with the $i$-th random seed of the underlying SARS-CoV-2 main protease data set $\mathfrak{D}$ in a $k$-fold cross validation scheme repeated with $m$ random seeds. The MMP sets $\mathfrak{M}^{i,j}_{\text{train}}, \mathfrak{M}^{i,j}_{\text{test}}, \mathfrak{M}^{i,j}_{\text{inter}}, \mathfrak{M}^{\text{i,j}}_{\text{cores}}$ differ via the relationship of their associated MMPs with the individual compounds in $\mathfrak{D}^{i,j}_{\text{train}}$ and $\mathfrak{D}^{i,j}_{\text{test}}$. These sets are rigorously defined in Section~\ref{subsec: qsar_ac_study_data_splitting} and visualised in Figure~\ref{fig:ac_pred_data_splitting_strategy}. We will shortly repeat their definitions here in intuitive terms: 

\begin{itemize}
	\item $\mathfrak{M}^{i,j}_{\text{train}}$ contains MMPs that are fully included in $\mathfrak{D}^{i,j}_{\text{train}}$.
	
	\item $\mathfrak{M}^{i,j}_{\text{inter}}$ contains MMPs with exactly one compound in $\mathfrak{D}^{i,j}_{\text{train}}$ and the other compound in $\mathfrak{D}^{i,j}_{\text{test}}$. It simulates a compound-optimisation scenario where one is searching for small modifications of known compounds that would give rise to ACs or half-ACs.
	
	\item $\mathfrak{M}^{i,j}_{\text{test}}$ contains MMPs that are fully included in $\mathfrak{D}^{i,j}_{\text{test}}$. It models a setting where one tries to discover novel ACs and half-ACs in the same area of chemical space that $\mathfrak{D}^{i,j}_{\text{train}}$ was sampled from.
	
	\item Finally, $\mathfrak{M}^{i,j}_{\text{cores}}$ is a subset of $\mathfrak{M}^{i,j}_{\text{test}}$ consisting of MMPs in $\mathfrak{D}^{i,j}_{\text{test}}$ that do not share structural cores with MMPs in $\mathfrak{M}^{i,j}_{\text{inter}}$ or $\mathfrak{M}^{i,j}_{\text{train}}$. It corresponds to the difficult task of predicting ACs and half-ACs within structurally novel MMPs that do not contain near analogs to MMP compounds involved in the training set.
\end{itemize}
As mentioned above, Proposition~\ref{prop: loss_function_order_inv} assures us that it is possible without loss of generality to always only consider one randomly chosen compound-ordering for each MMP in all MMP sets. The overall AC and PD-classification performance of each model is recorded via the average over $mk$ training and test runs for all data splits $\mathfrak{S}^{1,1}, ..., \mathfrak{S}^{m,k}$. For our experiments we set $(m,k) = (50,2)$. The number of $50*2 = 100$ repetitions is substantial and considerably large for deep-learning experiments due to their associated computational cost. This choice led to a runtime in the order of approximately $10$-$20$ hours per model; however, it significantly reduced the effects of stochastic fluctations on our results and led to increased experimental quality and reliability.

\subsubsection{Prediction Tasks and Prediction Strategies} \label{subsubsec: twin_nets_pred_strat}

As specified in Section~\ref{sec: twin_net_ac_pred_model_description}, each MMP $$(\mathcal{R}, \tilde{\mathcal{R}}) \in \mathfrak{M}^{i,j}_{\text{train}} \cup \mathfrak{M}^{i,j}_{\text{test}} \cup \mathfrak{M}^{i,j}_{\text{inter}} \cup \mathfrak{M}^{i,j}_{\text{cores}} $$
comes with a ternary label
$$\text{AC}(\mathcal{R}, \tilde{\mathcal{R}}) \in \{(1,0,0),(0,1,0), (0,0,1)\} \eqqcolon \{\text{AC}, \text{ half-AC},\text{ non-AC}\} $$
indicating whether $(\mathcal{R}, \tilde{\mathcal{R}})$ is an AC, half-AC or non-AC; and a binary label 
$$\text{PD}(\mathcal{R}, \tilde{\mathcal{R}}) \in \{0, 1\} \eqqcolon \{ \text{Right}, \text{Left} \} $$
indicating which of both compounds is more active. The goal of each model is to predict both $\text{AC}(\mathcal{R}, \tilde{\mathcal{R}})$ and $\text{PD}(\mathcal{R}, \tilde{\mathcal{R}})$ from the input MMP representation $(\mathcal{R}, \tilde{\mathcal{R}})$. The four twin neural network models are designed to output explicit probabilistic estimates of $\text{AC}(\mathcal{R}, \tilde{\mathcal{R}})$ in $\Delta_3$ and of $\text{PD}(\mathcal{R}, \tilde{\mathcal{R}})$ in $(0,1)$, and can therefore be directly used for AC and PD-classification. The two QSAR-modelling baselines ECFP-MLP and GIN-MLP, however, can only directly output estimates of the activity labels of individual molecules. Thus, when dealing with a QSAR model $Q$, we thresholded the predicted MMP activity-difference,
$$Q(\mathcal{R}) - Q(\tilde{\mathcal{R}})  \approx  \text{act}(\mathcal{R}) - \text{act}(\tilde{\mathcal{R}}),$$
to construct discrete predictions for $\text{AC}(\mathcal{R}, \tilde{\mathcal{R}})$ and $\text{PD}(\mathcal{R}, \tilde{\mathcal{R}})$. This was done in the same straightforward manner as described in Section~\ref{subsec: qsar_ac_study_pred_strategies} for our previous computational study, with the only exception that we extended our AC-classification stategy to now also include half-ACs. If $(\mathcal{R}, \tilde{\mathcal{R}}) \in \mathfrak{M}^{i,j}_{\text{train}} \cup \mathfrak{M}^{i,j}_{\text{test}} \cup \mathfrak{M}^{i,j}_{\text{cores}} $ then we performed AC-classification via
$$ (\mathcal{R}, \tilde{\mathcal{R}}) \mapsto
\begin{cases}
(0,0,1) \quad \text{if} \ \vert Q(\mathcal{R}) - Q(\tilde{\mathcal{R}}) \vert \leq 1, \\
(0,1,0) \quad \text{if} \ \vert Q(\mathcal{R}) - Q(\tilde{\mathcal{R}}) \vert \in (1,2), \\
(1,0,0) \quad \text{if} \ \vert Q(\mathcal{R}) - Q(\tilde{\mathcal{R}}) \vert \geq 2 \\
\end{cases} $$
and PD-classification via
$$ (\mathcal{R}, \tilde{\mathcal{R}}) \mapsto
\begin{cases}
1 \quad \text{if} \ Q(\mathcal{R}) \geq Q(\tilde{\mathcal{R}}), \\
0 \quad \text{else}.
\end{cases} $$
Similarly, if $(\mathcal{R}, \tilde{\mathcal{R}}) \in \mathfrak{M}^{i,j}_{\text{inter}}$ and we were \textit{a priori} given (say) the experimental activity $\text{act}(\tilde{\mathcal{R}})$, then we performed AC-classification via
$$ (\mathcal{R}, \tilde{\mathcal{R}}) \mapsto
\begin{cases}
(0,0,1) \quad \text{if} \ \vert Q(\mathcal{R}) - \text{act}(\tilde{\mathcal{R}}) \vert \leq 1, \\
(0,1,0) \quad \text{if} \ \vert Q(\mathcal{R}) - \text{act}(\tilde{\mathcal{R}}) \vert \in (1,2), \\
(1,0,0) \quad \text{if} \ \vert Q(\mathcal{R}) - \text{act}(\tilde{\mathcal{R}}) \vert \geq 2 \\
\end{cases} $$
and PD-classification via
$$ (\mathcal{R}, \tilde{\mathcal{R}}) \mapsto
\begin{cases}
1 \quad \text{if} \ Q(\mathcal{R}) \geq \text{act}(\tilde{\mathcal{R}}), \\
0 \quad \text{else}.
\end{cases}.$$

\subsubsection{Performance Measures} \label{subsubsec: twin_networks_performance_measures}

For the balanced PD-classification problem we employ the standard accuracy as a suitable performance measure:
$$ \frac{\text{number of correct predictions}}{\text{number of predictions}} \in [0,1] .$$
For each MMP set and each model, we measured PD-classification accuracy (1) on the whole MMP set, (2) on the subset of MMPs predicted by the model to be half-ACs, and (3) on the subset of MMPs predicted by the model to be ACs.

The ternary AC-classification task is naturally highly imbalanced and it is therefore necessary to choose a more nuanced set of performance measures in order to paint an adequate and detailed picture of model performance. For each class 
$$C \in \{\text{AC}, \text{ half-AC},\text{ non-AC}\},$$
let $n_C$ be the number of MMPs in class $C$, $p_C$ be the number of MMPs predicted to be in class $C$, and $p^{\text{true}}_C$ be the number of MMPs correctly predicted to be in class $C$. In our experiments, we recorded the \textit{sensitivity}
$$\frac{p^{\text{true}}_C}{n_C} \in [0,1] $$
and the \textit{precision}
$$\frac{p^{\text{true}}_C}{p_C} \in [0,1]  $$
of each model for each class $C$. There is an implicit trade-off between sensitivity and precision; improvement in one of both metrics usually leads to deterioration of the other. A model that shows both higher precision and sensitivity than another model for a class $C$ can be seen as a better classifier for this particular class. 

Finally, as a simple overall performance measure for ternary AC-classification we employed the multi-class version of the Matthews correlation coefficient~(MCC). Let $n_{\text{MMP}} \coloneqq n_{\text{AC}} + n_{\text{half-AC}} + n_{\text{non-AC}}$ be the total number of MMPs and $p^{\text{true}} = p^{\text{true}}_{\text{AC}} +  p^{\text{true}}_{\text{half-AC}} + p^{\text{true}}_{\text{non-AC}}$ be the total number of correct predictions. Then the multi-class MCC is given by
$$\frac{n_{\text{MMP}}p^{\text{true}} - \sum\limits_{C \in \{\text{AC}, \text{ half-AC},\text{ non-AC}\}}  n_C p_C   }{\sqrt{\Big(n_{\text{MMP}}^{2} - \sum\limits_{C \in \{\text{AC}, \text{ half-AC},\text{ non-AC}\}} p_C^2 \Big) \Big(n_{\text{MMP}}^{2} - \sum\limits_{C \in \{\text{AC}, \text{ half-AC},\text{ non-AC}\}} n_C^2 \Big)  }} \in [-1,1].$$
While the MCC can give a quick and rough assessement of AC-classification performance, it should be used with caution: A lot of subtle differences are lost entirely when boiling down the performance of a highly imbalanced multi-class prediction model into a single scalar performance metric such as the MCC. For an AC-classifier, one must therefore also take a close look at its individual sensitivity and precision values for each class in order to get accurate insights into its true performance and utility.

\subsubsection{Evaluated Models}

We included six distinct models in our computational experiments:

\begin{itemize}
	\item \textbf{ECFP + MLP Baseline:} A standard ECFP-MLP model trained on individual molecules for QSAR-prediction, as was used in the computational study from Chapter~\ref{chap: qsar_ac_study}.
	
	\item \textbf{GIN + MLP Baseline:} A standard GIN-MLP model trained on individual molecules for QSAR-prediction, as was used in the computational study from Chapter~\ref{chap: qsar_ac_study}.
	
	\item \textbf{ECFPs + Twin Network:} A twin neural network as visualised in Figure~\ref{fig:twin_network_architecture} that uses ECFPs for the featurisation of individual MMP compounds.
	
	\item \textbf{GINs + Twin Network:} A twin neural network as visualised in Figure~\ref{fig:twin_network_architecture} that uses GINs for the featurisation of individual MMP compounds.
		
	\item \textbf{ECFP-NFPs + Twin Network:} A twin neural network as visualised in Figure~\ref{fig:twin_network_architecture} that uses pre-trained ECFP-NFPs for the featurisation of individual MMP compounds.
			
	\item \textbf{GIN-NFPs + Twin Network:} A twin neural network as visualised in Figure~\ref{fig:twin_network_architecture} that uses pre-trained GIN-NFPs for the featurisation of individual MMP compounds.
\end{itemize}
Each of the four twin neural network models above corresponds to one of the four MMP featurisations described in Section~\ref{subsec: twin_net_ac_pred_model_description_input_features}. ECFPs, MLPs and GINs were implemented in \texttt{RDKit}~\citep{landrum2006rdkit}, \texttt{PyTorch}~\citep{paszke2019pytorch}, and \texttt{PyTorch Geometric}~\citep{fey2019fast}, respectively.

\subsubsection{Model Training and Hyperparameter Settings}

As mentioned, each model was evaluated within a $k$-fold cross validation scheme repeated with $m$ random seeds for $(m,k) = (50,2)$. This means that an independent version of each model was trained on each training space $(\mathfrak{D}^{i,j}_{\text{train}}, \mathfrak{M}^{i,j}_{\text{train}})$ for $i~\in~\{1,...,50\}$ and $j \in \{1,2\}$ and the results were then averaged over all $mk = 50*2 = 100$ trials. All models were trained on a single NVIDIA GeForce RTX 3060 GPU using AdamW optimisation~\citep{loshchilov2017decoupled}. QSAR models were trained via the mean squared error loss function and twin neural networks via the custom cross-entropy-based loss function $\mathcal{L}_{(\mathcal{R}, \tilde{\mathcal{R}})}$ introduced in Section~\ref{subsec: twin_net_ac_pred_model_description_training}. 

A detailed specification of the hyperparameter choices for the evaluated models and their training loops can be found in Table~\ref{tab: twin_models_hyperparams}. 
\begin{table}[!t]
	\caption[Hyperparameters of twin neural networks and baseline QSAR models.]{Training and model hyperparameters of twin neural networks and baseline QSAR methods.}
	\centering
	\small
	\begin{tabular}{ |p{14.5cm}|  }
		\hline \\
		\multicolumn{1}{|c|}{\textbf{ECFP + MLP}} \\ \\
		
		\textbf{Architecture:} $\text{arch}(\text{MLP}) = (1024, 1, 1024, 5) $
		\newline
		\textbf{Training:} batch size = $64$, learning rate = $10^{-3}$, learning rate decay = $\max \{0.98^{\text{epoch}}, 10^{-1}\}$, weight decay = $0.1$, dropout rate = $0.25$, epochs = $500$					\\
		\hline \\

		\multicolumn{1}{|c|}{\textbf{ECFPs + Twin Network}} \\ \\
		
		\textbf{Architecture:}  $\text{arch}(T_\theta) = (1024, 1024, 1024, 3)$, $\text{arch}(M_\gamma) = (1024, 3, 1024, 1)$, $\text{arch}(O_\eta) = (1024, 1, 1024, 1)$ 
		\newline
		\textbf{Training:} batch size = $2048$, learning rate = $10^{-4}$, learning rate decay = $\max \{0.98^{\text{epoch}}, 10^{-2}\}$, weight decay = $0.1$, dropout rate = $0.25$, epochs = $500$							\\
		\hline \\
		
		\multicolumn{1}{|c|}{\textbf{ECFP-NFPs + Twin Network}} \\ \\
		
		\textbf{Architecture:}  $\text{arch}(T_\theta) = (1024, 1024, 1024, 3)$, $\text{arch}(M_\gamma) = (1024, 3, 1024, 1)$, $\text{arch}(O_\eta) = (1024, 1, 1024, 1)$ 
		\newline
		\textbf{Training:} batch size = $2048$, learning rate = $10^{-4}$, learning rate decay = $\max \{0.98^{\text{epoch}}, 10^{-2}\}$, weight decay = $0.1$, dropout rate = $0.25$, epochs = $500$							\\
		\hline \\
		
		\multicolumn{1}{|c|}{\textbf{GIN + MLP}} \\ \\
		
		\textbf{Architecture:}  $\text{arch}(\text{GIN}) = \{R = 2, \ l = 128, \ \text{arch}(\phi_1) = (78, 128, 128, 2), \ \text{arch}(\phi_2) = (128, 128, 128, 2)\}$, $\text{arch}(\text{MLP}) = (128, 1, 256, 3)$
		\newline
		\textbf{Training:} batch size = $256$, learning rate = $10^{-2}$, learning rate decay = $\max \{0.98^{\text{epoch}}, 10^{-1}\}$, weight decay = $0.1$, dropout rate = $0.25$, epochs = $1000$							\\
		\hline \\
		
		\multicolumn{1}{|c|}{\textbf{GINs + Twin Network}} \\ \\
		
		\textbf{Architecture:} $\text{arch}(T^{\text{GIN-part}}_\theta) = \{R = 2, \ l = 128, \ \text{arch}(\phi_1) = (78, 128, 128, 2), \ \text{arch}(\phi_2) = (128, 128, 128, 2)\}$, $\text{arch}(T^{\text{MLP-part}}_\theta) = (128, 256, 256, 3)$, $\text{arch}(M_\gamma) = (256, 3, 256, 1)$, $\text{arch}(O_\eta) = (256, 1, 256, 1)$ 
		\newline
		\textbf{Training:} batch size = $2048$, learning rate = $10^{-3}$, learning rate decay = none, weight decay = $0.1$, dropout rate = $0.25$, epochs = $1500$							\\
		\hline \\
		
		\multicolumn{1}{|c|}{\textbf{GIN-NFPs + Twin Network}} \\ \\
		
		\textbf{Architecture:}  $\text{arch}(T_\theta) = (256, 256, 256, 3)$, $\text{arch}(M_\gamma) = (256, 3, 256, 1)$, $\text{arch}(O_\eta) = (256, 1, 256, 1)$ 
		\newline
		\textbf{Training:} batch size = $2048$, learning rate = $10^{-4}$, learning rate decay = none, weight decay = $0.1$, dropout rate = $0.25$, epochs = $500$							\\
		\hline
		
	\end{tabular}
	\label{tab: twin_models_hyperparams}
\end{table}
The architecture of MLPs is specified via quadruples of integers in $\mathbb{N}^{4}$ which specify input dimension, output dimension, hidden dimension and number of hidden layers; for example, the quadruple $(100, 1, 200, 5)$ describes an MLP with $100$ neurons in its input layer, followed by $5$ hidden layers with $200$ neurons each, followed by an output layer with $1$ neuron. The architecture of GINs is specified by their radius $R$, their fingerprint length $l$ and the architectures of the MLPs $\phi_r$ at each graph layer $r$ (see Section~\ref{subsec: gins_description}). 

All neural networks consistently used $\text{ReLU}(x) \coloneqq \max\{0,x\}$ as their hidden activation function with the exception of $O_\eta$ from Figure~\ref{fig:twin_network_architecture} which was equipped with $\arctan$-activations to preserve the desired symmetry-properties of the twin model~(see Proposition~\ref{prop: order_eq_pd_twin_net}). Furthermore, all neural networks employed trainable additive bias vectors after each weight-matrix multiplication. Batch normalisation~\citep{ioffe2015batch} was used in all models. The training and model hyperparameter settings of a QSAR method $Q$ for NFP generation (see Section~\ref{subsec: twin_net_ac_pred_model_description_input_features}) were chosen to be almost identical to the hyperparameter settings of its corresponding baseline QSAR method specified in Table~\ref{tab: twin_models_hyperparams}, with one difference being that batch-normalisation was dropped between the last hidden layer and the scalar output layer when using $Q$ to learn NFPs via QSAR-prediction. These architectural choices led to a dimensionality of $1024$ for ECFP-NFPs and of $256$ for GIN-NFPs.

Due to time constraints and the complexity of the twin neural network architecture, a full hyperparameter optimisation of all models was not feasible in this project. However, our hyperparameter choices for the baseline models ECFP-MLP and GIN-MLP generated strong QSAR-prediction results that were on par with the results achieved by the corresponding fully hyperparameter-optimised models from Chapter~\ref{chap: qsar_ac_study}. The ECFP-MLP model in this section reached a mean QSAR-MAE on $\mathfrak{D}_{\text{test}}$ of $0.426$ (in pIC\textsubscript{50} units) and the GIN-MLP model reached a mean QSAR-MAE of $0.442$. These results are as strong as the ones depicted in Figure~\ref{fig:ac_results_sarscov2mpro} which are also based on a $2$-fold cross validation scheme on the same SARS-CoV-2 data set, but included extensive hyperparameter-optimisation routines for the ECFP-MLP and the GIN-MLP model. We can thus conclude with a high degree of confidence that whenever a twin network beats a baseline QSAR model in the experiments in this chapter, then the same twin network would beat the same QSAR model in a related experiment involving full hyperparamter-optimisation of all models. This is because the results in Figure~\ref{fig:ac_results_sarscov2mpro} imply that hyperparameter optimisation would not improve the performance of the QSAR-modelling baselines in the experiments in this section; it could only possibly improve the performance of a twin network. Even without hyperparameter optimisation, our experiments are therefore still suitable to rigorously answer the question whether a twin neural network can beat a (hyperparameter-optimised) ECFP-MLP or GIN-MLP baseline at AC or PD-classification.

\subsection{Results and Discussion} \label{subsec: twin_net_ac_pred_model_description}

The results of our computational experiments are depicted in~\Cref{fig:twin_nets_ac_bars_all,fig:twin_nets_pd_bars_all}. Note that the empirical observations in this chapter are based exclusively on the SARS-CoV-2 main protease binding affinity data set. This data set was suitable for our experiments since it is composed of a single high-quality assay and has a high density of MMPs. However, we acknowledge the necessity to repeat our analysis with other data sets to obtain further evidence for the generalisability of our results.
\begin{figure}
	\centering
	\includegraphics[width=0.922\linewidth]{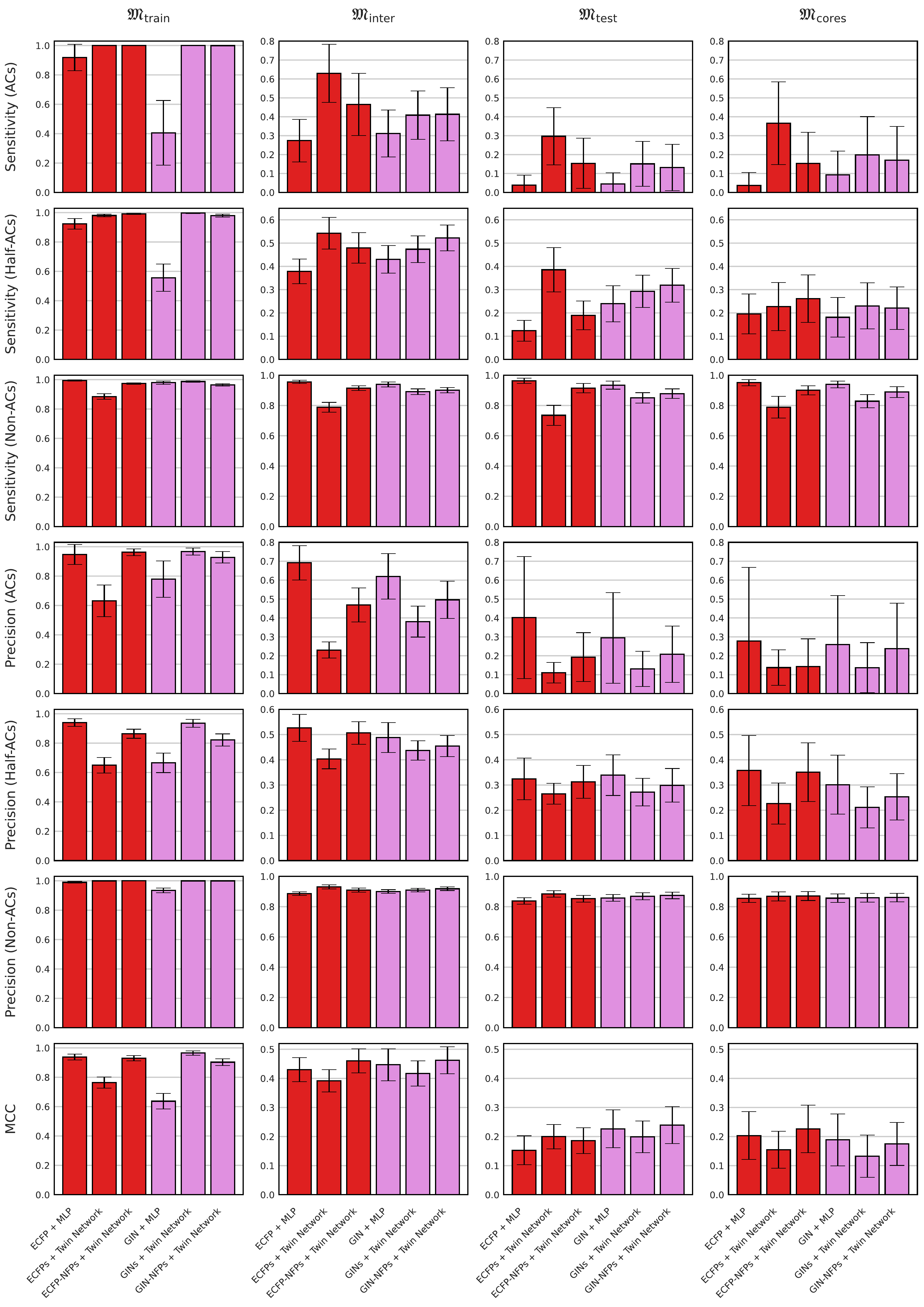}
	\caption[AC-classification results for twin neural network models.]{Activity-cliff~(AC) classification results on the SARS-CoV-2 main protease data set for two baseline QSAR models and four newly developed twin neural network models with distinct input featurisations. The red and violet bars correspond to ECFP-based and GIN-based models, respectively. The total length of each error bar is set to be equal to twice the standard deviation of the performance measured over all $mk = 50*2 = 100$ models.}
	\label{fig:twin_nets_ac_bars_all}
\end{figure}
\begin{figure}[!t]
	\centering
	\includegraphics[width=0.922\linewidth]{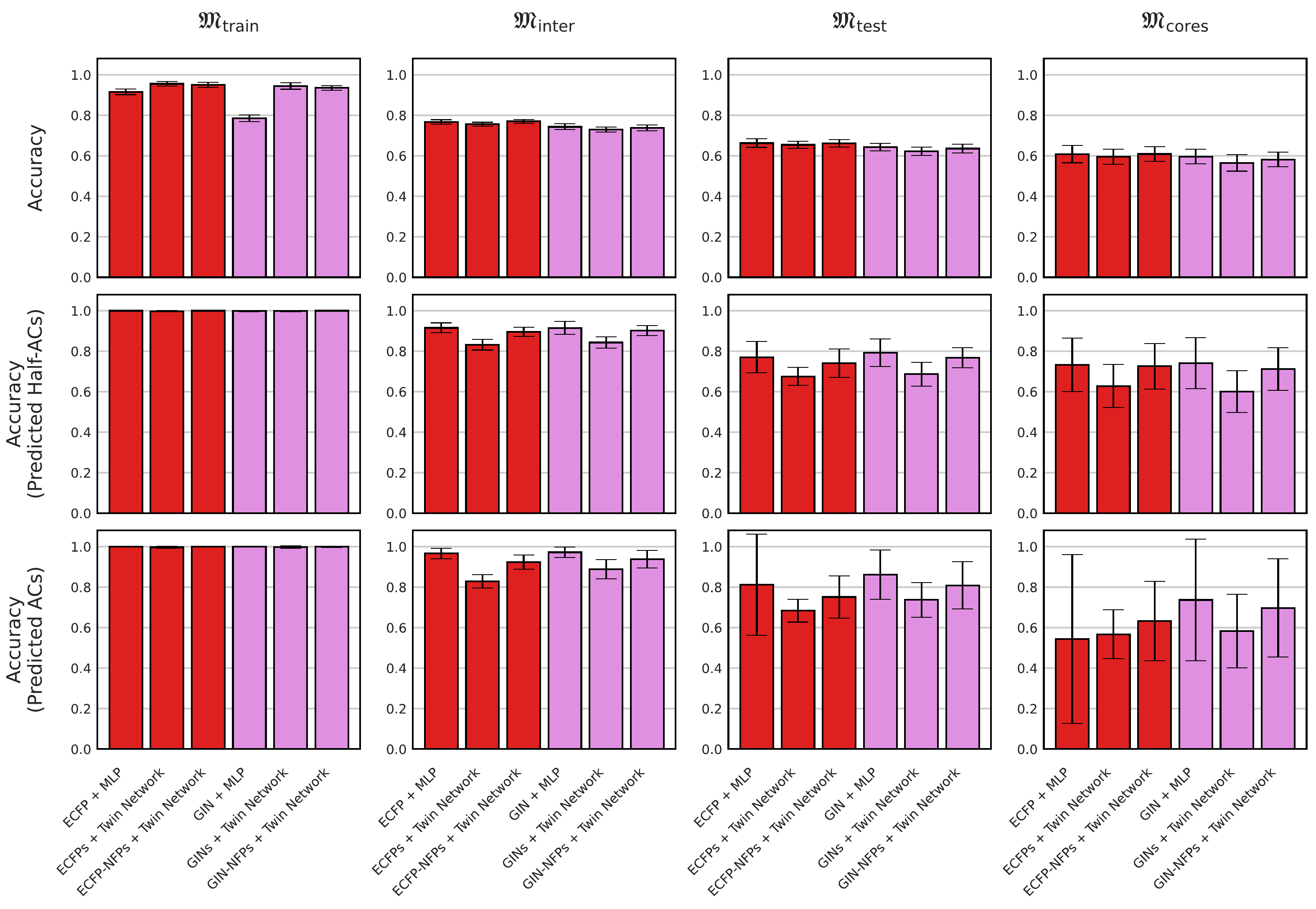}
	\caption[PD-classification results for twin neural network models.]{Potency-direction~(PD) classification results on the SARS-CoV-2 main protease data set for two baseline QSAR models and four newly developed twin neural network models with distinct input featurisations. The red and violet bars correspond to ECFP-based and GIN-based models, respectively. The total length of each error bar is set to be equal to twice the standard deviation of the performance measured over all $mk = 50*2 = 100$ models.}
	\label{fig:twin_nets_pd_bars_all}
\end{figure}

\begin{remark}[Error Bars]
When considering the plots in~\Cref{fig:twin_nets_ac_bars_all,fig:twin_nets_pd_bars_all}, it is easy to overinterpret overlapping error bars when comparing models. The purpose of an error bar in Figure~\ref{fig:twin_nets_ac_bars_all} or Figure~\ref{fig:twin_nets_pd_bars_all} is to communicate the standard deviation associated with the performance of a particular model across multiple random data splits to the reader. In this study, the length of each error bar was set to two standard deviations. While this choice is not uncommon, it is essentially arbitrary. For example, setting the error bar length to one standard deviation instead and pointing this out to the reader would remove many overlaps between error bars while conveying the same information about the stochastic fluctuations in model performance across random data splits. 
\end{remark}

\subsubsection{AC-Classification Performance} \label{subsubsec: ac_class_discussion}

Looking at overall AC-classification performance as quantified by the MCC in the last row of Figure~\ref{fig:twin_nets_ac_bars_all} shows that twin networks that use standard ECFPs or GINs as featurisers appear to exhibit slightly weaker overall MCC performance than their related baseline QSAR models, with the exception of ECFP-based twin networks on $\mathfrak{M}_{\text{test}}$ which clearly outperform ECFP-MLPs in this scenario. The reason might be that these twin networks are exclusively trained on MMPs and ignore compounds not associated with any MMP; unlike QSAR models, such twin networks can thus not leverage the additional SAR-information in isolated compounds in $\mathfrak{D}_{\text{test}}$. This substantially decreases the size of their training set relative to QSAR models that can train on all available compounds. A simple way to remove this data advantage for QSAR models would be to also training them exclusively on compounds involved in MMPs; however, such experiments would not truly simulate a realistic scenario where one uses all available data to produce the strongest possible model.

The overall picture changes when the input featurisations for the twin networks are enriched via transfer learning. The strongest results on $\mathfrak{M}_{\text{inter}}$ are achieved by ECFP-NFP-based and GIN-NFP-based twin networks. Similarly, the strongest results on $\mathfrak{M}_{\text{test}}$ are achieved by GIN-NFP-based twin networks and the strongest results on $\mathfrak{M}_{\text{cores}}$ are achieved by ECFP-NFP-based twin networks. Our computational study in Chapter~\ref{chap: qsar_ac_study} and the results in Figure~\ref{fig:ac_results_sarscov2mpro} indicate that ECFP-MLPs and GIN-MLPs are already amongst the strongest QSAR models for AC-classification. Our observations nevertheless suggest that twin networks when combined with a suitable transfer-learning approach for MMP featurisation (i.e.~NFPs) outcompete even such strong baseline QSAR models for the detection of AC across multiple distinct prediction scenarios. 

It is notable though that if the only examined metric is the overall MCC, then the advantage of the NFP-based twin networks over the QSAR-modelling baselines can in some cases appear modest. However, taking a closer look at the AC-sensitivity and AC-precision of the QSAR-modelling baselines reveals a more nuanced picture. ECFP-MLPs and GIN-MLPs exhibit a reasonably high MCC across MMP sets which should in theory reflect good overall performance for AC-classification. However, the AC-sensitivity of both QSAR models is exceedingly low, reaching around $0.04$ on $\mathfrak{M}_{\text{test}}$; in other words, around $96\%$ of ACs are not correctly classified by these two techniques. This weakness is to some extent compensated by comparatively high AC-precision values: for example, if an ECFP-MLP classifies an MMP in $\mathfrak{M}_{\text{test}}$ as an AC then the probability that this MMP truly is an AC is after all approximately $40\%$. However, the low AC-sensitivity of the tested QSAR models still casts doubt on their practical utility as AC-classifiers, in spite of their moderately high MCCs. Tuning the thresholds for AC-classification in the QSAR-modelling-based prediction strategies outlined in Section~\ref{subsubsec: twin_nets_pred_strat} could mitigate this problem by increasing AC-sensitivity at the cost of AC-precision, thus potentially leading to a more well-balanced and useful classifier. However, the originally set classification thresholds in Section~\ref{subsubsec: twin_nets_pred_strat} precisely reflect the conceptual definitions of non-ACs, half-ACs and ACs used in the literature and throughout this project. Changing these thresholds would thus invite paradoxical situations where for example MMPs with a predicted absolute activity difference of less than two orders of magnitude could be classified as ACs which by definition are MMPs that exhibit an absolute activity difference of more than two orders of magnitude. Therefore, tuning the classification thresholds for QSAR models in an attempt to balance out their AC-sensitivity and AC-precision, while potentially feasible in practice, would arguably be inconsistent and undesirable from a conceptual point of view.

The problem of imbalanced AC-sensitivity and AC-precision can be circumvented in an elegant manner by the twin neural network methodology. The function $w_{\text{AC}}$ that forms part of the twin network loss function $\mathcal{L}_{(\mathcal{R}, \tilde{\mathcal{R}})}$ introduced in Section~\ref{subsec: twin_net_ac_pred_model_description_training} allows one to directly control the class weights given to non-ACs, half-ACs and ACs respectively during twin network training. Assigning a higher relative weight to one class leads to a higher sensitivity (and usually lower precision) for this class in the final trained model. In the twin network model it is thus possible to tune the trade-off between sensitivity and precision for each class in a straightforward and conceptually consistent manner by simply modifying the weight function $w_{\text{AC}}$. In our experiments, we automatically chose class weights that reflected the relative frequencies of non-ACs, half-ACs and ACs in the SARS-CoV-2 data set (see Table~\ref{tab: qsar_ac_study_datsets_overview}). This generally leads to a much less skewed trade-off between AC-sensitivity and AC-precision than can be observed for the two baseline QSAR models. For instance, ECFP-NFP-based twin networks show the strongest overall MCC performance of $0.461$ out of all models on $\mathfrak{M}_{\text{inter}}$ while exhibiting an AC-sensitivity of $0.465$ and an AC-precision of $0.468$. In comparison, ECFP-MLPs on $\mathfrak{M}_{\text{inter}}$ reach a lower MCC of $0.430$, a much lower AC-sensitivity of $0.274$ and a much higher AC-precision of $0.692$. In summary, the twin networks tend to achieve much higher AC-sensitivity and somewhat higher half-AC-sensitivity compared to ECFP-MLPs or GIN-MLPs, at the cost of lower precision in both classes, leading to more balanced classifiers. This important trend cannot be seen from the MCC performance alone; it makes twin neural network models substantially more well-rounded AC-classifiers than ECFP-MLPs or GIN-MLPs which skew heavily in the direction of high AC-precision and very low AC-sensitivity.

Perhaps unsurprisingly, almost all of the evaluated models exhibit very high AC-classification performance on $\mathfrak{M}_{\text{train}}$ as they were trained on this set of MMPs. There seems to be one salient exception to this though: GIN-MLPs exhibit comparatively very low overall AC-classification performance on the set of training-MMPs, reaching an MCC of only $0.637$ and an AC-sensitivity of only $0.406$. This effect might simply be caused by insufficient overall QSAR-prediction performance and/or a lack of training time on the QSAR task. However, there are two reasons that speak against this explanation:
Firstly, each GIN-MLP was trained for no less than $1000$ epochs on $\mathfrak{D}_{\text{train}}$ which was sufficiently long for the training loss to converge to a stable plateau for the last several hundred epochs in all instances. Secondly, GIN-MLPs reach a similar QSAR-MAE for the prediction of individual molecular activities on $\mathfrak{D}_{\text{test}}$ as ECFP-MLPs ($0.442$ for GIN-MLPs vs.~$0.426$ for ECFP-MLPs), yet ECFP-MLPs show much higher AC-classification performance on $\mathfrak{M}_{\text{train}}$. In a future research project, it might be interesting to experiment with techniques to increase the AC-classification performance of GIN-MLP models on $\mathfrak{M}_{\text{train}}$ and study the effects of this on the performance on $\mathfrak{M}_{\text{inter}}$, $\mathfrak{M}_{\text{test}}$ and $\mathfrak{M}_{\text{cores}}$. A straightforward avenue to explore would be to increase the size of GIN-MLPs and train them for an unusually long time, potentially on the order of $10^{4}$ to $10^{5}$ epochs; while this would almost certainly lead to overfitting on the QSAR task and thus lower QSAR-prediction performance, one can hypothesise that it might nevertheless increase AC-sensitivity on $\mathfrak{M}_{\text{train}}$.

When analysing the AC-classification results with respect to molecular featurisation, we see that GIN-based methods consistently match or outperform ECFP-based methods according to MCCs on $\mathfrak{M}_{\text{inter}}$ and $\mathfrak{M}_{\text{test}}$. These findings agree with the trends observed in Figure~\ref{fig:ac_results_sarscov2mpro} from our previous computational study and extend them to the realm of twin neural networks. At first glance, our observations thus once again suggest that graph-based GIN features tend to be equal or superior to ECFPs for the detection of ACs, and this seems to be true both when the underlying predictor is an MLP or when it is a twin neural network. There are three caveats to this conclusion though.

Firstly, looking closely at sensitivity and precision values reveals that the MCC might be a misleading performance measure in some cases. For example, GIN-based twin networks reach an MCC of $0.199$ on $\mathfrak{M}_{\text{test}}$ while ECFP-based twin networks reach an essentially equivalent MCC of $0.200$ on the same MMP set. This suggests that both methods are equally good at ternary AC-classification. And indeed, the precision values of both methods on $\mathfrak{M}_{\text{test}}$ closely resemble each other for all three classes. However, the sensitivities of GIN-based twin networks for non-ACs/half-ACs/ACs are $0.850$ / $0.293$ / $0.151$ while the corresponding sensitivity of ECFP-based twin networks are $0.735$ / $0.385$ / $0.297$. Thus, since ACs and half-ACs are naturally of greater interest than non-ACs in almost all cases, ECFP-based twin networks seem to be strongly preferable to GIN-based twin networks, even though their respective MCCs suggest that both methods are equivalent. One way to express this advantage is that ECFP-based twin networks can generate a longer list of potential AC and half-AC-candidates from a given test data set than GIN-based twin networks while operating at an equal level of precision.

The second caveat to the idea that GIN-features are superior to ECFPs for AC-classification comes from the fact that this trend, while true on $\mathfrak{M}_{\text{inter}}$ and $\mathfrak{M}_{\text{test}}$, appears to reverse on $\mathfrak{M}_{\text{cores}}$. The MCC performance of all GIN-based methods drops in a predictable manner when moving from $\mathfrak{M}_{\text{test}}$ to the more difficult test set $\mathfrak{M}_{\text{cores}}$. However, surprisingly, the same is not true for ECFP-based methods: here the MCC performance of ECFP-MLPs and ECFP-NFP-based twin networks does in fact increase. ECFPs therefore appear to cope better than GINs with the distributional shift between $\mathfrak{M}_{\text{train}}$ and $\mathfrak{M}_{\text{cores}}$; at first sight, this suggest that ECFP-based methods might be able to extract more generalisable chemical knowledge that is not merely based on memorisation. Once again the picture becomes more refined though when looking at sensitivity and precision metrics on $\mathfrak{M}_{\text{cores}}$. The investigated models can be grouped into two sets according to their respective AC-precision values: ECFP-MLPs, GIN-MLPs and GIN-NFP-based twin networks share a similar level of AC-precision; and the same is true for ECFP-based twin networks, ECFP-NFP-based twin networks and GIN-based twin networks. The highest AC-sensitivity by far in the first group is exhibited by GIN-NFP-based twin networks and the same is true in the second group for ECFP-based twin networks. Thus, if the goal is to discover as many AC-candidates as possible in a new data set while maintaining a certain level of precision, then the best model is either a GIN-NFP-based or an ECFP-based twin network, even though neither of these two methods exhibits a comparatively high MCC.

\subsubsection{PD-classification Performance}

The overall accuracy-results for the balanced binary PD-classification task are depicted in the first row of Figure~\ref{fig:twin_nets_pd_bars_all}. We can see that the differences in performance between distinct models on $\mathfrak{M}_{\text{inter}}$, $\mathfrak{M}_{\text{test}}$ and $\mathfrak{M}_{\text{cores}}$ are minor. All models correctly classify the potency direction of approximately $75\%$ of MMPs in $\mathfrak{M}_{\text{inter}}$, $65\%$ of MMPs in $\mathfrak{M}_{\text{test}}$ and $60\%$ of MMPs in $\mathfrak{M}_{\text{cores}}$. Interestingly, GIN-MLPs perform noticeably worse than the other methods on $\mathfrak{M}_{\text{train}}$ but this does not seem to harm their relative predictive abilities on $\mathfrak{M}_{\text{inter}}$, $\mathfrak{M}_{\text{test}}$ or $\mathfrak{M}_{\text{cores}}$. This resembles the observation already discussed in Section~\ref{subsubsec: ac_class_discussion} that GIN-MLPs show much lower AC-sensitivity on $\mathfrak{M}_{\text{train}}$ than the other models while still exhibiting competitive AC-classification performance on the other MMP sets. Our previous discussion on this effect thus also applies to the analogous situation for PD-classification.

The accuracy-results for all models on the full MMP sets $\mathfrak{M}_{\text{inter}}$, $\mathfrak{M}_{\text{test}}$ and $\mathfrak{M}_{\text{cores}}$ might appear modest at first; note however that MMPs are by definition pairs of \textit{structurally similar} molecules and that this similarity regularly goes hand-in-hand with similar activities for a given pharmacological target. ACs and half-ACs are relatively rare exceptions to this rule which is known by medicinal chemists as the \textit{similarity principle}. Classifying which of two compounds is more active thus becomes a challenging prediction task in a setting where both compounds have almost identical chemical structures and therefore also frequently exhibit almost identical activities. In particular, the similarity principle along with the inherent noisiness of experimentally measured binding affinity values suggests that the discrete PD-classification learning signal derived from $\mathfrak{M}_{\text{train}}$ for the four twin networks likely contains a considerable amount of noise in the form of false binary labels. In light of these obstacles, the capabilities of the investigated methods to detect the potency direction of MMPs is nontrivial. The fact that all six distinct techniques reach almost equivalent levels of accuracy on $\mathfrak{M}_{\text{inter}}$, $\mathfrak{M}_{\text{test}}$ and $\mathfrak{M}_{\text{cores}}$ might potentially be caused by a performance ceiling that is rooted in the underlying data set itself and that is approached by all models.

The second and third row of Figure~\ref{fig:twin_nets_pd_bars_all} contain PD-classification accuracies in a scenario where the sets of test-MMPs are restricted to only include MMPs that were classified as half-ACs/ACs by a respective model. Arguably, these results are of higher practical relevance than the overall PD-classification results on the full MMP sets, since the predicted potency direction of an MMP is often of stronger interest if the predicted activity difference is substantial. This is for instance true in the case of compound optimisation. We overall see stronger PD-classification results on the restricted MMP sets of predicted half-ACs/ACs than on the full MMPs-sets; perhaps unsurprisingly, this suggests that it is easier to predict the potency direction of MMPs whose underlying activity difference is large. When MMPs are restricted to predicted half-ACs/ACs, we observe essentially perfect results for all models on $\mathfrak{M}_{\text{train}}$. We further see a tendency for the two baseline QSAR models to slightly outperform the four twin neural networks on the restricted versions of $\mathfrak{M}_{\text{inter}}$, $\mathfrak{M}_{\text{test}}$ and $\mathfrak{M}_{\text{cores}}$. For example, on $\mathfrak{M}_{\text{inter}}$ ECFP-MLPs predict close to $100\%$ of potency directions of predicted ACs correctly, while ECFP-based twin networks accurately predict only slightly more than $80\%$. 

Caution needs to be exercised though when directly comparing accuracy-results like these, since the exact sets of predicted half-ACs/ACs differ from model to model. Moreover, the lower the half-AC/AC-precision of a model, the more non-ACs infiltrate its set of predicted half-ACs/ACs and the harder PD-classification subsequently becomes on this set. And indeed, comparing~\Cref{fig:twin_nets_ac_bars_all,fig:twin_nets_pd_bars_all} reveals that models that exhibit a comparatively high half-AC/AC-precision (such as the two baseline QSAR models) also tend to exhibit a comparatively high PD-classification accuracy on predicted half-ACs/ACs. As mentioned before, a likely underlying reason for this is that the potency direction of half-ACs and ACs might be easier to predict than the potency direction of non-ACs.

Some models also tend to have a substantially larger set of prediced half-ACs/ACs than others. The number of (not necessarily correctly) predicted members of a class $C$ for a given model $M_1$ can be estimated from the true number of occurrences of $C$ in the test set and the sensitivity and precision values of the model for~$C$:
$$p_C(M_1) = n_{C} \frac{\text{sens}_C(M_1)}{\text{prec}_C(M_1)}. $$
Here we use the notation introduced in Section~\ref{subsubsec: twin_networks_performance_measures}. If $M_2$ is a second model, then it follows that
$$ \frac{p_C(M_1)}{p_C(M_2)} = \frac{\text{sens}_C(M_1)\text{prec}_C(M_2)}{\text{prec}_C(M_1)\text{sens}_C(M_2)}.$$
With this formula, we can for instance conclude that the (average) ratio of the numbers of predicted ACs for $M_1 \coloneqq $ ECFP-based twin networks and $M_2 \coloneqq$ GIN-based twin networks on $\mathfrak{M}_{\text{test}}$ is given by
$$\frac{p_{\text{AC}}(M_1)}{p_{\text{AC}}(M_2)} = \frac{\text{sens}_{\text{AC}}(M_1)\text{prec}_{\text{AC}}(M_2)}{\text{prec}_{\text{AC}}(M_1)\text{sens}_{\text{AC}}(M_2)} = \frac{0.297*0.131}{0.110*0.151} \approx 2.342. $$
This implies that the list of predicted AC-candidates generated by ECFP-based twin networks is more than twice as long as the equivalent list generated by GIN-based twin networks, even though both models operate at a similar level of AC-precision~($0.110$ vs.~$0.131$). However, the PD-classification accuracies of ECFP-based and GIN-based twin networks show a modest but noticeable difference on $\mathfrak{M}_{\text{test}}$ ($0.683$ vs.~$0.737$). ECFP-based twin networks can thus detect a much larger volume of potential AC-candidates in $\mathfrak{M}_{\text{test}}$ than GIN-based twin networks at an almost equivalent level of precision, but this advantage comes at the apparent cost of a slightly reduced PD-classification accuracy on predicted ACs.

\section{Conclusions} \label{sec: twin_net_ac_pred_conc}

In this chapter, we have introduced a novel twin neural network architecture for activity-cliff~(AC) and potency-direction~(PD) classification. The general design of this twin network enables it to seamlessly integrate with essentially any featurisation method for individual molecules. 
We first mathematically investigated the built-in symmetry properties of the twin architecture which were designed to serve as a useful inductive bias for our application. We then conducted a series of computational experiments on a data set of SARS-CoV-2 main protease inhibitors to compare the AC and PD-classification performance of two QSAR models and four versions of our twin model, based on either ECFPs or GINs, with or without pre-trained neural fingerprints~(NFPs) generated via supervised transfer learning. To the best of our knowledge, this work represents the first application of twin neural networks to the problems of AC and PD-classification and the first AC-prediction study that includes appropriate QSAR-modelling baselines for comparison. 

At PD-classification, the baseline QSAR models tend to be slightly more accurate than the twin networks if the MMP sets are restricted to predicted half-ACs/ACs. However, on the full MMP sets the differences in overall PD-classification accuracy amongst the six investigated models are negligible. This stands in contrast to the heterogeneous AC-classification results, where we have seen that the developed twin architecture is able to outperform both baseline QSAR models in a variety of important ways. NFP-based twin networks tend to reach the strongest overall MCC performance across distinct prediction scenarios. This demonstrates the positive effects of the transfer learning technique used to generate NFPs which enables the exploitation of additional chemical knowledge contained in compounds not involved in MMPs. Our observations suggest that NFP-based twin networks are a well-balanced choice for AC-classification that can outperform even strong baseline QSAR models. Moreover, twin networks generally appear to strike a substantially better balance between sensitivity and precision than the overly conservative QSAR methods. This may make them more attractive than standard QSAR models for applications such as compound optimisation and automatic SAR-knowledge acquisition. The improved AC-classification balance for the twin network is rooted in the fact that class weights can be assigned to non-ACs, half-ACs and ACs in its loss function.

Our observations have further demonstrate that a detailed look at sensitivity and precision metrics in addition to overall performance measures such as the MCC can lead to important insights in certain use cases. In particular, if the goal is to generate a long list of potential half-AC/AC-candidates at a certain level of precision, then sometimes one model can be preferable over another one even if both have essentially the same MCC. For instance, ECFP-based twin networks and GIN-based twin networks have very similar MCCs and AC-precisions on $\mathfrak{M}_{\text{test}}$ in our experiments; however, the AC-sensitivity of ECFP-based twin networks on this MMP set is much higher and they can thus generate a much longer list of AC-candidates at the same level of precision.

A promising pathway for future research might be to explore ways to further optimise the input-MMP featurisation for AC-classification, potentially by including three-dimensional information that can smooth out two-dimensional MMP cliffs, or by pre-training a neural feature extractor via a contrastive loss that explicitly incentivises the resolution of ACs in the embedding space.

\newpage $\text{}$ 
\newpage
\chapter[Beyond Hashing: Substructure-Pooling Techniques to Robustly Improve Extended-Connectivity Fingerprints]{Beyond Hashing: Substructure-Pooling \\ Techniques to Robustly Improve Extended-Connectivity Fingerprints} \label{chap: ecfps_sort_and_slice}

\section{Introduction}

The use of extended-connectivity fingerprints~(ECFPs) is omnipresent in current cheminformatics. As described in Section~\ref{sec: ecfps}, the ECFP algorithm transforms a molecule (usually given in the form of a SMILES string) into a high-dimensional binary vector whose components indicate the presence or absence of particular circular chemical substructures within the input compound.\footnote{\textit{ECFPs with counts} also exist, which do not simply take the form of binary vectors but integer vectors that indicate the exact number of occurrences of each substructure. In this work, however, unless specifically stated otherwise, we always focus on binary ECFPs without counts due to their more widespread use and conceptual simplicity.} A modern and widely recognised technical description of ECFPs was given in $2010$ by~\citet{rogers2010extended}, although the key ideas underlying ECFP generation were already introduced by~\citet{morgan1965generation} in $1965$. To name only a few applications, ECFPs have been used successfully for ligand-based virtual screening~\citep{riniker2013open}, the prediction of the aqueous solubility of molecular compounds~\citep{duvenaud2015convolutional}, the computational detection of cytotoxic substructures of molecules~\citep{webel2020revealing}, the prediction of the inhibitory activity of molecules against \textit{E.\ coli} enzymes~\citep{rogers2005using}, the identification of unknown binding targets of chemical compounds via similarity searching~\citep{alvarsson2014ligand}, and the prediction of quantum-chemical properties of small molecules~\citep{gilmer2017neural}. 

Perhaps the most typical use case of ECFPs is as a featurisation method for supervised molecular machine learning, i.e.~as a method to transform molecules into binary feature vectors for a given downstream molecular property prediction task. For this purpose, ECFPs are popular tools due to their conceptual simplicity, high interpretability, and low computational cost. Moreover, a growing corpus of literature suggests that ECFPs still regularly match or even surpass the predictive performance of trainable feature-extraction methods based on state-of-the-art message-passing graph neural networks (GNNs)~\citep{stepivsnik2021comprehensive,mayr2018large,menke2021using,chithrananda2020chemberta,winter2019learning,dablander2023exploring}.
These findings agree with our own results presented in Chapter~\ref{chap: qsar_ac_study} where we demonstrated via a series of rigorous computational experiments that ECFPs consistently beat modern GINs~\citep{xu2018powerful} at binding affinity prediction for a given protein (i.e.~QSAR-prediction).

From a bird's eye view, the ECFP algorithm outlined in Section~\ref{subsec: ecfp_and_fcfp_atom_features} can be decomposed into two steps:
a first step 
$$\mathcal{S} \mapsto \mathcal{I} \subseteq \{1,...,2^{32}\}  $$
in which the SMILES string $\mathcal{S}$ of a compound is mapped to a set of integer identifiers $\mathcal{I}$ which correspond to circular chemical subgraphs of varying radii $r~\in~\{0,..., R\}$ with atomic centers in $\mathcal{S}$ (see Figure~\ref{fig:circular_subgraphs_example}); and a second step
$$ \mathcal{I} \mapsto \mathcal{F} \in \{0,1\}^{l}$$
in which the set of integer identifiers (i.e.~the set if circular subgraphs) is transformed into a binary vector of predefined length $l$. We refer to the first step as \textit{substructure enumeration} and to the second step as \textit{substructure pooling}. In this chapter, we focus on substructure pooling and we will give a precise mathematical definition of it in the context of supervised molecular machine learning.

Note that the technical parallels between ECFPs and message-passing GNNs (Section~\ref{sec: gnns}) are striking: In both cases, a compound is first transformed into an unordered set representation whereby different compounds can be transformed to sets of different cardinalities. For ECFPs, this set representation is given by the set of initial and updated integer atom identifiers in $\mathcal{I}$ (which correspond to circular substructures) while for GNNs this set representation is given by the set of initial and updated atom feature vectors.\footnote{To be precise, one uses multisets (i.e.~sets with counts) instead of standard sets in the case of GNN architectures. This is to be able to distinguish identical atom feature vectors belonging to distinct atoms. Similarly, one would use multisets instead of sets when dealing with ECFPs with counts instead of binary ECFPs. However, as mentioned above, in this work we focus on binary ECFPs without counts.} For both methods, the pooling operation then plays the crucial role of reducing the given set representation to a single feature vector that describes the entire compound and that can readily be fed into a standard machine learning model.

While considerable work has been done to develop and investigate a variety of pooling methods for modern GNN architectures~\citep{navarin2019universal,cangea2018towards,lee2019self,ranjan2020asap,ma2020path}, almost no analogous research exists that describes and explores alternative substructure pooling methods for ECFPs. The canonical way for substructure pooling~\citep{rogers2010extended} for ECFPs is based on the use of a deterministic hash function as was already described in Section~\ref{sec: ecfps} of this thesis. The hash function is used to map a set of integer identifiers $\mathcal{I}$ into a set $\{1, ..., l\}$ of much smaller range that can then be transformed into a binary vectorial fingerprint $\mathcal{F} \in \{0,1\}^l$ of desired length $l$. This straightforward type of hashing is the current default substructure-pooling technique for ECFPs and is used almost universally in the molecular property prediction literature, although alternative hashing procedures for ECFPs and closely related circular fingerprints have been explored by~\citet{probst2018probabilistic} for analog searches in big data settings. In spite of its widespread use, the default hashing technique comes with a considerable downside: It is well-known that standard hash-based substructure pooling for ECFPs suffers from the technical problem of \textit{bit collisions}, which occur when distinct integer identifiers (i.e.~distinct substructures) are hashed to the same component of the output vector $\mathcal{F}$. Bit collisions necessarily occur when the fingerprint dimension $l$ is smaller than the number of detected circular substructures which is almost always the case in standard settings; moreover, the smaller the fingerprint dimension $l$ relative to the number of detected substructures, the more bit collisions emerge. The ambiguities introduced by bit collisions into the fingerprint not only compromise its interpretability but also its predictive performance in machine-learning applications.

\citet{gutlein2016filtered} published a high-quality study which represents one of the few existing works that systematically explores alternative substructure pooling strategies for ECFPs to circumvent the problem of bit collisions. They explore an advanced supervised substructure selection scheme (i.e.~a feature selection scheme) as a pooling method to construct what they referred to as \textit{filtered} fingerprints. We will discuss filtered fingerprints as introduced by~\citet{gutlein2016filtered} in more detail below. 

In this chapter, we describe an extremely simple and surprisingly effective alternative to hashing for substructure-pooling of ECFP substructures which we call \textit{Sort \& Slice}. In a nutshell, Sort \& Slice is based on first sorting all unique circular substructures in the training set according to their frequency of occurrence within training compounds and then slicing away the least frequent substructures to arrive at a fingerprint of desired length. From a formal point of view, Sort \& Slice can be seen as a very simple unsupervised feature selection scheme.
In spite of this simplicity, we are able to mathematically show a strong overlap between Sort \& Slice and a more complex unsupervised feature selection method based on the maximisation of information entropy~\citep{shannon1948mathematical,cover1991entropy}. 

Sort \& Slice fully removes bit collisions at the level of integer identifiers; each vectorial component in the final fingerprint corresponds to the presence or absence of one and only one integer identifier in $\mathcal{I}$. This implies that each fingerprint component is associated with a unique circular substructure (if one ignores the slim chance of hash collisions during the ECFP substructure enumeration that could lead to two substructures being assigned the same integer identifier). As a result, vectorial ECFP representations generated via Sort \& Slice are more straightforward to interpret than hashed ECFPs, which regularly contain components that correspond to multiple integer identifiers due to colliding bits. Furthermore, note that the slicing procedure, while massively reducing the dimension of the fingerprint, tends to preserve the vast majority of the chemical information contained in the training set. This is because in common real-world molecular data sets, the least frequent substructures usually only appear in a few compounds (very often only in a single compound); removing such substructures thus corresponds to the removal of almost-constant features with little-to-no information that could be leveraged by a prediction algorithm.

We show via a series of rigorous ECFP-based computational experiments that Sort \& Slice regularly and sometimes substantially outperforms (i) standard hash-based substructure pooling~\citep{rogers2010extended}, (ii) filtered fingerprints as developed by~\citet{gutlein2016filtered}, and (iii) a popular supervised feature selection scheme based on mutual-information maximisation~(MIM)~\citep{shannon1948mathematical,cover1991entropy}. In particular, we show that Sort \& Slice consistently leads to higher predictive performance than hashing, filtering, or MIM
\begin{itemize}
	\item across a diverse set of supervised molecular property prediction tasks involving regression as well as balanced and imbalanced classification,
	
	\item across distributional shifts caused by distinct data splitting strategies (random, scaffold),
	
	\item across machine learning models (random forest, multilayer perceptron),
	
	\item across fingerprint radii ($R \in \{1,2,3\}$),
	
	\item across fingerprint lengths ($l \in \{512, 1024, 2048, 4096\}$), and
	
	\item across initial atomic invariants (standard, pharmacophoric).
	
\end{itemize}
We are thus able to demonstrate that the predictive advantage provided by Sort \& Slice over both hashing and two advanced supervised feature selection techniques is highly robust and generalises across a large number of settings.

Note that we do not dare to claim that the Sort \& Slice technique we investigate in this chapter was necessarily first discovered and implemented by us; in fact, the simplicity of the method makes it possible that versions of it have already been applied by other researchers in the past in a variety of contexts. In particular, we acknowledge the recent work of~\citet{macdougall2022reduced} which we discovered during our literature search: he proposed a procedure that is almost identical to Sort \& Slice, with the only technical difference to our method appearing to be associated with the slicing procedure. While the slicing technique from~\citet{macdougall2022reduced} only allows limited control over the length of the fingerprint, our slicing scheme can generate fingerprints of any predefined arbitrary length. We will discuss this difference in more detail below when mathematically describing Sort \& Slice. Our goal is to provide the following novel contributions in this chapter:

\begin{enumerate}
	\item We give a precise and very general mathematical definition of substructure pooling and suggest it as a potential research avenue to boost the performance of structural fingerprints in molecular machine learning.
	
	\item We mathematically describe our version of Sort \& Slice as a straightforward alternative to hashing for substructure pooling that is very easy to implement, allows full control over the fingerprint length and exhibits markedly higher interpretability due to an absence of bit collisions.
	
	\item We show via a series of strict computational experiments that for ECFPs Sort \& Slice consistently leads to higher predictive performance than hashing and two relevant supervised feature selection schemes for molecular property prediction across a large number of scenarios; and that frequently the performance gains associated with Sort \& Slice are surprisingly large.
	
	\item We recommend that due to its technical simplicity, dimensional customisability, improved interpretability and superior predictive performance, Sort \& Slice should canonically replace hashing as the default substructure pooling method to vectorise ECFPs for supervised molecular machine learning.
	
\end{enumerate}

\

\section{Methods and Experimental Methodology}

\subsection{Substructure Pooling: Mathematical Description}
\label{subsec: subtructure_pooling}

\begin{definition}[Substructure Pooling] \label{def: substructure_pooling}
	Let
	$$\mathfrak{C} = \{\mathcal{C}_{1}, ..., \mathcal{C}_{m}\} $$
	be a (potentially very large) set of chemical substructures. The substructures $\mathcal{C}_1, ..., \mathcal{C}_m$ could take the form of SMILES strings, molecular graphs, or another computational representation such as hashed integer identifiers.
	Now let the power set of $\mathfrak{C}$, i.e.~the set of all possible subsets of $\mathfrak{C}$, be denoted by
	$$ P(\mathfrak{C}) \coloneqq \{\mathfrak{A} \ \vert \ \mathfrak{A} \subseteq  \mathfrak{C} \}. $$
	A substructure-pooling method of dimension $l \in \mathbb{N}$ for the chemical substructures in $\mathfrak{C}$ is an operator
	$$\Psi :  P(\mathfrak{C}) \to \mathbb{R}^l$$
	that maps subsets of $\mathfrak{C}$ to $l$-dimensional real-valued vectors.
\end{definition}

Substructure pooling naturally appears in the context of supervised molecular machine learning for the vectorisation of structural fingerprints for molecular featurisation. To see this, consider a supervised molecular property prediction task specified by a training set of $n$ compounds
$$\mathfrak{D} = \{\mathcal{R}_1, ..., \mathcal{R}_n\} \subset \mathfrak{R} $$
and an associated function
$$c : \mathfrak{D} \to \mathbb{R}$$
that assigns regression or classification labels to the training set.
The training compounds $\mathcal{R}_1, ..., \mathcal{R}_n$ form part of a larger chemical space $\mathfrak{R}$ whose elements could for example be represented via SMILES strings or molecular graphs. Analogous to Definition~\ref{def: substructure_pooling}, let
$$\mathfrak{C} = \{\mathcal{C}_{1}, ..., \mathcal{C}_{m}\} $$
be a set of $m$ chemical substructures of interest and let
$$ P(\mathfrak{C}) = \{\mathfrak{A} \ \vert \ \mathfrak{A} \subseteq  \mathfrak{C} \}$$
be its power set. Now let 
$$\varphi : \mathfrak{R} \to P(\mathfrak{C})$$ 
be a structural-fingerprinting algorithm that maps an input compound in $\mathfrak{R}$ to the set of substructures in $\mathfrak{C}$ that appear in the input compound. Via $\varphi$ one can transform each training compound $\mathcal{R}_i$ into a set representation consisting of $r_i$ substructures:
$$\varphi(\mathcal{R}_i) = \{\mathcal{C}_{i,1}, ..., \mathcal{C}_{i, r_i} \} \subseteq \mathfrak{C}. $$
The elements $\mathcal{C}_{i,1}, ..., \mathcal{C}_{i, r_i} \in \mathfrak{C}$ corresponds to chemical substructures that belong to the larger set of considered substructures $\mathfrak{C}$ and that are present in $\mathcal{R}_i$. 

A straightforward example for $\varphi$ is of course the ECFP algorithm described in Section~\ref{sec: ecfps}. In this case, $\mathfrak{C}$ is the set of all circular chemical fragments up to a predefined radius $R$ represented via their respective integer identifiers in the hash space $\{1, ..., 2^{32}\}$. Another option for $\varphi$ that has been frequently used in the past is the $166$-bit MACCS fingerprint~\citep{durant2002reoptimization} for which $\mathfrak{C}$ becomes a fixed set of $166$ chemical substructures represented via their respective SMARTS strings~\citep{smartstheorymanual}.

By composing a substructure-fingerprinting algorithm
$$\varphi : \mathfrak{R} \to P(\mathfrak{C})$$ 
with a substructure-pooling operator
$$\Psi :  P(\mathfrak{C}) \to \mathbb{R}^l$$
one can finally transform each molecular set representation $\varphi(\mathcal{R}_i)$ into a real-valued vector $\Psi(\varphi(\mathcal{R}_i)) \in \mathbb{R}^l$. The vectorised training data set $$\Psi(\varphi(\mathfrak{D})) \subset \mathbb{R}^l$$ can then be fed into a standard machine learning model such as a random forest or a multilayer perceptron that can be trained to predict the labels specified by~$c$. Note that the substructure-pooling operator $\Psi$ can be constructed leveraging knowledge from the training data; in other words, $\Psi$ is allowed to depend on $\mathfrak{D}$ and the labelling function~$c$. Furthermore, $\Psi$ does not necessarily need to be a fixed function; it could also be a trainable deep network. For instance, later in Section~\ref{sec: diff_substruc_pool} we will introduce a trainable substructure-pooling technique based on a differentiable self-attention mechanism.

The problem of substructure pooling can be directly translated into a mathematical problem that closely resembles GNN pooling as described in Section~\ref{sec: gnns}. This can be achieved if one employs a substructure-embedding function of some dimension $w$:
$$\gamma : \mathfrak{C} \to \mathbb{R}^{w}. $$
Using this embedding, one can straightforwardly transform sets of substructures into sets of vectors via
$$\Gamma_{\gamma} : P(\mathfrak{C}) \to \{ A \subset \mathbb{R}^w \ \vert \ A \ \text{is finite} \}, \quad \Gamma_{\gamma}(\{\mathcal{C}_{1}, ..., \mathcal{C}_{r}\}) = \{\gamma(\mathcal{C}_{1}), ..., \gamma(\mathcal{C}_{r})\}. $$ 
Given a substructure-fingerprinting algorithm $\varphi$, the composite function 
$$\Gamma_{\gamma} \circ \varphi : \mathfrak{R} \to \{ A \subset \mathbb{R}^w \ \vert \ A \ \text{is finite} \}$$ then maps a chemical compound to a set of real-valued vectors. This vectorial set representation can be seen analogously to the outputs of an iteratively applied GNN layer which computes a representation of the molecule in the form of layerwise sets $f_i(A)$ of initial and updated atom feature vectors 
$$f_0(A),...,f_R(A) \subset \{ A \subset \mathbb{R}^w \ \vert \ A \ \text{is finite} \}.$$
Whether the vectors within a set representation generated by $\Gamma_{\gamma} \circ \varphi$ can also be associated with specific radii and central atoms within the input compound like the feature vectors in $f_0(A),...,f_R(A)$ depends on the specifics of the fingerprinting-algorithm~$\varphi$. In the case of ECFPs, for example, such a layerwise and atomwise interpretation that is equivalent to its GNN counterpart is possible and highly natural since each ECFP-generated integer identifier is associated with a substructure of radius $r \in \{0,...,R\}$ around a central atom. Note though that even if such an interpretation is not possible, both techniques still lead to a representation of the input compound in terms of sets of vectors. We can thus see that all GNN pooling methods that correspond to graph-topology-independent permutation-invariant set-functions
$$\bigoplus : \{ A \subset \mathbb{R}^w \ \vert \ A \ \text{is finite} \} \to \mathbb{R}^l $$
can immediately be repurposed to vectorise the sets produced by $\Gamma_{\gamma} \circ \varphi$ for downstream machine learning. Given a suitable embedding $\gamma$ one can therefore reuse techniques from the literature on GNN pooling for substructure pooling. As an example, consider the sum operator that is frequently used for GNN pooling:
$$\sum :  \{ A \subset \mathbb{R}^w \ \vert \ A \ \text{is finite} \} \to \mathbb{R}^w, \quad \sum \left( \{v_1,...,v_r\} \right) = \sum_{i = 1}^r v_i.$$
By composing the sum operator $\Sigma$ with a substructure embedding $\gamma$ we immediately gain a substructure-pooling operator whose output dimension $l$ is equal to the embedding-dimension $w$:
$$\Psi :  P(\mathfrak{C}) \to \mathbb{R}^w, \quad \Psi(\{\mathcal{C}_{1}, ..., \mathcal{C}_{r}\}) = \Big(\sum \ \circ \ \Gamma_{\gamma} \Big) (\{\mathcal{C}_{1}, ..., \mathcal{C}_{r}\}) = \sum_{i = 1}^r \gamma(\mathcal{C}_{i}). $$
From this example it becomes clear that a pair $(\oplus, \gamma)$ consisting of a permutation-invariant set-function $\oplus$ and a substructure-embedding function $\gamma$ is sufficient to fully determine an associated substructure-pooling method $\Psi = \oplus \circ \Gamma_{\gamma}$. Note that just like in the case of GNN pooling, the operator $\oplus$ could correspond to simple summation or averaging, or could be modelled by a more complex trainable deep network as is the case for the differentiable graph pooling method proposed by~\citet{navarin2019universal}.

We now take a look at two natural example techniques for the embedding of substructures.

\begin{example}[One-Hot Embedding] \label{ex: one_hot_emb}
Perhaps the simplest possible substructure embedding is given via \textit{one-hot encoding}. Denote with $u_{m,i} \in \mathbb{R}^m$ the $m$-dimensional unit vector which is equal to $1$ only in its $i$-th component and equal to $0$ everywhere else. Furthermore, let 
$$s: \mathfrak{C} \to \{1,...,m\} $$
be a bijective function. Note that $s$ imposes a linear order on $\mathfrak{C}$ by assigning a unique rank $s(\mathcal{C}) \in \{1,...,m\}$ to each substructure $\mathcal{C} \in \mathfrak{C}$. Then the one-hot encoded substructure embedding associated with $s$ is simply given by
$$\gamma_s : \mathfrak{C} \to \mathbb{R}^{m}, \quad \gamma_s(\mathcal{C}) = u_{m, s(\mathcal{C})}. $$
We see that in the case of one-hot encoding, the embedding dimension $w$ is equal to the number of substructures $\vert \mathfrak{C} \vert = m$ and can thus be extremely large.
\end{example}

\begin{example}[Physicochemical Embedding]

Another natural way to embed substructures is given via physicochemical-descriptor vectors (PDVs) as described in Section~\ref{sec: pdvs}. For example, the $200$ descriptors specified in Table~\ref{tab: physchem_descriptors_rdkit} can be computed for each substructure, leading to a meaningful embedding function
$$\gamma_{\text{PDV}} : \mathfrak{C} \to \mathbb{R}^{200}. $$
Some advantages of this embedding over one-hot encoding are its high \textit{a priori} content of interpretable and potentially useful chemical information and the fact that similar substructures are likely to end up close in the embedding space. A potential disadvantage of physicochemical embeddings over one-hot embeddings might be that certain frequently-used permutation-invariant set functions such as the summation or averaging operator potentially incur a considerable loss of information when applied to sets of PDVs. In contrast, for sets of one-hot encoded vectors, summation and averaging are invertible operations that do not incur any loss of information. To see this, note that every positive entry in a vector that represents the sum or average of a set of one-hot encoded vectors corresponds to exactly one particular one-hot encoded vector in the original set.\footnote{In the case of ECFPs with counts, detected substructures in a compound would be encoded as multisets of one-hot encoded vectors instead of sets. Summing a multiset of one-hot encoded vectors is still invertible, but averaging is not. For instance, the multiset $\{(1,0),(1,0),(0,1),(0,1)\}_{\text{mul}}$ can easily be uniquely reconstructed from its sum $(2,2)$, but not from its average $(1/2,1/2)$, since the multiset $\{(1,0),(0,1)\}_{\text{mul}}$ corresponds to the same average.}
\end{example}

\subsection{Investigated Substructure-Pooling Techniques}

Note that substructure pooling as given in Definition~\ref{def: substructure_pooling} is a highly general operation that encompasses hashing, feature selection and more complex methods based on combining substructure-embeddings $\gamma$ with differentiable permutation-invariant set functions $\oplus$. In spite of this, substructure-pooling beyond hashing remains largely unexplored. Below we go on to describe four substructure-pooling methods for ECFPs that we chose to experimentally investigate: the canonically used hashing procedure, the Sort \& Slice method which is the main focus of this study, and two technically mature supervised feature selection schemes. A high-level overview of the four evaluated substructure-pooling techniques can be found in Figure~\ref{fig:substructurepoolingmethodsforecfps}.

The substructure-pooling methods we describe in this section can in principle be used with any structural fingerprint; however, in this study we focus on ECFPs. We therefore assume from now on that the set of chemical substructures under consideration $\mathfrak{C}$ consists of \textit{circular} substructures up to a predefined radius $R$ represented via a set of integer identifiers in a large hash space:
$$\mathfrak{C} = \{\mathcal{J}_1,...,\mathcal{J}_m \} \subseteq \{1,...,2^{32}\}. $$

\begin{remark}[Potential Hash Collisions during ECFP Substructure Enumeration]
Note that strictly speaking there is not always a perfect one-to-one correspondence between circular chemical substructures and integer ECFP identifiers in the hash space $\{1,...,2^{32}\}$; theoretically hash collisions could occur during the ECFP substructure enumeration process (within one compound or across distinct compounds) that could lead to two circular fragments being assigned the same integer identifier $\mathcal{J}_i$. However, hash collisions of this type are very rare~\citep{gutlein2016filtered, rogers2010extended}. For simplicity we therefore assume that each integer identifier $\mathcal{J}_i \in \mathfrak{C}$ can indeed always be mapped to a unique circular chemical fragment.
\end{remark}

\begin{figure}[H]
	\centering
	\includegraphics[width=0.97\linewidth]{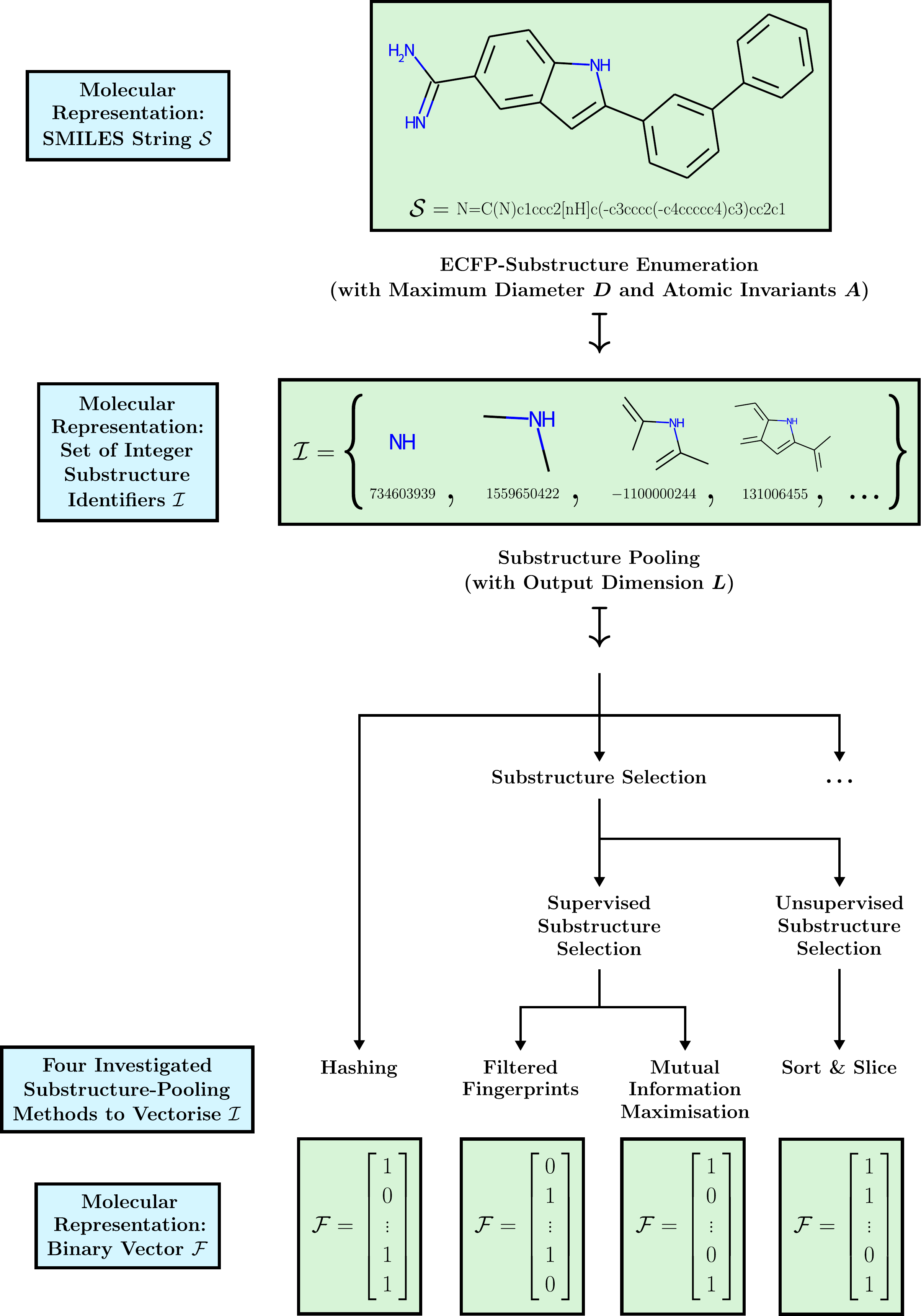}
	\caption[Overview of investigated substructure-pooling methods.]{Schematic overview of the four investigated substructure-pooling methods for the vectorisation of ECFPs.}
	\label{fig:substructurepoolingmethodsforecfps}
\end{figure}

We imagine that the described substructure-pooling methods are used in the context of a supervised molecular property-prediction task specified via a set of $n$ training compounds
$$\mathfrak{D} = \{\mathcal{R}_1, ..., \mathcal{R}_n\} \subset \mathfrak{R} $$
that are elements of some larger chemical space $\mathfrak{R}$ and that are given via their associated SMILES strings. We also assume that we are given a labelling function
$$c : \mathfrak{D} \to \mathbb{R}$$
that assigns regression or classification labels to the training compounds. The ECFP algorithm is denoted via
$$\varphi : \mathfrak{R} \to P(\mathfrak{C}) $$
and turns SMILES strings $\mathcal{R} \in \mathfrak{R}$ into sets of integer identifiers 
$$ \varphi(\mathcal{R}) = \{ \mathcal{J}_1, ..., \mathcal{J}_r \} \subseteq \mathfrak{C} $$ that correspond to circular chemical substructures that appear in the input compound~$\mathcal{R}$. For each substructure $\mathcal{J} \in \mathfrak{C}$, the set of all training compounds that contain $\mathcal{J}$ is called the support of $\mathcal{J}$ and is denoted via
$$ \text{supp}(\mathcal{J}) \coloneqq \{ \mathcal{R} \in \mathfrak{D} \ \vert \ \mathcal{J} \in \varphi(\mathcal{R}) \}.$$
The set
$$\mathfrak{C}_{\mathfrak{D}} \coloneqq \bigcup_{\mathcal{R} \in \mathfrak{D}} \varphi(\mathcal{R}) $$
contains all integer identifiers in the entire training set, i.e.~all circular substructures that appear in any of the $n$ training compounds.

\subsubsection{Hashing}

The canonical way for substructure pooling~\citep{rogers2010extended} is via a deterministic hash function
$$\tilde{h} : \{1, ..., 2^{32}\} \to \{1, ..., l\}$$
that compresses the set of integer identifiers in $\mathfrak{C}$ into a much smaller set of integers $\{1, ..., l\}$ whose cardinality corresponds to the desired fingerprint dimension~$l$. Formally, one can employ $h$ to define a substructure-pooling method via
$$\Psi :  P(\mathfrak{C}) \to \mathbb{R}^l, \quad \Psi(\{\mathcal{J}_1,...,\mathcal{J}_r\})_i = \left\{
\begin{array}{ll}
1 & \quad \exists \ k \in \{1,...,r\}: h(\mathcal{J}_k) = i , \\
0 & \quad \text{else}. \\
\end{array}
\right. $$
This means that the map
$$\Psi \circ \varphi : \mathfrak{R} \to \mathbb{R}^l $$
transforms a chemical compound $\mathcal{R} \in \mathfrak{R}$ into an $l$-dimensional binary vectors whose $i$-th component is $1$ if and only if (at least) one of the substructures in $\varphi(\mathcal{R}) = \{\mathcal{J}_1,...,\mathcal{J}_r\}$ gets hashed to the integer $i$. If $l$ gets small then hash collisions start to occur that can for instance lead to distinct substructures in $\mathfrak{C}_{\mathfrak{D}}$ being mapped by $h$ to the same component of the fingerprint. This degrades its interpretability and predictive performance. Note that hash-based substructure-pooling is independent of the training set $\mathfrak{D}$ and the labelling function $c$.

\subsubsection{Sort \& Slice} 

Let
$$f : \mathfrak{C} \to \{0,...,n\}, \quad f(\mathcal{J}) = \vert \text{supp}(\mathcal{J}) \vert =\vert \{ \mathcal{R} \in \mathfrak{D} \ \vert \ \mathcal{J} \in \varphi(\mathcal{R}) \} \vert, $$
be a function that maps every substructure-identifier $\mathcal{J} \in \mathfrak{C}$ to the number of training compounds in $\mathfrak{D} = \{\mathcal{R}_1, ..., \mathcal{R}_n\}$ which contain~$\mathcal{J}$ (i.e.~to its frequency in the training set). We can use $f$ to define a linear order $\prec$ on the set of all substructures~$\mathfrak{C}$ via
$$ \mathcal{J} \prec \tilde{\mathcal{J}} \iff f(\mathcal{J}) < f(\tilde{\mathcal{J}}) \ \text{or} \ [ f(\mathcal{J}) = f(\tilde{\mathcal{J}}) \ \text{and} \ \mathcal{J} < \tilde{\mathcal{J}} ].$$
We see that the order defined by $\prec$ considers a substructure larger than another substructure if it appears in more training compounds. If two substructures appear in the same number of training compounds, ties are broken using the (arbitrary) ordering defined by the integer identifiers themselves. The relation $\prec$ defines a total order on $\mathfrak{C}$ which means that each substructure $\mathcal{J} \in \mathfrak{C}$ can be assigned a unique rank with respect to $\prec$. Let
$$s : \mathfrak{C} \to \{1,...,m\} $$
be a bijective \textit{sorting function} that assigns the ranks determined by $\prec$. Here rank $1$ is assigned to the largest substructure with respect to $\prec$. If there are no ties then $s(\mathcal{J}) = 1$ implies that $\mathcal{J}$ appears in more training compounds than any other substructure in~$\mathfrak{C}$.
Now let
$$\gamma_s : \mathfrak{C} \to \mathbb{R}^{m}, \quad \gamma_s(\mathcal{J}) = u_{m, s(\mathcal{J})} $$
be the one-hot embedding associated with $s$ (note that one-hot embeddings are described in Example~\ref{ex: one_hot_emb}). Furthermore, let $m_{\mathfrak{D}} \coloneqq \vert \mathfrak{C}_{\mathfrak{D}} \vert $ be the total number of substructures that appear in the training set. Based on $m_{\mathfrak{D}}$ and the desired fingerprint length $l$ we define a \textit{slicing function}
$$\eta_{m_{\mathfrak{D}},l} : \mathbb{R}^m \to \mathbb{R}^l, \quad \eta_{m_{\mathfrak{D}},l}(v_1,...,v_m) = \left\{
\begin{array}{ll}
(v_1,...,v_l) & \quad m_{\mathfrak{D}} \geq l, \\
(v_1,...,v_{m_{\mathfrak{D}}},0,...,0) & \quad m_{\mathfrak{D}} < l. \\
\end{array}
\right.  $$
Then the \textit{Sort \& Slice} substructure-pooling operator is defined via
$$\Psi :  P(\mathfrak{C}) \to \mathbb{R}^l, \quad \Psi(\{\mathcal{J}_1,...,\mathcal{J}_r\}) = \eta_{m_{\mathfrak{D}},l} \Big( \sum_{i = 1}^r \gamma_s(\mathcal{J}_i) \Big).$$
The sum expression
$$\sum_{i = 1}^r \gamma_s(\mathcal{J}_i) \in \{0,1\}^m $$
is equal to a very long binary vector; each component of this vector indicates the presence or absence of a particular circular substructure in $\mathfrak{C}$ in the input compound. The substructures appear according to the frequency by which they occur in the training set such that substructures that occur in more training compounds appear earlier in the vector.

If $m_{\mathfrak{D}} \geq l$, which is usually the case, then the function $\eta_{m_{\mathfrak{D}},l}$ slices away the less frequent substructures from the vector to produce a final representation with the desired dimension $l$. In the unusual case where $l$ is set to be larger than the total number of substructures in the training set $m_{\mathfrak{D}}$, then all training substructures are contained in the fingerprint and it is simply padded with additional $0$s at the end to reach length $l$.

In simple terms, the map
$$\Psi \circ \varphi : \mathfrak{R} \to \mathbb{R}^l $$
outputs binary fingerprints of length $l$ that indicate the presence or absence of the $l$ substructures that appear most frequently in the training set $\mathfrak{D}$. In particular, note that these fingerprints do not exhibit hash collisions; each vectorial component can be assigned to a unique substructure. This clarity comes at the price of losing the information contained in the less frequent substructures in $\mathfrak{C}_{\mathfrak{D}}$ that are sliced away. Note however that in real-world chemical data sets a vast portion of substructures that exist in the training set only occur in a few compounds. In a machine-learning context, slicing away such highly infrequent substructures should thus lead to minimal loss of information since it corresponds to removing almost-constant features (i.e.~almost-constant columns the the training-set feature matrix). The Sort \& Slice operator $\Psi$ is dependent on the training set $\mathfrak{D}$ but not on the training labels $c$. It can be interpreted as a simple unsupervised feature selection technique that selects substructures according to the frequency with which they appear in the training set.

In our literature review we discovered that~\citet{macdougall2022reduced} proposed a version of the ECFP that closely resembles the ECFP generated via Sort \& Slice. The main difference between the substructure-pooling method proposed by MacDougall and ours appears to be in the slicing operation. MacDougall proposes to arrange substructures in discrete packages according to the level sets of $f$:
$$\mathfrak{C}_{j,f} \coloneqq \{\mathcal{J} \in \mathfrak{C} \ \vert \ f(\mathcal{J}) = j \}. $$
It is clear that the sets $\mathfrak{C}_{n,f}, ..., \mathfrak{C}_{0,f}$ form a partition of $\mathfrak{C}$. MacDougall only absorbs substructures in $\mathfrak{C}_{n,f}, ..., \mathfrak{C}_{n - k,f}$ into the final fingerprint whereby $k$ is a tunable \textit{slicing parameter}. In this setting, tuning $k$ gives limited control over the fingerprint dimension as going from $k$ to $k + 1$ leads to a discrete jump in fingerprint length of magnitude $\vert\mathfrak{C}_{n - (k+1),f} \vert $. This coarse-graining of the fingerprint length might be motivated by an effort to avoid dealing with ties when sorting substructures according to frequency. In contrast, our Sort \& Slice technique solves the problem of ties by introducing a second (arbitrary) layer of ordering based on the magnitude of the integer identifiers. This allows us to assign a unique rank to every substructure $\mathcal{J} \in \mathfrak{C}$. As a consequence, our framework easily allows full control over the fingerprint length~$l$ as the substructures in $\mathfrak{C}$ can be sorted in an unambiguous way. Subsequently, all but the top $l$ substructures can be sliced away for any desired $l$.

\begin{remark}[Information-Theoretic Interpretation of Sort \& Slice] \label{remark: sort_and_slice_info}

Let us assume that no substructure appears in more than half of the training compounds:
$$\max_{\mathcal{J} \in \mathfrak{C}_{\mathfrak{D}}} f(\mathcal{J}) \leq \frac{n}{2}.  $$
In other words, we assume that
$$\forall \mathcal{J} \in \mathfrak{C}_{\mathfrak{D}}: \quad p(\mathcal{J}) \coloneqq \frac{f(\mathcal{J})}{n} \leq \frac{1}{2}. $$
Here $p(\mathcal{J}) \in (0,1]$ is the probability to find substructure $\mathcal{J} \in \mathfrak{C}_{\mathfrak{D}}$ in a training compound that was chosen uniformly at random from $\mathfrak{D}$. It is also an empirical estimate of the probability to find substructure $\mathcal{J}$ in a compound that was sampled in the same way as the training compounds which are assumed to have been generated via independent and identically distributed draws from the larger chemical space $\mathfrak{R}$. In real-world data sets the assumption that $p(\mathcal{J}) \leq 1/2$ holds approximately true since usually only a very small fraction of all circular substructures in $\mathfrak{C}_{\mathfrak{D}}$ actually appear in the majority of compounds; in fact, in real-word data sets almost all substructures in $\mathcal{J} \in \mathfrak{C}_{\mathfrak{D}}$ tend to appear in only a few compounds. If indeed no substructure appears in the majority of training compounds, the ordering imposed on $\mathfrak{C}_{\mathfrak{D}}$ by the frequency-function $f$ is equivalent to the ordering imposed by the empirical \textit{information entropy} $H \circ p$ of a substructure with respect to $\mathfrak{D}$. The binary information entropy function $H$~\citep{shannon1948mathematical} is defined via
$$H : [0,1] \to [0,1], \quad H(p) = -p \log_2(p) - (1-p) \log_2(1-p) $$
with $0*\log_2(0) \coloneqq 0$.
We naturally define the empirical information entropy of a substructure in the training set~$\mathfrak{D}$ via
$$H \circ p: \mathfrak{C}_{\mathfrak{D}} \to [0,1], \quad (H \circ p)(\mathcal{J})  = H(p(\mathcal{J})).   $$
The quantity $H(p(\mathcal{J}))$ peaks at $p(\mathcal{J}) = 1/2$ and provides an empirical measure for how informative the substructure $\mathcal{J}$ is when used as a binary feature in a fingerprint. It represents a simple plug-in entropy estimator based on the empirical probability estimate $p(\mathcal{J})$ to find substructure $\mathcal{J}$ in a compound. $H(p)$ is strictly increasing for $p \in [0, 1/2]$ and is strictly decreasing for $p \in [1/2,1]$. To see that $f$ and $H \circ p$ induce the exact same ordering on $\mathfrak{C}_{\mathfrak{D}}$ note the following two facts:
\begin{align*}
&\forall \mathcal{J}, \tilde{\mathcal{J}} \in \mathfrak{C}_{\mathfrak{D}} : \quad f(\mathcal{J}) < f(\tilde{\mathcal{J}}) \iff p(\mathcal{J}) < p(\tilde{\mathcal{J}}) \iff (H \circ p)(\mathcal{J}) < (H \circ p)(\tilde{\mathcal{J}}), \\
&\forall \mathcal{J}, \tilde{\mathcal{J}} \in \mathfrak{C}_{\mathfrak{D}} : \quad f(\mathcal{J}) = f(\tilde{\mathcal{J}}) \iff p(\mathcal{J}) = p(\tilde{\mathcal{J}}) \iff (H \circ p)(\mathcal{J}) = (H \circ p)(\tilde{\mathcal{J}}).
\end{align*}
Here in each of both rows the first equivalence is trivial and the second equivalence holds due to our the strict monotonicity of $H$ on $[0,1/2]$ and our current assumption that $p(\mathcal{J}), p(\tilde{\mathcal{J}}) \in [0,1/2]$.

This argument shows that in realistic settings (i.e.~in settings where almost all substructures appear in no more than half of the training compounds) Sort \& Slice tends to automatically lead to a fingerprint that only contains the most informative substructures from an entropic point of view, all while being much simpler to theoretically understand, implement and interpret than an approach explicitly built on empirical information entropy.

\end{remark}

\subsubsection{Filtering}

In this section we describe the substructure-pooling technique proposed by~\citet{gutlein2016filtered} referred to as \textit{filtering}. Their original method was published in \texttt{Java}; we used the technical description from their article to create a reimplementation in \texttt{Python} that integrates with the rest of our code base. 

Let us first assume that we are given a binary molecular classification problem, i.e.~we assume that our labelling function is given by
$$c : \mathfrak{D} \to \{0,1\}. $$
If instead the initial labels specified by $c$ correspond to a continuous regression problem with labels in $\mathbb{R}$, then we set all labels below or above the label median to $0$ or $1$ respectively, to still arrive at a binary classification problem of the above form.

Now let $R$ be a random compound that was drawn from~$\mathfrak{R}$ according to some probability distribution. Moreover, we imagine that the given training set 
$$\mathfrak{D} = \{ \mathcal{R}_1,..., \mathcal{R}_n \} $$
represents a statistical sample of $n$ independent and identically distributed draws of $R$ from $\mathfrak{R}$. Then the available training labels
$$ \hat{c} \coloneqq (c(\mathcal{R}_1),..., c(\mathcal{R}_n)) \in \{0,1\}^n $$
form a statistical sample of size $n$ for the random labels of $R$. Furthermore, for each substructure $\mathcal{J} \in \mathfrak{C}$, the expression
$$g_{\mathcal{J}}(R) = \left\{
\begin{array}{ll}
1 & \quad \mathcal{J} \in \varphi(R), \\
0 & \quad \text{else}, \\
\end{array}
\right.  $$
forms a random variable that is equal to $1$ if and only if substructure $\mathcal{J}$ is contained in $R$. The binary sequence
$$ \hat{g}_{\mathcal{J}} \coloneqq  (g_{\mathcal{J}}(\mathcal{R}_1),...,g_{\mathcal{J}}(\mathcal{R}_n))\in \{0,1\}^n$$
represents a statistical sample of size $n$ for $g_{\mathcal{J}}(R)$. 

We now define a function
$$ f : \mathfrak{C} \to [0,1], \quad f(\mathcal{J}) =  p_{\chi^2}(\hat{c}, \hat{g}_{\mathcal{J}}),$$
that assigns to each substructure $\mathcal{J} \in \mathfrak{C}$ its $p$-value in a statistical $\chi^2$ independence-test~\citep{pearson1900chi2} between $\hat{c}$ and $\hat{g}_{\mathcal{J}}$. As is the case for Sort \& Slice, this function allows one to define a total order on $\mathfrak{C}$ via
$$ \mathcal{J} \prec \tilde{\mathcal{J}} \iff f(\mathcal{J}) > f(\tilde{\mathcal{J}}) \ \text{or} \ [ f(\mathcal{J}) = f(\tilde{\mathcal{J}}) \ \text{and} \ \mathcal{J} > \tilde{\mathcal{J}} ].$$
The larger the $p$-value, the smaller the substructure according to $\prec$. We now reduce the number of substructures we consider in our fingerprint to the desired dimension $l$ via the following scheme:

\begin{itemize}
	
	\item \textbf{Step 0:} The set of selected substructures is initialised via $\mathfrak{C}_l \coloneqq \mathfrak{C}.$
	
	\item \textbf{Step 1:}. A substructure $\mathcal{J} \in \mathfrak{C}_l$ that fulfills $\vert \text{supp}(\mathcal{J}) \vert \leq 1$ is randomly chosen and removed from $\mathfrak{C}_l$. This is repeated until all substructures in $\mathfrak{C}_l$ appear in at least two training compounds or until $\vert \mathfrak{C}_l \vert = l$.
	
	\item \textbf{Step 2:} A substructure $\mathcal{J} \in \mathfrak{C}_l$ that is \textit{non-closed} is randomly chosen and removed from $\mathfrak{C}_l$. This is repeated until all remaining substructures in $\mathfrak{C}_l$ are closed or until $\vert \mathfrak{C}_l \vert = l$. Note that a substructure $\mathcal{J} \in \mathfrak{C}_l$ is called \textit{non-closed} if there exists another substructure $\tilde{\mathcal{J}} \in \mathfrak{C}_l$ such that $\text{supp}(\mathcal{J}) = \text{supp}(\tilde{\mathcal{J}})$ and $\mathcal{J}$ contains a proper subgraph that is isomorphic to $\tilde{\mathcal{J}}$.
	
	\item \textbf{Step 3:} The smallest element of $\mathfrak{C}_l$ with respect to the linear order $\prec$ is chosen and removed. This is repeated until $\vert \mathfrak{C}_l \vert = l$.
\end{itemize}
Step $1$ is performed to remove almost-constant substructural features that contain little information. Step $2$ represents a graph-theoretic attempt to reduce feature redundancy via the removal of substructures that contain smaller substructures that match the exact same set of training compounds. Finally, Step $3$ is performed to select the $l$ substructures that show the strongest statistical dependence on the training label as quantified by a $\chi^2$-test. Using the selection of substructures $\mathfrak{C}_l$ one can construct a one-hot embedding (see Example~\ref{ex: one_hot_emb})
$$\gamma_s : \mathfrak{C} \to \mathbb{R}^{l}, \quad \gamma_s(\mathcal{J}) = \left\{
\begin{array}{ll}
u_{l, s(\mathcal{J})} & \quad \mathcal{J} \in \mathfrak{C}_l, \\
0 & \quad \text{else}, \\
\end{array}
\right.  $$
whereby
$$s : \mathfrak{C}_l \to \{1,...,l\} $$
is some arbitrary bijective sorting function. Substructure pooling by means of filtered fingerprints can now be described via:
$$\Psi :  P(\mathfrak{C}) \to \mathbb{R}^l, \quad \Psi(\{\mathcal{J}_1,...,\mathcal{J}_r\}) = \sum_{i = 1}^r \gamma_s(\mathcal{J}_i).$$
The map
$$\Psi \circ \varphi : \mathfrak{R} \to \mathbb{R}^l $$
transforms chemical compounds into hash collision-free binary fingerprints that only indicate the presence or absence of substructures in~$\mathfrak{C}_l$. Note again that $\mathfrak{C}_l$ contains the $l$ substructures that exhibit the lowest $p$-values in a $\chi^2$-test with respect to the training label; these substructures thus have a comparatively high statistical dependence with the target variable and might therefore be useful features for a machine-learning system. $\Psi$ depends on both the training compounds in $\mathfrak{D}$ and the training labels~$c$. Filtered fingerprints form a type of supervised feature selection scheme.

\subsubsection{Mutual-Information Maximisation}

In this section we continue to assume the same setting as in the previous section, i.e.~we assume that our labelling function is binary (or has been binarised),
$$c : \mathfrak{D} \to \{0,1\}, $$
and we consider the binary sequences $\hat{c} \in \{0,1\}^n$ and $\hat{g}_{\mathcal{J}} \in \{0,1\}^n$ for $\mathcal{J} \in \mathfrak{C}$ to be statistical samples of size $n$ of two random variables that describe whether or not the binary label of a randomly chosen compound is positive and whether or not the compound contains substructure $\mathcal{J}$. Based on these samples derived from our training set, we compute the empirical \textit{mutual information} $I$~\citep{shannon1948mathematical,cover1991entropy} between substructure $\mathcal{J}$ and the training label via
$$I(\hat{c}, \hat{g}_{\mathcal{J}}) = H(\hat{c}) + H(\hat{g}_{\mathcal{J}}) - H(\hat{c}, \hat{g}_{\mathcal{J}}).$$
Here $H$ denotes an empirical estimate of the information entropy of a random variable based on a statistical sample. $H$ could be implemented using a variety of strategies for entropy estimation; since we are dealing with the relatively easy case of discrete binary variables, we implement $H$ as the simple plug-in entropy estimator based on the relative frequencies of binary outcomes that was already used in Remark~\ref{remark: sort_and_slice_info}. $I(\hat{c}, \hat{g}_{\mathcal{J}})$ is a nonnegative, symmetric and nonlinear measure of the statistical dependence between $\hat{c}$ and $\hat{g}_{\mathcal{J}}$. The larger $I(\hat{c}, \hat{g}_{\mathcal{J}})$, the more information the presence of substructure $\mathcal{J}$ in a compound conveys about the value of its training label and \textit{vice versa}.

We now define a function
$$ f : \mathfrak{C} \to \left[0, \infty\right), \quad f(\mathcal{J}) =  I(\hat{c}, \hat{g}_{\mathcal{J}}),$$
that assigns to each substructure $\mathcal{J} \in \mathfrak{C}$ its empirical mutual information with the training label. Once again this function allows one to define a total order on $\mathfrak{C}$ via
$$ \mathcal{J} \prec \tilde{\mathcal{J}} \iff f(\mathcal{J}) < f(\tilde{\mathcal{J}}) \ \text{or} \ [ f(\mathcal{J}) = f(\tilde{\mathcal{J}}) \ \text{and} \ \mathcal{J} < \tilde{\mathcal{J}} ].$$
We go on to reduce the number of substructures we consider in our fingerprint to the desired length $l$ via the following scheme:

\begin{itemize}
	
	\item \textbf{Step 0:} The set of substructures is initialised via $\mathfrak{C}_l \coloneqq \mathfrak{C}.$
	
	\item \textbf{Step 1:} If two substructures $\mathcal{J}, \tilde{\mathcal{J}} \in \mathfrak{C}_l$ appear in the exact same set of training compounds, i.e.~if $\text{supp}(\mathcal{J}) = \text{supp}(\tilde{\mathcal{J}})$, then one of the substructures is chosen uniformly at random and removed from $\mathfrak{C}_l$. This is repeated until no two substructures have the same support or until $\vert \mathfrak{C}_l \vert = l$.
	
	\item \textbf{Step 2:} The smallest element of $\mathfrak{C}_l$ with respect to the linear order $\prec$ is chosen and removed. This is repeated until $\vert \mathfrak{C}_l \vert = l$.
\end{itemize}
Step $1$ is performed in an attempt to reduce feature redundancy via the removal of substructural features that are identical in the training set. Then, Step $2$ is performed to select only the $l$ most informative substructures with respect to the training label. Using the selection $\mathfrak{C}_l$ one can construct a one-hot embedding (see Example~\ref{ex: one_hot_emb})
$$\gamma_s : \mathfrak{C} \to \mathbb{R}^{l}, \quad \gamma_s(\mathcal{J}) = \left\{
\begin{array}{ll}
u_{l, s(\mathcal{J})} & \quad \mathcal{J} \in \mathfrak{C}_l, \\
0 & \quad \text{else}, \\
\end{array}
\right.  $$
whereby
$$s : \mathfrak{C}_l \to \{1,...,l\} $$
is some arbitrary bijective sorting function. Substructure pooling based on mutual-information maximisation (MIM) can now be described with the following operator:
$$\Psi :  P(\mathfrak{C}) \to \mathbb{R}^l, \quad \Psi(\{\mathcal{J}_1,...,\mathcal{J}_r\}) = \sum_{i = 1}^r \gamma_s(\mathcal{J}_i).$$
The map
$$\Psi \circ \varphi : \mathfrak{R} \to \mathbb{R}^l $$
transforms chemical compounds into hash-collision-free binary fingerprints that exclusively indicate the presence or absence of substructures in the tailored set~$\mathfrak{C}_l$. Remember that $\mathfrak{C}_l$ contains the $l$ substructures that exhibit the highest mutual information with the training label and should thus be highly predictive in a supervised machine-learning setting. $\Psi$ depends on the training compounds in $\mathfrak{D}$ and the training labels~$c$. It can be seen as a supervised feature selection scheme.

\subsection{Experimental Setup}

We computationally evaluated the predictive performance of the four substructure-pooling techniques introduced in the previous section (Hash, Sort \& Slice, Filter and MIM) using five molecular property prediction data sets~(see Table~\ref{tab: data_sets_substruc_pool}). The data sets were chosen to cover a diverse set of chemical regression and binary classification tasks: the prediction of lipophilicity, aqueous solubility, binding affinity, and mutagenicity. We also included a LIT-PCBA virtual screening data set~\citep{tran2020lit} that represents a highly imbalanced binary classification problem.
\begin{table}[!b]
	
	\centering
	{\renewcommand{\arraystretch}{1.5}
		\begin{tabular}{| p{5.05cm}| p{2.35cm} | p{2.58cm} | p{3.3cm} | }

			\multicolumn{1}{l}{\textbf{Prediction Task}} 
			
			& \multicolumn{1}{l}{\textbf{Task Type}} 
			
			& \multicolumn{1}{l}{\textbf{Compounds}}
			
			& \multicolumn{1}{l}{\textbf{Source}} \\
			
			\hline
			
			Lipophilicity [logD] & Regression & $4200$  & MoleculeNet~\citep{wu2018moleculenet} \\ \hline 
			
			Aqueous Solubility [logS] & Regression & $9335$  & \citet{sorkun2019aqsoldb} \\ \hline 
			
			SARS-CoV-2 Main Protease Binding Affinity [pIC\textsubscript{50}] & Regression & $1924$  & COVID Moonshot Project~\citep{achdout2020covid} \\ \hline 
			
			Ames Mutagenicity & Classification & $3496$ positives \newline $3009$ negatives & \citet{hansen2009benchmark} \\ \hline 
			
			Estrogen Receptor Alpha Antagonism & Classification & $88$ positives \newline $3833$ negatives  & LIT-PCBA~\citep{tran2020lit} \\ \hline
			
	\end{tabular}}
	
	\caption[Data sets for substructure-pooling experiments.]{Overview of the five cleaned molecular property prediction data sets used to experimentally evaluate the predictive performance of distinct substructure-pooling techniques for ECFPs.}
	
	\label{tab: data_sets_substruc_pool}
	
\end{table}
Also note that we once again experimented with the same SARS-CoV-2 main protease binding affinity data set that we already explored in Chapters~\ref{chap: qsar_ac_study}~and~\ref{chap: twin_net_ac_pred}.

All data sets were cleaned in the following manner: 
SMILES strings were algorithmically standardised and desalted using the ChEMBL structure pipeline~\citep{bento2020open}. This step also removed solvents and isotopic information. Afterwards, SMILES strings that generated error messages upon being turned into an \texttt{RDKit} mol object were deleted. Furthermore, a scan for duplicate SMILES strings was performed; if two SMILES strings were found to be identical, one of the SMILES strings was deleted uniformly at random along with its training label. Finally, we also detected rare instances where SMILES strings appeared to encode several disconnected fragments instead of one connected compound; such SMILES strings were too deleted from the data.

As a data splitting strategy, we implemented $k$-fold cross validation repeated with $m$ random seeds using $(k,m) = (2,3)$; thus each model was independently trained and tested $km = 6$-times. Performance results were recorded as the average and standard deviation over these $6$ splits, using the mean absolute error (MAE) for the three regression data sets and the area under the receiver operating characteristic curve~(AUROC) for the balanced mutagenicity classification data set. To measure performance on the LIT-PCBA estrogen receptor alpha antagonism classification data set, we used the area under the precision recall curve (AUPRC) which quantifies the tradeoff between sensitivity ($=$ recall) and precision; the AUPRC is a suitable and commonly used metric for highly imbalanced problems in which positives are of stronger natural interest than negatives. We experimented with two distinct splitting techniques within our cross-validation framework: standard uniform \textit{random} splitting and \textit{scaffold} splitting~\citep{bemis1996properties}. For the LIT-PCBA classification data set we used \textit{stratified} random splitting instead of standard random splitting in order to stabilise the small number of positives across training and test sets (we still refer to this split simply as a random split). Unlike random splitting, scaffold splitting generates a partition of a chemical data set in which the molecular scaffolds of all training-set compounds are distinct from the molecular scaffolds of all test-set compounds. This creates a distributional shift between training and test set which leads to a more challenging prediction scenario where a model is trained in one structural area of chemical space but tested in another.

As prediction algorithms we selected two standard machine-learning models: random forests (RFs) and multilayer perceptrons (MLPs). The hyperparameter choices for both models are listed in Table~\ref{tab: substruc_pool_exper_params_rf_mlp}. 
\begin{table}[!t]
	
	\centering
	{\renewcommand{\arraystretch}{1.5}
		\begin{tabular}{|p{7.15cm}| p{7.15cm} | }

			\multicolumn{2}{c}{\textbf{Machine-Learning Model Hyperparameters}}

			\\
			
			\multicolumn{1}{c}{\textbf{Random Forest}} 
			
			& \multicolumn{1}{c}{\textbf{Multilayer Perceptron}} \\
			
			\hline
			
			\footnotesize

			NumberOfTrees = $100$
			
			MaxDepth = None
			
			MinSamplesLeaf= $1$
			
			MinSamplesSplit = $2$
			
			Bootstrapping = True
			
			MaxFeatures = Sqrt
			
			Criterion (regression) = SquaredError
			
			Criterion (classification) = Gini& 
			
			\footnotesize 
			
			NumberOfHiddenLayers = $5$
			
			NeuronsPerHiddenLayer = $512$
			
			HiddenActivation = ReLU
			
			UseBiasVectors = True
			
			DropoutRateHiddenLayers = $0.25$
			
			BatchNormHiddenLayers = True
			
			BatchSize = $64$
			
			LearningRate = 1e-3
			
			LRDecayFactor = $\max\{0.98^{\text{epoch}}, \text{1e-2}\}$
			
			WeightDecayFactor = $0.1$
			
			NumberOfEpochs = $250$
			
			Optimiser = AdamW~\citep{loshchilov2017decoupled}
			
			OutputActivation (regression) = Identity
			
			Loss (regression) = MeanSquaredError
			
			OutputActivation (classification) = Sigmoid
			
			Loss (classification) = BinaryCrossEntropy \\  \hline 
			
	\end{tabular}}
	
	\caption[RF and MLP hyperparameters for substructure-pooling experiments.]{Selected hyperparameters for the two prediction models used in our substructure-pooling experiments: random forests~(RFs) and multilayer perceptrons~(MLPs).}
	
	\label{tab: substruc_pool_exper_params_rf_mlp}
	
\end{table}
All MLPs were trained on a single NVIDIA GeForce RTX 3060 GPU. In the case of RFs, the chosen hyperparameters are identical to the default ones from scikit-learn~\citep{pedregosa2011scikit} with the exception for MaxFeatures which was set to ``Sqrt" instead of $1.0$ in the case of RF regressors in order to add randomness.\footnote{The fact that the default scikit-learn setting of MaxFeatures = $1.0$ for RF regressors actually does not generate a classical random forest based on trees built via random subsets of features but rather simply a set of bagged decision trees via bootstrap aggregation was pointed out in a social media post by Greg Landrum on Twitter via @dr\_greg\_landrum on 1:57 PM, Feb 28, 2023: \textit{I assume there's a reason for it, but I really don't think it's a feature that the default parameters for a scikit-learn RandomForestRegressor don't actually build a random forest.}}

We conducted a thorough investigation of the ECFP hyperparameter space. For each data set, for each data splitting technique~(random vs.~scaffold), and for each prediction model~(RF vs.~MLP), we evaluated $24$ different ECFPs based on a complete exploration of the following grid:

\begin{itemize}
	
\item fingerprint dimension $l \in \{512, 1024, 2048, 4096\}$, \item substructure diameter $D \in \{2, 4, 6\}$,
\item atomic invariants $\in \{\text{standard (ECFP), pharmacophoric (FCFP)}\}$,
\item active tetrahedral R-S chirality flags.

\end{itemize}
Each of the $24$ ECFP versions was further combined with all four substructure-pooling methods (Hash, Sort \& Slice, Filter, MIM). This resulted in $96$ distinct vectorial ECFPs used for each combination of data set, splitting type and machine-learning model. From a bird's-eye view, the conducted experiments are organised according to a robust combinatorial methodology of the following form:

\begin{gather*}
\vert \{\text{Lipophilicity Data Set}, ..., \text{Estrogen Receptor Alpha Antagonism Data Set}\} \vert \\
 \times \\
\vert \{\text{Random Data Split},\ \text{Scaffold Data Split}\} \vert \\
\times \\
\vert \{\text{Random Forest},\ \text{Multilayer Perceptron}\} \vert \\
\times \\
\vert \{512\text{-Bit ECFP2}, ..., 4096\text{-Bit FCFP6}\} \vert \\
\times \\
\vert \{\text{Hash},\ \text{Filter},\ \text{MIM},\ \text{Sort \& Slice}\} \vert \\
= \\
5 * 2 * 2 * 24 * 4 = 1920.
\end{gather*}
Each of the $1920$ modelling scenarios resulting from this combinatorial setup was evaluated on the chosen data set via $2$-fold cross validation with $3$ random seeds as mentioned before. In total we therefore trained $1920 * 6 = 11520$ models, half of which were deep-learning models.

\section{Results and Discussion}

The detailed experimental results for each data set can be found in~\Cref{fig:moleculenetlipophilicity,fig:aqsoldbsolubility,fig:posterasarscov2mpro,fig:amesmutagenicity,fig:litpcbaesrant}. A comprehensive overview of all results is depicted in Figure~\ref{fig:boxplotswithtitles} where it becomes evident that Sort \& Slice outperforms hashing in almost all scenarios and that frequently the achieved performance gains are non-negligible. For example, a random forest trained on a random split of the AqSolDB solubility data set~\citep{sorkun2019aqsoldb} achieves a median MAE of about $1.045$ when combined with the hashed versions of the $24$ investigated ECFPs but a median MAE of about $0.998$ if the ECFPs are vectorised via Sort \& Slice instead. This corresponds to a relative improvement of the median MAE of approximately $4.5\%$.
\begin{figure}
	\centering
	\includegraphics[width=0.98\linewidth]{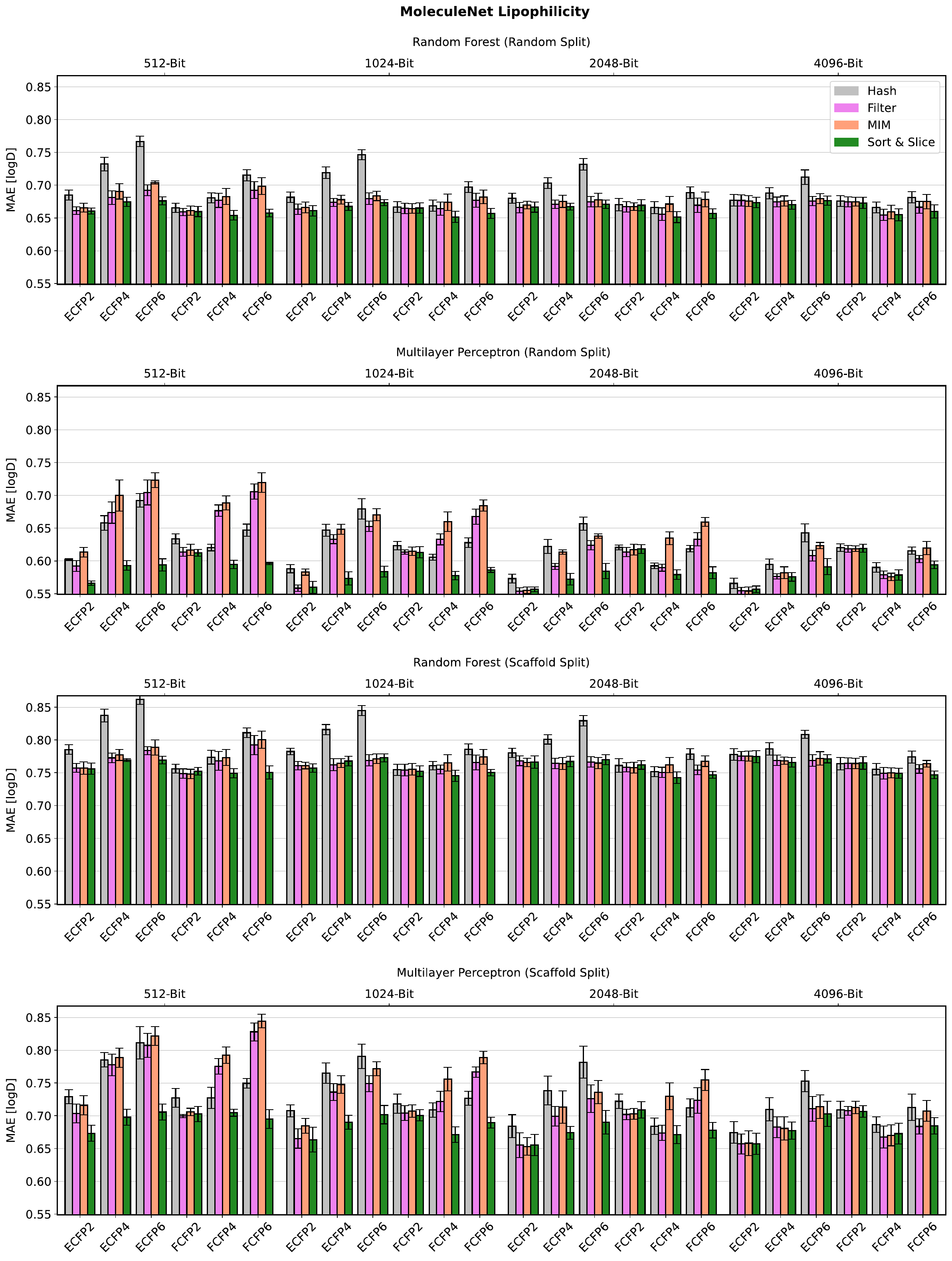}
	\caption[Substructure-pooling experiments (lipophilicity).]{Predictive performance of the four substructure-pooling methods (indicated by colours) for the lipophilicity regression data set using varying data splitting techniques, prediction models and ECFP hyperparameters. Each coloured bar shows the average mean absolute error~(MAE) of the respective model across a $k$-fold cross validation scheme repeated with $m$ random seeds for $(m,k) = (3,2)$. The length of each error bar equals twice the standard deviation of the performance measured over the $mk = 6$ trained models.}
	\label{fig:moleculenetlipophilicity}
\end{figure}
\begin{figure}
	\centering
	\includegraphics[width=0.98\linewidth]{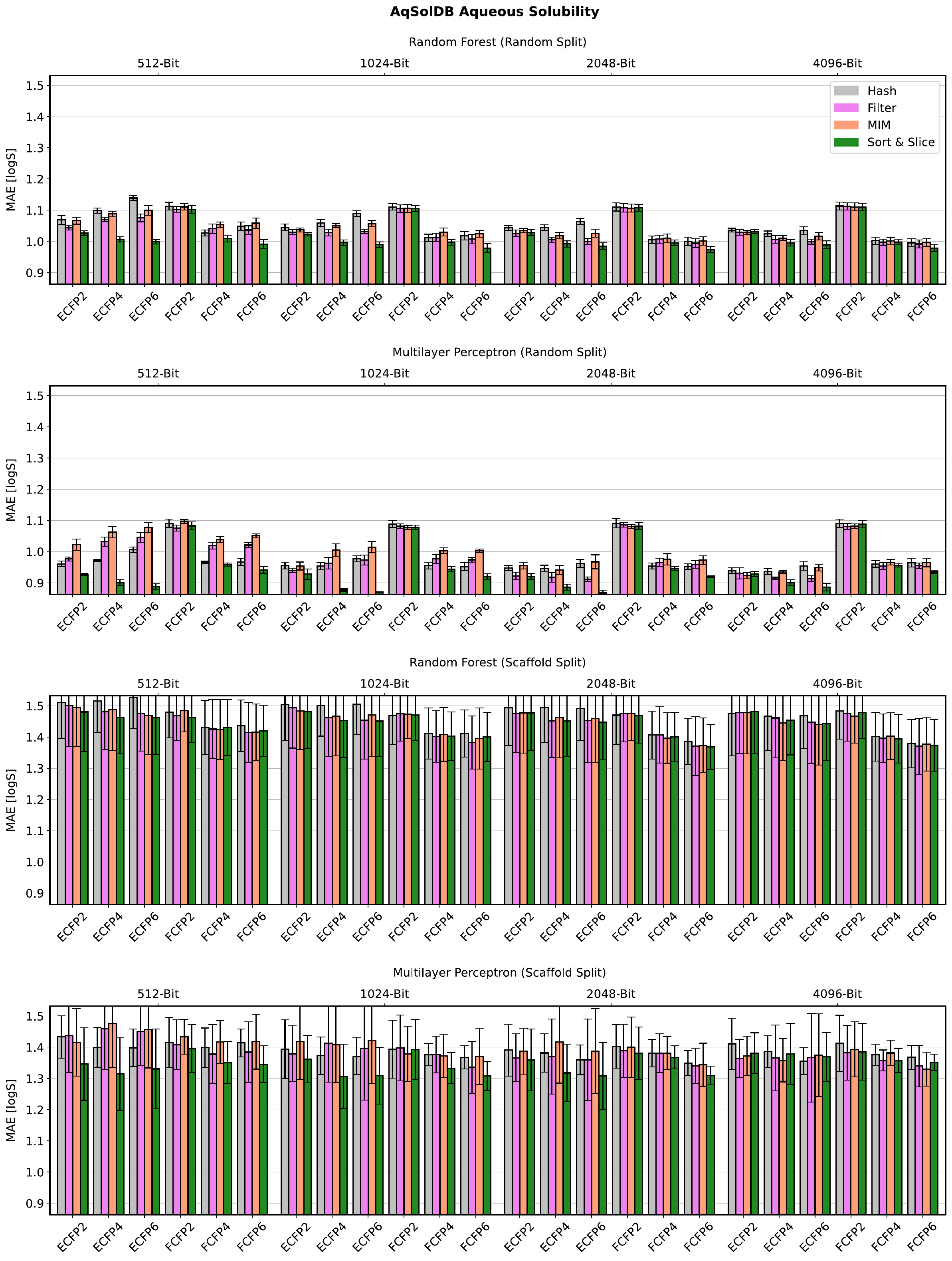}
	\caption[Substructure-pooling experiments (aqueous solubility).]{Predictive performance of the four substructure-pooling methods (indicated by colours) for the solubility regression data set using varying data splitting techniques, prediction models and ECFP hyperparameters. Each coloured bar shows the average mean absolute error~(MAE) of the respective model across a $k$-fold cross validation scheme repeated with $m$ random seeds for $(m,k) = (3,2)$. The length of each error bar equals twice the standard deviation of the performance measured over the $mk = 6$ trained models.}
	\label{fig:aqsoldbsolubility}
\end{figure}
\begin{figure}
	\centering
	\includegraphics[width=0.98\linewidth]{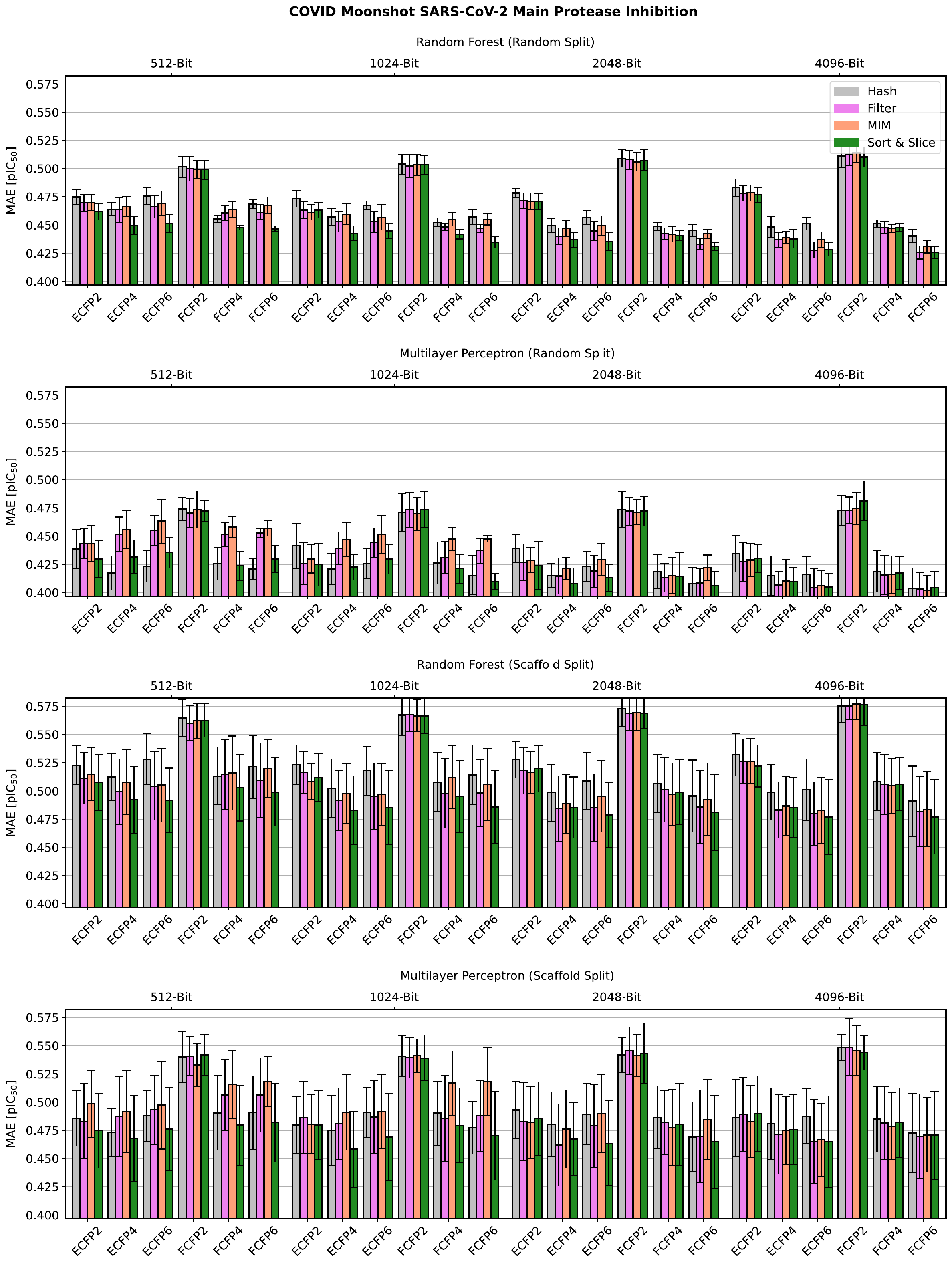}
	\caption[Substructure-pooling experiments (SARS-CoV-2 main protease).]{Predictive performance of the four substructure-pooling methods (indicated by colours) for the SARS-CoV-2 main protease binding affinity regression data set using varying data splitting techniques, prediction models and ECFP hyperparameters. Each coloured bar shows the average mean absolute error~(MAE) of the respective model across a $k$-fold cross validation scheme repeated with $m$ random seeds for $(m,k) = (3,2)$. The length of each error bar equals twice the standard deviation of the performance measured over the $mk = 6$ trained models.}
	\label{fig:posterasarscov2mpro}
\end{figure}
\begin{figure}
	\centering
	\includegraphics[width=0.98\linewidth]{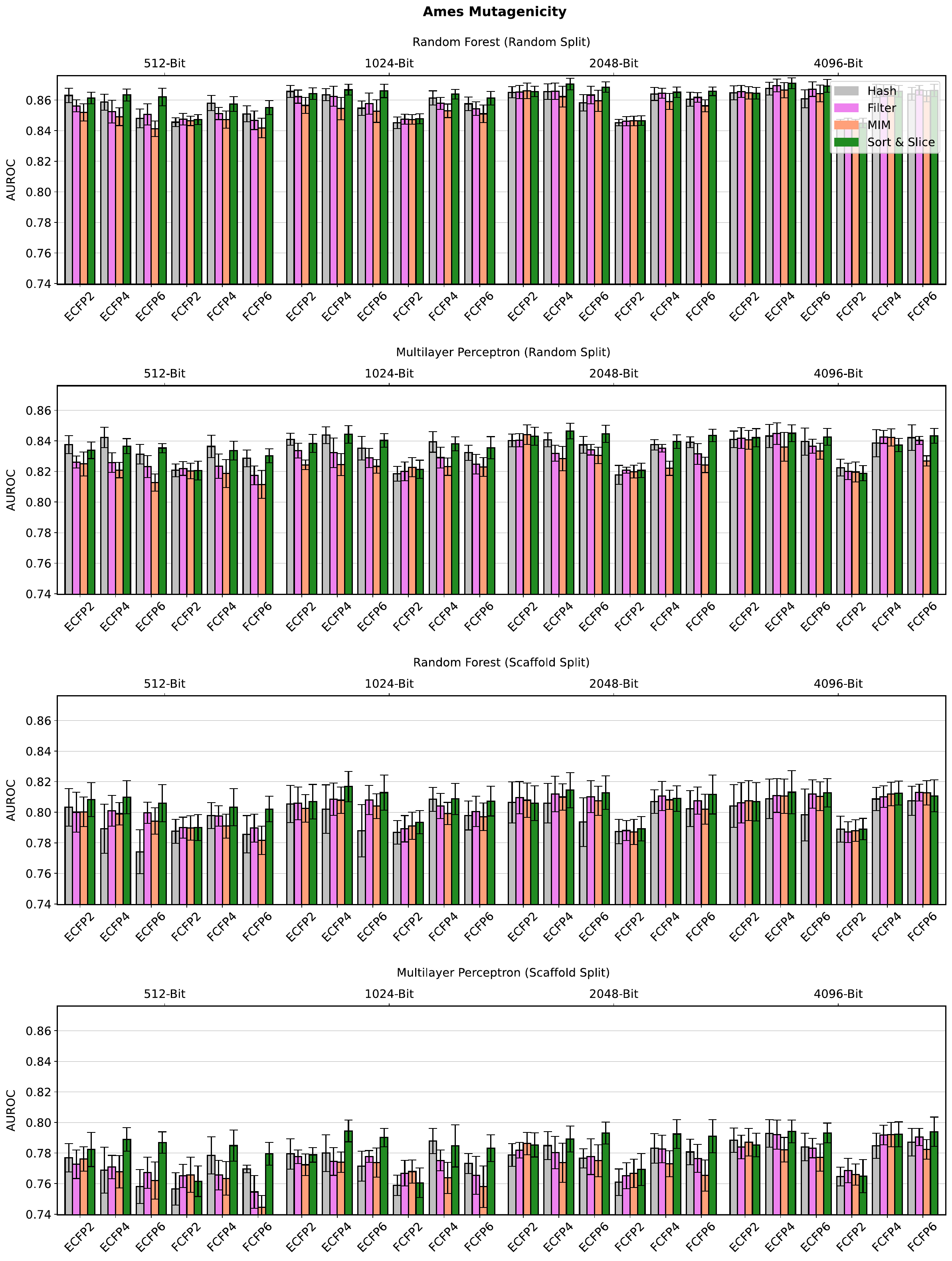}
	\caption[Substructure-pooling experiments (mutagenicity).]{Predictive performance of the four substructure-pooling methods (indicated by colours) for the balanced mutagenicity classification data set using varying data splitting techniques, prediction models and ECFP hyperparameters. Each coloured bar shows the average area under the receiver operating characteristic curve~(AUROC) of the respective model across a $k$-fold cross validation scheme repeated with $m$ random seeds for $(m,k) = (3,2)$. The length of each error bar equals twice the standard deviation of the performance measured over the $mk = 6$ trained models.}
	\label{fig:amesmutagenicity}
\end{figure}
\begin{figure}
	\centering
	\includegraphics[width=0.98\linewidth]{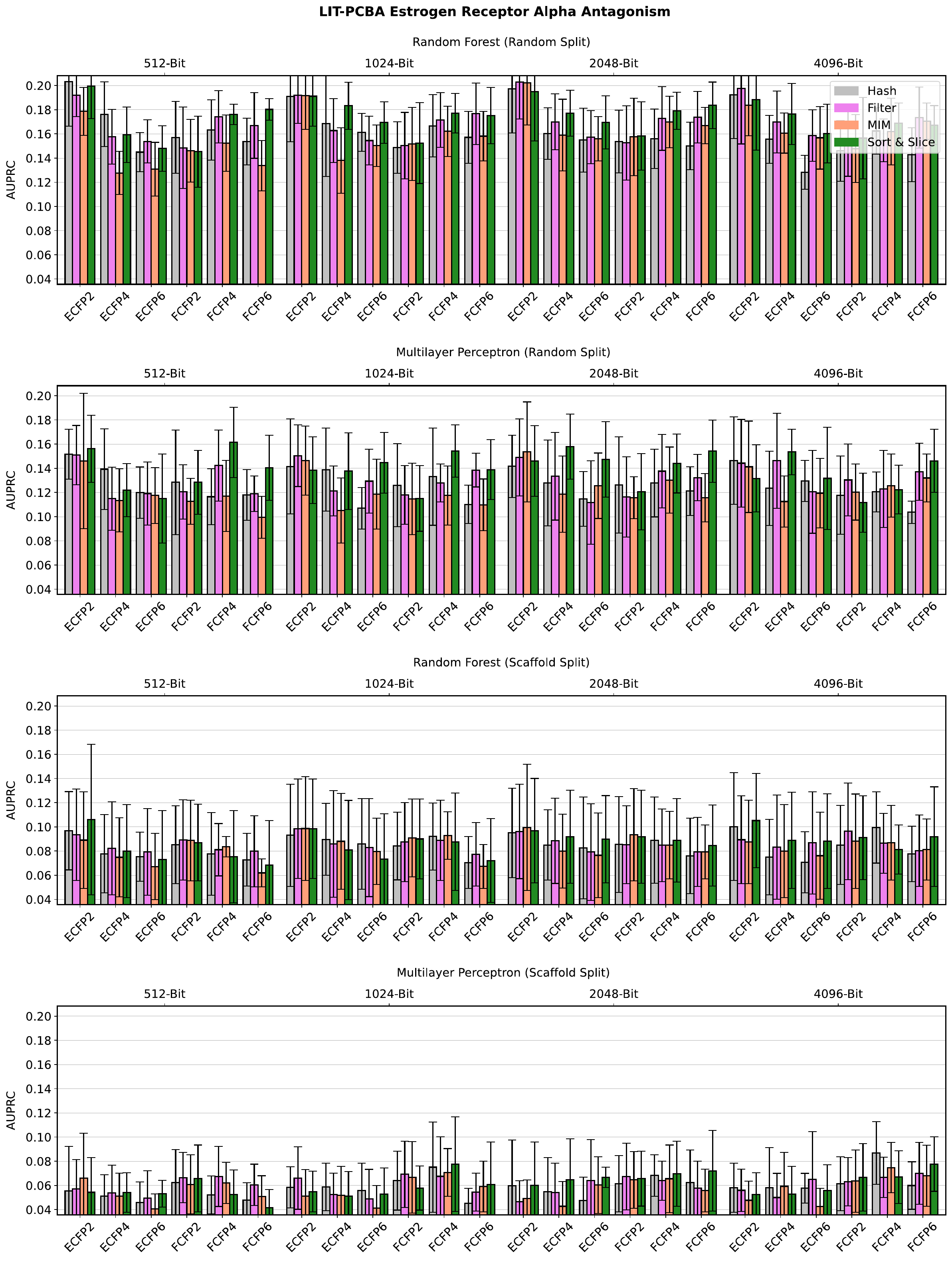}
	\caption[Substructure-pooling experiments (estrogen receptor alpha antagonism).]{Predictive performance of the four substructure-pooling methods (indicated by colours) for the imbalanced estrogen receptor alpha antagonism classification data set using varying data splitting techniques, prediction models and ECFP hyperparameters. Each coloured bar shows the average area under the precision recall curve~(AUPRC) of the respective model across a $k$-fold cross validation scheme repeated with $m$ random seeds for $(m,k) = (3,2)$. The length of each error bar equals twice the standard deviation of the performance measured over the $mk = 6$ trained models.}
	\label{fig:litpcbaesrant}
\end{figure}
\begin{figure}
	\centering
	\includegraphics[width=0.89\linewidth]{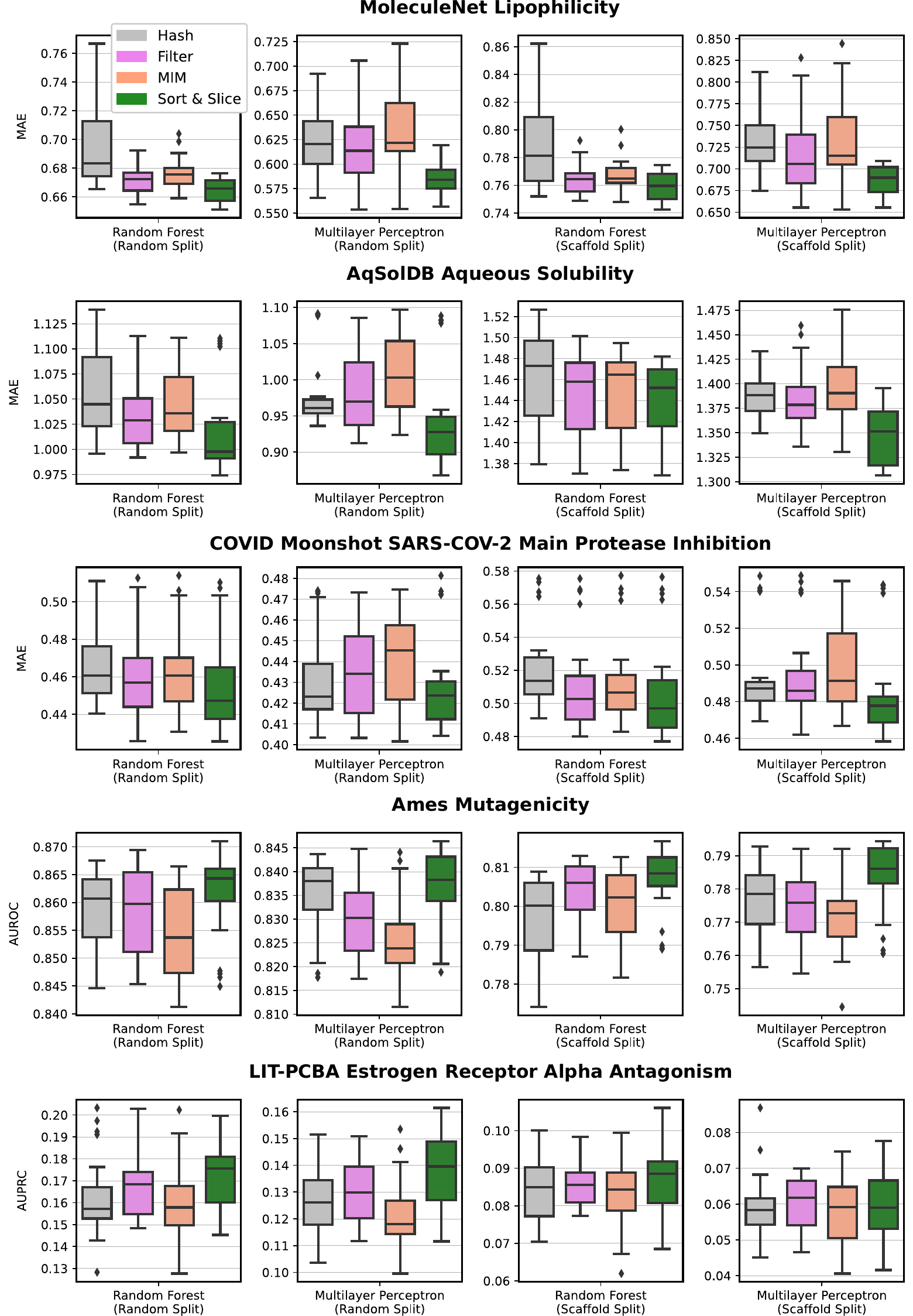}
	\caption[Overview of results of all substructure-pooling experiments.]{Overview of the predictive performance of the four investigated substructure-pooling methods (indicated by colours) across regression and classification data sets, data splitting techniques and prediction models. Each boxplot visualises the performance of a substructure-pooling method on top of $24$ distinct ECFP-types generated by combining fingerprint dimensions $l \in \{512, 1024, 2048, 4096\}$, fingerprint diameters $D \in \{2,4,6\}$ and initial atomic invariants~$\in~\{\text{standard, pharmacophoric}\}$.}
	\label{fig:boxplotswithtitles}
\end{figure}

The fine-grained results in~\Cref{fig:moleculenetlipophilicity,fig:aqsoldbsolubility,fig:posterasarscov2mpro,fig:amesmutagenicity,fig:litpcbaesrant} give more detailed insights into the relative performance of substructure-pooling techniques for various ECFP hyperparameters. In particular, this allows us to track the performance of the highly popular $1024$-bit ECFP4 which has been used as one of the common fingerprints in countless applications. We see that the Sort \& Slice version of the $1024$-bit ECFP4 surpasses the predictive performance of the hashed $1024$-bit ECFP4 in all but a few cases. For instance, Figure~\ref{fig:moleculenetlipophilicity} shows that replacing hashing with Sort \& Slice when using a $1024$-bit ECFP4 with an MLP on a random split of the lipophilicity data set leads to a rather remarkable relative MAE improvement of $11.37\%$.

The results in Figure~\ref{fig:litpcbaesrant} suggest that the advantage of Sort \& Slice over hashing remains robust even in a highly imbalanced classification scenario. The superior AUPRC of Sort \& Slice indicates a better trade-off between sensitivity and precision in a setting with very few positives where standard algorithms tend to generate predictions that are heavily biased towards the negative class (i.e.~heavily biased towards extremely high precision and extremely low sensitivity). Our observations are consistent with the hypothesis that Sort \& Slice exerts a mitigating effect on this bias relative to hashing by increasing sensitivity favourably at the cost of precision, leading to a more balanced classifier with stronger overall performance. The predictive power of the RFs and MLPs trained on the imbalanced LIT-PCBA estrogen receptor alpha antagonism data set could potentially be further increased by combining them with computational techniques explicitly tailored to counteract class imbalance (oversampling, undersampling, tree-wise undersampling, weighted loss, balanced training batches, ...). Note though that commonly used balancing techniques such as oversampling or undersampling may interact with methods like Sort \& Slice by changing the relative frequencies of substructures in the training set. Since the primary goal of this study was to evaluate the relative performance of substructure-pooling techniques in a clear and technically straightforward setting, we thus decided not to add this additional layer of complexity to our experiments. However, combining techniques to counteract class imbalance with substructure-pooling methods like Sort \& Slice might reveal unknown synergies and could form an interesting project for future research.

Parts of~\Cref{fig:moleculenetlipophilicity,fig:aqsoldbsolubility,fig:posterasarscov2mpro,fig:amesmutagenicity,fig:litpcbaesrant} seem to suggest that the improvements achieved via Sort \& Slice over hashing tend to become more pronounced as the fingerprint length decreases, the fingerprint diameter increases and as standard atomic invariants are used instead of pharmacophoric invariants. These observations strongly support the idea that the predictive advantage of Sort \& Slice over hashing stems at least partially from an absence of bit collisions: unlike ECFPs generated via Sort \& Slice, hashed ECFPs exhibit more and more bit collisions as the dimension of the fingerprint descreases relative to the number of substructures identified in the data set. The number of identified substructures in turn increases with the fingerprint diameter and when switching from pharmacophoric to standard atomic invariants. An increase in bit collisions then seems to degrade the predictive performance of hashed ECFPs relative to those generated by Sort \& Slice.

A question that remains, however, is whether the predictive advantage of Sort \& Slice is merely a product of the general avoidance of bit collisions via the selection of a subset of substructures instead of the hashing of all substructures; or if and to what extent the particular unsupervised substructure selection scheme underlying Sort \& Slice (i.e.~sorting of substructures according to their frequency in the training set and subsequent exclusion of rare substructures) independently contributes to the performance gain. Surprisingly, Figure~\ref{fig:boxplotswithtitles} shows that Sort \& Slice not only beats hashing, but it also consistently outperforms the two other investigated substructure-pooling techniques (filtering and MIM), whose respective performance measurements tend to fall between hashing and Sort \& Slice. Note that just like Sort \& Slice, both of these techniques are based on substructure selection and lead to fingerprints that are entirely free of bit collisions. 

These observations reveal two points. Firstly, the performance gains provided by Sort \& Slice, filtering and MIM over standard hashing are not purely the result of avoiding bit collisions via substructure selection; but the specific strategy by which substructures are selected for the final fingerprint does indeed make an important difference for downstream predictive performance. Secondly, and perhaps remarkably, the extremely simple frequency-based substructure selection strategy implemented by Sort \& Slice outperforms the more technically advanced substructure selection schemes underlying filtering and MIM. This is in spite of the fact that, unlike MIM and filtering, Sort \& Slice is an unsupervised technique that does not utilise any information associated with the training label. It appears surprising that Sort \& Slice would beat technically sophisticated supervised feature selection methods such as filtering or MIM that select substructures using task-specific information. While the reasons for this are not obvious, it is generally conceivable that exploiting the training label when selecting substructures could potentially harm the generalisation abilities of a machine-learning system by contributing to its risk of overfitting to the training data (just like any other aspect of supervised model training could).

A natural extension of our study for a future research project could be to include additional substructure-selection techniques. One interesting approach would be to only select substructures that maximise feature variance based on the training set.
If $p(\mathcal{J}) \in (0,1]$ represents the fraction of training compounds in which a detected substructure $\mathcal{J}$ is present, then its associated empirical feature variance in the context of a binary fingerprint is given by $p(\mathcal{J})(1-p(\mathcal{J}))$. The maximisation of feature variance would correspond to the removal of almost-constant columns of the feature matrix and would reflect a common data preparation strategy from traditional QSAR modelling. Both the empirical feature-variance $p(\mathcal{J})(1-p(\mathcal{J}))$ of a substructure $\mathcal{J}$ and its empirical information entropy $H(p(\mathcal{J}))$ as introduced in Remark~\ref{remark: sort_and_slice_info} peak when $p(\mathcal{J}) = 1/2$, i.e.~when $\mathcal{J}$ is present in exactly half of all training compounds. It is easy to prove that, in the case of binary fingerprints, ranking substructures according to $p(\mathcal{J})(1-p(\mathcal{J}))$ is equivalent to ranking them according to $H(p(\mathcal{J}))$. In this sense, feature-variance maximisation is equivalent to entropy maximisation and both methods translate to the removal of high-frequency as well as low-frequency substructures. At first glance, this approach might seem significantly different from Sort \& Slice which is based on the exclusion of only low-frequency substructures. However, note that if there are no high-frequency substructures, then naturally Sort \& Slice, feature-variance maximisation and entropy maximisation all simply slice away low-frequency substructures from the binary fingerprint and are thus all the same. In Remark~\ref{remark: sort_and_slice_info}, we have given a mathematical proof that Sort \& Slice, entropy maximisation, and therefore also feature-variance maximisation are in fact strictly equivalent under the assumption that no substructure appears in more than half of all training compounds. Since in common chemical data sets it is usually true that only very few substructures exist in more than half of all training compounds, Sort \& Slice should be expected to closely approximate feature-variance maximisation (and entropy maximisation) in realistic settings while arguably being somewhat easier to describe and implement. It would be interesting to computationally compare the performance of substructure selection via feature-variance maximisation with Sort \& Slice, to check whether in practice the exclusion of a small number of high-frequency substructures has a significant effect after all, or whether indeed both methods lead to a very similar level of performance as suggested by the theoretical arguments in Remark~\ref{remark: sort_and_slice_info}.

Another compelling feature selection technique that could be used for substructure pooling is given by \textit{conditional} MIM as described by~\citet{fleuret2004fast}. Conditional MIM can be seen as a more sophisticated version of MIM that iteratively selects features that maximise the mutual information with the training label conditional on the information contained in any feature already picked. While MIM and conditional MIM both select features that are individually informative about the training label, conditional MIM is additionally designed to reduce redundancy by selecting features that also exhibit weak pairwise dependence and thus contain distinct pieces of information about the target variable. Conditional MIM can be a stronger choice than MIM in scenarios where there is a large informational overlap between features; on the other hand, if all features are perfectly independent, then MIM and conditional MIM become mathematically equivalent. A limitation of conditional MIM in the context of substructure pooling for ECFPs is its computational cost when it comes to the selection of large numbers of features; even the fast implementation of conditional MIM provided by~\citet{fleuret2004fast} may be slow to select hundreds or even thousands of substructures out of an even larger substructure pool. This might make the generation of vectorial ECFPs with usual lengths such as $1024$ or $2048$ bits impractical or even intractable when conditional MIM is used for substructure pooling. One way to address this problem is by instead using simple MIM as we did in our study; MIM can be interpreted as a natural simplification of conditional MIM that remains computationally feasible even in very high feature dimensions at the price of potentially leading to more feature redundancy. In a future study, it might still be worthwhile to explore the predictive abilities of low-dimensional vectorial ECFPs generated via substructure-pooling operators based on conditional MIM; it is conceivable that conditional MIM could generate a short yet effective and information-dense ECFP vectorisation whose performance may match or possibly even surpass the one of much longer hashed ECFPs.

Finally, note that the results for the SARS-CoV-2 main protease data set in Figure~\ref{fig:posterasarscov2mpro} that are based on a random data split are fully comparable to the QSAR-prediction results of the nine models that we investigated in our computational study in Chapter~\ref{chap: qsar_ac_study} (see~Figure~\ref{fig:ac_results_sarscov2mpro}). In both studies, we used the same data set, the same data splitting technique (random split), and the same evaluation scheme ($2$-fold cross validation repeated with the same $3$ random seeds across both studies). One difference is that, in the previous study from Chapter~\ref{chap: qsar_ac_study}, we fully optimised the kNN, RF and MLP hyperparameters, but only experimented with a single type of ECFP (hashed $2048$-bit ECFP$4$), while in the current study we kept the RF and MLP hyperparameters constant but explored a large chunk of the ECFP hyperparameter space. We can see that the strongest QSAR-predictor in~Figure~\ref{fig:ac_results_sarscov2mpro} is given by a hashed $2048$-bit ECFP$4$ combined with a hyperparameter-optimised MLP which reaches an MAE slightly above $0.42$.
In contrast, Figure~\ref{fig:posterasarscov2mpro} shows that the same $2048$-bit ECFP$4$ combined with an MLP based on our intuitively set hyperparameters from Table~\ref{tab: substruc_pool_exper_params_rf_mlp} reaches an MAE slightly below $0.42$. This shows that in this setting our custom MLP hyperparameter choice is essentially as performant as the computationally optimised MLP hyperparameters from our previous study. We further see in Figure~\ref{fig:posterasarscov2mpro} that Sort \& Slice once again leads to slightly better performance than hashing for the $2048$-bit ECFP$4$ combined with our custom MLP on a random split. This suggests that the predictive performance of the best QSAR-predictor for SARS-CoV-2 main protease binding affinity from our previous computational study from Chapter~\ref{chap: qsar_ac_study} could still have been slightly improved by vectorising the used $2048$-bit ECFP$4$s via Sort \& Slice instead of hashing.

\section{Conclusions}

We have introduced a general mathematical framework for the vectorisation of structural fingerprints via a formal operation we refer to as \textit{substructure pooling}. For structural fingerprints, substructure pooling is the natural analogue to node feature vector pooling in modern GNN architectures. Unlike GNN pooling, substructure pooling remains largely unexplored and is almost always performed via the hashing of substructures into a vector of predefined length. Our proposed mathematical framework encompasses hash-based substructure pooling, but also pooling operations based on a diverse set of alternative techniques such as supervised and unsupervised feature selection. Trainable permutation-invariant set functions operating on sets of substructure embeddings also fit into the given framework, and the future exploration of such advanced substructure-pooling methods might form an interesting opportunity for novel research. For example, in Section~\ref{sec: diff_substruc_pool} below we introduce our idea of a novel trainable substructure-pooling technique based on a differentiable self-attention mechanism.

As part of our work, we have mathematically described and experimentally evaluated a method we refer to as \textit{Sort \& Slice} as an alternative to hashing for substructure pooling of ECFP substructures. In a nutshell, Sort \& Slice is based on first ranking all identified substructures in the training set according to their frequency of occurrence and then constructing a binary fingerprint that only indicates the presence or absence of the most frequent substructures. Sort \& Slice is easy to implement and interpret and leads to increased predictive performance for supervised molecular machine learning tasks. Formally, Sort \& Slice can be interpreted as a simple unsupervised feature selection scheme. We have given a mathematical proof that, under reasonable theoretical assumptions that are approximately valid for realistic data sets, Sort \& Slice filters out all but the most informative substructures from an information-entropic perspective.

Due to its natural simplicity, variations of Sort \& Slice might have already been used by other researchers in practical scenarios in the past. However, we are not aware of any occurrence of our version of Sort \& Slice in a formal research paper. In particular, we are not aware of any rigorous experimental comparison of Sort \& Slice and standard hash-based substructure pooling outside of this work. To the best of our knowledge, only one other version of Sort \& Slice has been systematically explored~\citep{macdougall2022reduced}; however, unlike our version of Sort \& Slice, this slightly different technique only allows limited control over the dimension of the vectorial fingerprint and was evaluated in a less general experimental setting.

In summary, our experiments show that Sort \& Slice tends to generate higher (and sometimes substantially higher) downstream predictive performance than hashing for a variety of molecular property prediction tasks. This predictive advantage seems to exist across regression and classification data sets, balanced and imbalanced tasks, data splitting techniques, machine-learning models, and ECFP hyperparameters, and appears to increase with the expected number of bit collisions in the hashed ECFP. Perhaps surprisingly, Sort \& Slice not only seems to outcompete hashing but also two more technically sophisticated supervised substructure selection schemes. 
This suggests that simply sorting substructures according to frequency of occurrence in the training set and then discarding infrequent substructures is a relatively (and maybe unexpectedly) strong feature selection strategy. Based on the predictive advantage of Sort \& Slice, its technical simplicity, and its ability to improve fingerprint interpretability by avoiding bit collisions, we recommend that it should canonically replace hashing as the standard substructure-pooling technique for supervised molecular machine learning.

\newpage $\text{}$ 
\newpage
\chapter[Future Directions]{Future Directions} \label{chap: future_research}

In this Chapter, we briefly describe two ideas we developed that could potentially form the seeds for two future research projects.

\section[A Graph-Based Self-Supervised Learning Strategy to Make Classical Molecular Featurisations Trainable]{A Graph-Based Self-Supervised Learning \\ Strategy to Make Classical Molecular \\ Featurisations Trainable}

\label{sec: gnn_pretraining_non_diff_to_diff}

In our study from Chapter~\ref{chap: qsar_ac_study}, we showed that classical ECFPs consistently outperform trainable GINs at QSAR-prediction in a rigorous evaluation setting involving a robust series of random data splits and full hyperparameter-optimisation loops. This runs counter to the hopes that message-passing GNNs might be able to beat classical featurisations at molecular property prediction via their abilities to extract chemical knowledge directly from molecular graphs in a differentiable manner.

To enable graph-based featurisation methods to reach their full predictive potential and possibly break through the performance ceiling posed by ECFPs, we have developed a novel self-supervised learning strategy for GNNs based on predicting precomputed ECFPs from a potentially giant corpus of unlabelled molecular graphs. The pre-trained GNN can then be seamlessly combined with an ECFP-MLP model trained on a supervised molecular property prediction task such as QSAR-prediction. Our suggested learning strategy is partially motivated by recent observations that self-supervised pre-training followed by task-specific supervised fine-tuning can lead to impressive results in the image domain~\citep{chen2020big}. 

Our proposed scheme is divided into three steps that are visualised in Figure~\ref{fig:gnn_ecfp_self_supervised}.
\begin{figure}[!t]
	\centering
	\includegraphics[width=0.99\linewidth]{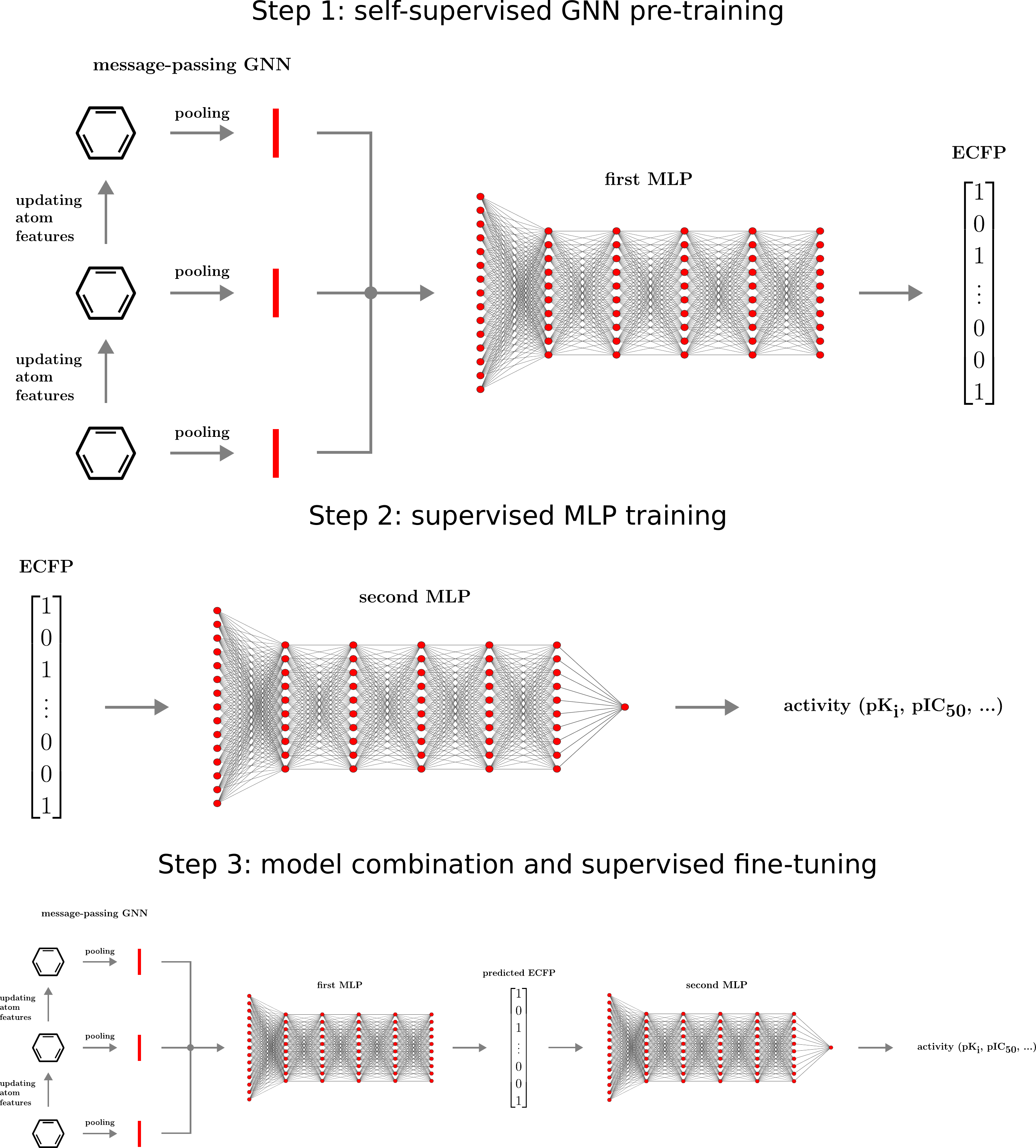}
	\caption[Self-supervised pre-training and fine-tuning strategy for GNNs.]{Step 1: Self-supervised pre-training of a graph neural network~(GNN) to predict precomputed extended-connectivity fingerprints~(ECFPs) from a large corpus of unlabelled molecular graphs. Step 2: Supervised training of a standard ECFP-based multilayer perceptron~(MLP) on a given molecular property prediction task of interest. Step 3: Combination of both pre-trained models and fine-tuning of the resulting end-to-end model on the supervised task from Step~2.}
	\label{fig:gnn_ecfp_self_supervised}
\end{figure}
Step 1 is based on pre-training a GNN to predict ECFPs from a large number of unlabelled molecular graphs, Step 2 represents training of a standard ECFP-MLP model on a supervised molecular property prediction task, and Step 3 involves plugging together both models from the two previous steps to create a graph-based predictor. The performance of this graph-based predictor must necessarily match the one of the standard ECFP-MLP model if the GNN part has indeed managed to successfully learn to generate ECFPs from molecular graphs. Notably, in Step 3 the combined end-to-end model can be further fine-tuned on the supervised task whereby the training signal then flows directly from the molecular graph to the training label. This final fine-tuning step might improve the performance of the combined model above that of´ the standard ECFP-MLP model and in this manner beat the state of the art. By using Sort \& Slice ECFPs instead of hashed ECFPs, we can build on our already improved baseline.

During the supervised fine-tuning process, the learnt ECFP representation generated by the GNN is expected to morph in a task-specific manner that benefits the predictive performance of the larger model. From this perspective, the proposed scheme can be interpreted as a way to make non-trainable classical precomputed molecular featurisations such as ECFPs differentiable and trainable. Note that the training strategy outlined in Figure~\ref{fig:gnn_ecfp_self_supervised} is not limited to ECFPs but can also be employed with PDVs, MACCS fingerprints or any other classical molecular featurisation method, as long as it can easily be generated for a large number of compounds.

The features extracted by the early layers of the pre-trained GNN that are close to the molecular graph can be seen as a novel type of neural fingerprint whose information content and predictive power could be explored. Finally, attempting to use message-passing GNNs to learn a differentiable mapping from molecular graphs to ECFPs might reveal their practical (in)abilities to correctly decipher chemical substructures; such insights could guide the way to further improvements of graph-based molecular featurisation methods in drug discovery.

\section[Trainable Substructure Pooling via Differentiable Self-Attention]{Trainable Substructure Pooling via \\ Differentiable Self-Attention} \label{sec: diff_substruc_pool}

The four substructure-pooling methods investigated in our study in Chapter~\ref{chap: ecfps_sort_and_slice} are all either based on hashing or on some type of supervised or unsupervised feature selection strategy. However, it is also possible to devise more complex differentiable substructure-pooling operators based on trainable deep networks. To the best of our knowledge, this research avenue is currently unexplored.

We propose to investigate substructure pooling via \textit{self-attention}. Self-attention is the key deep learning component in the famous transformer model that was introduced in the seminal paper from~\citet{vaswani2017attention} and is still leading to state-of-the-art results in natural language processing. Given a set of input vectors, self-attention intuitively enables the updating of the representation of each input vector in a learnable and context-sensitive manner, i.e.~in a manner that not only depends on the vector itself but also on the learnt interactions between the vector and all the other vectors in the input set. In this sense, each element in the set of input vectors metaphorically \textit{pays attention} to all other elements that are present, or from another perspective, the set of input vectors as a whole pays attention to \textit{itself} by considering the interactions between all of its elements (hence the name \textit{self}-attention).

To explore self-attention in the context of substructure pooling we once again consider the formal setting from Section~\ref{subsec: subtructure_pooling}. Let
$$\mathfrak{C} = \{\mathcal{C}_{1}, ..., \mathcal{C}_{m}\} $$
be a (potentially very large) set of $m$ chemical substructures and let
$$ P(\mathfrak{C}) = \{\mathfrak{A} \ \vert \ \mathfrak{A} \subseteq  \mathfrak{C} \}$$
be its power set. Furthermore, let
$$\{\mathcal{C}_{1}, ..., \mathcal{C}_{r} \} \in P(\mathfrak{C})$$
be a representation of some input compound $\mathcal{R}$ as a set of $r$ substructures in $\mathfrak{C}$. We imagine that $\mathcal{R}$ was transformed into $\{\mathcal{C}_{1}, ..., \mathcal{C}_{r} \}$ via some structural fingerprinting-method such as the ECFP or the MACCS-algorithm. We can now use some (injective) substructure embedding
$$\gamma : \mathfrak{C} \to \mathbb{R}^{w} $$
to generate a representation of $\mathcal{R}$ as a set of vectors:
$$\{\gamma(\mathcal{C}_{1}), ..., \gamma(\mathcal{C}_{r}) \} \subset \mathbb{R}^w.$$
The embedding $\gamma$ could for instance be based on one-hot encoding of substructures or on physicochemical substructure descriptors. Our goal is to update the representations of the vectors in $\{\gamma(\mathcal{C}_{1}), ..., \gamma(\mathcal{C}_{r}) \}$ in a trainable way using self-attention. 

In its simplest form, a classical self-attention layer~\citep{vaswani2017attention,lee2019set} is defined via three trainable weight matrices $W^{Q}, W^{K} \in \mathbb{R}^{w_q \times w}$ and $W^{V} \in \mathbb{R}^{w_v \times w}$. We now focus on a specific substructure representation $\gamma(\mathcal{C}_{i}) \in \mathbb{R}^w$ whose representation we want to update using these weight matrices. We start by generating a \textit{query} vector
$$q_i \coloneqq W^{Q}\gamma(\mathcal{C}_{i}) \in \mathbb{R}^{w_q}, $$
a \textit{key} vector
$$k_i \coloneqq W^{K}\gamma(\mathcal{C}_{i}) \in \mathbb{R}^{w_q}, $$
and a \textit{value} vector
$$v_i \coloneqq W^{V}\gamma(\mathcal{C}_{i}) \in \mathbb{R}^{w_v}. $$
We proceed by computing weights $$\alpha_{i,1},...,\alpha_{i,r} \in \mathbb{R}$$ by calculating the dot product between the query vector $q_i$ and the key vectors $k_1,...,k_r$ of all the other vectors in the input set:
$$\forall j \in \{1,...,r\}: \quad \alpha_{i,j} \coloneqq q_i^{T}k_j \in \mathbb{R}. $$
Each quantity $\alpha_{i,j}$ can be intuitively interpreted as a measure for how much attention the vector $\gamma(\mathcal{C}_{i})$ pays to the vector $\gamma(\mathcal{C}_{j})$ during its updating process. The attention weights are usually further normalised via a nonlinear softmax activation function:
$$(\bar{\alpha}_{i,1},...,\bar{\alpha}_{i,r}) \coloneqq \text{softmax} (\alpha_{i,1},...,\alpha_{i,r}), $$
such that 
$$\bar{\alpha}_{i,1},...,\bar{\alpha}_{i,r} > 0 \quad \text{and} \quad \sum_{j = 1}^{r} \bar{\alpha}_{i,j} = 1. $$ Finally, the updated representation of $\gamma(\mathcal{C}_{i})$ is given by a weighted sum of all value vectors:
$$\gamma(\mathcal{C}_{i})_{\text{upt}} \coloneqq \sum_{j = 1}^r \bar{\alpha}_{i,j} v_j \in \mathbb{R}^{w_v}. $$
The transformation
$$\mathbb{R}^w \supset \{\gamma(\mathcal{C}_{1}), ..., \gamma(\mathcal{C}_{r}) \} \mapsto \{\gamma(\mathcal{C}_{1})_{\text{upt}}, ..., \gamma(\mathcal{C}_{r})_{\text{upt}} \} \subset \mathbb{R}^{w_v} $$
is interpreted as the application of one self-attention layer to the set of input vectors $\{\gamma(\mathcal{C}_{1}), ..., \gamma(\mathcal{C}_{r}) \}.$ This layer can be trained like any other deep learning component by adapting its defining weight matrices $W^{Q}, W^{K}, W^{V}$ via some form of gradient descent. Several self-attention layers can naturally be stacked on top of each other to eventually generate a final vector-set representation of the input compound $\mathcal{R}$ in the form of iteratively updated substructure representations that encode structural and contextual information:
$$\{\gamma(\mathcal{C}_{1})_{\text{final}}, ..., \gamma(\mathcal{C}_{r})_{\text{final}} \} \subset \mathbb{R}^{w_{\text{final}}}.$$
Using a standard pooling function 
$$\bigoplus : \{ A \subset \mathbb{R}^{w_{\text{final}}} \ \vert \ A \ \text{is finite} \} \to \mathbb{R}^l, $$
i.e.~a permutation-invariant set function $\bigoplus$ such as summation, averaging or componentwise maximum, one can finally represent $\mathcal{R}$ as a single vector
$$ \bigoplus \{\gamma(\mathcal{C}_{1})_{\text{final}}, ..., \gamma(\mathcal{C}_{r})_{\text{final}} \} \in \mathbb{R}^l$$
that can be fed into a standard multilayer perceptron for further processing. Note that iteratively updating substructural embeddings via stacked self-attention layers and then aggregating the final substructural representations via a permutation-invariant set function formally defines a (trainable, differentiable) substructure-pooling method:
$$\Psi :  P(\mathfrak{C}) \to \mathbb{R}^l, \quad \Psi(\{\mathcal{C}_{1},...,\mathcal{C}_{r}\}) = \bigoplus \{\gamma(\mathcal{C}_{1})_{\text{final}}, ..., \gamma(\mathcal{C}_{r})_{\text{final}} \}.$$
The operator $\Psi$ once again satisfies Definition~\ref{def: substructure_pooling} introduced in Chapter~\ref{chap: ecfps_sort_and_slice}. This underlines the generality of our definition of substructure pooling which encompasses techniques such as hashing, unsupervised and supervised feature selection, and the trainable updating of sets of substructural embeddings via modern deep learning architectures.

Self-attention-based substructure pooling on top of structural fingerprints has several properties that could potentially make it an interesting featurisation method for chemical prediction tasks. The self-attention mechanism should provide a useful inductive bias to learn substructural representations that are influenced by molecular context, i.e.~by the presence or absence of other substructures. As a result, self-attention should explicitly support the learning of task-specific compound-level featurisations that depend not only on individually present substructures but also on their interactions. Note that this includes long-range interactions between substructures located at physically distant parts of the input compound. This might represent an important advantage over molecular featurisation via message-passing GNNs: the receptive field of GNNs is strictly local and thus does not allow for information flow between physically distant parts of an input compound during message-passing.

We conducted a literature search and were only able to identify one other work that has explored a technique similar to the one proposed in this section: \citet{kim2022substructure} investigate a dual architecture that combines a GNN branch operating on molecular graphs and a self-attention branch operating on substructural embeddings. They pre-train their architecture to predict precomputed physicochemical descriptors using a large corpus of unlabelled compounds and obtain encouraging results when fine-tuning their model on a range of supervised molecular property prediction tasks. Further work of this kind could attempt to refine substructural self-attention mechanisms, for example by developing more powerful pre-training schemes or by studying the effects of different types of initial substructural embeddings (such as one-hot embeddings versus physicochemical embeddings).

It might also be particularly interesting to explore the abilities of models involving self-attention-based substructure pooling to correctly predict \textit{non-additivity}~\citep{gogishvili2021nonadditivity,kwapien2022implications}. In its most narrow form, non-additivity refers to a phenomenon observed in protein-ligand binding where the change of two substructures in a ligand results in much higher or lower binding affinity than would be expected from the respective additive contributions of the single changes alone. From a theoretical point of view, self-attention-based substructure pooling appears to be well-suited to detect such non-additivity events via its ability to learn distinct representations for a given substructure conditional on the presence or absence of other substructures.

\newpage $\text{}$ 
\newpage
\chapter[Conclusions and Further Thoughts]{Conclusions and Further Thoughts} \label{chap: conclusions}

In this work, we have studied classical and graph-based molecular featurisation methods in a variety of important machine-learning scenarios for computational drug discovery. We have put a particular focus on the under-researched challenge of activity-cliff prediction which is of natural interest in compound optimisation and the elucidation of structure-activity relationships. We have (i) systematically explored the capabilities of physicochemical-descriptor vectors, extended-connectivity fingerprints and graph isomorphism networks for the prediction of quantitative structure-activity relationships, activity cliffs and potency directions, (ii) have designed a novel twin neural network model that can naturally learn to featurise compound pairs for the prediction of activity cliffs and potency directions, and (iii) have described an easily implementable method for the vectorisation of extended-connectivity fingerprints that robustly outperforms hashing at supervised molecular property prediction. We have also outlined two further research ideas in the area of molecular featurisation that can be seen as two distinct attempts to bring together the strengths of classical non-trainable featurisers and trainable deep learning components such as graph neural networks and self-attention.

Detailed conclusions from our main research projects can be found at the respective ends of~\Cref{chap: qsar_ac_study,chap: twin_net_ac_pred,chap: ecfps_sort_and_slice}. Overall, our investigations provide further evidence for the vital role that molecular featurisation plays in the performance of molecular machine learning tasks. In Chapter~\ref{chap: qsar_ac_study} we saw that switching from one featurisation technique to another can easily lead to substantial shifts in performance for both quantitative structure-activity relationship and activity-cliff prediction. The balanced activity-cliff classification performance of our twin neural network model from Chapter~\ref{chap: twin_net_ac_pred}, compared to the imbalanced performance of the evaluated baseline quantitative structure-activity relationship predictors at the same task, supports the idea that it might generally pay off to naturally adapt the featuriser to the given problem rather than trying to adapt the problem to a pre-existing featuriser. Our results from Chapter~\ref{chap: ecfps_sort_and_slice} show how even seemingly minor technical decisions such as the procedure chosen to vectorise a set of identified substructures can have a surprisingly large and consistent impact on the predictive accuracy of a structural fingerprint.

Extracting powerful features from arbitrary molecular structures is a difficult research challenge. Recent work has shown that in some cases self-supervised pre-training strategies on unlabelled molecular graphs can substantially boost the performance of graph neural networks, and graph isomorphism networks in particular~\citep{wang2021molclr,hu2019strategies}. Considering these results and the accessibility of large databases with millions of unlabelled compounds, it may be worthwhile to continue exploring the limits of self-supervised pre-training of graph neural networks in the molecular domain, for example by investigating the pre-training strategy we propose in Section~\ref{sec: gnn_pretraining_non_diff_to_diff}.

However, although in this work we have only experimented with graph isomorphism networks as prototypical examples of graph neural networks in the $1$-WL class, we still hypothesise that differentiable graph-based message-passing, while relatively useful in certain contexts such as activity-cliff prediction, might not yet be the correct learning paradigm to truly and substantially outperform the technically related and more traditional extended-connectivity fingerprints in the same way that convolutional neural networks have outperformed classical feature-engineering methods in computer vision. A fruitful area for future research might be the development of methods to overcome some of the technical shortcomings shared by both extended-connectivity fingerprints and graph neural networks such as a strictly local circular receptive field and limited theoretical expressivity. For example, a notable attempt in this direction has recently been made by~\citet{bouritsas2022improving} who managed to increase the theoretical expressivity of message-passing graph neural networks via a technique based on subgraph isomorphism counting. Another interesting avenue has been explored by~\citet{ying2021transformers} who introduced a transformer-based graph featuriser with a global receptive field that has achieved strong results across a variety of benchmarks.

It is also worth noting that molecular graphs are not entirely general but rather obey certain constraints dictated by the laws of chemistry; current message-passing graph neural networks, on the other hand, are highly general architectures that can essentially operate on any graph structure. One can speculate that it might be possible to somehow constrain the neural architecture of graph-based machine learning methods in a way that more directly leverages the chemical rules that govern the structure of molecules. 

Finally, it might be useful to consider that one of the central limitations of current molecular featurisation methods may not be in the technical details of the featurisation itself, but in the information content of the original molecular representation from which the features are extracted. In the vast majority of cases, molecular featurisations for supervised prediction tasks are derived either from molecular string representations such as SMILES strings or from molecular graphs. Both of these representations usually fully encode the chemical composition and two-dimensional connectivity structure of an input compound, along with simple 3D attributes such as tetrahedral R-S chirality and E-Z double bond geometry. While this appears comprehensive, it is possible to imagine that a real molecule might have other relevant physicochemical properties that cannot be easily derived from these pieces of information alone, such as more complex stereochemical features based on its ensemble of conformers or even quantum-chemical characteristics associated with its electronic structure. Developing novel featurisation methods adapted to more realistic molecular representations whose information content strictly surpasses that of molecular graphs and SMILES strings may be a promising area for future research.

\newpage $\text{}$ 
\newpage
\chapter*{Summary of Research Contributions}
\addcontentsline{toc}{chapter}{Summary of Research Contributions}

\section*{Published Peer-Reviewed Research Papers}
\begin{itemize}
	
	\item Markus Dablander, Thierry Hanser, Renaud Lambiotte, and Garrett M.~Morris. Exploring QSAR models for activity-cliff prediction. \textit{Journal of Cheminformatics}, 15(1),~47,~2023. \href{https://doi.org/10.1186/s13321-023-00708-w}{Link to paper.}
	
	\item Julius Berner, Markus Dablander, and Philipp Grohs. Numerically solving parametric families of high-dimensional Kolmogorov partial differential equations via deep learning. \textit{Advances in Neural Information Processing Systems}, 33, 16615-16627, 2020. \href{https://arxiv.org/pdf/2011.04602}{Link to paper.}
	
	My friend and colleague Julius Berner and I wrote this NeurIPS paper as shared first authors under the supervision of Prof.~Philipp Grohs from the University of Vienna. This independent research project was conducted by us in parallel to my main doctoral studies.
\end{itemize}

\section*{Technical Reports from Industrial Study Groups}

\begin{itemize}
	
	\item Ann Smith, Markus Dablander, Constantin Octavian Puiu, Brady Metherall, William Lee, Ruzanna Ab Razak, and Noriszura Ismail. Tourism Forecasting and Environment. \textit{Mathematics in Industry Reports}, 2023. \href{https://doi.org/10.33774/miir-2024-rqvbr}{Link to ESGI report.}
	
	\item Simone Appella, Anvarbek Atayev, Oliver Bond, Ben Collins, Markus Dablander, Nikolai Fadeev, Andrew Lacey, Piotr Morawiecki, Hilary Ockendon, Davide Polvara, Ellen Powell, Lorenzo Quintavalle Laval, Eddie Wilson, and Yang Zhou. Determining the conductance of networks created by randomly dispersed cylinders. \textit{Mathematics in Industry Reports}, 2021. \href{https://doi.org/10.33774/miir-2021-3pqt1-v2}{Link to ESGI report.}
\end{itemize}

\section*{Conference Presentations}

\begin{itemize}
	
	\item Markus Dablander, Thierry Hanser, Renaud Lambiotte, and Garrett M.~Morris. Exploring molecular machine learning models for activity-cliff prediction. Poster presentation at the 10th International Congress on Industrial and Applied Mathematics (ICIAM). In-person, Tokyo, 2023. \href{http://dx.doi.org/10.13140/RG.2.2.35914.34241}{Link to poster.}
	
	\item Markus Dablander, Thierry Hanser, Renaud Lambiotte, and Garrett M.~Morris. Siamese neural networks work for activity cliff prediction. Poster presentation at the 4th RSC-BMCS / RSC-CICAG Artificial Intelligence in Chemistry Symposium. Virtual, 2021. \href{http://dx.doi.org/10.13140/RG.2.2.18137.60000}{Link to poster.} 
	
	\item Julius Berner, Markus Dablander, and Philipp Grohs. Numerically solving parametric families of high-dimensional Kolmogorov partial differential equations via deep learning. Poster presentation at the Thirty-fourth Conference on Neural Information Processing Systems (NeurIPS). Virtual, 2020. \href{http://dx.doi.org/10.13140/RG.2.2.20514.85443}{Link to poster.}
	
\end{itemize}

\section*{Visited Industrial Study Groups}

\begin{itemize}
	
	\item ESGI 171 in Edinburgh, UK, in-person (2023).
	
	\item ESGI 156 in Ålesund, Norway, in-person (2022).
	
	\item ESGI 165 in Durham, UK, virtual (2021).
	
	\item ESGI 162 in Leeds, UK, virtual (2020).
\end{itemize}

\section*{Awards and Prizes}
\begin{itemize}
	
	\item Winner of InFoMM Doctoral Prize Scheme. Associated with InFoMM-funded post-doctoral research position at the Mathematical Institute, University of Oxford.
	
	\item Second Prize at the 2021 Smith Institute's TakeAIM Competition for showcasing the potential impact of my computational research on activity cliffs via a short text accessible to non-experts.
	
	\item Winner of the Royal Society of Chemistry Prize for the Best Scientific Poster at the 4th RSC-BMCS / RSC-CICAG Artificial Intelligence in Chemistry Symposium.
	
\end{itemize}

\section*{Published Code}
\begin{itemize}
	
	\item Codebase to reproduce and extend the computational experiments from our published paper~\textit{Exploring QSAR Models for Activity-Cliff Prediction}~\citep{dablander2023exploring}.  \href{https://github.com/MarkusFerdinandDablander/QSAR-activity-cliff-experiments}{Link to GitHub repository.}

\end{itemize}

\bibliographystyle{unsrtnat}
\bibliography{refs.bib}

\end{document}